\documentclass[10pt]{article}

\usepackage{arxiv}
\usepackage{times}

\usepackage{epsfig}
\usepackage{graphicx}
\usepackage{amsmath}
\usepackage{amssymb}
\usepackage{booktabs}
\usepackage{enumitem}

\usepackage{amsmath,amsfonts,bm, amsthm,algorithm, algpseudocode}

\def\Figref#1{Figure~\ref{#1}}

\def\secref#1{section~\ref{#1}}
\def\Secref#1{Section~\ref{#1}}

\def\eqref#1{equation~\ref{#1}}

\def\Algref#1{Algorithm~\ref{#1}}

\def\tableref#1{table~\ref{#1}}
\def\Tableref#1{Table~\ref{#1}}

\def\1{\bm{1}}

\def\eps{{\epsilon}}

\def\rva{{\mathbf{a}}}

\def\rvx{{\mathbf{x}}}
\def\rvy{{\mathbf{y}}}

\def\rmU{{\mathbf{U}}}

\def\va{{\bm{a}}}

\def\vw{{\bm{w}}}

\def\eva{{a}}

\DeclareMathAlphabet{\mathsfit}{\encodingdefault}{\sfdefault}{m}{sl}
\SetMathAlphabet{\mathsfit}{bold}{\encodingdefault}{\sfdefault}{bx}{n}

\def\sM{{\mathbb{M}}}

\def\sP{{\mathbb{P}}}

\def\sR{{\mathbb{R}}}

\newcommand{\inner}[2]{\left \langle #1 \right \rangle #2}

\theoremstyle{plain}
\newtheorem{thm}{\protect\theoremname}
\theoremstyle{remark}

\theoremstyle{definition}

\theoremstyle{plain}

\theoremstyle{plain}

\theoremstyle{definition}
\newtheorem{definition}{Def.}[section]

\providecommand{\corollaryname}{Corollary}
\providecommand{\lemmaname}{Lemma}
\providecommand{\problemname}{Problem}
\providecommand{\remarkname}{Remark}
\providecommand{\theoremname}{Theorem}

\usepackage{color}
\usepackage{multirow}
\usepackage{pgfplots}
\usepackage{standalone}
\usepackage{wrapfig}
\usepackage{caption}

\newcommand{\etal}{\emph{et al.}}

\usepackage[pagebackref=true,breaklinks=true,letterpaper=true,colorlinks,bookmarks=false]{hyperref}

\pgfplotsset{compat=1.14}
\begin{document}

\title{Semantic Adversarial Attacks: Parametric Transformations That Fool Deep Classifiers}

\author{Ameya Joshi \qquad Amitangshu Mukherjee \qquad Soumik Sarkar \qquad Chinmay Hegde\thanks{This work was supported in part by NSF grants CCF-1750920, CNS-1845969, DARPA AIRA grant PA-18-02-02, AFOSR YIP Grant FA9550-17-1-0220, an ERP grant from Iowa State University, a GPU gift grant from NVIDIA corporation, and faculty fellowships from the Black and Veatch Foundation.}
\\
Iowa State University\\
{\tt\small \{ameya, amimukh, soumiks, chinmay\}@iastate.edu}
}

\maketitle

\begin{abstract}
Deep neural networks have been shown to exhibit an intriguing vulnerability to adversarial input images corrupted with imperceptible perturbations. However, the majority of adversarial attacks assume global, fine-grained control over the image pixel space. In this paper, we consider a different setting: what happens if the adversary could only alter specific attributes of the input image? These would generate inputs that might be perceptibly different, but still natural-looking and enough to fool a classifier. We propose a novel approach to generate such ``semantic'' adversarial examples by optimizing a particular adversarial loss over the range-space of a parametric conditional generative model. We demonstrate implementations of our attacks on binary classifiers trained on face images, and show that such natural-looking semantic adversarial examples exist. We evaluate the effectiveness of our attack on synthetic and real data, and present detailed comparisons with existing attack methods. We supplement our empirical results with theoretical bounds that demonstrate the existence of such parametric adversarial examples. 
\end{abstract}
\section{Introduction}
\label{sec:intro}

\begin{figure}[htp]
    \centering
    \setlength{\tabcolsep}{1pt}
    \renewcommand{\arraystretch}{0.2}
    \resizebox{\linewidth}{!}{
    \begin{tabular}{c c  c  c c c c c }
    \includegraphics[width=0.25\linewidth]{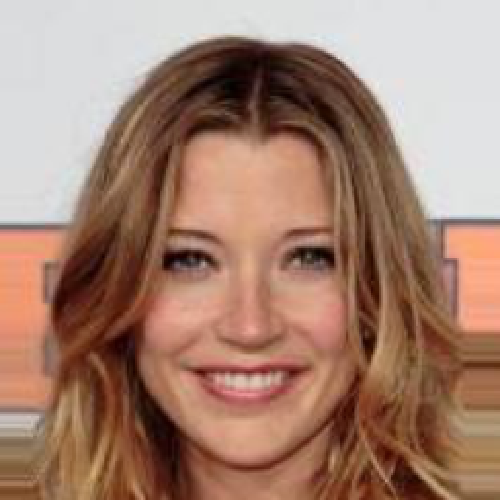} & \includegraphics[width=0.25\linewidth]{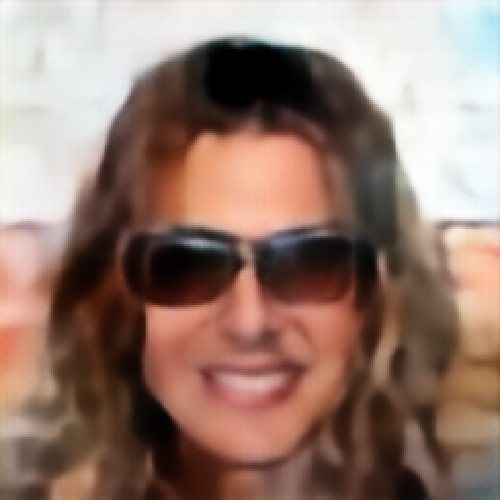} & & 
    \includegraphics[width=0.25\linewidth]{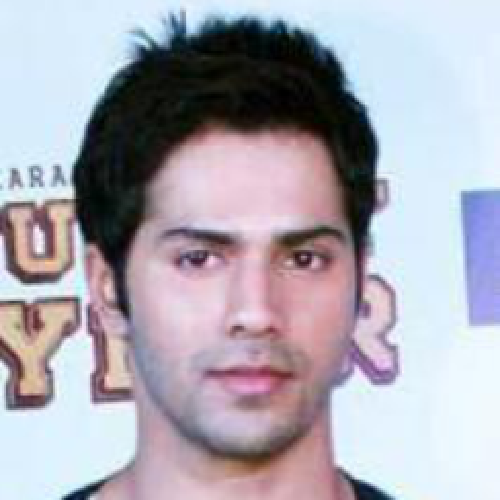} & \includegraphics[width=0.25\linewidth]{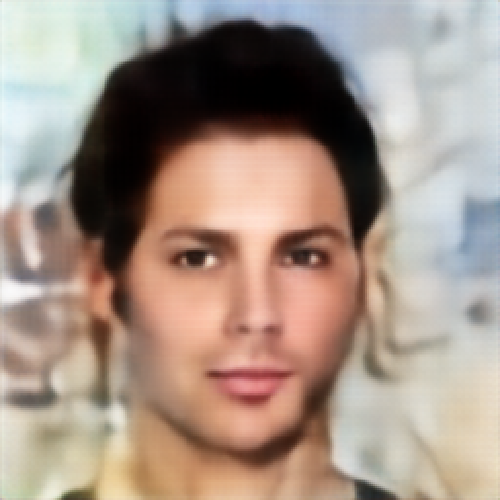} & &
    \includegraphics[width=0.25\linewidth]{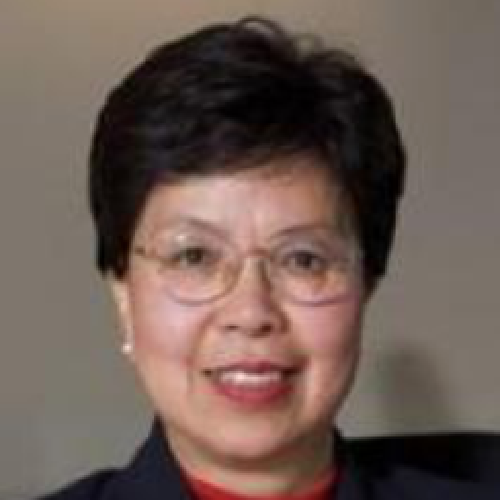} &  \includegraphics[width=0.25\linewidth]{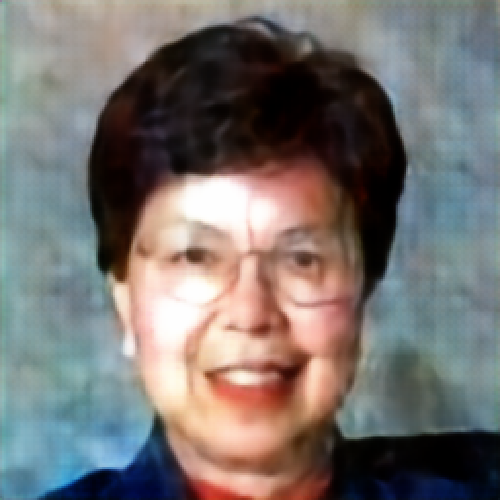} \\
    \includegraphics[width=0.25\linewidth]{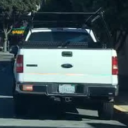} &
    \includegraphics[width=0.25\linewidth]{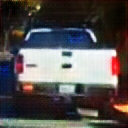} & &
    \includegraphics[width=0.25\linewidth]{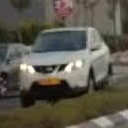} & \includegraphics[width=0.25\linewidth]{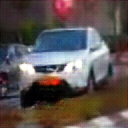} & &
     \includegraphics[width=0.25\linewidth]{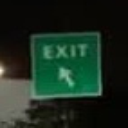} &
    \includegraphics[width=0.25\linewidth]{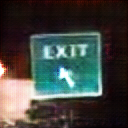} \\
    \end{tabular}
    }
    \caption{\sl Examples of semantic adversarial attacks with a single modifiable attribute. The first and third columns are original images. Semantic adversarial examples (Columns 2 and 4) are generated by optimizing over the manifold of a parametric generative model to fool a binary gender classifier.}
    \label{fig:single_attrib}
\end{figure}

The existence of adversarial inputs for deep neural network-based classifiers has been well established by several recent works~\cite{Carlini2017cwl2, Dathathri2017minadvexample, adversarialexamples2015, Goodfellow2018existence, Szegedy2014intriguing, MoosaviDezfooli2017UniversalAP}. The adversary typically confounds the classifier by adding an \emph{imperceptible} perturbation to a given input image, where the range of the perturbation is defined in terms of bounded pixel-space $\ell_p$-norm balls. Such adversarial ``attacks'' appear to catastrophically affect the performance of state-of-the-art classifiers~\cite{Athalye2018obfuscated,Ilyas2018alimitedqueries,Ilyas2018blackbox, Shafahi2017inevitable}.

Pixel-space norm-constrained attacks reveal interesting insights about  generalization properties of deep neural networks. However, imperceptible attacks are certainly not the only means available to an adversary. Consider an input example that comprises salient, invariant features along with modifiable attributes. An example would be an image of a face, which consists of invariant features relevant to the identity of the person, and variable attributes such as hair color and presence/absence of eyeglasses. Such adversarial examples, though perceptually distinct from the original input, appear natural and acceptable to an oracle or a human observer but would still be able to subvert the classifier. Unfortunately, the large majority of adversarial attack methods do not port over to such {natural} settings. 


A systematic study of such attacks is paramount in safety-critical applications that deploy neural classifiers, such as face-recognition systems or vision modules of autonomous vehicles. These systems are required to be immune to a limited amount of variability in input data, particularly when these variations are achieved through natural means. Therefore, a method to generate adversarial examples using natural perturbations, such as facial attributes in the case of face images, or different weather conditions for autonomous navigation systems, would shed further insights into the real-world robustness of such systems. We refer to such perceptible attacks as ``semantic'' attacks.

This setting fundamentally differs from existing attack approaches and has been (largely) unexplored thus far. Semantic attacks utilize nonlinear generative transformations of an input image instead of linear, additive techniques (such as image blending). 
Such complicated generative transformations display higher degrees of nonlinearity in corresponding attacks, the effects of which warrant further investigation. In addition, the role of the number of modifiable attributes (parameters in the generative models) in the given input is also an important point of consideration.  

\textbf{Contributions:} We propose and rigorously analyze a framework for generating adversarial examples for a deep neural classifier by modifying \emph{semantic} attributes. 

We leverage generative models such as Fader Networks~\cite{lample2017fader} that have semantically meaningful, tunable attributes corresponding to parameters into a continuous bounded space that implicitly define the space of ``natural'' input data. Our approach exploits this property by treating the range space of these attribute models as a manifold of semantic transformations of an image. 

We pose the search for adversarial examples on this \emph{semantic manifold} as an optimization problem over the parameters conditioning the generative model. Using face image classification as a running test case, we train a variety of parametric models (including Fader Networks and Attribute GANs), and demonstrate the ability to generate semantically meaningful adversarial examples using each of these models. In addition to our empirical evaluations, we also provide a theoretical analysis of a simplified semantic attack model to understand the capacity of parametric attacks that typically exploit a significantly lower dimensional attack space compared to the classical pixel-space attacks.


Our specific contributions are as follows: 
\begin{enumerate}[wide,nolistsep, parsep=1pt, labelindent=0pt]
    \item We propose a novel optimization based framework to generate semantically valid adversarial examples using parametric generative transformations.
    \item We explore realizations of our approach using variants of multi-attribute transformation models: Fader Networks~\cite{lample2017fader} and Attribute GANs~\cite{He2017AttGANFA} to generate adversarial face images for a binary classifier trained on the CelebA dataset~\cite{liu2015faceattributes}. Some of our modified multi-attribute models are non-trivial and may be of independent interest.
    \item We present an empirical analysis of our approach and show that increasing the \emph{dimensionality} of the attack space results in more effective attacks. In addition, we investigate a sequence of increasingly nonlinear attacks, and demonstrate that a higher degree of \emph{nonlinearity} (surprisingly) leads to weaker attacks.
    \item Finally, we provide a preliminary theoretical analysis by providing upper bounds for the \textit{classification error} for a simplified surrogate model under adversarial condition~\cite{schmidt2018adversarially}. This analysis supports our empirical observations regarding the dimensionality of the attack space.
\end{enumerate}

We demonstrate the effectiveness of our attacks on simple deep classifiers trained over complex image datasets; hence, our empirical comparisons are significantly more realistic than popular attack methods such as FGSM~\cite{adversarialexamples2015} and PGD~\cite{Kurakin2017,madry2018towards} that primarily have focused on simpler datasets such as MNIST~\cite{mnist} and CIFAR. 
Our approach also presents an interesting use-case for multi-attribute generative models which have been used solely as visualization tools thus far. 

\noindent\textbf{Outline:} We begin with a review of relevant literature in \Secref{sec:relatedwork}. We describe our proposed framework, \emph{Semantic Adversarial Generation}, in \secref{sec:semadv}. In \Secref{sec:implementation} we describe two variants of our framework to show different methods of ensuring the semantic constraint. We provide empirical analysis of our work in \Secref{sec:expts}. We further present empirical analysis and theoretical qualification on the dimensionality of the parametric attack space in \Secref{sec:dimparams}, and conclude with possible extensions in \Secref{sec:discussion}.

\begin{figure*}
    \centering
    \includegraphics[width=0.85\textwidth]{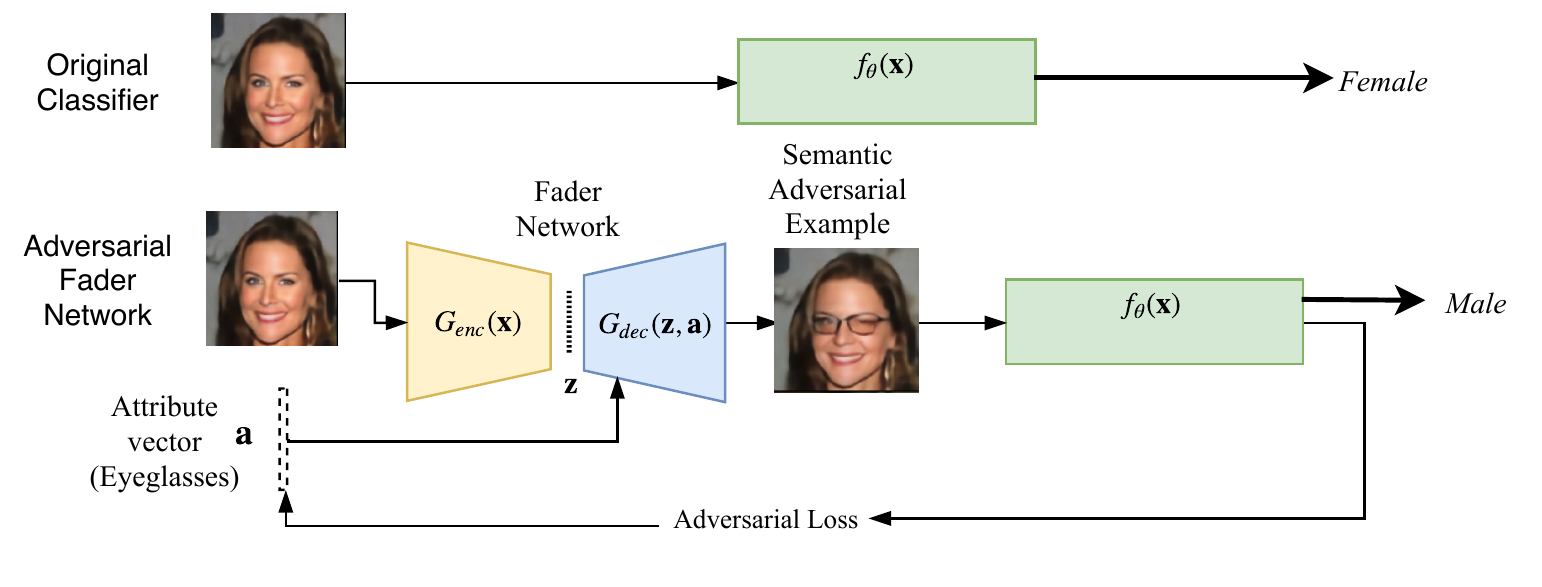}
    \caption{\sl An single-attribute \textbf{Adversarial} Fader Network. The semantic adversarial attack framework optimizes an adversarial loss to generate an adversarial direction. Backpropagating the adversarial direction through the Fader Network with respect to the attribute vector, $\va$, ensures that the adversarial example is only generated for that specific attribute. Here, the adversarial algorithm generates eyeglasses on a face of a Female by optimizing $\va$, thus forcing the gender classifier to misclassify the image as Male.}
    \label{fig:attackfader}
\end{figure*}

\section{Related Work}
\label{sec:relatedwork}
Due to space constraints coupled with the large amount of recent progress in the area of adversarial machine learning, our discussion of related work is necessarily incomplete. We defer a more detailed discussion to the appendix. 

Our focus is on \emph{white box, test-time} attacks on deep classification systems; other families of attacks (such as backdoor attacks, data poisoning schemes, and black-box attacks) are not directly relevant to our setting, and we do not discuss those methods here. 

\noindent\textbf{Adversarial Attacks:} Evidence that deep classifiers are susceptible to imperceptible adversarial examples can be attributed to Szegedy \etal~\cite{Szegedy2014intriguing}. Goodfellow~\etal~\cite{adversarialexamples2015} and Kurakin~\etal~\cite{Kurakin2017} extend this line of work using the Fast Gradient Sign Method (FGSM) and its iterative variants. Carlini and Wagner~\cite{Carlini2017cwl2} devise state-of-the-art attacks under various pixel-space $l_p$ norm-ball constraints by proposing multiple adversarial loss functions. Athalye~\etal~\cite{Athalye2018obfuscated} further analyze several defense approaches against pixel-space adversarial attacks, and demonstrate that most existing defenses can be surpassed by approximating gradients over defensively trained models.

Such attacks perturb the pixel-space under an imperceptibility constraint. On the contrary, we approach the problem of generating adversarial examples that have perceptible yet semantically valid modifications. Our method considers a smaller `parametric' space of modifiable attributes that have physical significance. 

\noindent\textbf{Parametric Adversarial Attacks:}
Parametric attacks are a recently introduced class of attacks in which the attack space is defined by a set of parameters rather than the pixel space. Such approaches result in more ``natural'' adversarial examples as they target the image formation process instead of the pixel space. Recent works by Athalye~\etal~\cite{Athalye2018SynthesizingRA} and Liu~\etal~\cite{liu2018beyond} use optimization over geometric surfaces in 3D space to create adversarial examples. Zhang~\etal~\cite{zhang2019camou} demonstrate the existence of adversarially designed textures that can camouflage vehicles. Zhao~\etal~\cite{zhao2018generating} generate adversarial examples by using the parametric input latent space of GANs\cite{GAN}. Xiao~\etal~\cite{Xiao2018SpatiallyTA} employ spatial transforms to perturb image geometry for creating adversarial examples. 
Sharif~\etal~\cite{Sharif2018advgennets} propose a generative model to alter images of faces with eyeglasses in order to confound a face recognition classifier. Contrary to these methods, we consider the inverse approach of using a pre-trained multi-attribute generative model to  transform inputs over multiple attributes for generating adversarial examples. 

Song~\etal~\cite{song2018constructing} optimize over the latent space of a conditional GAN to generate unrestricted adversarial examples for a gender classifier. While our approach is thematically similar, we fundamentally differ in the context of being able to generate adversarial counterparts for given test samples while providing a finer degree of control using multi-attribute generative models.
We discuss relevant literature regarding such multi-attribute generative models below.

\noindent\textbf{Attribute-Based Conditional Generative Models:} Generative Adversarial Networks (GAN)~\cite{GAN} are a popular approach for the generation of samples from a real-world data distribution. Recent advancements~\cite{Radford2016UnsupervisedRL, UnsupervisedIT, wang2016generative, chen2016infogan} in GANs allow for creation of high quality realistic images. Chen~\etal~\cite{chen2016infogan} introduce the concept of a attribute learning generative model where visual features are parametrized by an input vector. 

Perarnaue \etal~\cite{InvertibleCG} use a Conditional Generative Adversarial Network~\cite{Mirza2014ConditionalGA} and an encoder to learn the attribute invariant latent representation for attribute editing. Fader Networks~\cite{lample2017fader} improve upon this using an auto-encoder with a latent discriminator. He~\etal~\cite{He2017AttGANFA} argue that such an attribute invariant constraint is too constrictive and replace it an attribute classification constraint and a reconstruction loss instead to alter only the desired attributes preserving attribute-excluding features. 
These models are primarily used for generation of a large variety of facial images. We provide a secondary (and perhaps practical) use case for such attribute models in the context of understanding generalization properties of neural networks.

\section{Semantic Attacks}
\label{sec:semadv}

Conceptually, producing an adversarial semantic (``natural'') perturbation of a given input depends on two algorithmic components: (i) the ability to navigate the manifold of parametric transformations of an input image,  and (ii) the ability to perform optimization over this manifold that maximizes the classification loss with respect to a given target model. We describe each component in detail below. 


\textbf{Notation:}
We assume a white-box threat model, where the adversary has access to a target model $f(\rvx):\sR^d\to\{0,1\}^c$ and the gradients associated with it. The model classifies an input image, $\rvx$ into one of $c$ classes, represented by a one-hot output label, $\rvy$. In this paper, we focus on binary classification models ($c=2$) while noting that our framework transparently extends to multi-class models. Let $G(\rvx, \va):\sR^d\times\sR^k \to \sR^d$ denote parametric transformations, conditioned on a parameter vector, $\va$. Here, each element of $\va$ (say, $\eva_i$) is a real number that corresponds to a specific semantic attribute. For example, $\eva_0$ may correspond to facial hair, with a value of zero (or negative) denoting absence and a positive value denoting presence of hair on a given face example.
We define a semantic adversarial attack as the deliberate process of transforming an input image, $\rvx$ via a parametric model to produce a new example $\tilde{\rvx} = G(\rvx, \rva)$ such that $f(\tilde{\rvx}) \neq f(\rvx)$. 

\subsection{Parametric Transformation Models}
\label{subsec:paramgen}
First, let us consider the problem of generating semantic transformations of a given input example. In order to create semantically transformed examples, the defined parametric generative model $G(.)$ should satisfy two properties:
$G(.)$ should reconstruct the invariant data in an image, and
$G(.)$ should be able to independently perturb the semantic attributes while minimally changing the invariant data.

The parametric transformation model therefore, is trained to reconstruct the original example while disentangling the semantic attributes. This involves conditioning the generative model on a set of parameters corresponding to the modifiable attributes. The \emph{semantic parameter vector} consists of these parameters and is input to the parametric model to control the expression of semantic attributes.

We argue that the range-space of such a model approximates the manifold of the semantic transformations of input images. Therefore, the transformation model can be used a projection operator to ensure that a solution to an optimization problem will lie in the set of semantic transformations of an input image. We also observe that the semantic parameter vectors will be much lower in dimension than the original image.
 
In this paper, we consider two variants of such conditional generative models: Fader Networks~\cite{lample2017fader} and AttributeGANs (AttGAN)~\cite{He2017AttGANFA}. 
 
\subsection{Adversarial Parameter Optimization}
\label{subsec:advparam_opt}

\begin{algorithm}[!t]
\begin{small}
\caption{Adversarial Parameter Optimization}
\label{alg:attack}
\begin{algorithmic}[1]
\Require \small $\rvx_0$:Input image, $\rva_0$:Initial attribute vector, $E(.)$: Attribute encoder, $G(., .)$:Pre-trained parametric transformation model, $f(.)$: Target classifier, $\rvy:$ Original label

\State $h_{0} \gets f(\rvx)$, 
 $l_{\mathrm{adv}} \gets \infty$, 
 $i \gets 0$
\State success = 0
\While $\ l_{\mathrm{adv}} \neq 0\ \text{ and } \ i \leq \text{MaxIter}$
    \State $\bar{\rva} \gets E(\rva)$
    \State $h_i \gets f(G(\rvx_i, \bar{\rva_i}))$
    \State $l_{\mathrm{adv}} \gets L_{\mathrm{adv}}(\rvy, h_i)$
    \State $\rva_{i+1} \gets \text{BackProp}\left\{\rva_i, \nabla l_{\mathrm{adv}}\left(f(G(\rvx_i, E(\rva_i)))\right)\right\}$ 
    \State $ \tilde{\rvx} \gets G(\rvx, E(\rva_{i+1}))$
    \If {$f(\rvx) \neq f(\tilde{\rvx})$}
        \State \text{return success,}~$\tilde{\rvx}$
    \EndIf
    \State $i \gets i+1$
\EndWhile
\end{algorithmic}
\end{small}
\end{algorithm}

The problem of generating a semantic adversarial example essentially can be thought of as finding the right set of attributes that a classifier is adversarially susceptible to. In our approach, we model this as an optimization problem over the semantic parameters.

The generation of adversarial examples is generally modelled as an optimization problem that can be broken down into two sub problems:
(1) Optimization of an adversarial loss over the target network to find the direction of an adversarial perturbation. (2) Projection of the adversarial vector on the viable solution-space.

In the first step, we optimize over an adversarial loss, $L_{\mathrm{adv}}$. 
We model the second step as a projection of the adversarial vector onto the range space of a parametric transformation model. This is achieved by cascading the output of the transformation function to the input of our target network. The optimization problem can then be solved by back-propagating over both the network and the transform.
We also modify the Carlini-Wagner untargeted adversarial loss~\cite{Carlini2017cwl2} as shown in~\eqref{eq:cwloss} to include our semantic constraint:
\begin{eqnarray}
\label{eq:cwloss}
   & \max\left(0, \max\limits_{t \neq i} \left(f(\tilde{\rvx})_t\right) - f(\tilde{\rvx})_i\right)  \\
   & \text{s.t.}~~~\tilde{\rvx} = G(\rvx, \rva) \nonumber
\end{eqnarray}
where $i$ is the original label index and $t$ are the class label indices for any of the other classes. 

In comparison to the grid search method presented in Zhao~\etal~\cite{zhao2018generating} and Engstrom~\etal~\cite{engstrom2019a}, our optimization algorithm scales better. In addition, we create semantic adversarial transformations with multiple attributes for a specific input allowing for a fine-grained analysis of the generalization capacities of the target model. 

\section{Semantic Transformations}
\label{sec:implementation}

While our semantic attack framework is applicable to any parametric transformation model that enables gradient computations, we instantiate it by constructing adversarial variants of two recently proposed generative models: Fader networks~\cite{lample2017fader} and AttributeGANs (AttGAN)~\cite{He2017AttGANFA}. 

\subsection{Adversarial Fader Network}
\label{subsec:advfn}

A Fader Network~\cite{lample2017fader} is an encoder-decoder architecture trained to modify images with continuously parameterized attributes. They achieve this by learning an invariance over the encoded latent representation while disentangling the semantic information of the images and attributes. 
The invariance of the attributes is learnt by an adversarial training step in the latent space with the help of a latent discriminator which is trained to identify correct attributes corresponding to each training sample. 

Using our framework, we can adapt any pre-trained Fader Network to model the manifold of semantic perturbations of a given input. We note that minor adjustments are needed in our setting, since the parameter vector required by the approach of~\cite{lample2017fader} requires each scalar attribute, $a_i$, to be represented by a tuple, $(1-\eva_i, \eva_i)$. Since there is a one-to-one mapping between the two representations, we can project any real-valued parameter vector $\rva$ into this tuple form via an additional, fixed affine transformation layer. Given this extra ``attribute encoding'' step, all gradient computations proceed as before.  
We quantitatively study the effect of allowing the attacker access to single or multiple semantic attributes. In particular, we construct three approaches for generating semantic adversarial examples:
(i)    A single attribute Fader Network; 
(ii)    A multi-attribute Fader Network; and
(iii)    A cascaded sequence of single attribute Fader Networks.

\noindent\textbf{Single Attribute Attack:}
For the single attribute attack, we use the range-space of a pre-trained, single attribute  Fader Network to constrain our adversarial attack. The single attribute attack constrains an attacker to only modify a specified attribute for all the images. In the case of face images, such attributes might include presence/absence of eyeglasses, hair color, and nose shape.

In our experiments, we present examples of attacks on a gender classifier using three separate single attributes: eyeglasses, age, and skin complexion. \Figref{fig:attackfader} describes the mechanism of a single-attribute adversarial Fader Network that generates an adversarial example by adding eyeglasses. 

\noindent\textbf{Multi-Attribute Attack:}
Similar to the single-attribute case, we may also use pre-trained multi-attribute Fader Networks to model cases where the adversary has access to multiple modifiable traits. 

A limitation of multi-attribute Fader Networks lies in the difficulty of their training. This is because a Fader Network is required to learn disentangled representations of the attributes while in practice, semantic attributes cannot be perfectly decoupled. We resolve this using a novel conditional generative model described as follows.

\noindent\textbf{Cascaded Attribute Attack:}
We propose a novel method to simulate multi-attribute attacks by stage-wise concatenation pre-trained single attribute Fader networks. The benefit is that the computational burden of learning disentangled representations is now removed.

Each single-attribute model exposes a attribute latent vector. During execution of Alg. \ref{alg:attack} we jointly optimize over all the attribute vectors. The optimal adversarial vector is then segmented into corresponding attributes for each Fader Network to generate an adversarial example.


\subsection{Adversarial AttGAN}
\label{subsec:advattgan}
A second encoder-decoder architecture ~\cite{He2017AttGANFA}, known as AttGAN, achieves a similar goal as Fader Networks of editing attributes by manipulating the encoded latent representation; however, AttGAN disentangles the semantic attributes from the underlying invariances of the data by considering both the original and the flipped labels while training. This is achieved by training a latent discriminator and classifier pair to classify both the original and the transformed image to ensure invariance.

In order to generate semantic adversarial examples using AttGAN, we use a pretrained generator conditioned on $13$ attributes. The attribute vector in this case, is encoded to be a perturbation of the original sequence of attributes for the image. 
We consider the two sets of attributes listed in \Tableref{t:fader_perf} to generate adversarial examples. In our experience, the AttGAN architecture provides a more stable reconstruction, thus allowing for more modifiable parameters. 


\section{Experimental Results}
\label{sec:expts}

\begingroup
\begin{figure*}
    \centering
    \setlength{\tabcolsep}{1pt}
    \renewcommand{\arraystretch}{0.2}
    \begin{tabular}{c c c c  p{0.4cm}  c c  p{0.4cm}  c c c }
    \includegraphics[width=0.1\textwidth]{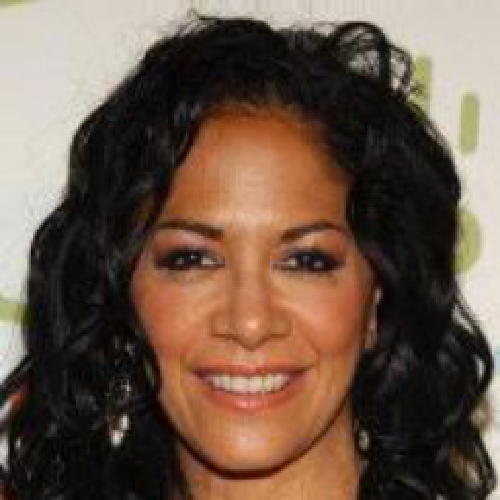} & \includegraphics[width=0.1\textwidth]{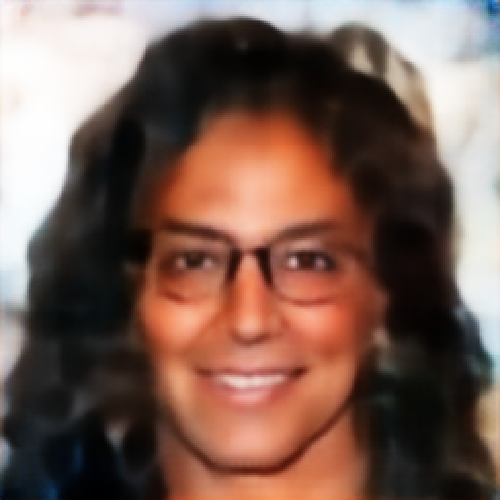} & 
    \includegraphics[width=0.1\textwidth]{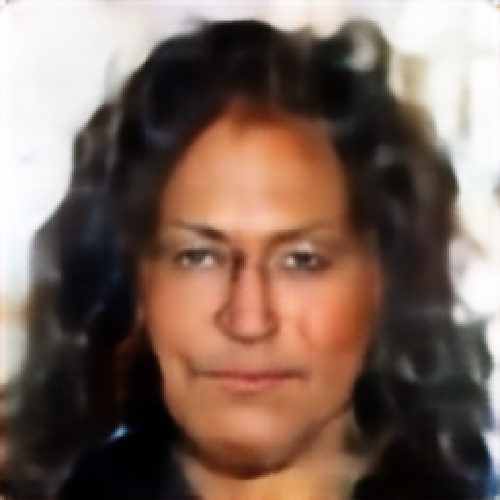} &
    \includegraphics[width=0.1\textwidth]{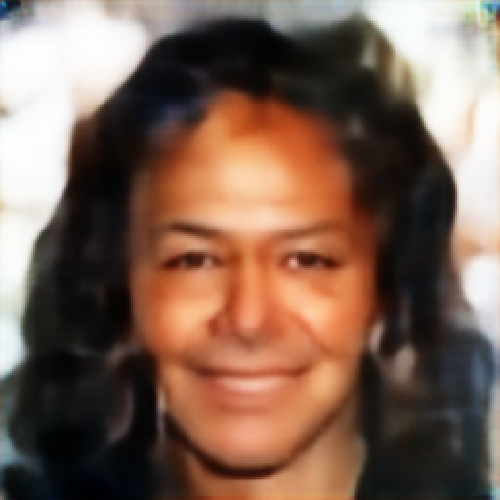} &\ &
     \includegraphics[width=0.1\textwidth]{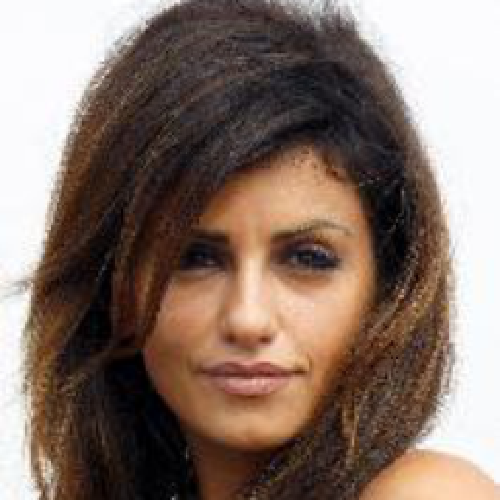} &
    \includegraphics[width=0.1\textwidth]{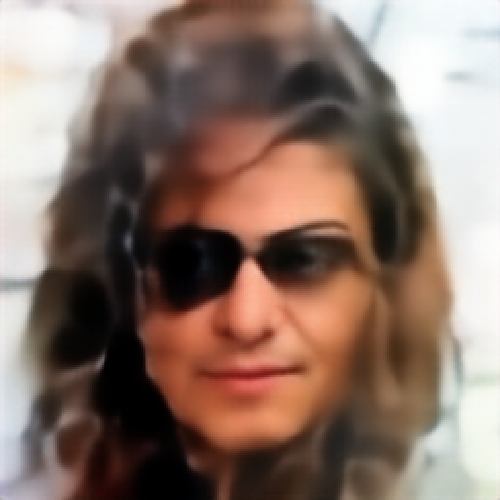} &\ &
    \includegraphics[width=0.1\textwidth]{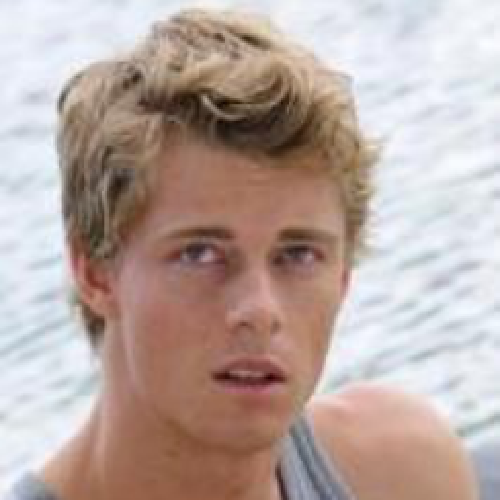} & \includegraphics[width=0.1\textwidth]{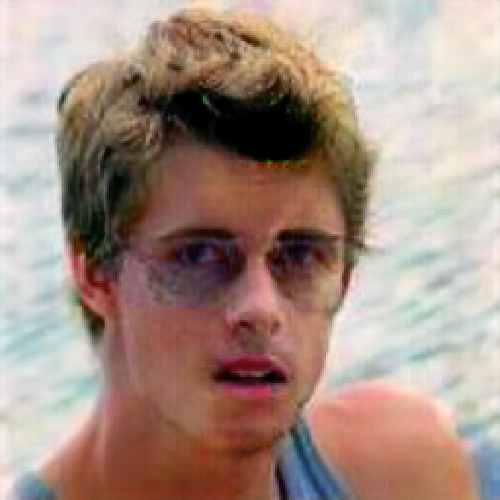} & 
    \includegraphics[width=0.1\textwidth]{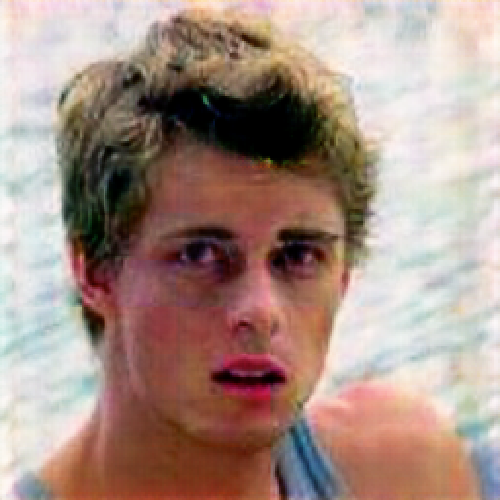}\\
     \includegraphics[width=0.1\textwidth]{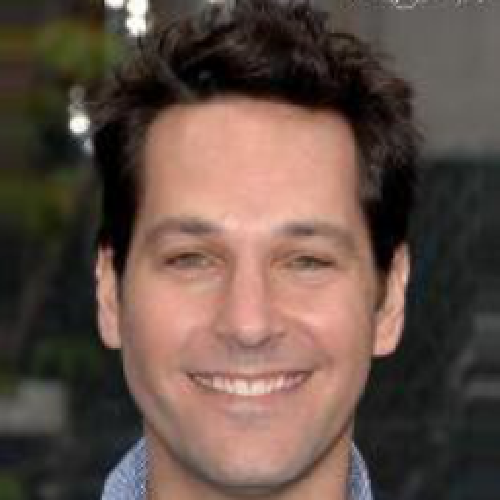} & \includegraphics[width=0.1\textwidth]{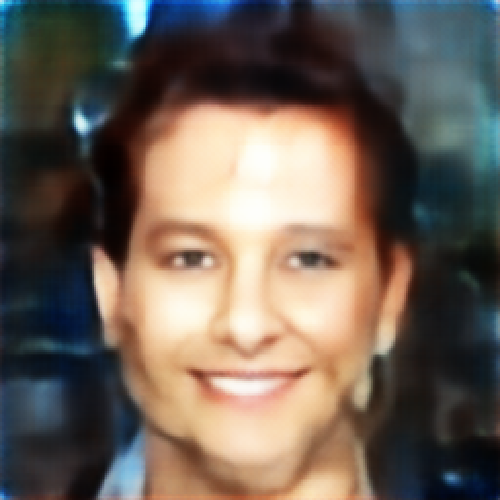} & 
    \includegraphics[width=0.1\textwidth]{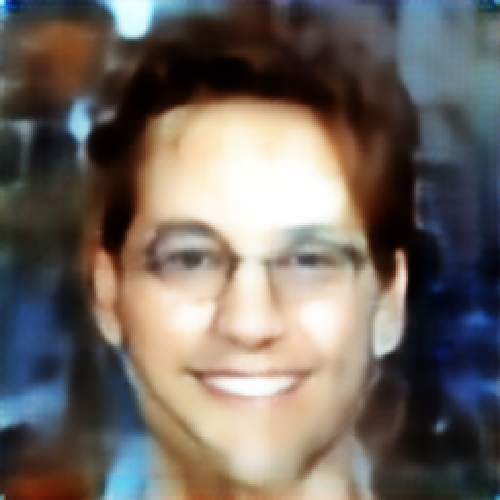} &
    \includegraphics[width=0.1\textwidth]{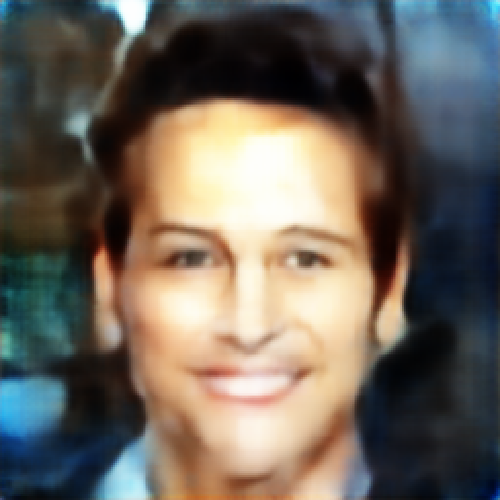} &\ &
      \includegraphics[width=0.1\textwidth]{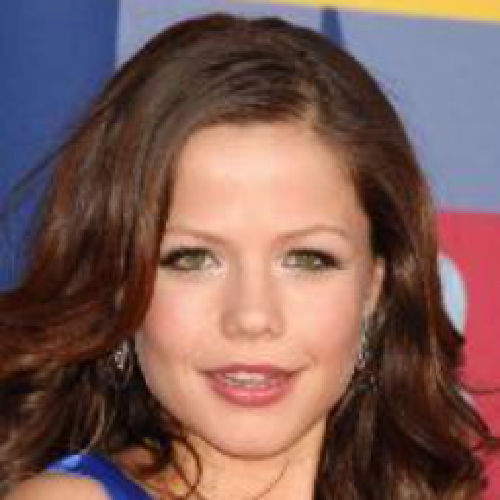} &
    \includegraphics[width=0.1\textwidth]{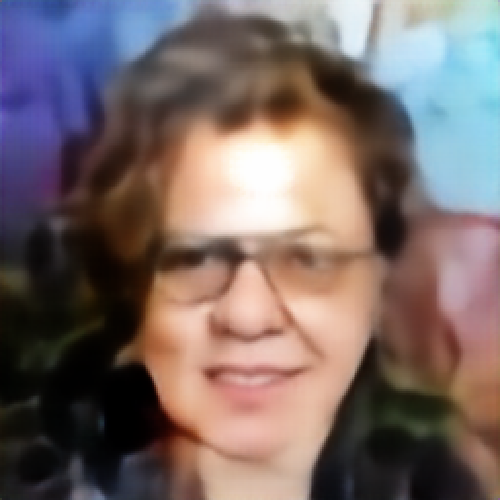} &\ &
    \includegraphics[width=0.1\textwidth]{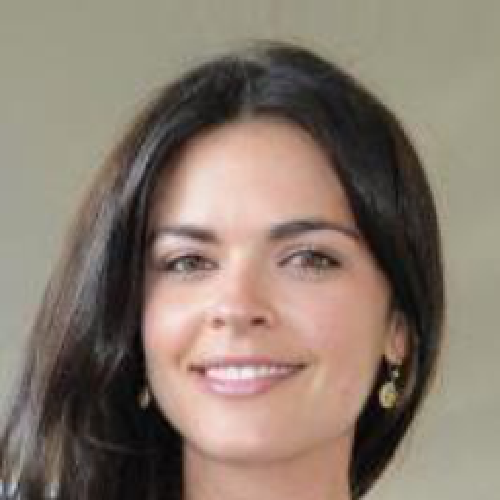} & \includegraphics[width=0.1\textwidth]{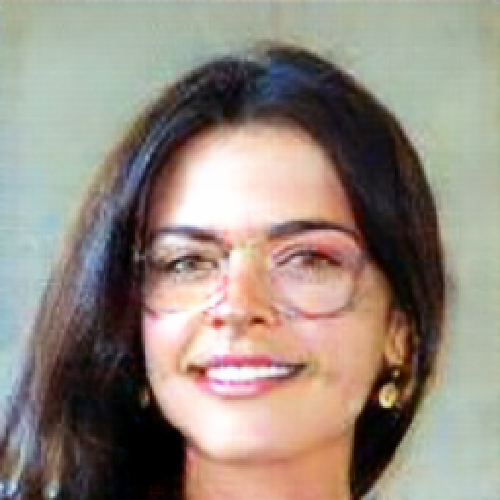} & 
    \includegraphics[width=0.1\textwidth]{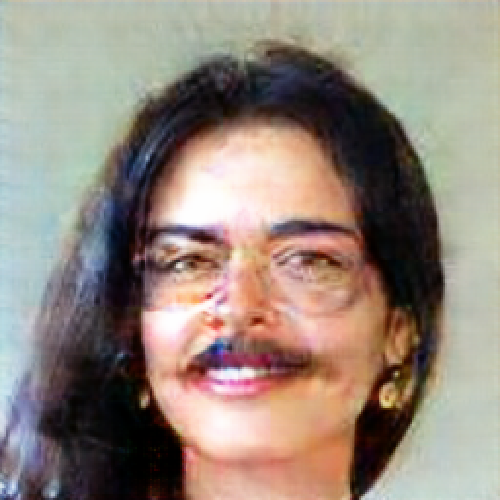}\\
     \includegraphics[width=0.1\textwidth]{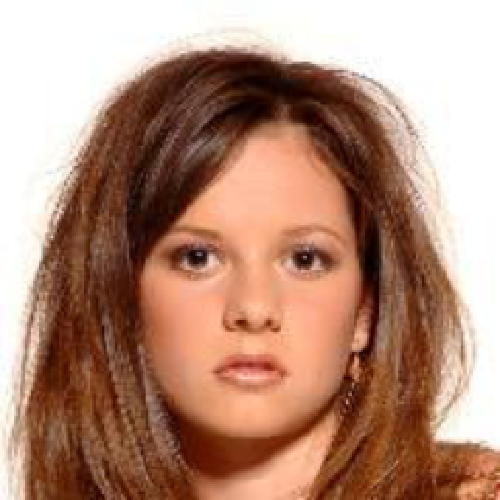} & \includegraphics[width=0.1\textwidth]{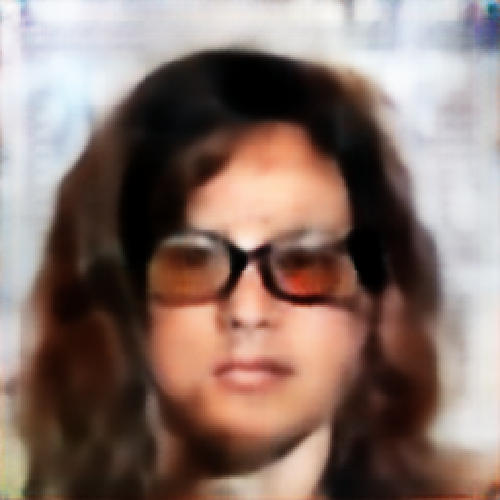} & 
    \includegraphics[width=0.1\textwidth]{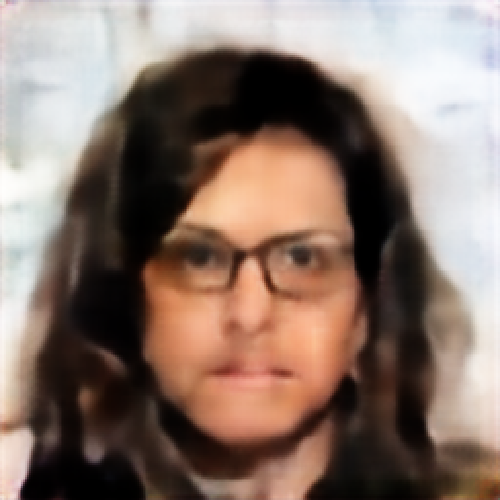} &
    \includegraphics[width=0.1\textwidth]{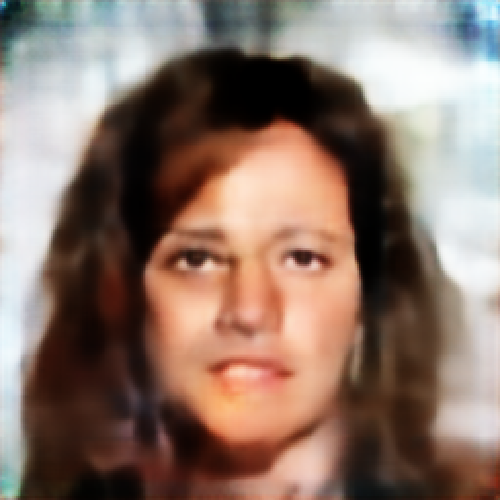} &\ &
      \includegraphics[width=0.1\textwidth]{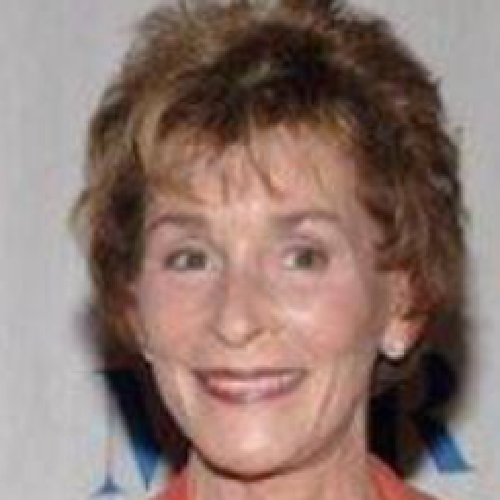} &
    \includegraphics[width=0.1\textwidth]{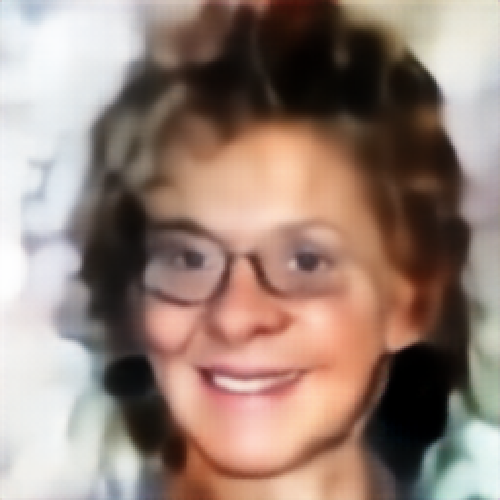} &\ &
    \includegraphics[width=0.1\textwidth]{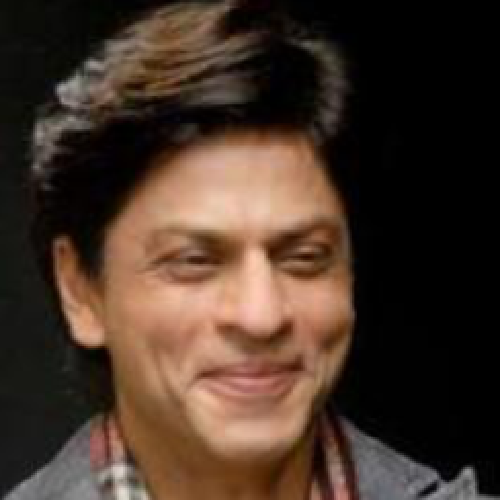} & \includegraphics[width=0.1\textwidth]{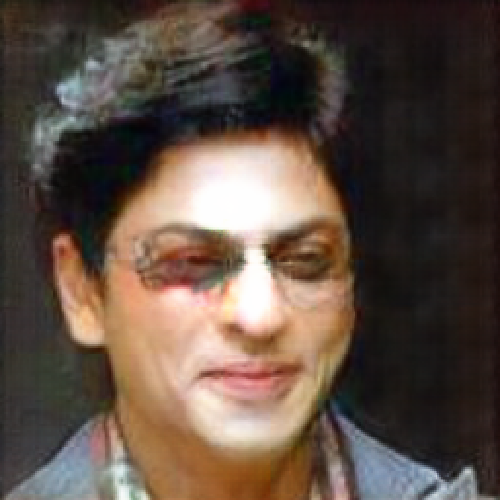} & 
    \includegraphics[width=0.1\textwidth]{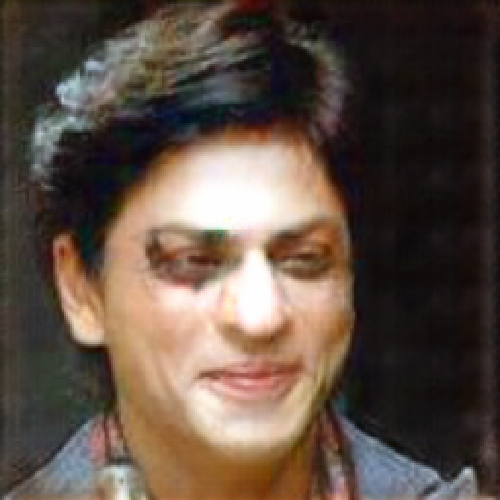} \\
     \includegraphics[width=0.1\textwidth]{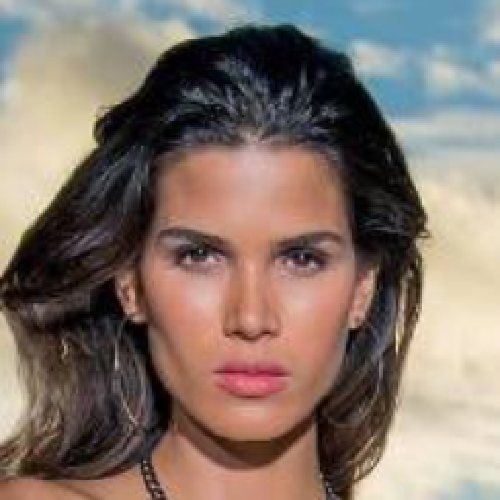} & \includegraphics[width=0.1\textwidth]{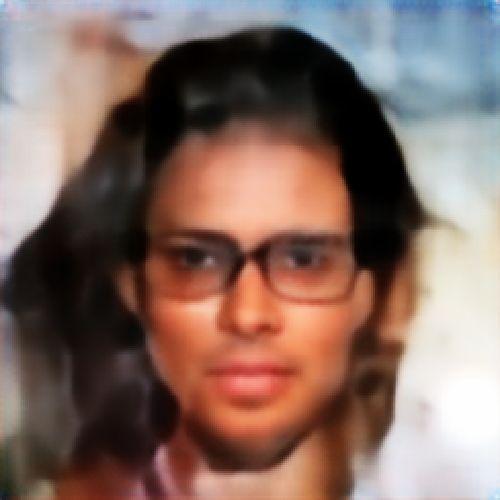} & 
    \includegraphics[width=0.1\textwidth]{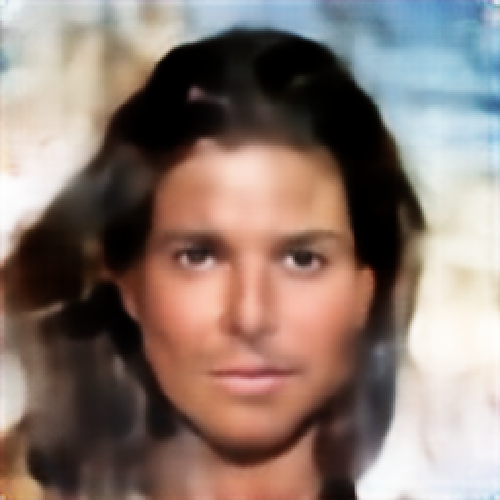} &
    \includegraphics[width=0.1\textwidth]{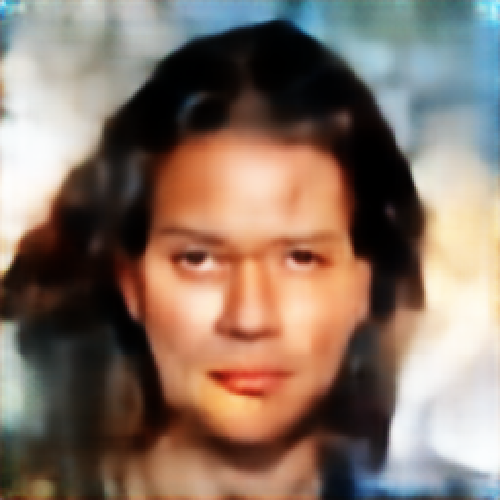} &\ &
      \includegraphics[width=0.1\textwidth]{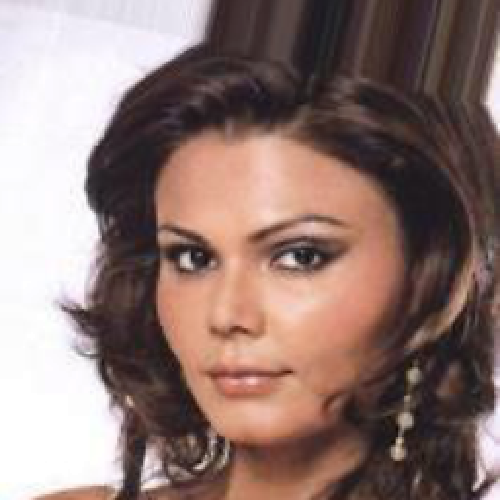} &
    \includegraphics[width=0.1\textwidth]{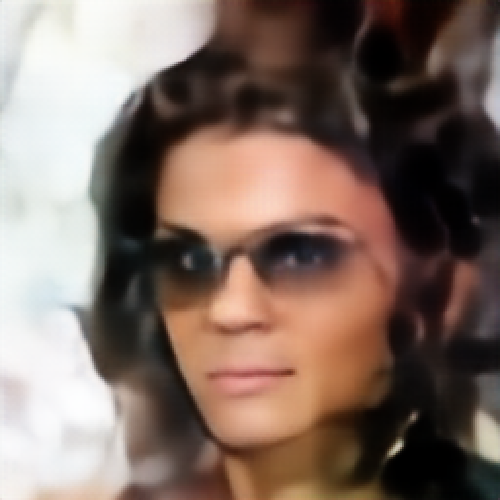} &\ &
    \includegraphics[width=0.1\textwidth]{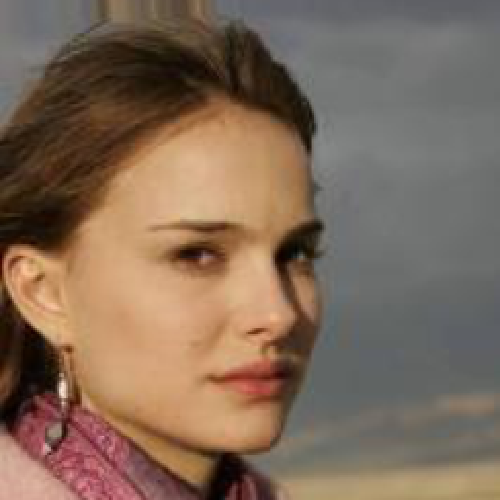} & \includegraphics[width=0.1\textwidth]{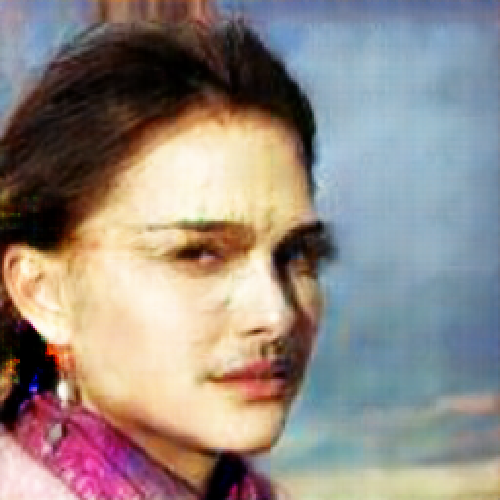} & 
    \includegraphics[width=0.1\textwidth]{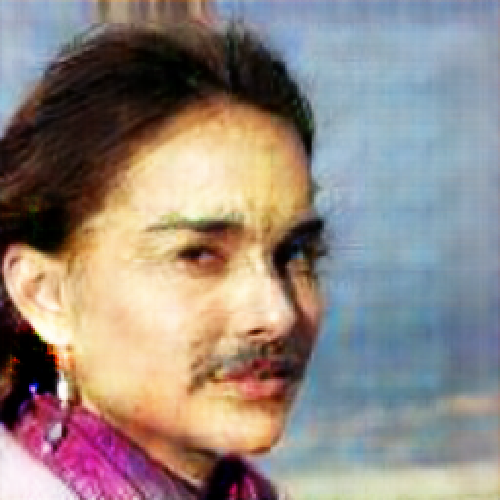}\\
     (a) & (b) & (c) & \multicolumn{2}{c}{(d)}   & (e) & \multicolumn{2}{c}{(f)} & (g) & (h) & (i)
    \end{tabular}
    \caption{\sl Semantic adversarial examples generated with multiple attribute semantic models as in \tableref{t:fader_perf}. Columns (a), (e) and (g) are original images. Columns (b)\{Attribute category: A1,A5,A6\} (c)\{Attribute category: A1,A2,A7\} and (d)\{Attribute category: A2,A5,A6\} show examples generated using multi-attribute Fader Networks as semantic transforms. Examples in (f)\{Attribute category: A1-A2-A3\} were generated using cascaded single attribute Fader Network. Columns (h)\{Attribute category: A1,A2,A6,A8,A10\} and (i)\{Attribute category: A1,A2,A6,A8,A9,A10\} are images transformed using an AttGAN with 5 and 6 attributes respectively. 
    Additional results of semantic attacks on multi-class classifiers for traffic scenes are provided in the appendix.}
    \label{fig:sem_adv_examples}
\end{figure*}
\endgroup

We showcase our semantic adversarial attack framework using a binary (gender) classifier as the target model trained on the CelebA dataset~\cite{liu2015faceattributes}. While we restrict ourselves to results on binary classifers on faces in this paper, additional results with multi-class classifiers on the Berkeley Deep Drive dataset~\cite{bdd100k} can be found in the appendix (refer Fig. 8). 
All experiments were performed on a single workstation equipped with an NVidia Titan Xp GPU in PyTorch~\cite{pytorch} v1.0.0.\footnote{Code\;and\;models:\;\url{https://github.com/ameya005/Semantic_Adversarial_Attacks}}
We train the classifier using the ADAM optimizer~\cite{kingma2015adam} over the categorical cross-entropy loss.
The training data is augmented with random horizontal flipping to ensure that the classifier does not overfit. The target model achieves a (standard) accuracy of 99.7\% on the test set (10\% of the dataset). 

\begin{table}[!tp]
    \centering
   \resizebox{\columnwidth}{!}{
    \begin{tabular}{p{0.3\columnwidth} p{0.4\columnwidth} p{0.25\columnwidth} p{0.25\columnwidth}}
    \toprule[2pt]
   Attack Type &  Attributes & Accuracy of target model (\%) & Random Sampling (\%)\\
    \midrule[1.5pt]
    \multirow{3}{0.3\columnwidth}{Single Attribute Attack}        & A1 & 52.0 & 89.00\\
                                                    & A2 & 35.0 & 96.00\\
                                                    & A3 & 14.0 & 90.00\\
    \midrule
    \multirow{3}{0.3\columnwidth}{Multi Attribute Attack}         & A1,A5,A6 & 3.00 & 89.00 \\
                                                    & A2,A5,A6 & 1.00 & 81.00\\
                                                    & A1,A2,A7 & 3.00 & 87.00\\
    \midrule
    \multirow{2}{0.3\columnwidth}{Cascaded Multi Attribute Attack} & A1-A2-A3 & 18.00 & 55.6 \\
                                                    & A2-A3-A4 & 20.00 & 93.00 \\
    \midrule
    \multirow{2}{0.3\columnwidth}{Multi Attribute AttGAN Attack} & A1,A2,A6,A8,A10 &  70.40 & 32.80 \\
                                                & A1,A2,A6,A8,A9,A10 & 39.40 & 40.40 \\
    \bottomrule[1.5pt]
    \end{tabular}
    }
    \caption{\sl Performance of the Semantic Adversarial Example under multiple implementations. Legend for attributes: A1-Eyeglasses, A2-Age, A3-Nose shape, A4-Eye shape, A5-Chubbiness, A6-Pale Skin, A7-Smiling, A8-Mustache, A9-Eyebrows, A10-Hair color. As the number of attributes increase, semantic attacks are more effective. Our optimization-based attack fares better as compared to worst-of-10 random sampling~\cite{engstrom2019a}, showing the former's efficacy at finding semantic adversarial examples. 
    }
    
    \label{t:fader_perf}
\end{table}
Our goal is to break this classifier model using semantic attacks. To do so, we use a subset of {500} randomly selected images from the test set. Each image is transformed by our algorithm using the various parametric transformation families described in~\Secref{sec:implementation}. 
Our metric of comparison for all adversarial attacks is the target model accuracy on the generated adversarial test set. 

\noindent\textbf{Adversarial Fader Networks:} We consider the three approaches documented in \secref{subsec:advfn}. For every image in our original test set, we generate adversarial examples by optimizing the adversarial loss in \eqref{eq:cwloss} with respect to the corresponding attribute parameters. 

In the cases of single-attribute and cascaded sequential attacks, we use the pre-trained single-attribute models provided by Lample~\etal~\cite{lample2017fader} to represent the manifold of semantic transformations. For the multi-attribute attack, we train 3 multi-attribute Fader Networks with the attributes presented in ~\Tableref{t:fader_perf}. We create an adversarial test set for each our approaches as described in \Secref{subsec:advfn} using our algorithm as defined in \Algref{alg:attack}. 

Our experiments show that Adversarial Fader Networks successfully generate examples that confound the binary classifier in all cases; see~\Tableref{t:fader_perf}. Visual adversarial examples are displayed in~\Figref{fig:single_attrib} and~\Figref{fig:sem_adv_examples}. We also observe that multi-attribute attacks outperform single-attribute attacks, which conforms with intuition; a more systematic analysis of the effect of the number of semantic attributes on attack performance is provided below in \Secref{sec:dimparams}.


\noindent\textbf{Adversarial AttGAN:} 
We perform a similar set of experiments using the multi-attribute AttGAN implementation of He~\etal\cite{He2017AttGANFA}. 
We record the performance over two experiments: one using 5 attributes, and the second using 6 attributes, as seen in \Tableref{t:fader_perf}. We observe a significant improvement in performance as the number of semantic attributes increases (in particular, adding the eyebrows attribute results in nearly a 30\% drop in model accuracy). 


\noindent\textbf{Comparison with parameter-space sampling:} We compare our method with a previously-proposed approach that investigates parametric attacks ~\etal~\cite{engstrom2019a}. They propose picking $s$ random samples from the parameter space and choose the adversarial example generated by the sample giving the worst cross entropy loss (we use $s=10$). 

We showcase the results in \Tableref{t:fader_perf}, and observe that in all cases (but one), our semantic adversarial attack algorithm outperforms random sampling. In addition, the table also reveals that random examples in the range of Fader Networks or AttGANs are mostly classified correctly. This suggests that the target model is generally invariant to the low reconstruction error incurred by the parametric transformation models\footnote{We do not compare our work with other approaches such as the Differentiable Renderer~\cite{liu2018beyond} and 3D adversarial attacks~\cite{Zeng20173dattack}, since these papers expect oracle access to a 3D rendering environment. We also do not compare with Song \etal\cite{song2018constructing} since they generate adversarial examples from scratch, whereas our attack targets specific inputs.}. 


\begin{table}[htp]
	\begin{minipage}{0.5\linewidth}
    \centering
    \begin{tabular}{c c}
    \toprule[1.5pt]
    Attack ($\epsilon=1.74$) & Accuracy(\%) \\
    \midrule
    \textbf{Single Att. Semantic Attack} & \textbf{14.01} \\
    \textbf{Multi Att. Semantic Attack} & \textbf{1.00} \\
    FGSM~\cite{adversarialexamples2015} & 91.6  \\
    PGD~\cite{madry2018towards, Kurakin2017} &  26.2 \\
    CW-$l_\infty$~\cite{Carlini2017cwl2} & 0.00  \\
    Spatial~\cite{engstrom2019a} &  41.00 \\
    \bottomrule
    \end{tabular}
    \caption{\sl Comparison of adversarial attacks with other attack strategies. A lower target accuracy corresponds to a better attack. The pixel space attacks are allowed to generate adversarial examples under the $l_\infty$ distance corresponding to our best performing multi-attribute attack model. Observe that semantic attacks are comparable to the state of the art pixel-space attack.}
    \label{t:advcomparison}
    \end{minipage}
	\hfill
	\begin{minipage}{0.5\linewidth}
			\centering
			\includegraphics[width=0.65\columnwidth]{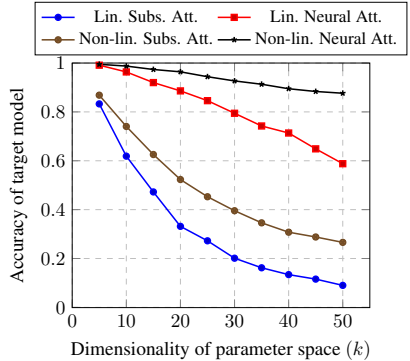} 
			\captionof{figure}{\sl Effect of dimensionality of the parametric attack space. Considering subspace and rank constrained transforms to generate adversarial examples, note that the target model accuracy decreases as the dimensionality of the attack space increases. The additive attack  (surrogate to PGD) is more effective than multiplicative attack(similar to our approach) over all values of $k$. 
			}
			\label{fig:synth_plots}
	\end{minipage}
\end{table}

\noindent\textbf{Comparison with pixel-space attacks:} In addition to our analyses described above, we also compare our attacks with the state-of-the-art Carlini-Wagner(CW) $l_\infty$-attack~\cite{Carlini2017cwl2} as well as several other attack techniques~\cite{adversarialexamples2015, Kurakin2017, engstrom2019a} in  \Tableref{t:advcomparison}. To ensure fair comparison, we consider the maximum $l_\infty$ distance over our multi-attribute attacks as the bound parameter $\epsilon$ for all pixel-norm based attacks. From the table, we observe that the CW attack is extremely effective; on the other hand, our semantic attacks are able to outperform other methods such as FGSM~\cite{Goodfellow2018existence} and PGD~\cite{madry2018towards}. 

We also compare our approach to \emph{Spatial Attacks} of~\cite{engstrom2019a}, which uses a grid search over affine transformations of an input to generate adversarial examples;  $\ell_\infty$ constraints do not apply here, and instead we use default  parameters provided in~\cite{engstrom2019a}. 
Our proposed attack methods are considerably more successful. We provide additional detailed experiments on binary and multi-class classifiers for other attributes as well as other datasets in the appendix.

\section{Analysis: Impact of Dimensionality}
\label{sec:dimparams}
From our experiments, we observe that limiting the adversary to a low-dimensional, semantic parametric transformation of the input leads to less-effective attacks than pixel-space attacks (at least when the same loss is optimized). Moreover, single-attribute semantic attacks are more powerful than multi-attribute attacks. This observation makes intuitive sense: the dimension of the manifold of perturbed inputs effectively represents the \emph{capacity} of the adversary, and hence a greater number of degrees of freedom in the perturbation should result in more effective attacks. In pixel-space attacks, the adversary is free to search over a high-dimensional $\ell_p$-ball centered around the input example, which is perhaps why $\ell_p$-norm attacks are so hard to defend against~\cite{Athalye2018obfuscated}. 

In this section, we provide experimental and theoretical analysis that precisely exposes the impact of the dimensionality of the attribute parameters. While our analysis is stylized and not directly applicable to deep neural classifiers, it constitutes a systematic first attempt towards upper bounds on what a semantically constrained adversary can possibly hope to achieve.


\subsection{Synthetic Experiments}
\label{subsec:rank}

We propose and analyze the following synthetic setup which enables explicit control over the dimension of the semantic perturbations. 

\textbf{Data:} We construct a dataset of $n = 500$ samples from a mixture of Gaussians (MoG) with 10 components (denoted by $\sP_d$) defined over $(\rvx, y) \in \sR^d \times \{ \pm 1 \}$. Each data sample is obtained by uniformly sampling one of the mixture component means, and then adding random Gaussian noise with standard deviation $\sigma \leq \sqrt{d}$. The component means are chosen as 10 randomly selected images (1 for each digit) from the MNIST dataset~\cite{mnist} rescaled to $10\times10$ (i.e., the ambient dimension is $d = 100$). 

\textbf{Target Model:} We artificially define two classes: the first class containing images generated from digits 0-4 and the second class containing images from samples 5-9. 
We train a simple two-layer fully connected network, $f(\rvx):\sR^d \to \{\pm1\}$ as the target model. The classifier is trained by optimizing cross-entropy using ADAM~\cite{kingma2015adam} for 50 epochs, resulting in training accuracy of 100\%, validation accuracy of 99.8\%, and test accuracy of 99.6\%.



\textbf{Parametric Transformations:} We consider a stylized transformation  function, $G(\rvx, \delta):\sR^d \times \sR^k \to \sR^d$. We study the effect of varying $k$ for two specific parametric transformation models.

\noindent\emph{Subspace attacks}: We first consider an additive (linear) attack model. Here, the manifold of semantic perturbations is constrained to lie a $k$-dimensional subspace spanned by an arbitrary matrix $\rmU \in \sR^{d \times k}$, whose columns are assumed to be orthonormal, and $\delta \in \sR^k$
    \begin{equation}
        \label{eq:synth_add}
        G(\rvx, \delta) := \tilde{\rvx} = \rvx + \rmU\rmU^T\delta
    \end{equation}

\noindent\emph{Neural attacks}: 
Next, we consider a \emph{multiplicative} attack model. Here the manifold of perturbations corresponds to a rank-$k$ transformation of the input. \begin{equation}
    \label{eq:synth_mul}
    \small
    G(\rvx, \delta) := \tilde{\rvx} = \rmU.diag(\delta).\rmU^T\rvx
\end{equation}
Here, $\rmU$ and $\delta$ follow the definition presented earlier. This transformation can be interpreted as the action of a shallow (two-layer) auto-encoder network with $k$ hidden neurons with scalar activations parameterized by $\delta$.

\noindent\emph{Nonlinear ReLU variants}: We also consider each of the above two attacks in the \emph{rectified} setting where the transformation is passed through a rectified linear unit: 
\(\tilde{\rvx} = \textrm{ReLU}\left(G(\rvx, \delta)\right).\)


\textbf{Results:} We analyse the effect of the dimensionality of the attack space($k$) by considering the performance of the subspace and neural attacks on the target binary classifier. \Figref{fig:synth_plots} shows the comparison of the constrained attacks for the linear and non-linear cases.

We infer the following: (i) as expected, increasing dimensionality of the semantic attack space leads to less accurate target models; (ii) adding a non-linearity to the transformation function \emph{reduces} the viability of both subspace- and rank-constrained attacks; (iii) subspace-constrained attacks are more powerful than neural attacks. In general, the degree of ``nonlinearity'' in the transformation model appears to be inversely proportional to the power of the corresponding semantic attack. We believe this phenomenon is somewhat surprising, and defer further analysis to future work.


\subsection{Theory}
\label{subsec:theory}

In the case of subspace attacks, we explicitly derive upper bounds on the generalization behavior of target models. Our derivation follows the recent approach of Schmidt \etal ~\cite{schmidt2018adversarially}, who consider a simplified version of the data model defined in \Secref{subsec:rank} and bound the performance of a linear classifier in terms of its \emph{robust classification error}.

\begin{definition}[Robust Classification Error]
Let $\sP_d:\sR^d\times \{\pm1\} \to \sR$ be a distribution and let $\mathcal{S}$ be any set containing $\rvx$. Then the $\mathcal{S}$-robust classification error of any classifier $f:\sR^d \to \{\pm 1\}$ is defined as $\beta = P_{(\rvx, y)\sim\sP_d}\left[ \exists \tilde{\rvx} \in \mathcal{S}: f(\tilde{\rvx}) \neq y\right]$.  
\end{definition}

Using this definition, we analyze the efficacy of subspace attacks on a simplified linear classifier trained using a mixture of two spherical Gaussians. 
Consider a dataset with samples $(\rvx, y) \in \sR^d \times \{\pm 1\}$ sampled from a mixture of two Gaussians with component means $\pm \theta^{\star}$ and standard deviation $\sigma \leq \sqrt{d}$. We assume a linear classifier $f_{\hat{\vw}}$, defined by the unit vector $\hat{w}$, as $f_{\hat{\vw}}(\rvx) = \text{sign}(\inner{\hat{\vw}, \rvx)}$. 

Let $\mathcal{S}_\epsilon = \{ \tilde{\rvx}~|~\tilde{\rvx} = \rvx + \rmU\rmU^T\delta,~|| \rvx - \tilde{\rvx} ||_\infty \leq \epsilon \}$. Assuming that the target classifier is well-trained (i.e., $\hat{\vw}$ is sufficiently well-correlated with the true component mean $\theta^{\star}$), we can upper bound the probability of error incurred by the classifier when subjected to any subspace attack.

\begingroup
\begin{figure}[!t]
    \centering
    \setlength{\tabcolsep}{1.5pt}
    \begin{tabular}{c c c c c}
    \includegraphics[width=0.15\columnwidth]{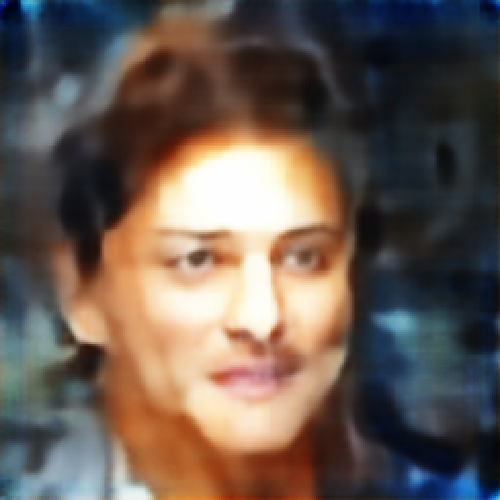} &
    \includegraphics[width=0.15\columnwidth]{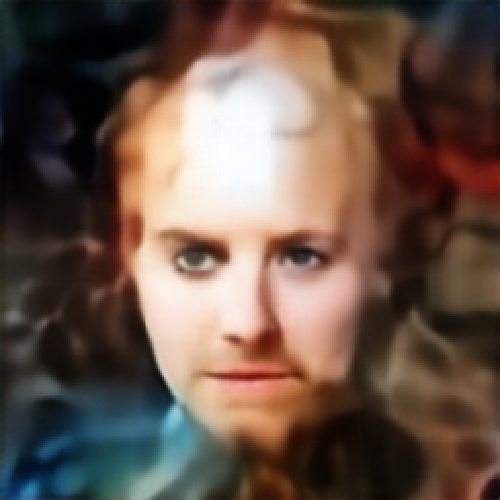} &
    \includegraphics[width=0.15\columnwidth]{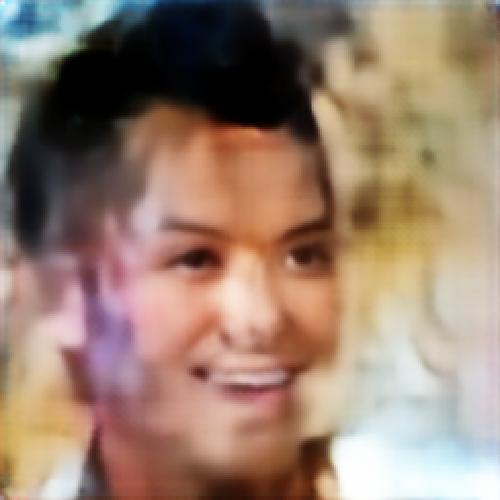} &
    \includegraphics[width=0.15\columnwidth]{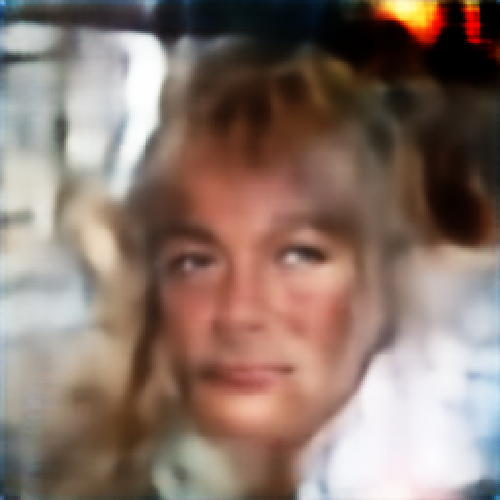} &
    \includegraphics[width=0.15\columnwidth]{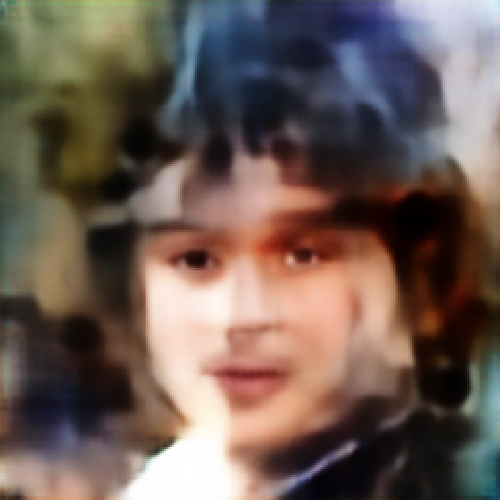}
    \end{tabular}
    \caption{\sl Semantically transformed single-attribute examples which are classified correctly by the target model but show severe artifacts. This shows that neural networks are immune to significant changes in the semantic domain unlike the pixel domain.}
    \label{fig:failures}
\end{figure}
\endgroup

\begin{thm}[Robust classification error for subspace attacks]
Let $\hat{\vw}$ be such that $\inner{\hat{\vw}, \theta^{\star}} \geq k ||\rmU||_{\infty,1} ||\hat{\vw}^T\rmU||_\infty \epsilon$.
Then, the linear classifier $f_{\hat{\vw}}$ has a $\mathcal{S}_\epsilon$-robust classification error upper bounded as:
\begin{equation}
\small
\label{eq:r_upperbound}
\beta \leq \exp\left(-\frac{\left(\inner{\hat{\vw}, \theta^{\star}} - k ||\rmU||_{\infty,1} ||\hat{\vw}^T\rmU||_\infty \epsilon \right)^2}{2\sigma^2}\right)
\end{equation}
\end{thm}

The proof is deferred to the appendix, but we provide some intuition. Lemma 20 of \cite{schmidt2018adversarially} recovers a similar result, albeit with the $\sqrt{k}$ term in the exponent being replaced by $\sqrt{d}$. This is because they only consider bounded $\ell_\infty$-perturbations in pixel-space, and hence their bound on the robust classification error scales exponentially according to the \emph{ambient dimension} $d$, while our bound is expressed in terms of the \emph{number of semantic attributes} $k \ll d$. A natural next step would be to derive \emph{sample complexity} bounds analogous to \cite{schmidt2018adversarially} but we do not pursue that direction here. 

\section{Discussion and Conclusions}
\label{sec:discussion}
We conclude with possible obstacles facing our approach and directions for future work.
We have provided evidence that there exist adversarial examples for a deep neural classifier that may be perceptible, yet are semantically meaningful and hence difficult to detect. A key obstacle is that parameters associated with semantic attributes are often difficult to decouple. 
This poses a practical challenge, as it is difficult to train a conditional generative model with independent latent semantic dimensions.
However, the success of recent efforts in this direction, including Fader Networks~\cite{lample2017fader}, AttGans~\cite{He2017AttGANFA}, and StarGANs~\cite{StarGAN2018} demonstrate promise of our approach: any newly developed conditional generative models can be used to mount a semantic attack using our framework. 

Despite the existence of semantic adversarial examples, we have found that enforcing semantic validity confounds the adversary's task, and that 
target models are generally able to classify a significant subset of the examples generated under our semantic constraint. \Figref{fig:failures} are examples of images generated with severe artifacts, yet that are successfully classified. This presents the question: is ``naturalness'' a strong defense? 

This intuition is the premise of a recent defense strategy called DefenseGAN~\cite{samangouei2018defensegan}. Indeed, our approach can be viewed as converse of this strategy: DefenseGAN uses the range-space of a generative model (specifically, a GAN) to \emph{defend} against pixel-space attacks, while conversely, we use the same principle to \emph{attack} trained target models. A closer look into the interplay between the two approaches is worthy of future study.  

\section*{Acknowledgements}

We thank Gauri Jagatap, Mohammedreza Soltani, and Anuj Sharma for helpful discussions.

\nocite{*}
{\small
\bibliographystyle{ieee}
\bibliography{egbib}

\begin{thebibliography}{10}\itemsep=-1pt

\bibitem{Athalye2018obfuscated}
A.~Athalye, N.~Carlini, and D.~A. Wagner.
\newblock Obfuscated gradients give a false sense of security: Circumventing
  defenses to adversarial examples.
\newblock In {\em ICML}, 2018.

\bibitem{Athalye2018SynthesizingRA}
A.~Athalye, L.~Engstrom, A.~Ilyas, and K.~Kwok.
\newblock Synthesizing robust adversarial examples.
\newblock In {\em ICML}, 2018.

\bibitem{boyd2004}
S.~Boyd and L.~Vandenberghe.
\newblock {\em Convex Optimization}.
\newblock Cambridge University Press, 2004.

\bibitem{Brown2017advpatch}
T.~B. Brown, D.~Man{\'e}, A.~Roy, M.~Abadi, and J.~Gilmer.
\newblock Adversarial patch.
\newblock {\em arXiv preprint arXiv:1712.09665}, 2017.

\bibitem{Carlini2017cwl2}
N.~Carlini and D.~A. Wagner.
\newblock Towards evaluating the robustness of neural networks.
\newblock {\em 2017 IEEE Symposium on Security and Privacy (SP)}, 2017.

\bibitem{chen2016infogan}
X.~Chen, Y.~Duan, R.~Houthooft, J.~Schulman, I.~Sutskever, and P.~Abbeel.
\newblock Infogan: Interpretable representation learning by information
  maximizing generative adversarial nets.
\newblock In {\em NeurIPS}, 2016.

\bibitem{Chen2017backdoor}
X.~Chen, C.~Liu, B.~Li, K.~Lu, and D.~Song.
\newblock {Targeted Backdoor Attacks on Deep Learning Systems Using Data
  Poisoning}.
\newblock {\em arxiv preprint}, abs/1712.05526, 2017.

\bibitem{StarGAN2018}
Y.~Choi, M.~Choi, M.~Kim, J.-W. Ha, S.~Kim, and J.~Choo.
\newblock Stargan: Unified generative adversarial networks for multi-domain
  image-to-image translation.
\newblock In {\em CVPR}, 2018.

\bibitem{Dabouei2019FastGA}
A.~Dabouei, S.~Soleymani, J.~M. Dawson, and N.~M. Nasrabadi.
\newblock Fast geometrically-perturbed adversarial faces.
\newblock {\em WACV}, 2019.

\bibitem{Dathathri2017minadvexample}
S.~Dathathri, S.~Zheng, S.~Gao, and R.~Murray.
\newblock {Measuring the Robustness of Neural Networks via Minimal Adversarial
  Examples}.
\newblock In {\em NeurIPS-W}, volume~35, 2017.

\bibitem{Rey-de-Castro2018targettedattacks}
R.~R. de~Castro and H.~A. Rabitz.
\newblock Targeted nonlinear adversarial perturbations in images and videos.
\newblock {\em arxiv preprint}, abs/1809.00958, 2018.

\bibitem{engstrom2019a}
L.~Engstrom, D.~Tsipras, L.~Schmidt, and A.~Madry.
\newblock A rotation and a translation suffice: Fooling cnns with simple
  transformations.
\newblock {\em arxiv preprint}, abs/1712.02779, 2017.

\bibitem{Eykholt2018RobustPA}
K.~Eykholt, I.~Evtimov, E.~Fernandes, B.~Li, A.~Rahmati, C.~Xiao, A.~Prakash,
  T.~Kohno, and D.~X. Song.
\newblock Robust physical-world attacks on deep learning visual classification.
\newblock {\em CVPR}, 2018.

\bibitem{Fawzi2018AdversarialVF}
A.~Fawzi, H.~Fawzi, and O.~Fawzi.
\newblock Adversarial vulnerability for any classifier.
\newblock In {\em NeurIPS}, 2018.

\bibitem{Fawzi2018advrobustness}
A.~Fawzi, O.~Fawzi, and P.~Frossard.
\newblock Analysis of classifiers' robustness to adversarial perturbations.
\newblock {\em Machine Learning}, 107, 2018.

\bibitem{adversarialexamples2015}
I.~Goodfellow, J.~Shlens, and C.~Szegedy.
\newblock Explaining and harnessing adversarial examples.
\newblock In {\em ICLR}, 2015.

\bibitem{Goodfellow2018existence}
I.~J. Goodfellow.
\newblock Defense against the dark arts: An overview of adversarial example
  security research and future research directions.
\newblock {\em arxiv preprint}, abs/1806.04169, 2018.

\bibitem{GAN}
I.~J. Goodfellow, J.~Pouget-Abadie, M.~Mirza, B.~Xu, D.~Warde-Farley, S.~Ozair,
  A.~C. Courville, and Y.~Bengio.
\newblock Generative adversarial nets.
\newblock In {\em NeurIPS}, 2014.

\bibitem{BadNets}
T.~Gu, B.~Dolan-Gavitt, and S.~Garg.
\newblock Badnets: Identifying vulnerabilities in the machine learning model
  supply chain.
\newblock {\em arxiv preprint}, abs/1708.06733, 2017.

\bibitem{He2017AttGANFA}
Z.~He, W.~Zuo, M.~Kan, S.~Shan, and X.~Chen.
\newblock Attgan: Facial attribute editing by only changing what you want.
\newblock {\em arxiv preprint}, 2017.

\bibitem{rmsprop}
G.~Hinton, N.~Srivastava, and K.~Swersky.
\newblock Lecture 6a, overview of mini-batch gradient descent.

\bibitem{Ilyas2018alimitedqueries}
A.~Ilyas, L.~Engstrom, A.~Athalye, and J.~Lin.
\newblock Black-box adversarial attacks with limited queries and information.
\newblock In {\em PMLR}, volume~80, 2018.

\bibitem{Ilyas2018blackbox}
A.~Ilyas, L.~Engstrom, and A.~Madry.
\newblock Prior convictions: Black-box adversarial attacks with bandits and
  priors.
\newblock {\em arxiv preprint}, abs/1807.07978, 2018.

\bibitem{Kaneko2017GenerativeAC}
T.~Kaneko, K.~Hiramatsu, and K.~Kashino.
\newblock Generative attribute controller with conditional filtered generative
  adversarial networks.
\newblock {\em CVPR}, 2017.

\bibitem{Kim2017UnsupervisedVA}
T.~Kim, B.~Kim, M.~Cha, and J.~Kim.
\newblock Unsupervised visual attribute transfer with reconfigurable generative
  adversarial networks.
\newblock {\em arxiv preprint}, abs/1707.09798, 2017.

\bibitem{kingma2015adam}
D.~Kingma and J.~Ba.
\newblock Adam: a method for stochastic optimization (2014).
\newblock In {\em ICLR}, 2015.

\bibitem{Kingma2014}
D.~P. Kingma and M.~Welling.
\newblock Auto-encoding variational bayes.
\newblock {\em arxiv preprint}, abs/1312.6114, 2014.

\bibitem{koh2017understanding}
P.~W. Koh and P.~Liang.
\newblock Understanding black-box predictions via influence functions.
\newblock In {\em JMLR}, volume~70, 2017.

\bibitem{Kurakin2017}
A.~Kurakin, I.~J. Goodfellow, and S.~Bengio.
\newblock Adversarial examples in the physical world.
\newblock {\em arxiv preprint}, abs/1607.02533, 2017.

\bibitem{lample2017fader}
G.~Lample, N.~Zeghidour, N.~Usunier, A.~Bordes, L.~Denoyer, et~al.
\newblock Fader networks: Manipulating images by sliding attributes.
\newblock In {\em NeurIPS}, 2017.

\bibitem{AutoencodingBP}
A.~B.~L. Larsen, S.~K. S{\o}nderby, and O.~Winther.
\newblock Autoencoding beyond pixels using a learned similarity metric.
\newblock In {\em ICML}, 2016.

\bibitem{mnist}
Y.~LeCun and C.~Cortes.
\newblock {MNIST} handwritten digit database, 2010.

\bibitem{Li2016}
M.~Li, W.~Zuo, and D.~Zhang.
\newblock Convolutional network for attribute-driven and identity-preserving
  human face generation.
\newblock {\em arxiv preprint}, abs/1608.06434, 2016.

\bibitem{Li2016DeepIT}
M.~Li, W.~Zuo, and D.~Zhang.
\newblock Deep identity-aware transfer of facial attributes.
\newblock {\em arxiv preprint}, abs/1610.05586, 2016.

\bibitem{liu2018beyond}
H.-T.~D. Liu, M.~Tao, C.-L. Li, D.~Nowrouzezahrai, and A.~Jacobson.
\newblock Beyond pixel norm-balls: Parametric adversaries using an analytically
  differentiable renderer.
\newblock In {\em ICLR}, 2019.

\bibitem{UnsupervisedIT}
M.-Y. Liu, T.~Breuel, and J.~Kautz.
\newblock Unsupervised image-to-image translation networks.
\newblock In {\em NeurIPS}, 2017.

\bibitem{liu2015faceattributes}
Z.~Liu, P.~Luo, X.~Wang, and X.~Tang.
\newblock Deep learning face attributes in the wild.
\newblock In {\em ICCV}, 2015.

\bibitem{AttCycle}
Y.~Lu, Y.-W. Tai, and C.-K. Tang.
\newblock Attribute-guided face generation using conditional cyclegan.
\newblock In {\em ECCV}, 2018.

\bibitem{madry2018towards}
A.~Madry, A.~Makelov, L.~Schmidt, D.~Tsipras, and A.~Vladu.
\newblock Towards deep learning models resistant to adversarial attacks.
\newblock In {\em ICLR}, 2018.

\bibitem{Mirza2014ConditionalGA}
M.~Mirza and S.~Osindero.
\newblock Conditional generative adversarial nets.
\newblock {\em arxiv preprint}, abs/1411.1784, 2014.

\bibitem{MoosaviDezfooli2017UniversalAP}
S.-M. Moosavi-Dezfooli, A.~Fawzi, O.~Fawzi, and P.~Frossard.
\newblock Universal adversarial perturbations.
\newblock {\em CVPR}, 2017.

\bibitem{DeepFool}
S.-M. Moosavi-Dezfooli, A.~Fawzi, and P.~Frossard.
\newblock Deepfool: A simple and accurate method to fool deep neural networks.
\newblock {\em CVPR}, 2016.

\bibitem{MoosaviDezfooli2018RobustnessVC}
S.-M. Moosavi-Dezfooli, A.~Fawzi, J.~Uesato, and P.~Frossard.
\newblock Robustness via curvature regularization, and vice versa.
\newblock In {\em CVPR}, 2019.

\bibitem{Mopuri2018NAGNF}
K.~R. Mopuri, U.~Ojha, U.~Garg, and R.~V. Babu.
\newblock Nag: Network for adversary generation.
\newblock {\em CVPR}, 2018.

\bibitem{Odena2017}
A.~Odena, C.~Olah, and J.~Shlens.
\newblock Conditional image synthesis with auxiliary classifier gans.
\newblock In {\em ICML}, 2017.

\bibitem{Papernot2016}
N.~Papernot, P.~D. McDaniel, S.~Jha, M.~Fredrikson, Z.~B. Celik, and A.~Swami.
\newblock The limitations of deep learning in adversarial settings.
\newblock {\em EuroS\&P}, 2016.

\bibitem{pytorch}
A.~Paszke, S.~Gross, S.~Chintala, G.~Chanan, E.~Yang, Z.~DeVito, Z.~Lin,
  A.~Desmaison, L.~Antiga, and A.~Lerer.
\newblock Automatic differentiation in pytorch.
\newblock In {\em NeurIPS-W}, 2017.

\bibitem{InvertibleCG}
G.~Perarnau, J.~van~de Weijer, B.~Raducanu, and J.~M. {\'A}lvarez.
\newblock Invertible conditional gans for image editing.
\newblock {\em arxiv preprint}, abs/1611.06355, 2016.

\bibitem{Radford2016UnsupervisedRL}
A.~Radford, L.~Metz, and S.~Chintala.
\newblock Unsupervised representation learning with deep convolutional
  generative adversarial networks.
\newblock {\em arxiv preprint}, abs/1511.06434, 2016.

\bibitem{rauber2017foolbox}
J.~Rauber, W.~Brendel, and M.~Bethge.
\newblock Foolbox: A python toolbox to benchmark the robustness of machine
  learning models.
\newblock {\em arXiv preprint arXiv:1707.04131}, 2017.

\bibitem{samangouei2018defensegan}
P.~Samangouei, M.~Kabkab, and R.~Chellappa.
\newblock Defense-{GAN}: Protecting classifiers against adversarial attacks
  using generative models.
\newblock In {\em ICLR}, 2018.

\bibitem{schmidt2018adversarially}
L.~Schmidt, S.~Santurkar, D.~Tsipras, K.~Talwar, and A.~Madry.
\newblock Adversarially robust generalization requires more data.
\newblock In {\em NeurIPS}, 2018.

\bibitem{Shafahi2018poisonfrogs}
A.~Shafahi, W.~R. Huang, M.~Najibi, O.~Suciu, C.~Studer, T.~Dumitras, and
  T.~Goldstein.
\newblock Poison frogs! targeted clean-label poisoning attacks on neural
  networks.
\newblock In {\em NeurIPS}, 2018.

\bibitem{Shafahi2017inevitable}
A.~Shafahi, W.~R. Huang, C.~Studer, S.~Feizi, and T.~Goldstein.
\newblock Are adversarial examples inevitable?
\newblock In {\em ICLR}, 2019.

\bibitem{Sharif2018advgennets}
M.~Sharif, S.~Bhagavatula, L.~Bauer, and M.~K. Reiter.
\newblock Adversarial generative nets: Neural network attacks on
  state-of-the-art face recognition.
\newblock {\em arxiv preprint}, abs/1801.00349, 2018.

\bibitem{Shen2017LearningRI}
W.~Shen and R.~Liu.
\newblock Learning residual images for face attribute manipulation.
\newblock {\em CVPR}, 2017.

\bibitem{song2018constructing}
Y.~Song, R.~Shu, N.~Kushman, and S.~Ermon.
\newblock Constructing unrestricted adversarial examples with generative
  models.
\newblock In {\em NeurIPS}, 2018.

\bibitem{Szegedy2014intriguing}
C.~Szegedy, W.~Zaremba, I.~Sutskever, J.~Bruna, D.~Erhan, I.~Goodfellow, and
  R.~Fergus.
\newblock Intriguing properties of neural networks.
\newblock {\em arXiv preprint arXiv:1312.6199}, 2013.

\bibitem{Tange2011a}
O.~Tange.
\newblock Gnu parallel - the command-line power tool.
\newblock {\em ;login: The USENIX Magazine}, 36(1):42--47, Feb. 2011.

\bibitem{Tramr2017EnsembleAT}
F.~Tram{\`e}r, A.~Kurakin, N.~Papernot, D.~Boneh, and P.~D. McDaniel.
\newblock Ensemble adversarial training: Attacks and defenses.
\newblock {\em arxiv preprint}, abs/1705.07204, 2017.

\bibitem{tran2018}
B.~Tran, J.~Li, and A.~Madry.
\newblock Spectral signatures in backdoor attacks.
\newblock In {\em NeurIPS}, 2018.

\bibitem{turner2019}
A.~Turner, D.~Tsipras, and A.~Madry.
\newblock Clean-label backdoor attacks, 2019.

\bibitem{2017DeepFI}
P.~Upchurch, J.~R. Gardner, G.~Pleiss, R.~Pless, N.~Snavely, K.~Bala, and K.~Q.
  Weinberger.
\newblock Deep feature interpolation for image content changes.
\newblock {\em CVPR}, 2017.

\bibitem{wang2016generative}
X.~Wang and A.~Gupta.
\newblock Generative image modeling using style and structure adversarial
  networks.
\newblock In {\em ECCV}, 2016.

\bibitem{Xiao2018SpatiallyTA}
C.~Xiao, J.-Y. Zhu, B.~Li, W.~He, M.~Liu, and D.~X. Song.
\newblock Spatially transformed adversarial examples.
\newblock {\em arxiv preprint}, abs/1801.02612, 2018.

\bibitem{Xiao2015}
H.~Xiao, B.~Biggio, B.~Nelson, H.~Xiao, C.~M. Eckert, and F.~Roli.
\newblock Support vector machines under adversarial label contamination.
\newblock {\em Neurocomputing}, 160, 2015.

\bibitem{Xiao2012}
H.~Xiao, H.~Xiao, and C.~M. Eckert.
\newblock Adversarial label flips attack on support vector machines.
\newblock In {\em ECAI}, 2012.

\bibitem{Xiao2018DNAGANLD}
T.~Xiao, J.~Hong, and J.~Ma.
\newblock Dna-gan: Learning disentangled representations from multi-attribute
  images.
\newblock {\em arxiv preprint}, abs/1711.05415, 2018.

\bibitem{bdd100k}
F.~Yu, W.~Xian, Y.~Chen, F.~Liu, M.~Liao, V.~Madhavan, and T.~Darrell.
\newblock Bdd100k: A diverse driving video database with scalable annotation
  tooling.
\newblock {\em arXiv preprint arXiv:1805.04687}, 2018.

\bibitem{Zeng20173dattack}
X.~Zeng, C.~Liu, Y.-S. Wang, W.~Qiu, L.~Xie, Y.-W. Tai, C.-K. Tang, and A.~L.
  Yuille.
\newblock Adversarial attacks beyond the image space.
\newblock {\em arxiv preprint}, abs/1711.07183, 2017.

\bibitem{zhang2019camou}
Y.~Zhang, H.~Foroosh, P.~David, and B.~Gong.
\newblock Camou: Learning physical vehicle camouflages to adversarially attack
  detectors in the wild.
\newblock In {\em ICLR}, 2019.

\bibitem{zhao2018generating}
Z.~Zhao, D.~Dua, and S.~Singh.
\newblock Generating natural adversarial examples.
\newblock In {\em ICLR}, 2018.

\bibitem{Zhou2017GeneGANLO}
S.~Zhou, T.~Xiao, Y.~Yang, D.~Feng, Q.~He, and W.~He.
\newblock Genegan: Learning object transfiguration and attribute subspace from
  unpaired data.
\newblock {\em arxiv preprint}, abs/1705.04932, 2017.

\bibitem{Zhu2017UnpairedIT}
J.-Y. Zhu, T.~Park, P.~Isola, and A.~A. Efros.
\newblock Unpaired image-to-image translation using cycle-consistent
  adversarial networks.
\newblock {\em ICCV}, 2017.

\end{thebibliography}
}

\appendix

\section{Related Work}

\noindent\textbf{Adversarial Examples and Attacks.} \\
In 2014, Szegedy \etal~\cite{Szegedy2014intriguing} shows that deep neural networks had mainly two counter intuitive properties, stating that the space described by higher layers of neural networks captures semantic information and there exists adversarial examples which questioned the generalization ability of a neural network. They generate such adversarial examples under the \textit{$L_2$} distance constraint which look similar to the original images but are classified with a different label by the classifier using a box constrained L-BFGS attack. 

Goodfellow et. al~\cite{adversarialexamples2015} and Kurakin~\etal~\cite{Kurakin2017} generate adversarial examples using Fast Gradient Sign method and its iterative variant under the \textit{$l_{\infty}$} constraint in less computation time. Other methods similar to FGSM have been mentioned in ~\cite{Tramr2017EnsembleAT}.  

Papernot \etal~\cite{Papernot2016} implements an attack under the $l_0$ constraint where they modify the pixel having the most significant contribution in changing the classification of the model to the target class. 
Moosavi-Dezfooli \etal~\cite{DeepFool} describe an untargeted attack algorithm under the $L_2$ constraint with the assumption that neural networks are linear in nature which they further extend to non-linear neural networks. Another family of attacks relates to a single universal adversarial direction for a dataset. Moosavi-Dezfooli~\etal~\cite{MoosaviDezfooli2017UniversalAP} prove the existence of an image-agnostic adversarial perturbation. Fawzi~\etal~\cite{Fawzi2018advrobustness} extend this to theoretically show that every classifier is vulnerable to adversarial attacks. Moosavi-Dezfooli~\etal further consider the effect of the curvature of the decision boundaries on the existence of adversarial examples in \cite{MoosaviDezfooli2018RobustnessVC}.    

Carlini and Wagner~\cite{Carlini2017cwl2} propose three attacks for adversarial image generation and shows that defensive distillation is not an effective defence mechanism. They devise attacks under the three norms in literature $l_1$, $l_2$ and $l_{\infty}$ to measure the deviation of adversarial perturbation from the original sample over seven different surrogate loss functions and finally selecting one of them which we use in our attack algorithm as well. The attack that they implement in this work is proven to be the most effective attack in literature and is a benchmark for comparison.

The primary difference between the aforementioned attacks and our attack is that these attacks perturb the image and make imperceptible changes in the pixel space and thereby not modifying the image in a semantic way. On the other hand, our attack focuses making naturalistic perceptible changes to the image which are semantic in nature and realistic.  

\noindent\textbf{Parametric adversarial attacks.}\\
The use of parametric transformations to generate adversarial examples has been tackled by several previous works. Most of these parametric attacks target the image formation process to create adversarial example. A recent work by Liu~\etal perturbs geometrical surfaces or lighting by optimizing over the relevant parameters for a 3D environment. They show convincing results with realistic looking adversarial examples. Zeng~\etal~\cite{Zeng20173dattack} use FGSM to perturb 3D models of objects to create adversarial examples. The primary caveat to such approaches is that they require precise 3D models of the objects that they create adversarial examples.
Athalye~\etal~\cite{Athalye2018SynthesizingRA} demonstrate the creation of a real-world adversarial 3D model using optimization over affine transformations corresponding to real-world realizations. Eykhol~\etal~\cite{Eykholt2018RobustPA} also provide mechanisms for real-world realizable adversarial examples for stop signs using designed adversarial stickers.

Mopuri~\etal~\cite{Mopuri2018NAGNF} train a generative adversarial network to generate adversarial attacks for classifiers. 
Zhao~\etal~\cite{zhao2018generating} show an interesting use of a GAN and an inverter network where they search over the input space of the GAN to generate semantically valid adversarial examples. These approaches are morally similar to our approach though we focus on specific physically perturbed attributes of images rather than imperceptible perturbations. CAMOU~\cite{zhang2019camou} is a more recent work that learns a neural approximator for physical camouflage and then optimizes over the same to generate an adversarial version to fool object detectors.

The space of generating adversarial examples using GANs for face recognition systems has also been touched upon by Dabouei~\etal~\cite{Dabouei2019FastGA} and Sharif~\etal~\cite{Sharif2018advgennets} which train generative networks for the specific purpose of creating adversarial examples. Sharif ~\etal especially show a realizable attack by adding glasses using a generative network to fool a face recognition classifier. We, in comparison, provide a more diverse attack space allowing for various semantic attributes. In addition, since our attack involves physically realizable perceptible attributes, it can be used to characterize a classifier's performance against physical adversarial attacks as well. 

Song~\etal~\cite{song2018constructing} uses an Auxiliary Class Generative Adversarial Network (AC-GAN)~\cite{Odena2017} to generate unrestricted adversarial examples from noise and then optimizes over the latent space of the conditional GAN to find such adversarial examples which get missclassified by a gender classifier. The paper describes the use of \textit{Mechanical Turk} as a checker for naturalness and validation for the generated images belonging to the desired class. We approach the more complex problem of finding an adversarial transformation for an input image instead of generating a random semantic adversarial example.   

\noindent\textbf{Attribute based generative models.} \\

Our approach relies on the use of attribute based generative models for enforcing the semantic constraint and representing attributes as a real-valued semantic variable. we discuss a few relevant approaches published recently.

As mentioned in ~\cite{He2017AttGANFA}, the literature related to facial attribute editing can be broadly divided into two sections, optimization based approaches and learning based approaches. Optimization approaches include Li~\etal~\cite{Li2016} and Gardner\etal~\cite{2017DeepFI} where the former optimizes the CNN feature difference between the input face image and the face images with the desired attributes with respect to the input face while the latter optimizes the input face in order to match the deep feature along the direction vector between the faces with and without the attributes. 

Li \etal~\cite{Li2016DeepIT} describe a method to optimize over an adversarial attribute loss and a deep identity feature loss in order to train a deep identity aware transfer model to add or remove facial attributes to/from a face. Shen \etal~\cite{Shen2017LearningRI} learn the difference between images before and after manipulation to simultaneously train two networks for respectively adding and removing a specific attribute. 

Generative Adversarial Networks(GAN)~\cite{GAN} are a popular approach for the generation of samples from a real-world data distribution. Recent advancements~\cite{Radford2016UnsupervisedRL, UnsupervisedIT, wang2016generative, chen2016infogan} in GANs allow for creation of high dimensional, high quality realistic images. These have been incorporated into the several attribute swapping generative models. Zhou \etal~\cite{Zhou2017GeneGANLO} recombine the information of the latent information of two images to swap a specific attribute between the given images. Liu \etal~\cite{UnsupervisedIT} generate high quality images by coupling GANs in order to learn a shared latent representation in order to tackle several unsupervised image translation tasks including domain adaptation and face image translation.

For multiple attribute swapping, models based on Kingma \etal~\cite{Kingma2014}, Goodfellow \etal~\cite{GAN},Larsen~\etal~\cite{AutoencodingBP}, Mirza \etal~\cite{Mirza2014ConditionalGA}, Radford \etal~\cite{Radford2016UnsupervisedRL} have become quite popular recently. 
Perarnaue \etal~\cite{InvertibleCG} uses a Conditional Generative Adversarial Network~\cite{Mirza2014ConditionalGA} and encoder to learn the attribute invariant latent representation for attribute editing. Similar work has been seen in Fader Networks~\cite{lample2017fader} where the model learns the attribute invariant latent space in order to identify a face as one and the same with or without a specific attribute. On the other hand, AttGAN~\cite{He2017AttGANFA} argues that such attribute invariant constraint is a bit too excessive and imposes an attribute classification constraint and a reconstruction loss instead to alter only the desired attributes preserving attribute-excluding features. StarGAN ~\cite{StarGAN2018} uses a cyclic consistency loss to preserve information and instead of learning a latent representation, it trains a conditional attribute transfer network to modify attributes. 
Chen \etal~\cite{chen2016infogan} and Odena \etal~\cite{Odena2017} map the generated images back to the conditional signals with the help of an auxiliary classifier to learn this conditional generation of the images. Kaneko \etal~\cite{Kaneko2017GenerativeAC} uses a conditional filtered generative adversarial network to present a generative attribute controller to edit attributes of an image while preserving the variations of an attribute. 

Xiao \etal~\cite{Xiao2018DNAGANLD} swaps blocks of the latent distribution containing relevant attributes between a given pair of images. A similar approach has been seen in Kim\etal~\cite{Kim2017UnsupervisedVA} where the latent representation is divided in blocks corresponding different attributes and these latent blocks are swapped in order to achieve multiple attribute swapping.

\noindent\textbf{Data poisoning.}\\
Much of the prior work mentioned discuss about adversarial attacks during inference. Data poisoning is a technique where the adversary injects false data to hinder the generalization capability of a deep neural network. Koh~\etal ~\cite{koh2017understanding} present the seminal work on data poisoning for deep neural networks where they construct approximate upper bounds to provide certificates to a large class of attacks. Xiao~\etal~\cite{Xiao2012} and Xiao~\etal~\cite{Xiao2015} also present a similar approach but on shallow learning models. Another class of data poisoning attack is referred to as a \textit{backdoor attack}, where an adversary corrupts the model to misclassify either a specific input or a group of inputs to a target label thus engineering a \textit{backdoor} that can be used to corrupt the learned model. Gu~\etal~\cite{BadNets} demonstrate a method to train a network maliciously with good performance on training and validation datasets but persistent poor performance on inputs associated with backdoor triggers.

These attacks can be realistic in nature, for e.g., a stop sign can be identified by the classifier as a speed limit sign in the presence of backdoor triggers which are mainly special markers added to the inputs by the adversary. Turner~\etal~\cite{turner2019} show that an adversary is able to gain whole control over the target model during inference, by training with samples generated with a GAN. More recently Tran~\etal~\cite{tran2018} identify a property related to all backdoor attacks known as spectral signatures with which poisoned examples from real image datasets can be detected and removed effectively. Chen~\etal~\cite{Chen2017backdoor} demonstrate an application of such backdoor attacks on a visual recognition system where they were able to break a weak threat model with a limited number of poisoned data examples with semantic attribute changes. This is perhaps the first attempt at considering the effect of semantic changes.


\section{Theoretical Results}

\subsection*{Robust classification error for subspace attacks}

We present a proof for the upper bound of the robust classification error in the case of subspace attacks. Recall the data model we use; a Mixture of Gaussians data model, $\sP_d(\theta^{\star}, \sigma)~\sim~\sR^d \times {\pm1}$ with two components and $\sigma \le \sqrt{d}$. Each of the components are regarded as classes. We additionally  assume a linear classifier, $f_{\vw}:\sR^d \to \{\pm1\}$ defined by the unit vector, $\hat{\vw} = sgn(\inner{\hat{\vw}, \rvx})$. 

Let $\mathcal{S}_\epsilon = \{\tilde{\rvx}~|~\tilde{\rvx} = \rvx+\rmU\rmU^T\delta,~|| \tilde{\rvx} - \rvx||_\infty \leq \epsilon \} $ and $rank\ \rmU = k$.

Under the assumption that the linear classifier is well trained, \textit{i.e.}, $\hat{\vw}$ is sufficiently correlated with the true component mean, $\theta^{\star}$, we upper bound the robust classification error. This involves considering the sample generalization error of a linear classifier on Gaussian data. We adapt arguments from  Schmidt~\etal~\cite{schmidt2018adversarially} for the case of subspace attacks. The theorem statement is repeated here for convenience. 

\setcounter{thm}{0}
\begin{thm}
Let $\hat{\vw}$ be such that $\inner{\hat{\vw}, \theta^{\star}} \geq k ||U||_{\infty,1} ||\hat{\vw}^T\rmU||_\infty \epsilon$.
Then, the linear classifier $f_{\hat{\vw}}$ has a $\mathcal{S}_\epsilon$-robust classification error upper bounded as:
\begin{equation}
\small
\label{eq:r_upperbound_apdx}
\beta \leq \exp\left(-\frac{\left(\inner{\hat{\vw}, \theta^{\star}} - k ||U||_{\infty,1} ||\hat{\vw}^T\rmU||_\infty \epsilon \right)^2}{2\sigma^2}\right)
\end{equation}
\end{thm}

\begin{proof}
For proving the above statement, we consider the probability of adversarial misclassification under a rank constrained attack.

Given $(\rvx, y) \in \sR^d \times {\pm 1}$ where $\rvx \sim MoG(\theta^\star, \sigma$; $\sigma \leq \sqrt{d}$, we consider a linear additive attack under a rank constraint,
\begin{equation}
    \label{eq:attack_add}
    \tilde{\rvx} = \rvx + \rmU\rmU^T\delta
\end{equation}
Here, $\rmU \in \sM_{d,k}$ is a random matrix with the columns forming an orthonormal basis of dimensionality $k$.
In addition, we consider that the adversarial example thus created is constrained to be in the norm ball, $\mathcal{B}_\infty^\eps$, which implies that 
\begin{equation}
    \label{eq:normconstraint}
    ||\tilde{\rvx} - \rvx||_\infty \leq \epsilon
\end{equation}

We attempt to bound the probability that a rank constrained adversarial example, $\tilde{\rvx}$, created using \eqref{eq:attack_add}, exists under the constraint defined by ~\eqref{eq:normconstraint}.
\begin{align*}
    \beta_\infty^\epsilon = &  \sP \left(\exists \tilde{\rvx}, \tilde{\rvx}\in\mathcal{B}_\infty^\epsilon s.t. \inner{y\tilde{\rvx},\hat{\vw}} \leq 0 \right) \\
    = & \sP(\exists \delta\ : \ ||\rmU\rmU^T \delta||_\infty \leq \epsilon, \inner{y (\rvx + \rmU\rmU^T \delta) , \hat{\vw}} \leq 0) \\
    = & \sP(\inner{y\rvx, \hat{\vw}} + \min\limits_{||\rmU\rmU^T \delta||_\infty \leq \epsilon} \inner{y\rmU\rmU^T\delta, \hat{\vw}} \leq 0)
\end{align*}

Let $\delta'\triangleq \rmU^T\delta$.

Now,
\begin{equation}
  \label{eq:minprob}
  \beta_\infty^\epsilon = \sP\left(\inner{y\rvx, \hat{\vw}} + \min\limits_{||\rmU\delta'|| \leq \epsilon} \inner{y\rmU\delta', \hat{\vw}} \leq 0\right)
\end{equation}

Consider the domain of the minimization, 
\begin{align*}
    & ||\rmU\delta'||_\infty \leq \epsilon \\
\end{align*}
Now using the definition of the $(1,\infty)$ operator norm for rectangular matrices~(See \cite{boyd2004}, Sec A.1.5) and the fact that $\rmU$ is orthonormal, 


\begin{align*}
    ||\delta'||_\infty \leq ||\delta'||_1 = ||\rmU^T \rmU \delta'||_1 \leq ||\rmU^T||_{1,\infty}||\rmU \delta'||_\infty \leq ||\rmU||_{\infty, 1}\epsilon
\end{align*}


Let set $A \triangleq \{\delta' : ||\delta'||_\infty \leq ||\rmU||_{\infty, 1}\epsilon \}$ and set $B \triangleq \{\delta': ||\rmU\delta'||_\infty \leq \epsilon\}$. We can clearly see that $B \subseteq A$. 
Now considering the $f=\inner{y\rmU\delta',\hat{w}}$, as 
\begin{equation}
    \label{eq:minim_incl}
    \min_A f \leq \min_B f \nonumber
\end{equation}
Thus we show that,
\begin{equation}
\inner{y\rvx, \hat{\vw}} + \min_A f \leq \inner{y\rvx, \hat{\vw}} + \min_B f
\end{equation}

From the above inequality, $RHS \leq 0 \implies LHS \leq 0$ but not vice versa.

By using the inclusion argument of probability measure, we can therefore show that, 
\begin{equation}
    \label{eq:minim_inclusion}
    \sP\left(\inner{y\rvx, \hat{\vw}} + \min_B f \leq 0 \right) \leq \sP\left(\inner{y\rvx, \hat{\vw}} + \min_A f \leq 0\right)
\end{equation}

We now upper bound the $RHS$ term using the same argument as that of Lemma 20 in \cite{schmidt2018adversarially}.

\begin{align*}
    & \sP\left( \inner{y\rvx, \hat{\vw}} + \min\limits_{||\delta'||_\infty \leq ||\rmU||_{\infty, 1}\epsilon} \inner{y\rmU\delta', \hat{\vw}} \leq 0 \right) \\
    = & \sP\left( \inner{y\rvx, \hat{\vw}} + \min\limits_{||\delta'||_\infty \leq ||\rmU||_{\infty, 1}\epsilon} y\hat{\vw}^T\rmU\delta' \leq 0 \right) \\
    & \text{Let $\bar{\vw} \triangleq \hat{\vw}^T\rmU$}\\
    = & \sP\left( \inner{y\rvx, \hat{\vw}} + \min\limits_{||\delta'||_\infty \leq ||\rmU||_{\infty, 1}\epsilon} y\bar{\vw}\delta' \leq 0 \right)
\end{align*}
We now drop $y$ as the constraint is symmetric and use definition of dual norm,
\begin{align*}
    & \sP\left(\inner{y\rvx, \hat{\vw}} - ||\rmU||_{\infty, 1}||\bar{\vw}||_{\infty}^\star\epsilon \leq 0\right) \\
    = & \sP\left(\inner{y\rvx, \hat{\vw}} \leq ||\rmU||_{\infty, 1}||\bar{\vw}||^\star_\infty \epsilon\right) \\
    = & \sP\left(\inner{y\rvx, \hat{\vw}} \leq k ||\rmU||_{\infty, 1}||\bar{\vw}||_\infty \epsilon\right)
\end{align*}

We now invoke Lemma 17 from \cite{schmidt2018adversarially} with $\mu=\theta^\star$ and $\rho=||\rmU||_{\infty, 1}\epsilon||\bar{\vw}||_\infty$ to bound the $RHS$, 
\begin{equation}
\label{eq:proof_upperbound}
\beta \leq \exp\left(-\frac{\left(\inner{\hat{\vw}, \theta^{\star}} - k ||\rmU||_{\infty,1} ||\hat{\vw}^T\rmU||_\infty \epsilon \right)^2}{2\sigma^2}\right)
\end{equation}

\end{proof}

\section{Details of Experiments}

\noindent\textbf{Dataset:} For our experiments, we use the CelebA dataset~\cite{liu2015faceattributes}. The dataset has approximately 200k images of faces. Each image is annotated with $65$ binary attributes. Examples of these attributes are gender, age and skin complexion. We preprocess the images by cropping the central $178\times178$ sub-image and resizing each crop to $256\times256$. The resized images are then normalized to be between $-1$ and $1$.\\ 

\noindent\textbf{Target Binary Classifier:} We attack a pre-trained gender binary classifier using our approach. The architecture used for the classifier is shown in \Tableref{t:arch}. We train the classifier with 70\% of the CelebA dataset~\cite{liu2015faceattributes} as training data and 20\% as validation data using categorical cross-entropy. We use ADAM~\cite{kingma2015adam} as our optimizer. Our model is 95.6\% accurate on the test set (10\% of the dataset).  We additionally train a binary \emph{age} classifier with the same architecture.

\begin{table}[htp]
    \centering
    \begin{tabular}{c c}
        \toprule 
        \textbf{Layers} & \textbf{Size} \\
        \midrule 
        Convolutional Layer with Relu & 32x3x3 \\
        Maxpooling Layer & 2x2 \\
        Convolutional Layer with Relu & 64x3x3 \\
        Maxpooling Layer & 2x2 \\
        Convolutional Layer with Relu & 128x3x3 \\
        Maxpooling Layer & 2x2 \\
        Fully Connected Layer & 1024 \\
        Fully Connected Layer & 2  \\
        \bottomrule
    \end{tabular}
    \caption{\sl Architecture for the binary classifier. }
     \label{t:arch}
\end{table}

\subsection*{Adversarial Fader Networks}

\noindent\textbf{Architecture of Fader Networks.} Fader Networks are an encoder-decoder architecture that disentangles semantic attributes during the reconstruction process. This is achieved by training a discriminator on the encoded latent vector while simultaneously reconstructing the original image from the concatenated latent vector and the semantic attribute vector. \Figref{fig:arch_attrib} shows the architecture of the Fader Networks. An intriguing effect of the training process is that the attribute vector space can be treated as a continuous and bounded space. We further can optimize over this space to generate adversarial examples. \\

\begin{figure}[ht]
    \centering
    \begin{minipage}{0.48\linewidth}
        \includegraphics[width=\linewidth]{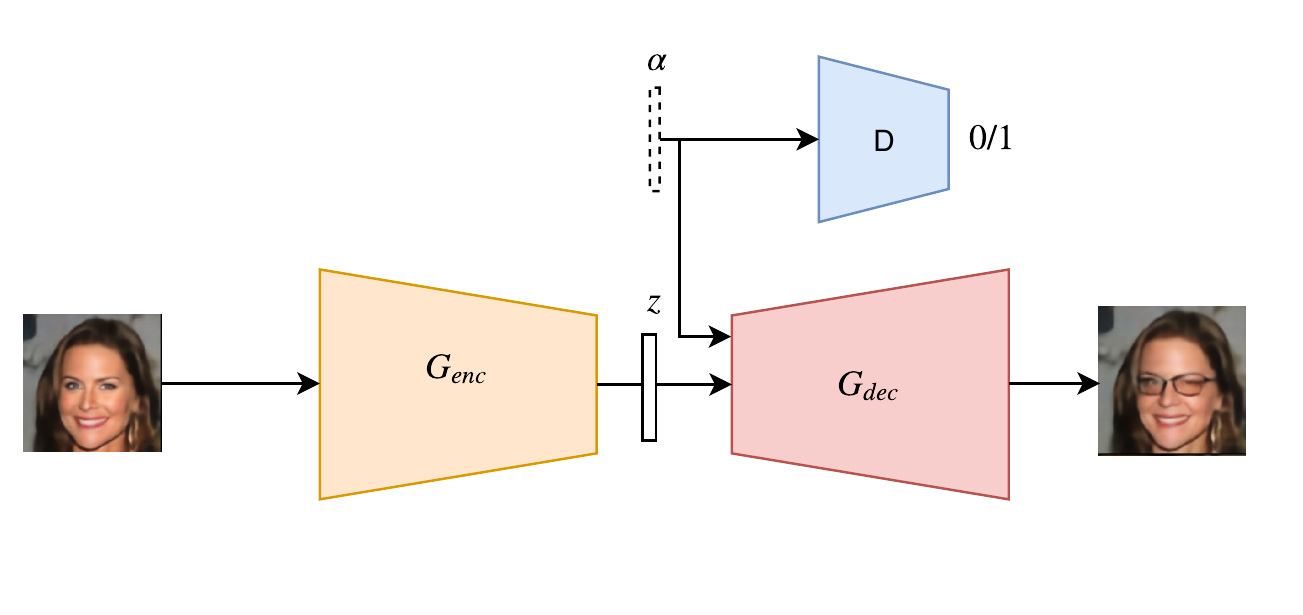} 
    \caption{\sl Architecture of Fader Networks. The encoder converts the input image to a latent vector. The decoder takes as input the latent and the attribute vectors to generate the transformed image. Here, the discriminative classifier acts as an adversarial network to decouple the underlying invariant data from the semantic attributes.}
    \label{fig:arch_attrib}
    \end{minipage}
    \begin{minipage}{0.5\linewidth}
     \centering
            \includegraphics[width=0.95\columnwidth]{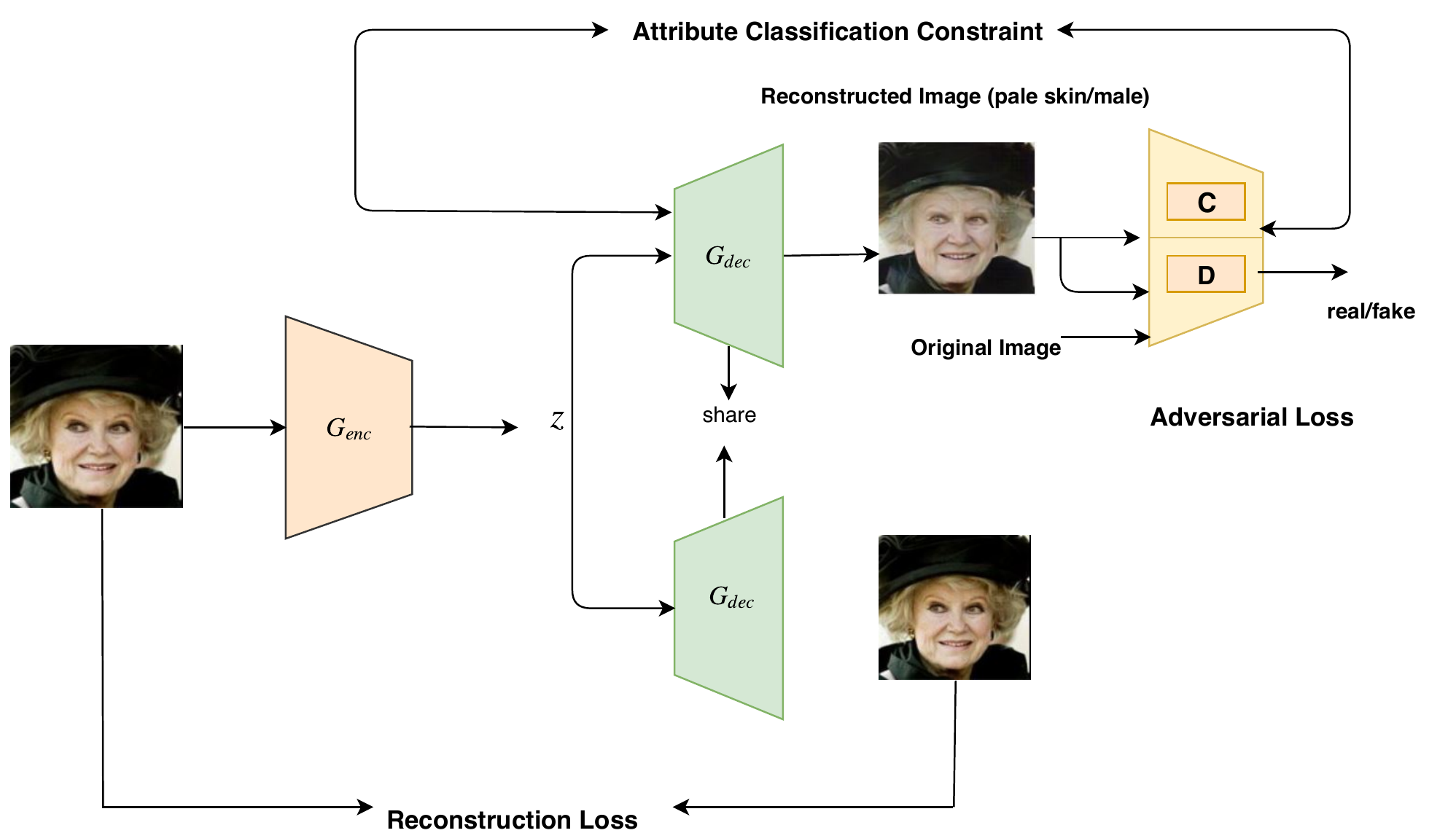} 
    \caption{\sl Architecture of AttGAN Networks. As compared to Fader Networks, the discriminator/classifier pair is used to analyse the reconstruction of the image with the original semantic attributes. Enforcing the decoder to construct both the original and the semantically transformed image results in the decoupling of the semantic and the invariant data.}
     \label{fig:attgan_arch}
     \end{minipage}
\end{figure}

\begin{table*}
    \centering
   \resizebox{\linewidth}{!}{
    \begin{tabular}{p{0.3\columnwidth} p{0.4\columnwidth} p{0.25\columnwidth} p{0.25\columnwidth}}
    \toprule[2pt]
   Attack Type &  Attributes & Accuracy of target model (\%) & Random Sampling (\%)\\
    \midrule[1.5pt]
    \multirow{3}{0.3\columnwidth}{Single Attribute Attack}        & A1 & 70.0 & 87.0 \\
                                                    & A2 & 61.0 & 93.0 \\
                                                    & A3 & 48.0 & 88.0 \\
    \midrule
    \multirow{3}{0.3\columnwidth}{Multi Attribute Attack}         & A1,A5,A6 & 12.0 & 86.0 \\
                                                    & A2,A5,A6 & 7.00 & 85.0 \\
                                                    & A1,A2,A7 & 28.0 & 84.0 \\
    \midrule
    \multirow{2}{0.3\columnwidth}{Cascaded Multi Attribute Attack} & A1-A2-A3 & 30.0 & 68.0 \\
                                                    & A1-A3-A4 & 31.0 & 80.0 \\ 
                                                    & A2-A3-A4 & 42.0 & 68.0 \\
                                                    
    \bottomrule[1.5pt]
    \end{tabular}
    }
    \caption{\sl Performance of the Semantic Adversarial Example under various Adversarial Fader Networks implementations for the binary age classifier. Legend for attributes: A1-Eyeglasses, A2-Gender, A3-Nose shape, A4-Eye shape, A5-Chubbiness, A6-Pale Skin, A7-Smiling. Observe that as the number of perturbed attributes increase, the semantic attacks become more effective. 
    In comparison to worst-of-10 random sampling~\cite{engstrom2019a} of the attribute space, our optimization framework is more effective at finding semantic adversarial examples. 
    Note that the performance of our semantic attacks fare very well to decrease the accuracy of the age classifier as well like the gender classifier.  
    }
    \label{t:fader_perf_apdx}
\end{table*}

 \noindent\textbf{Single and Multi Attribute Attacks.} We train three multi-attribute Fader networks with attributes presented in~\tableref{t:fader_perf_apdx}. The pre-trained Fader networks are used as semantic constraints with the attribute vectors as the optimization variables. We then process examples from the CelebA test set with the semantic attack algorithm to generate adversarial examples. In order to make our optimization algorithm compatible with Fader Networks, we create a non-parametric forward model to convert the attribute vector to a compatible form. We call this forward model ``Attribute Encoding''. 
 
 We generate semantic adversarial images by optimizing over a modified Carlini-Wagner loss~\cite{Carlini2017cwl2} with respect to the attribute vectors using ADAM~\cite{kingma2015adam} with a learning rate of $0.01$. We also experimented with various other optimizers including stochastic gradient descent, RMSProp~\cite{rmsprop}, but find that ADAM generates sharper images as well is the most successful. 

Our experiments show that successful multi-attribute models tend to be deeper and wider. In addition, these networks are extremely susceptible to mode collapse unless the hyperparameters are carefully tuned. We hypothesize that this is an effect of the strong coupling of facial attributes, thus making the generator-discriminator optimization difficult. An unconditioned generative neural network generally learns to associate these entangled representations to a latent vector space where dimension represents some combination of attributes. In order to get past this, we model the multi-attribute perturbation problem as a sequential perturbation of single attributes.\\

\noindent\textbf{Cascaded Attribute Attack.}
For the cascaded attribute attack, we cascade several smaller single attribute models one after the other to sequentially transform the input image. In this case, the problem of decoupling facial features from the underlying invariant data is divided among multiple models. The transformed image is then input to the target model. We generate adversarial examples as in the previous two cases for the CelebA test set by optimizing the Carlini-Wagner loss. In this case, we also modify the attribute encoding module to treat each attribute tuple separately.

We find that the semantically transformed images tend to be less sharp as compared to the ones generated single or multi-attribute attacks. This can be attributed to the concatenation of several reconstruction steps. Sequential reconstruction leads to loss of information and the reconstruction error compounding.\\

\noindent\textbf{Attribute Encoding.} Each attribute is represented by a tuple of real numbers that sum up to one. These tuples are concatenated into an attribute vector.  To ensure that this structure is preserved over the optimization framework, we use a non-parametric forward model to algebraically manipulate our optimization variables to this specific representation. The encoding module also implements the box constraint for the optimized attribute values to lie between $-3.0$ and $3.0$ in order to ensure that the generated images are valid.

\subsection*{Adversarial Attribute GANs}

\noindent\textbf{Architecture of Attribute GANs.} Attribute GANs~\cite{He2017AttGANFA} improve upon Fader Networks by using a discriminator-classifier pair to analyse the reconstructed images (Refer \Figref{fig:attgan_arch} for the architecture). They optimize over a combination of a reconstruction loss, an adversarial loss and an attribute constraint loss to ensure the editing of the exact desired attribute while preserving the attribute excluding details at the same time. The encoded latent vector is conditioned on the attribute vector during the decoding process. This results in the decoupling of semantic attributes from the underlying identity data. AttGAN takes as input an image and an attribute vector where each element represents an attribute. We select $k$ attributes to perturb for our semantic attack. 

\begin{figure}[ht]
   
\end{figure}

We use a pretrained AttGAN model with $13$ semantic attributes. For our experiments we consider $5$ and $6$ attributes respectively for transforming input images. 

\noindent\textbf{Attacks.} We adapt our adversarial Fader Network approach to the AttGANs by modifying the ``Attribute Encoding'' module to mask attributes that we do not perturb. The encoding module also constrains the elements to lie between $-1.0$ and $1.0$ as required by our algorithm to generate valid images.


\section{Results}

Our additional experiments on the binary age classifier show that our approach is able to generate adversarial examples for other classifiers trained on the CelebA dataset (See \Tableref{t:fader_perf_apdx}). Note that our observations regarding the increasing effectiveness of our attack approach as the number of attributes we perturb increase, holds even for a new classifier. We also compare the performance of our attack with worst-of-10 random sampling (similar to the approach in~\cite{engstrom2019a}.) This proves that our approach is successful at generating semantic adversarial examples. 

\begingroup
\begin{figure}[htp]
    \centering
    \setlength{\tabcolsep}{1pt}
    \renewcommand{\arraystretch}{0.5}
    \begin{tabular}{c c c c || c c c c ||  c c c }
    \includegraphics[width=0.090\textwidth]{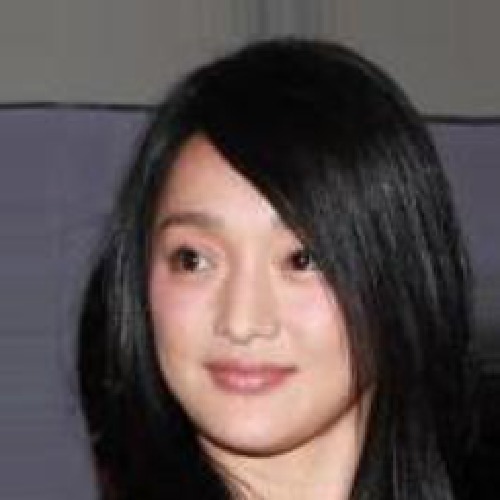} &
    \includegraphics[width=0.090\textwidth]{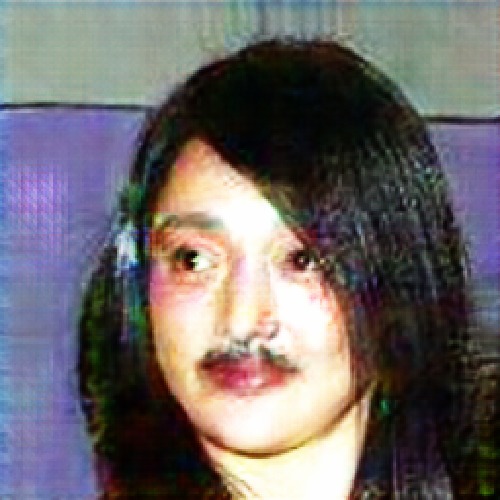} & 
    \includegraphics[width=0.090\textwidth]{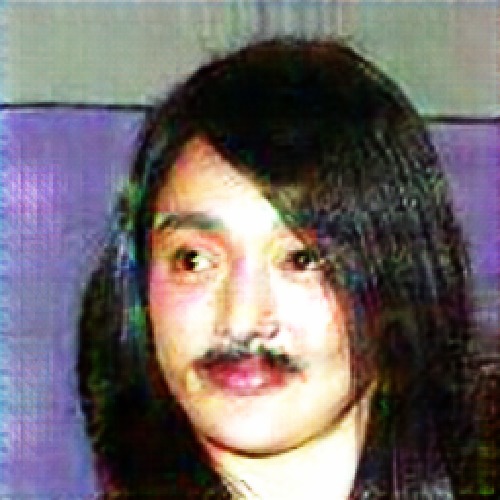} & &
    \includegraphics[width=0.090\textwidth]{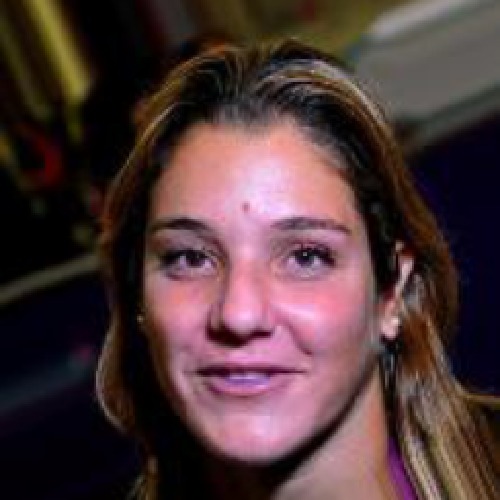} & 
    \includegraphics[width=0.090\textwidth]{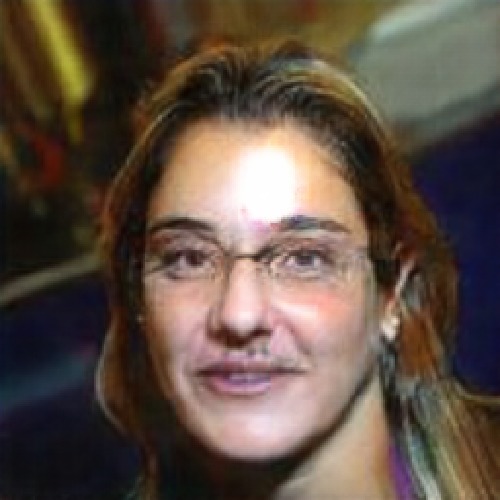} &
    \includegraphics[width=0.090\textwidth]{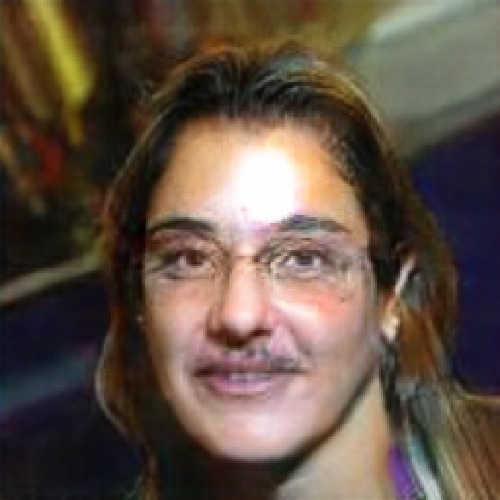} & &
    \includegraphics[width=0.090\textwidth]{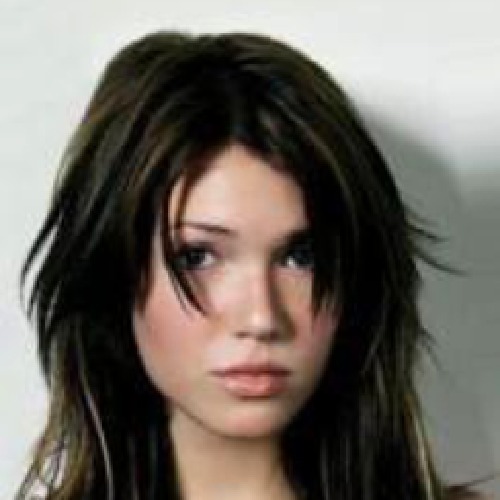} &
    \includegraphics[width=0.090\textwidth]{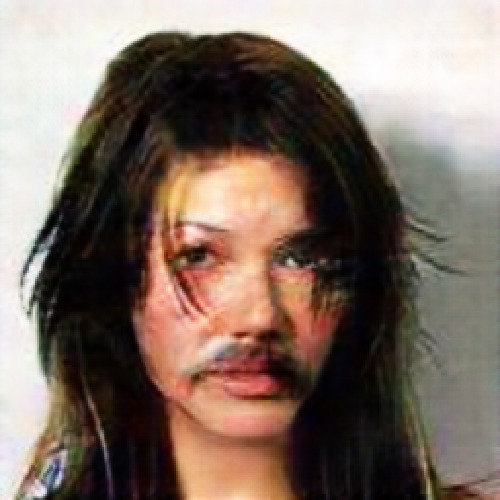} &
    \includegraphics[width=0.090\textwidth]{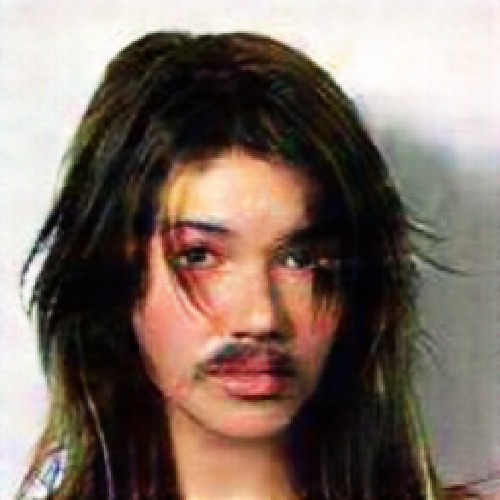} \\ 
    \includegraphics[width=0.090\textwidth]{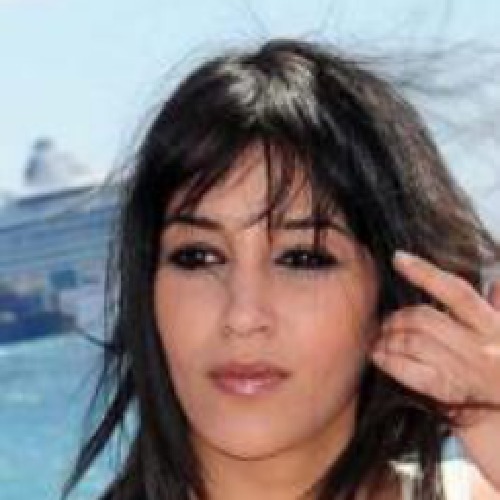} &
    \includegraphics[width=0.090\textwidth]{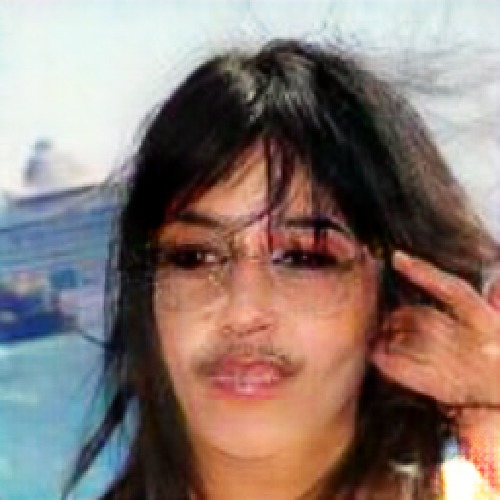} & 
    \includegraphics[width=0.090\textwidth]{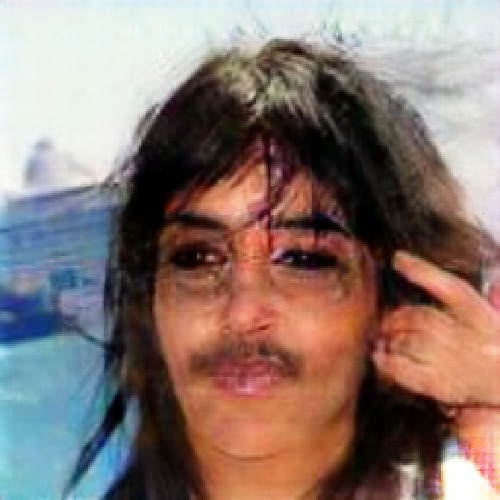} & &
    \includegraphics[width=0.090\textwidth]{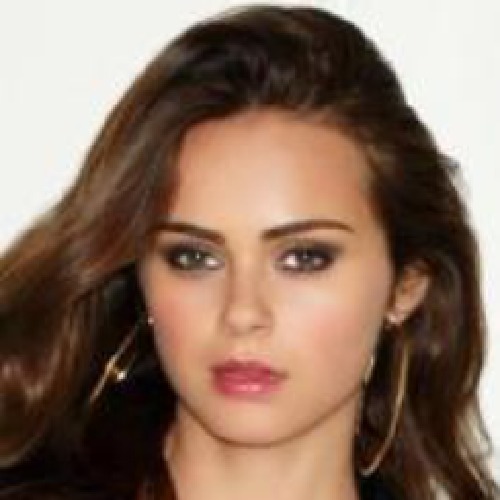} & 
    \includegraphics[width=0.090\textwidth]{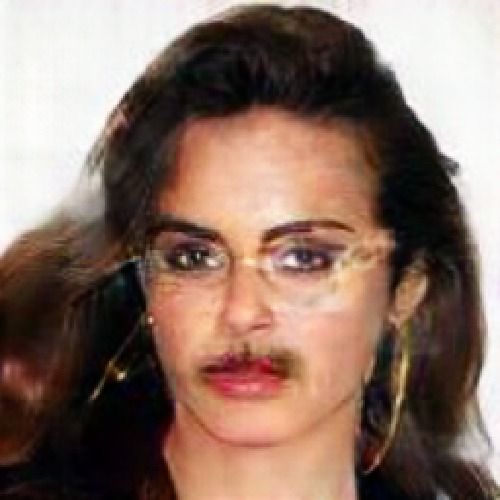} &
    \includegraphics[width=0.090\textwidth]{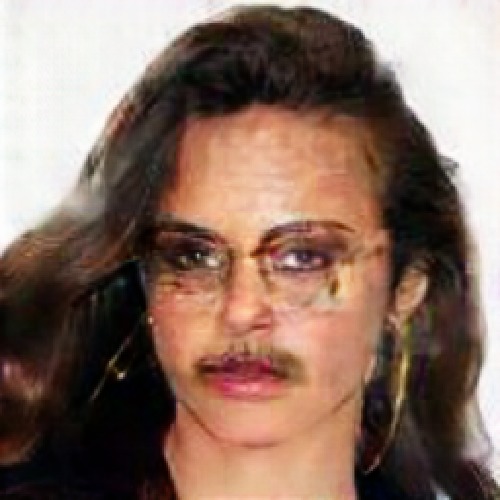} & &
    \includegraphics[width=0.090\textwidth]{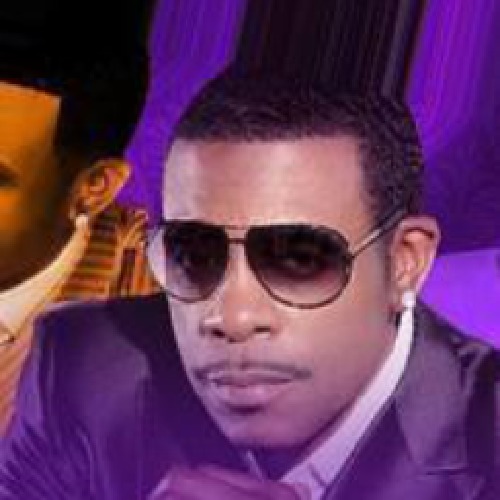} &
    \includegraphics[width=0.090\textwidth]{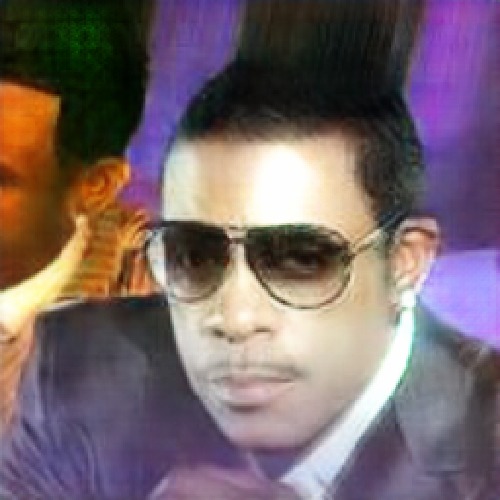} &
    \includegraphics[width=0.090\textwidth]{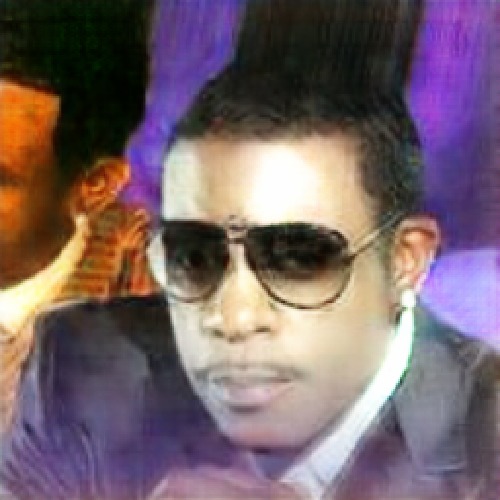} \\ 
    \includegraphics[width=0.090\textwidth]{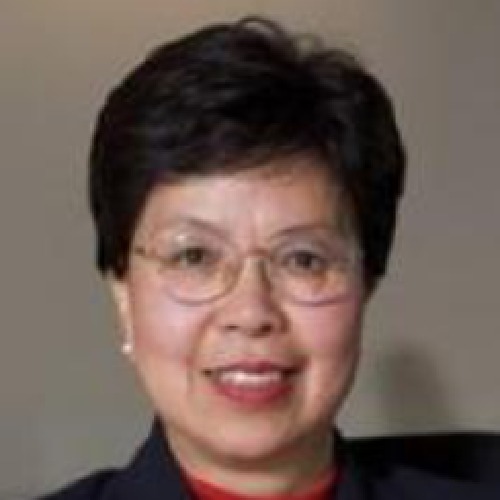} &
    \includegraphics[width=0.090\textwidth]{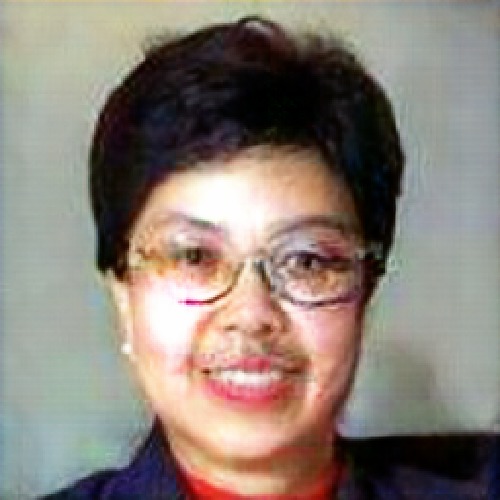} & 
    \includegraphics[width=0.090\textwidth]{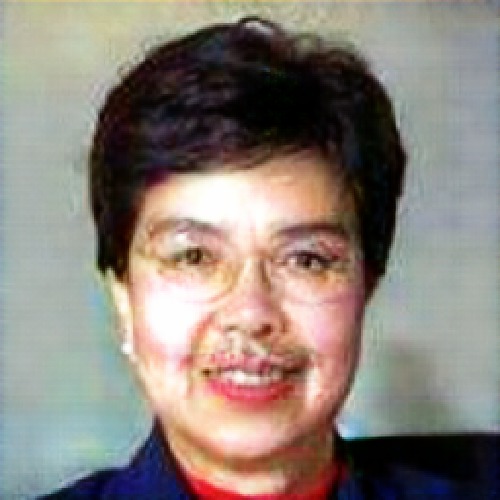} & &
    \includegraphics[width=0.090\textwidth]{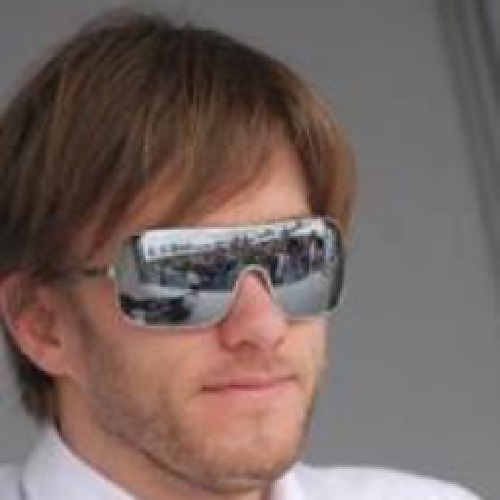} & 
    \includegraphics[width=0.090\textwidth]{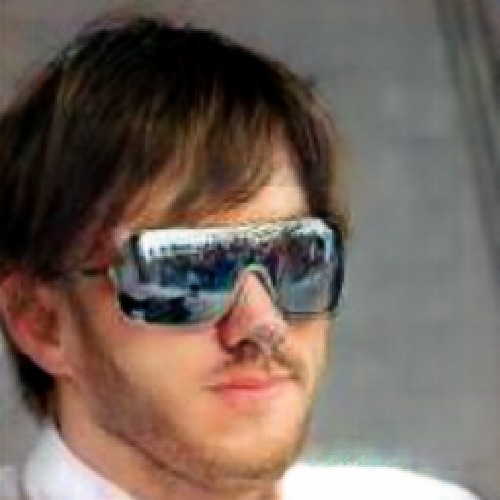} &
    \includegraphics[width=0.090\textwidth]{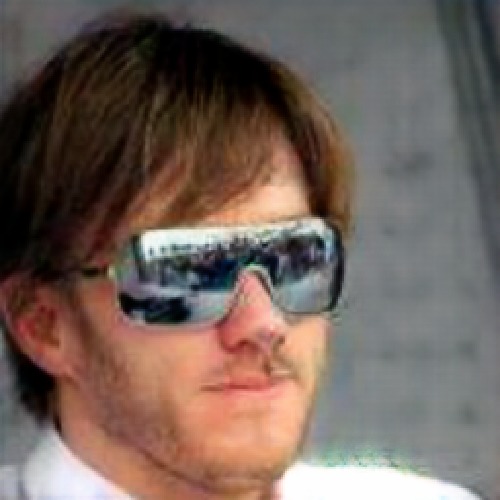} & &
    \includegraphics[width=0.090\textwidth]{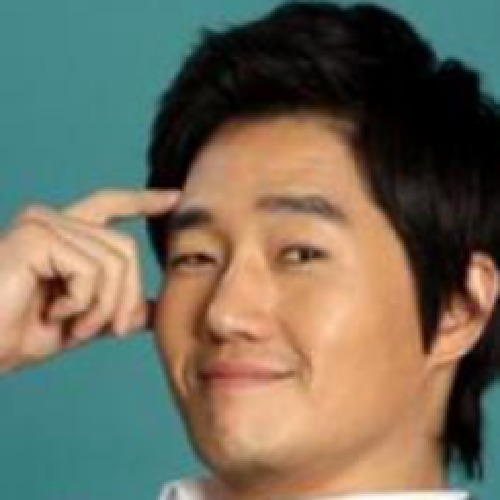} &
    \includegraphics[width=0.090\textwidth]{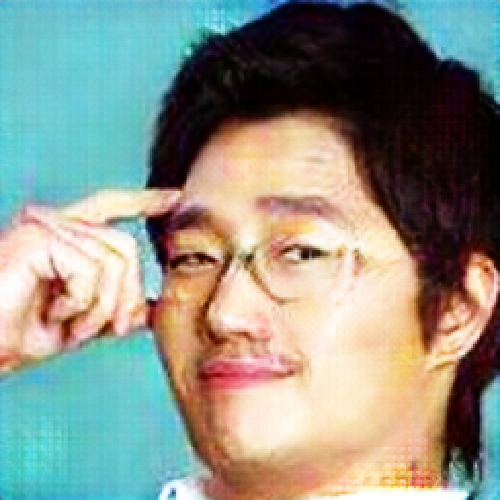} &
    \includegraphics[width=0.090\textwidth]{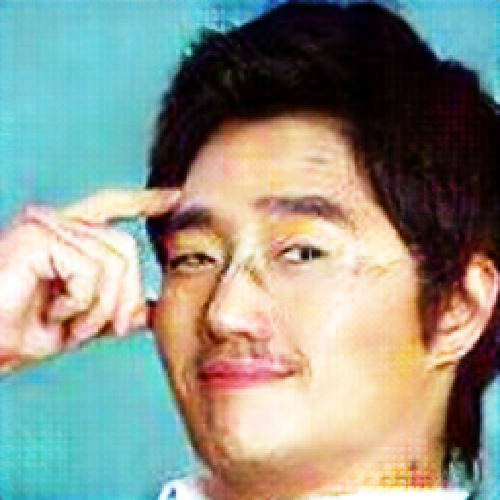} \\ 
    \includegraphics[width=0.090\textwidth]{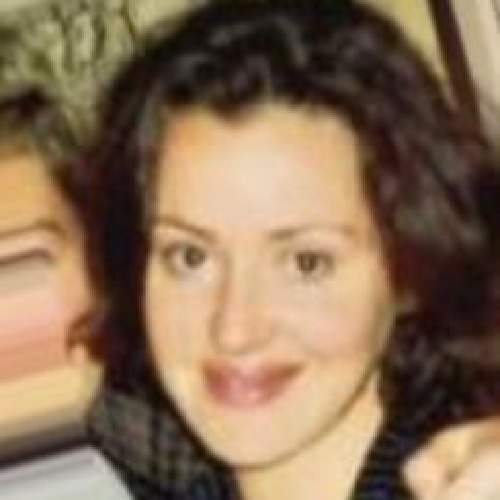} &
    \includegraphics[width=0.090\textwidth]{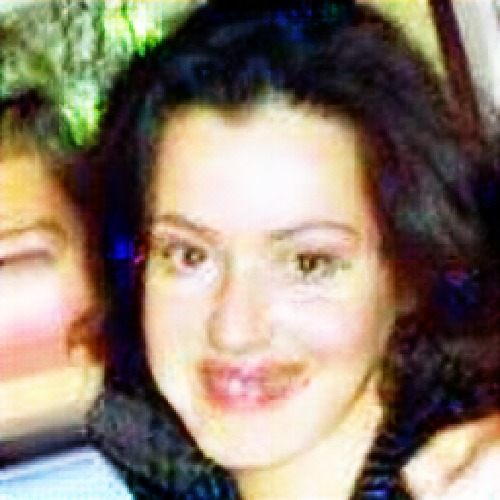} & 
    \includegraphics[width=0.090\textwidth]{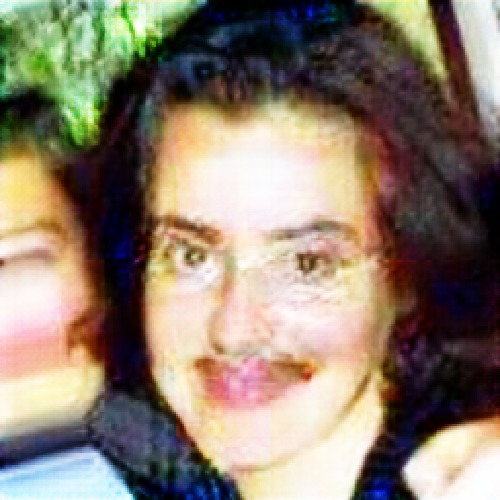} & &
    \includegraphics[width=0.090\textwidth]{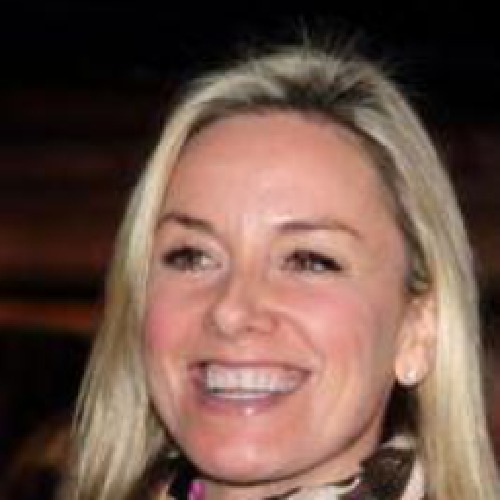} & 
    \includegraphics[width=0.090\textwidth]{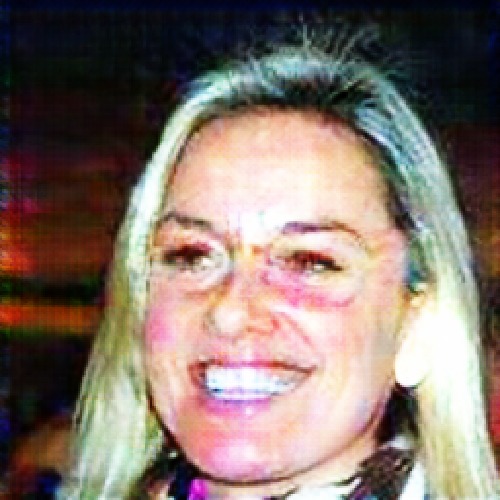} &
    \includegraphics[width=0.090\textwidth]{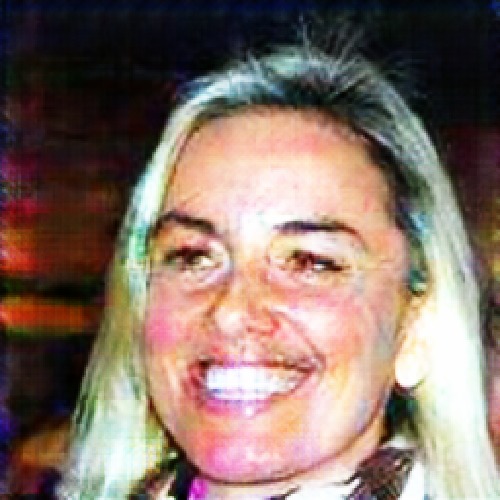} & &
    \includegraphics[width=0.090\textwidth]{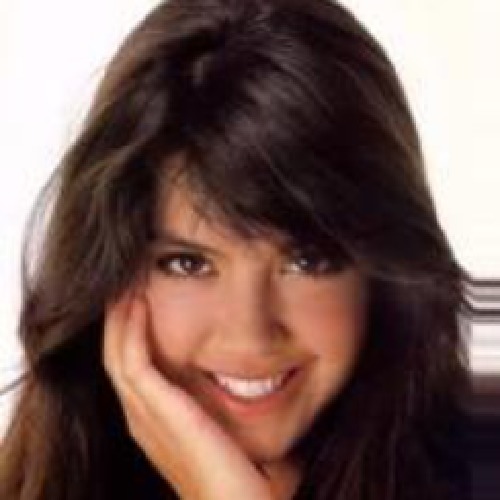} &
    \includegraphics[width=0.090\textwidth]{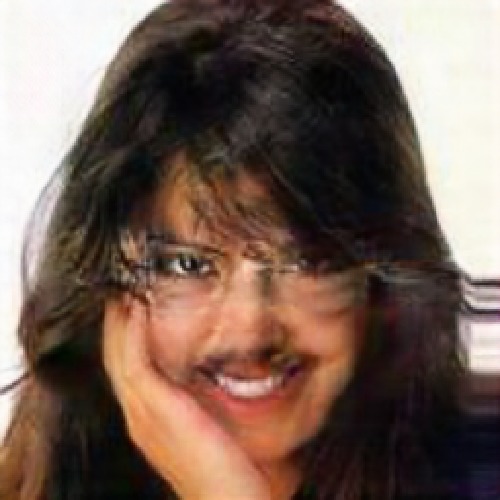} &
    \includegraphics[width=0.090\textwidth]{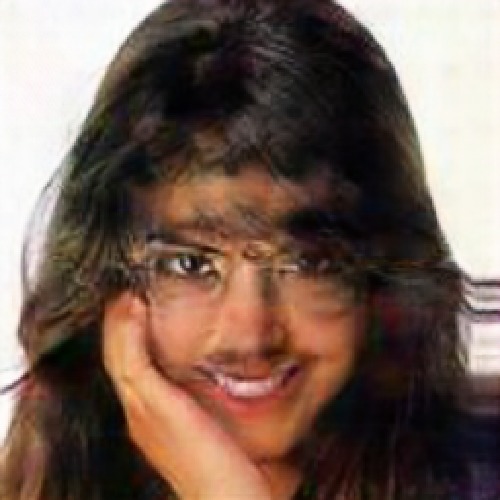} \\ 
    \includegraphics[width=0.090\textwidth]{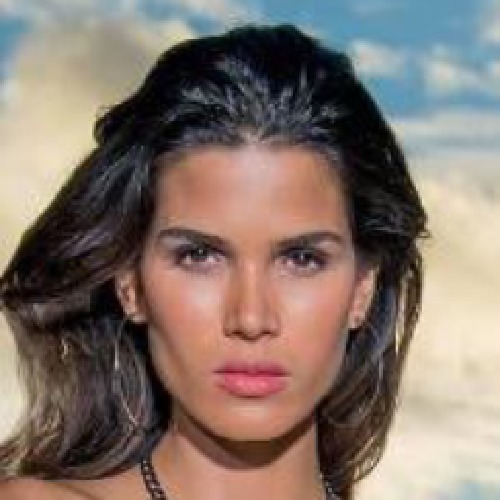} &
    \includegraphics[width=0.090\textwidth]{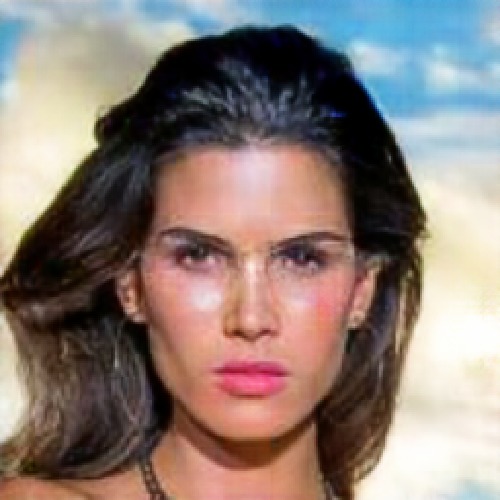} & 
    \includegraphics[width=0.090\textwidth]{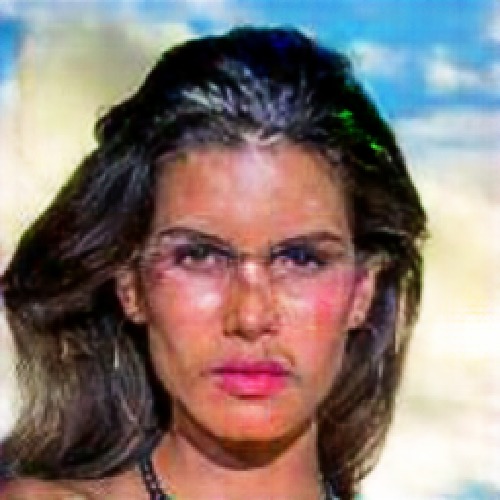} & &
    \includegraphics[width=0.090\textwidth]{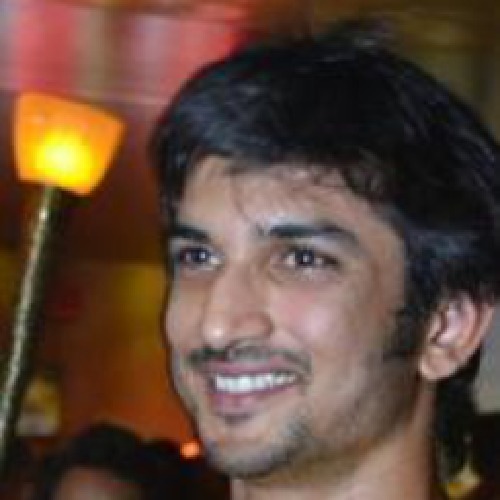} & 
    \includegraphics[width=0.090\textwidth]{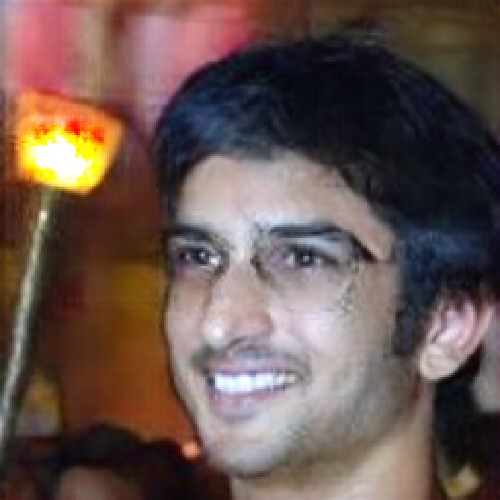} &
    \includegraphics[width=0.090\textwidth]{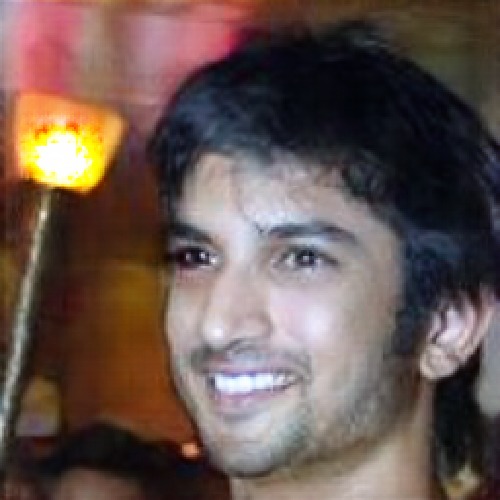} & &
    \includegraphics[width=0.090\textwidth]{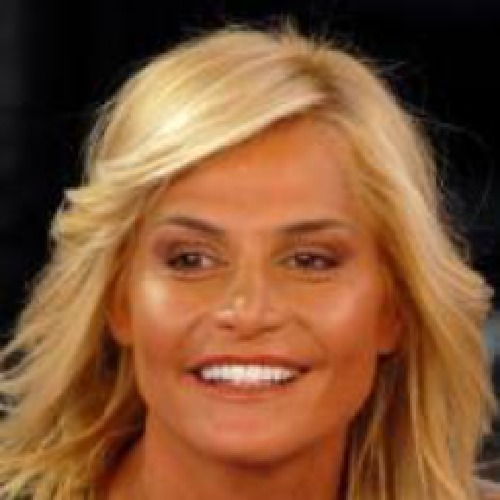} &
    \includegraphics[width=0.090\textwidth]{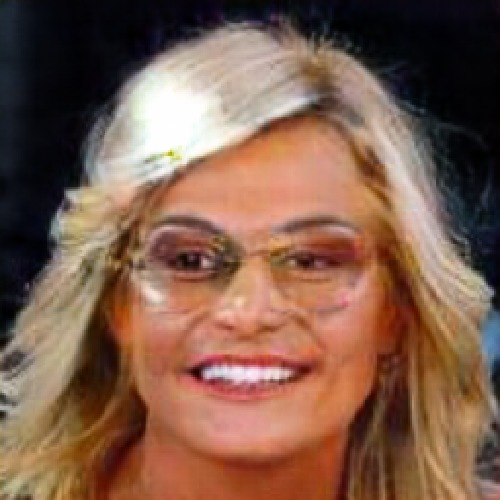} &
    \includegraphics[width=0.090\textwidth]{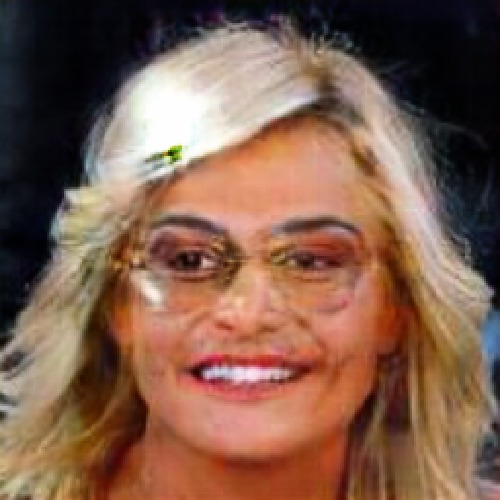} \\ 
    \includegraphics[width=0.090\textwidth]{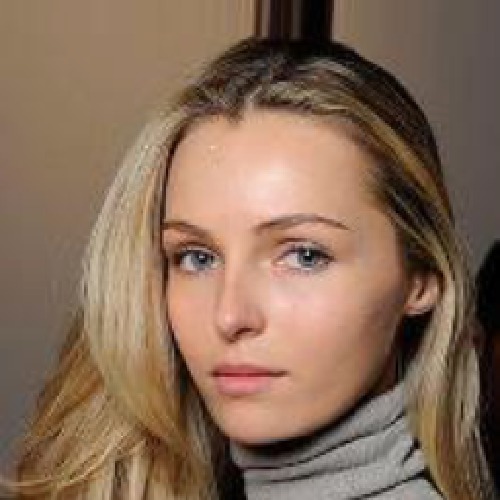} &
    \includegraphics[width=0.090\textwidth]{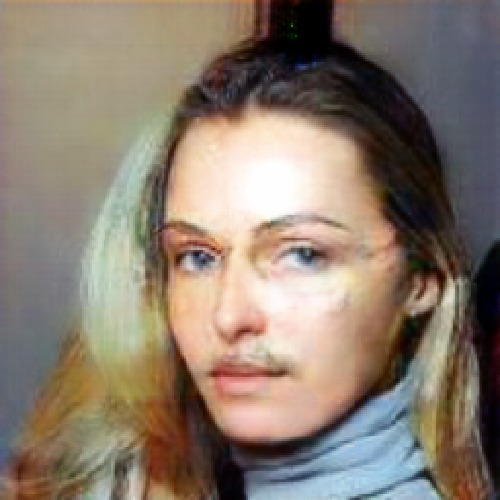} & 
    \includegraphics[width=0.090\textwidth]{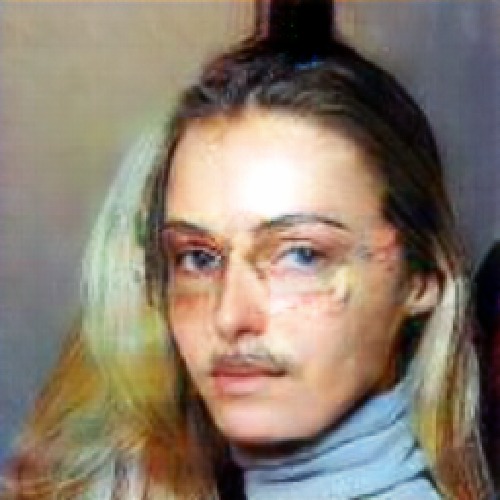} & &
    \includegraphics[width=0.090\textwidth]{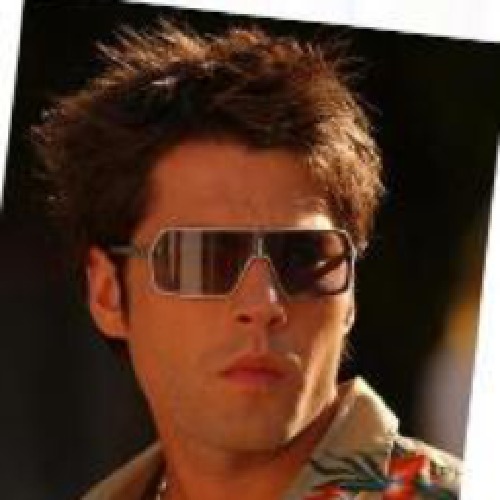} & 
    \includegraphics[width=0.090\textwidth]{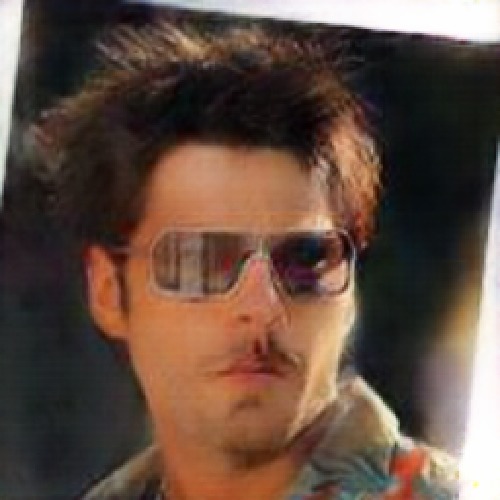} &
    \includegraphics[width=0.090\textwidth]{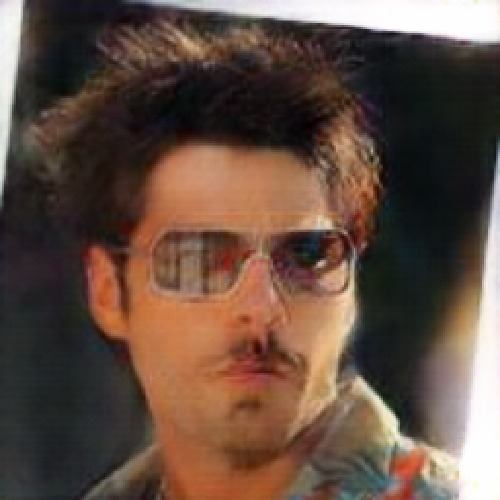} & &
    \includegraphics[width=0.090\textwidth]{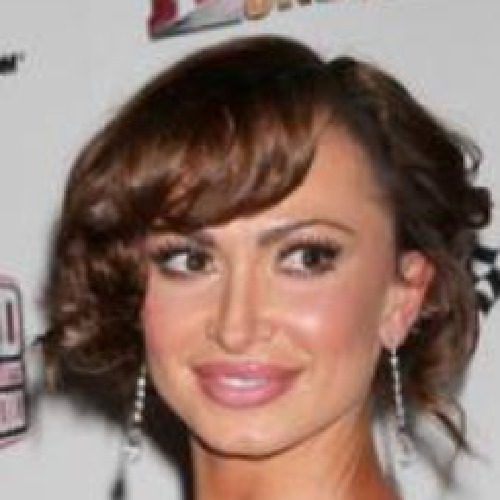} &
    \includegraphics[width=0.090\textwidth]{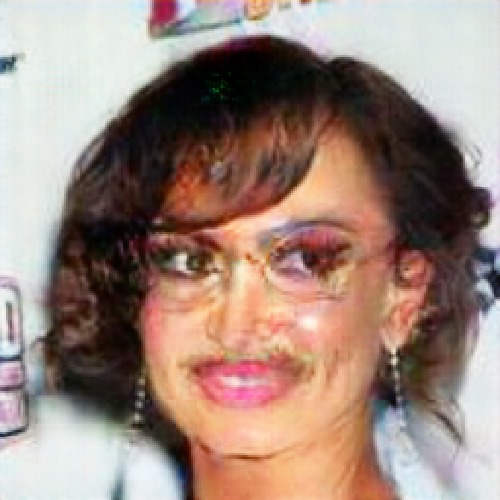} &
    \includegraphics[width=0.090\textwidth]{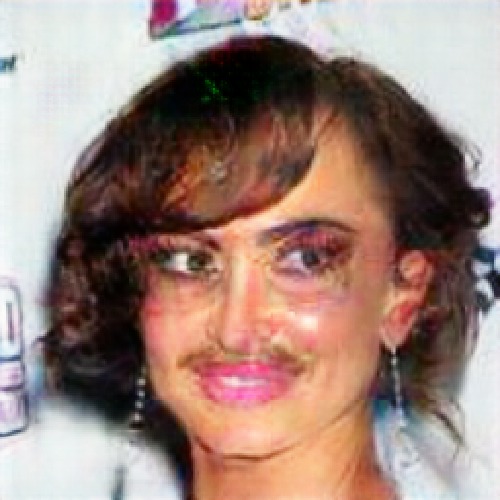} \\ 
    \includegraphics[width=0.090\textwidth]{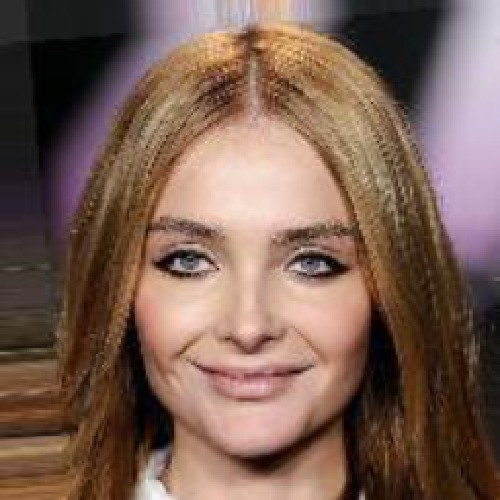} &
    \includegraphics[width=0.090\textwidth]{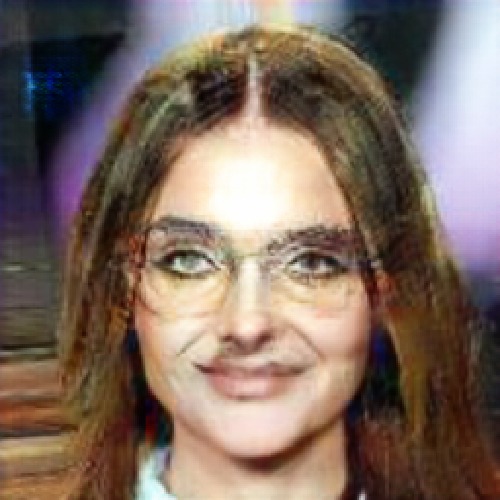} & 
    \includegraphics[width=0.090\textwidth]{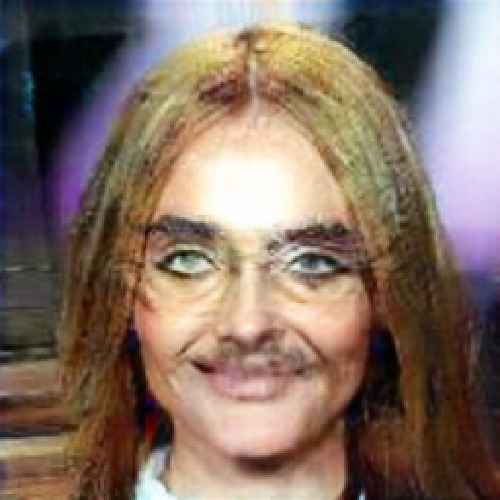} & &
    \includegraphics[width=0.090\textwidth]{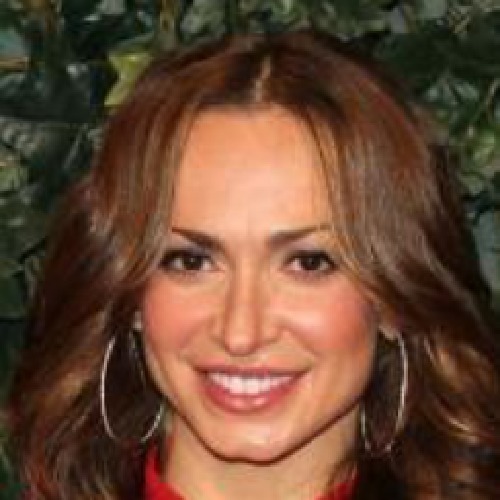} & 
    \includegraphics[width=0.090\textwidth]{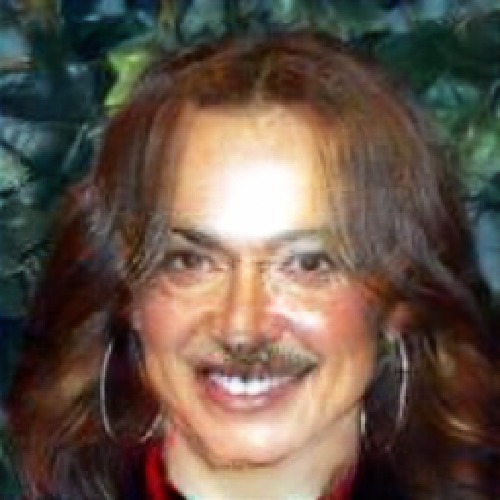} &
    \includegraphics[width=0.090\textwidth]{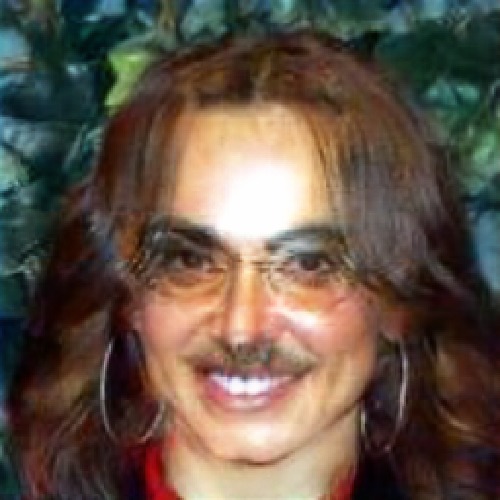} & &
    \includegraphics[width=0.090\textwidth]{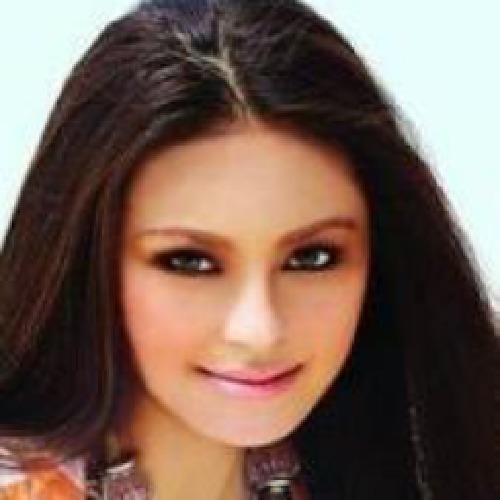} &
    \includegraphics[width=0.090\textwidth]{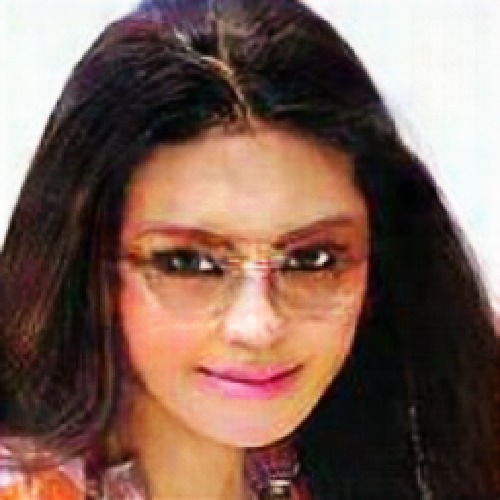} &
    \includegraphics[width=0.090\textwidth]{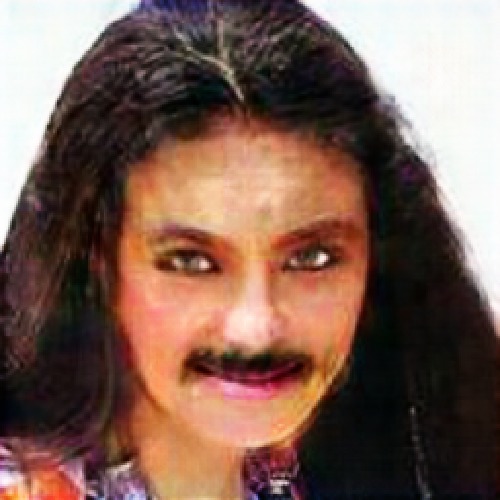} \\ 
    \includegraphics[width=0.090\textwidth]{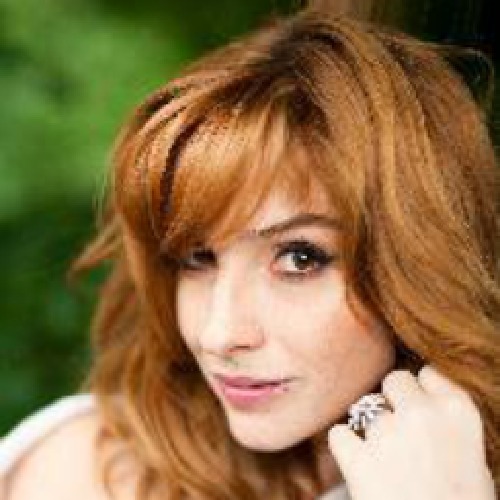} &
    \includegraphics[width=0.090\textwidth]{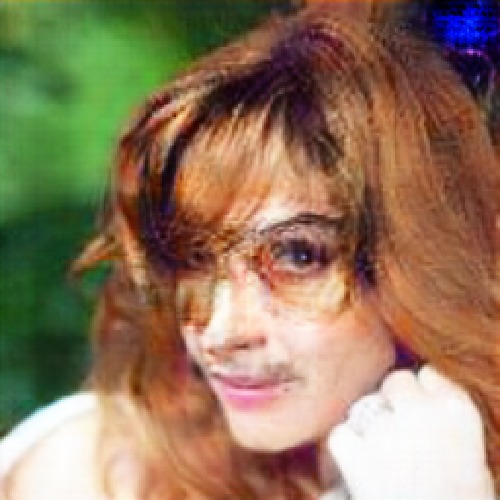} & 
    \includegraphics[width=0.090\textwidth]{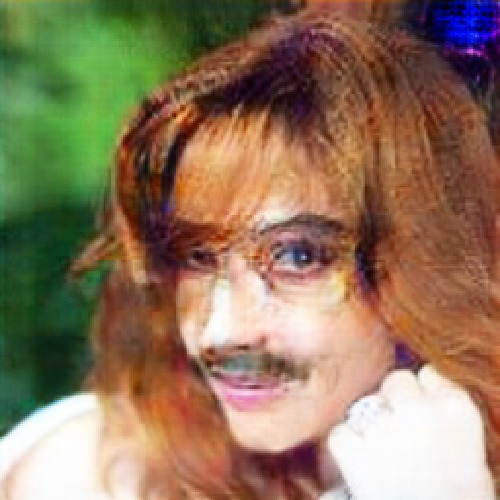} & &
    \includegraphics[width=0.090\textwidth]{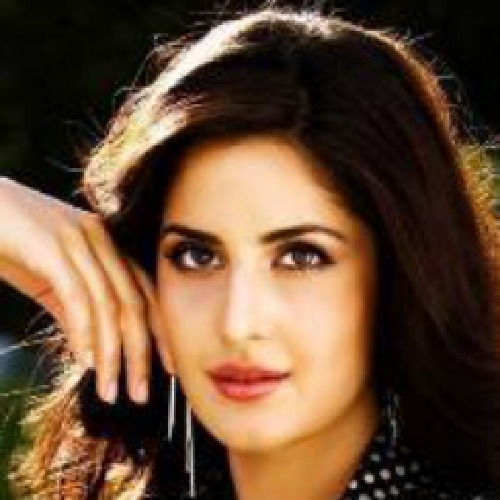} & 
    \includegraphics[width=0.090\textwidth]{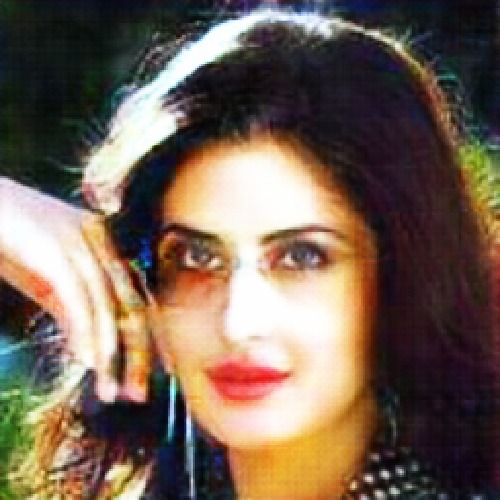} &
    \includegraphics[width=0.090\textwidth]{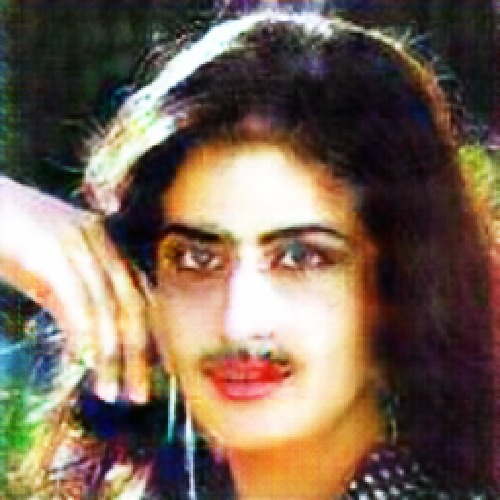} & &
    \includegraphics[width=0.090\textwidth]{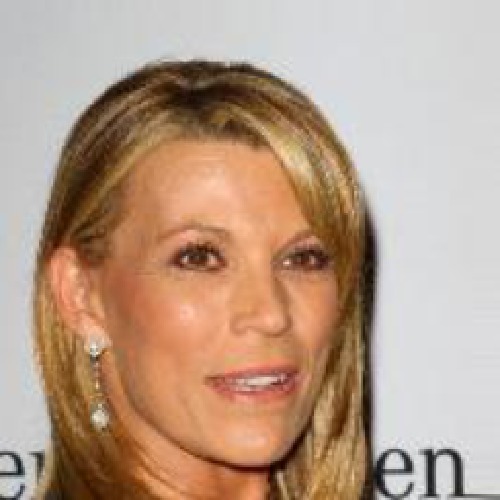} &
    \includegraphics[width=0.090\textwidth]{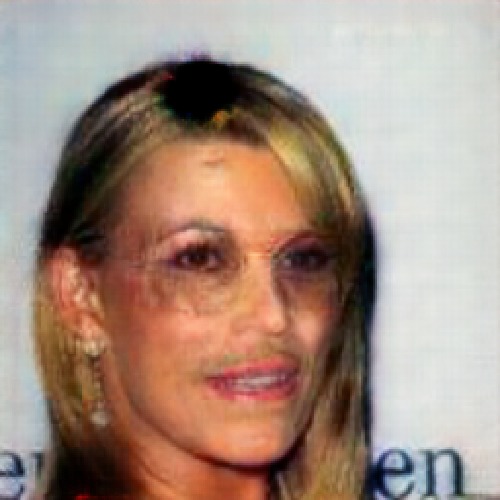} &
    \includegraphics[width=0.090\textwidth]{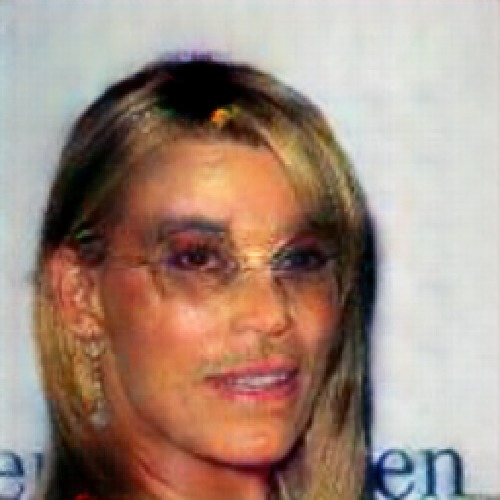} \\ 
    \includegraphics[width=0.090\textwidth]{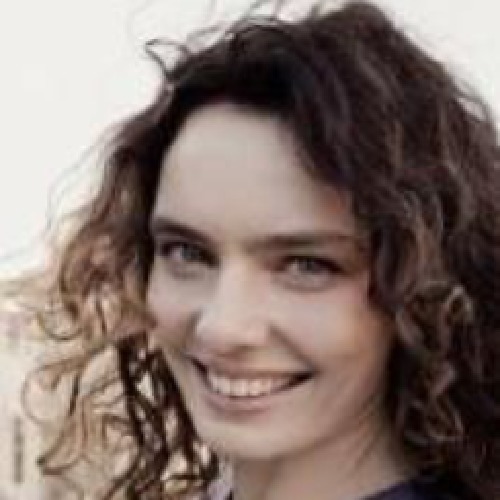} &
    \includegraphics[width=0.090\textwidth]{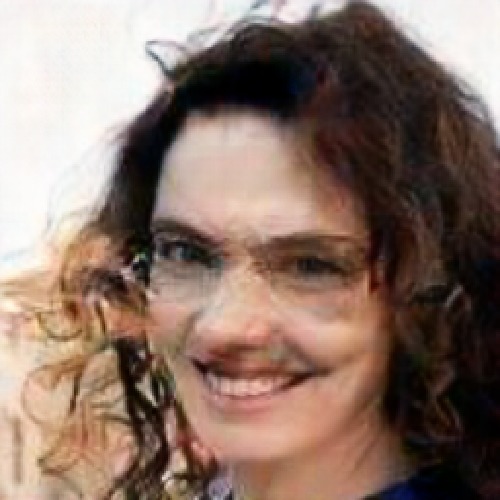} & 
    \includegraphics[width=0.090\textwidth]{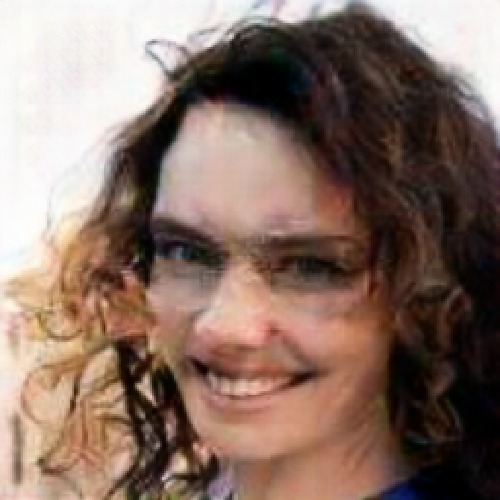} & &
    \includegraphics[width=0.090\textwidth]{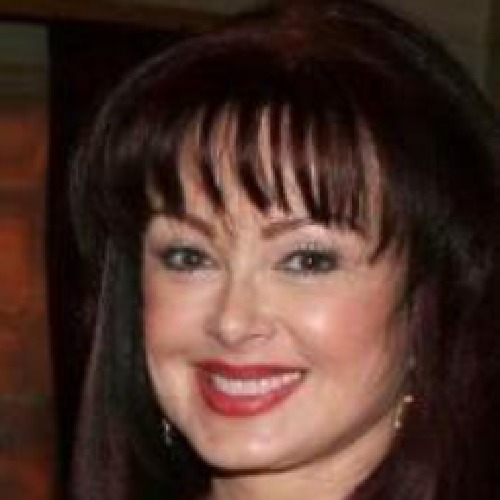} & 
    \includegraphics[width=0.090\textwidth]{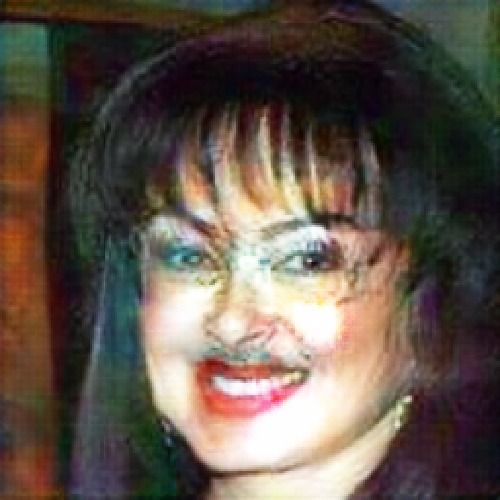} &
    \includegraphics[width=0.090\textwidth]{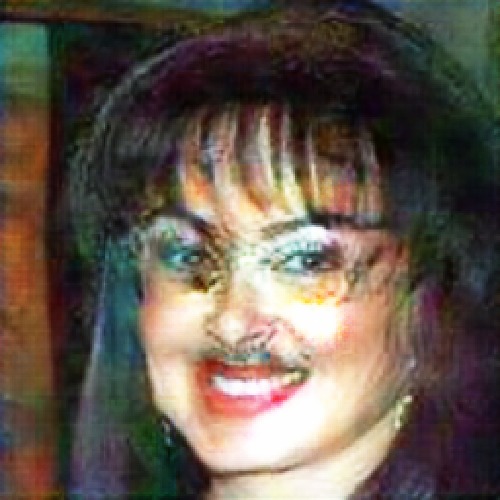} & &
    \includegraphics[width=0.090\textwidth]{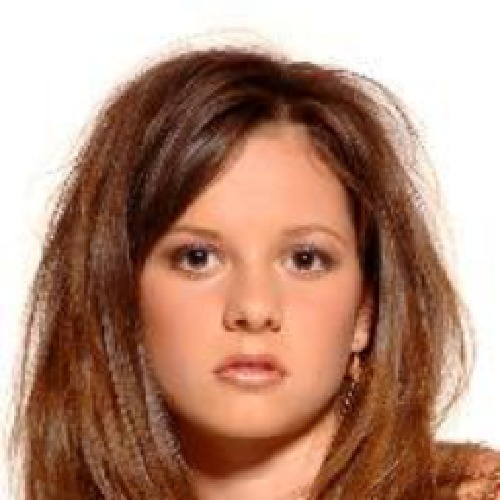} &
    \includegraphics[width=0.090\textwidth]{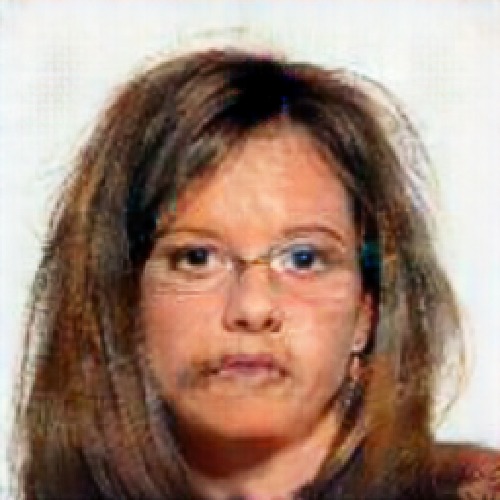} &
    \includegraphics[width=0.090\textwidth]{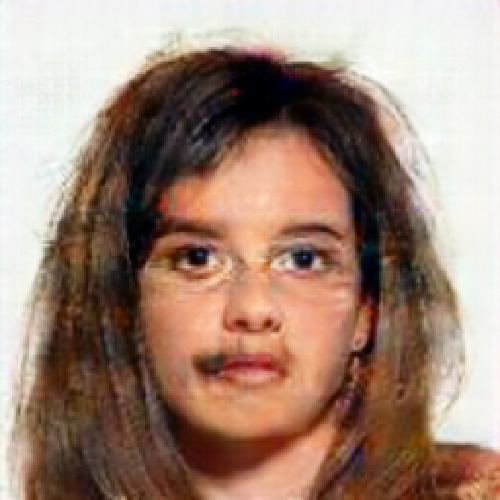} \\ 
    \includegraphics[width=0.090\textwidth]{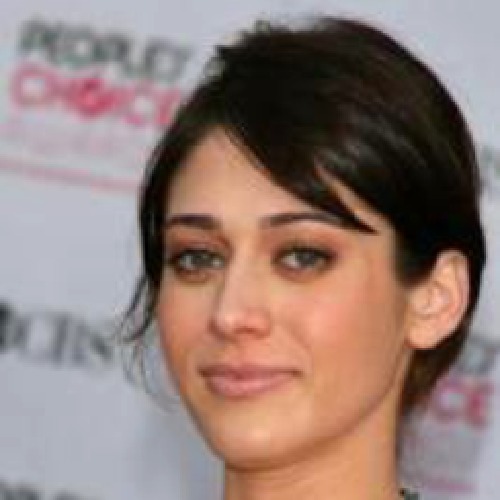} &
    \includegraphics[width=0.090\textwidth]{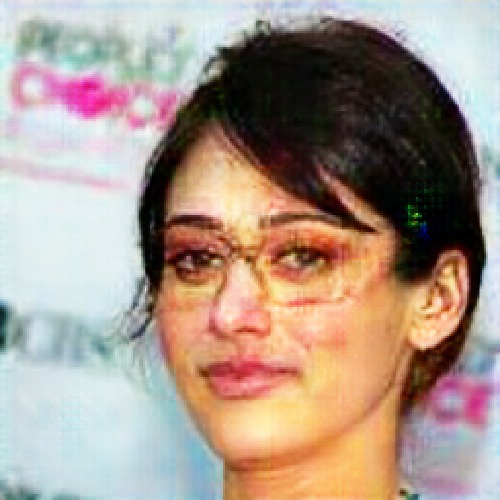} & 
    \includegraphics[width=0.090\textwidth]{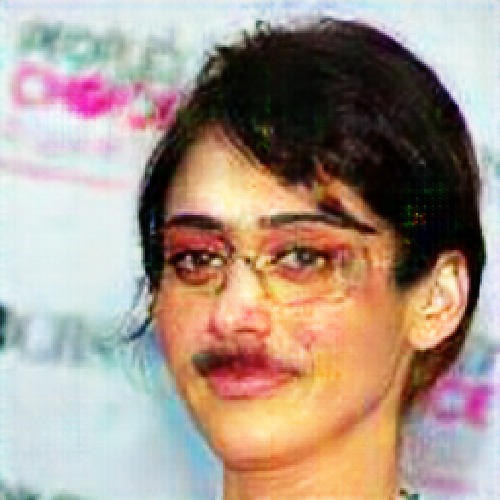} & &
    \includegraphics[width=0.090\textwidth]{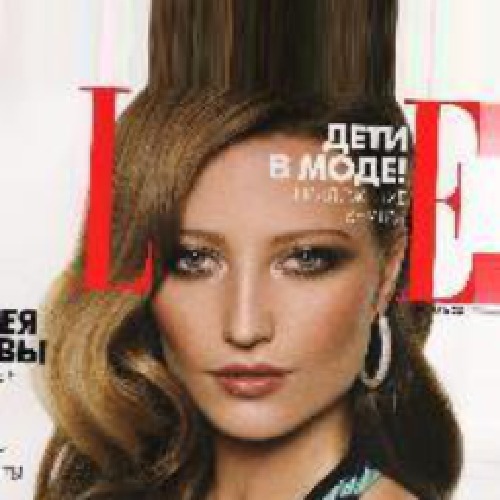} & 
    \includegraphics[width=0.090\textwidth]{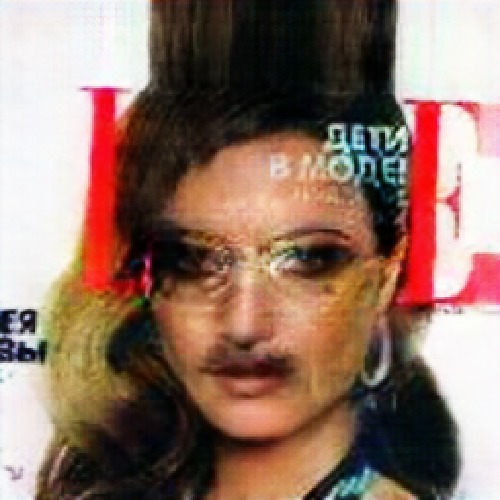} &
    \includegraphics[width=0.090\textwidth]{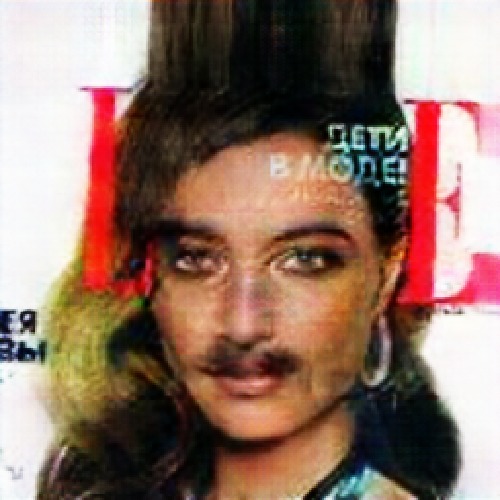} & &
    \includegraphics[width=0.090\textwidth]{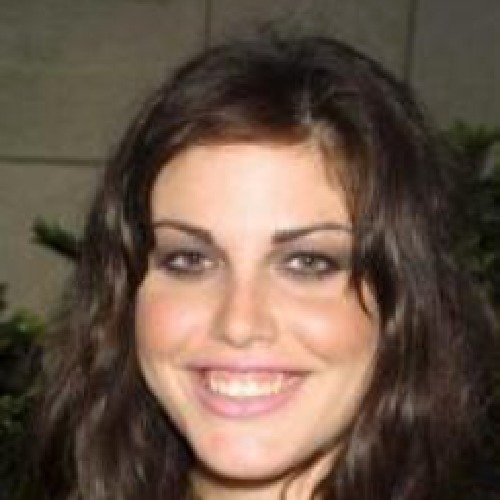} &
    \includegraphics[width=0.090\textwidth]{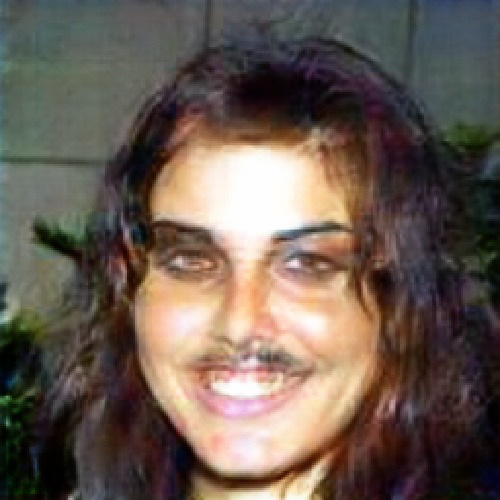} &
    \includegraphics[width=0.090\textwidth]{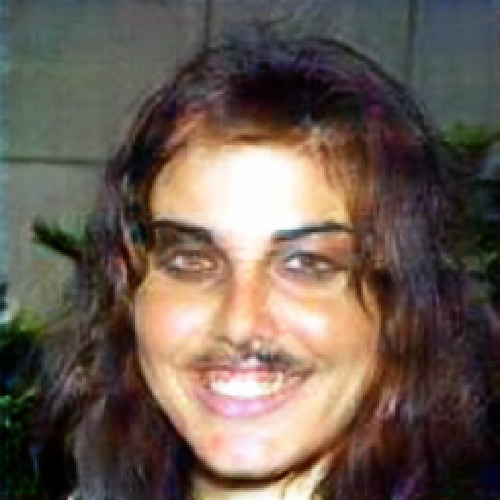} \\ 
    \includegraphics[width=0.090\textwidth]{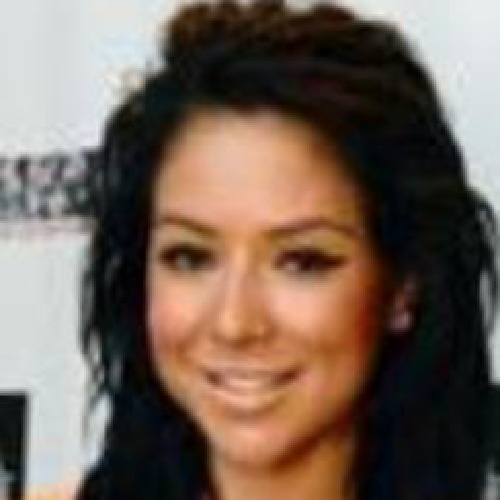} &
    \includegraphics[width=0.090\textwidth]{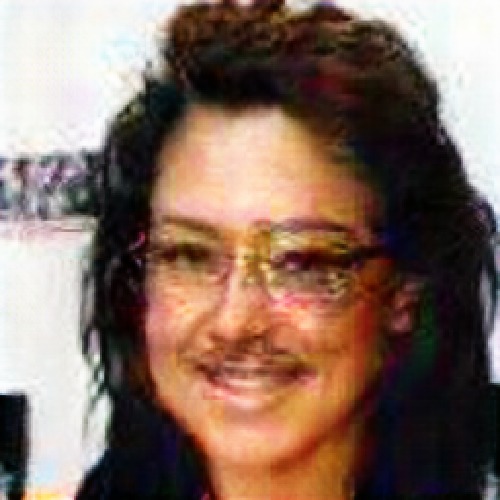} & 
    \includegraphics[width=0.090\textwidth]{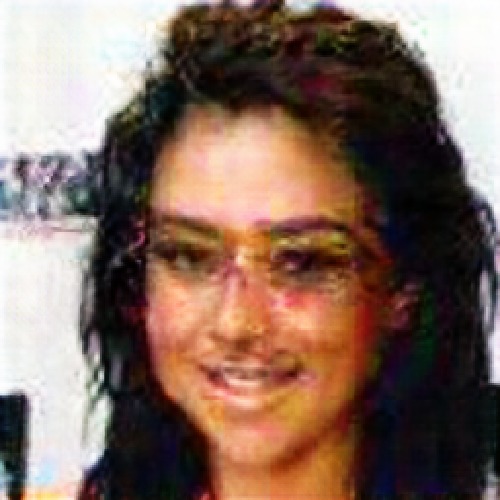} & &
    \includegraphics[width=0.090\textwidth]{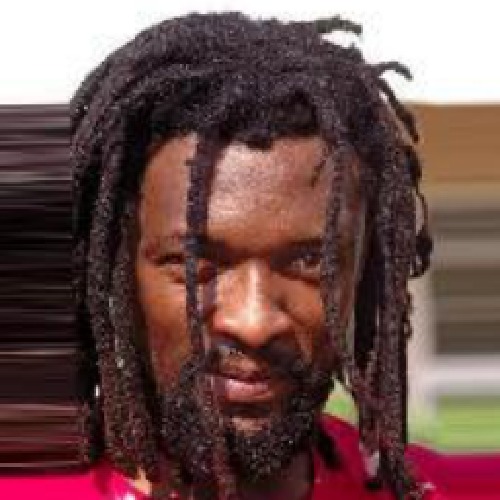} & 
    \includegraphics[width=0.090\textwidth]{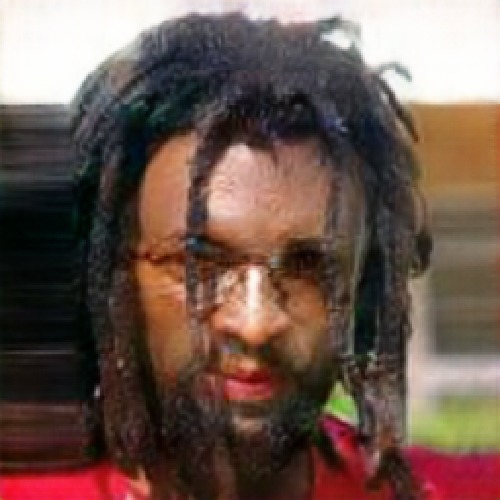} &
    \includegraphics[width=0.090\textwidth]{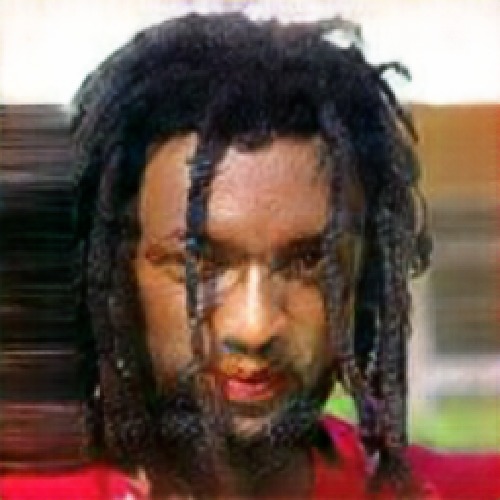} & &
    \includegraphics[width=0.090\textwidth]{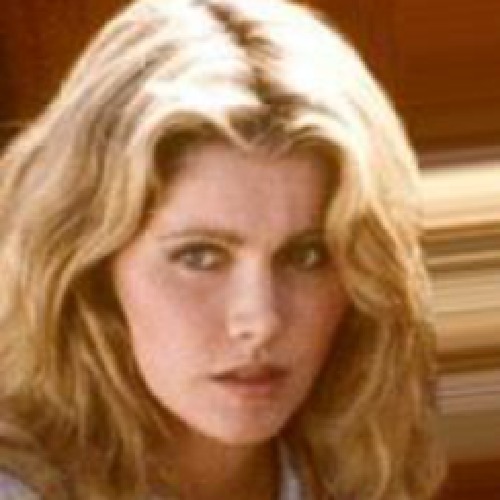} &
    \includegraphics[width=0.090\textwidth]{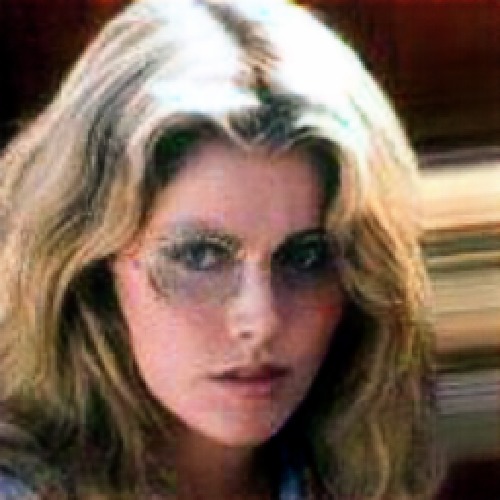} &
    \includegraphics[width=0.090\textwidth]{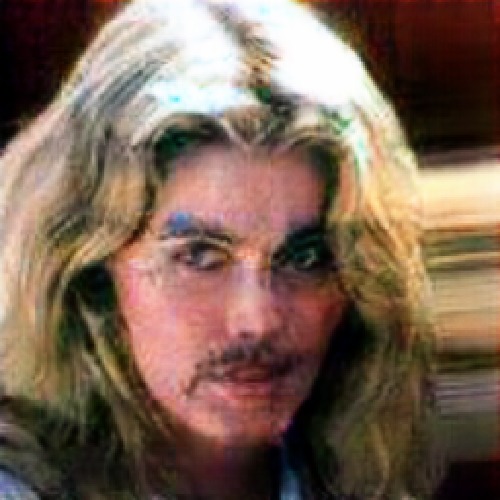} \\ 
    \includegraphics[width=0.090\textwidth]{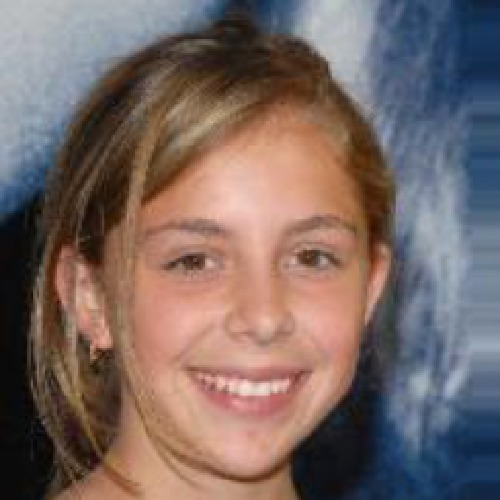} &
    \includegraphics[width=0.090\textwidth]{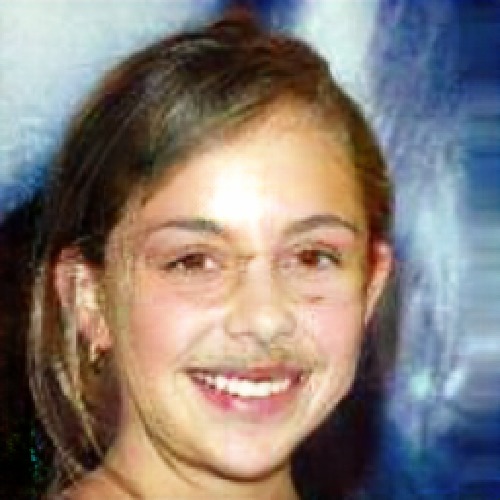} & 
    \includegraphics[width=0.090\textwidth]{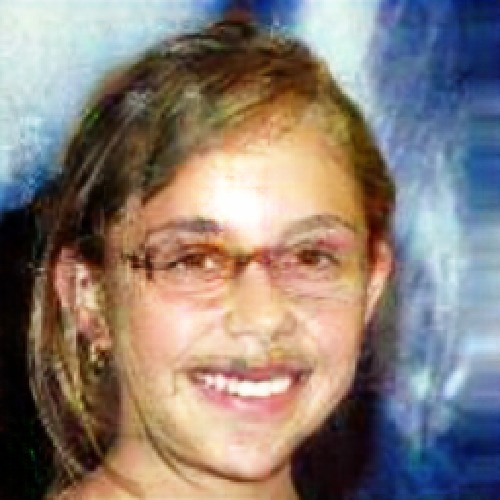} & &
    \includegraphics[width=0.090\textwidth]{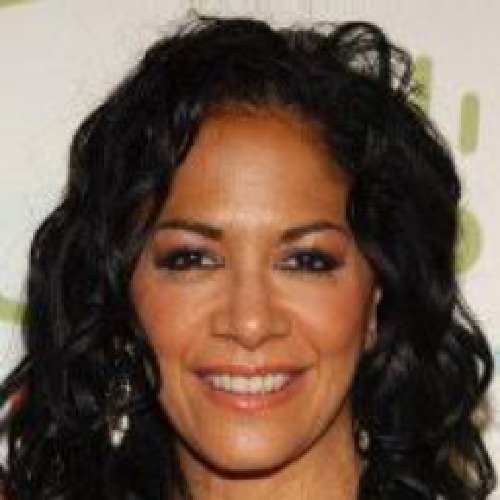} & 
    \includegraphics[width=0.090\textwidth]{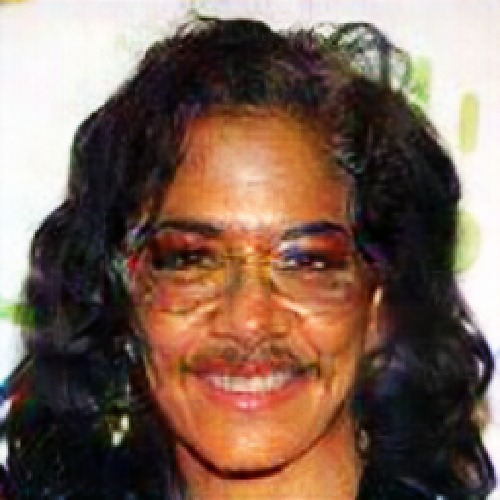} &
    \includegraphics[width=0.090\textwidth]{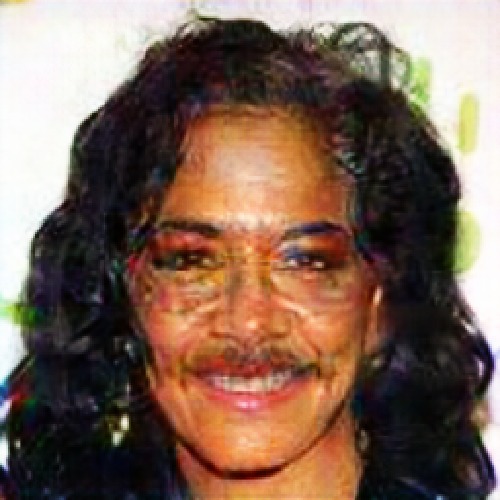} & &
    \includegraphics[width=0.090\textwidth]{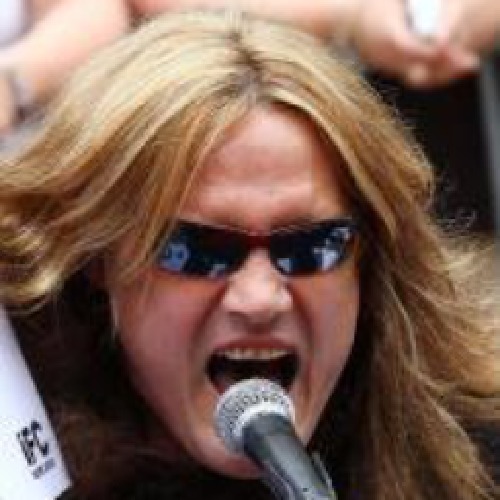} &
    \includegraphics[width=0.090\textwidth]{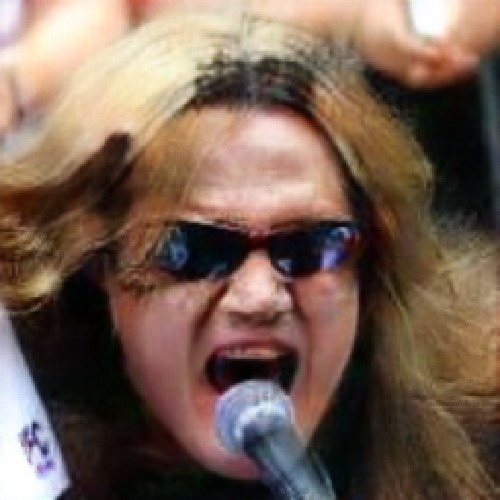} &
    \includegraphics[width=0.090\textwidth]{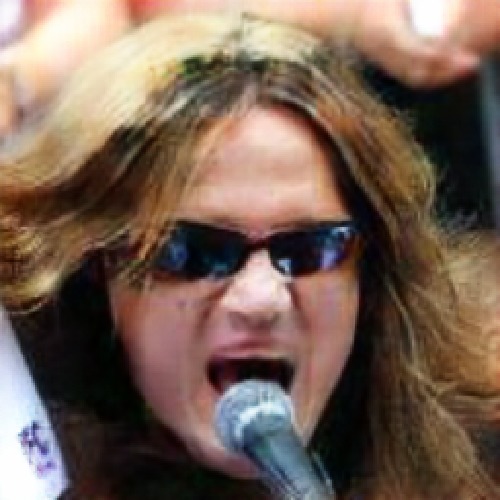} \\

    (a) &  (b) & (c) & & (d) & (e) & (f) & & (g) & (h) & (i) 
    \end{tabular}
    
    \caption{\sl Semantic adversarial examples generated with multiple attribute implementation using Adversarial AttGAN. The first, fourth and seventh columns contain the original images. We show adversarial examples generated under the attributes: (b),(e) and (h) Eyeglasses-Mustache-Age-Pale Skin-Young-Black Hair and (c),(f) and (i)Eyeglasses-Mustache-Pale skin-Age-Bushy eyebrows-Black hair. The quality of the images produced by the Adversarial AttGAN are sharper than those produced by the Adversarial Fader Networks. }
    \label{fig:attgan1}
\end{figure}
\endgroup

\begingroup
\begin{figure}[htp]
    \centering
    \setlength{\tabcolsep}{1pt}
    \renewcommand{\arraystretch}{0.5}
    \begin{tabular}{c c c c || c c c c ||  c c c }
    \includegraphics[width=0.090\textwidth]{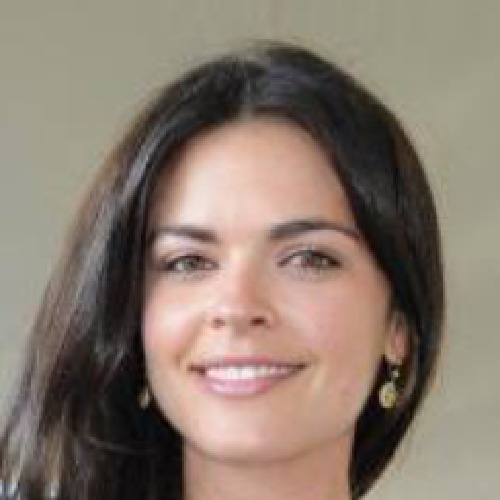} &
    \includegraphics[width=0.090\textwidth]{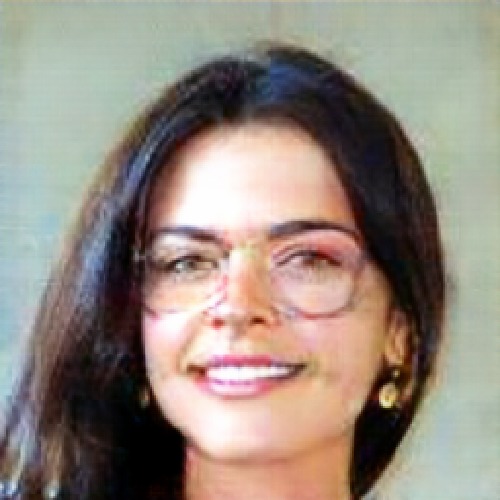} & 
    \includegraphics[width=0.090\textwidth]{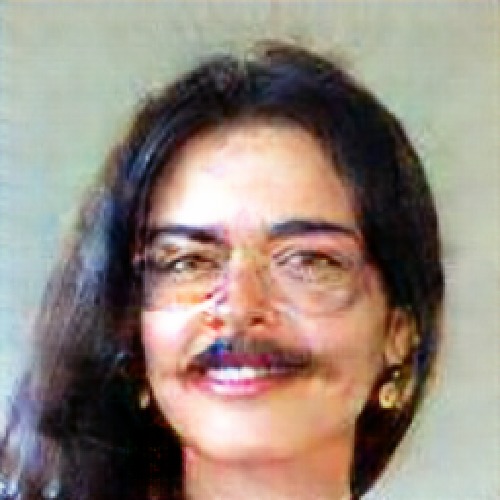} & &
    \includegraphics[width=0.090\textwidth]{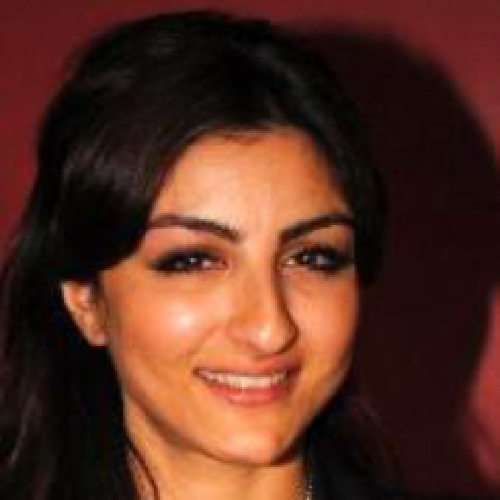} & 
    \includegraphics[width=0.090\textwidth]{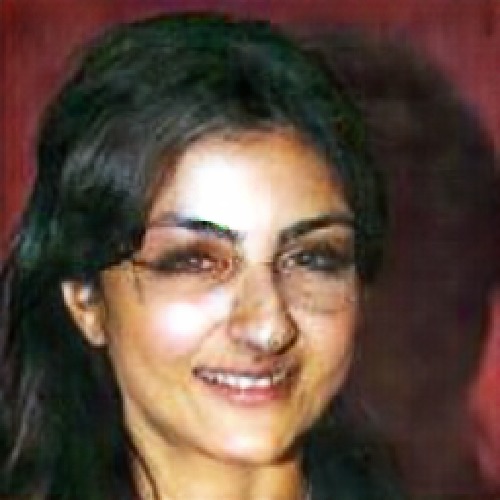} &
    \includegraphics[width=0.090\textwidth]{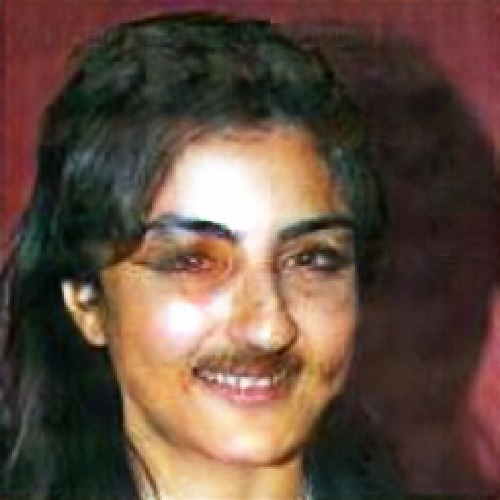} & &
    \includegraphics[width=0.090\textwidth]{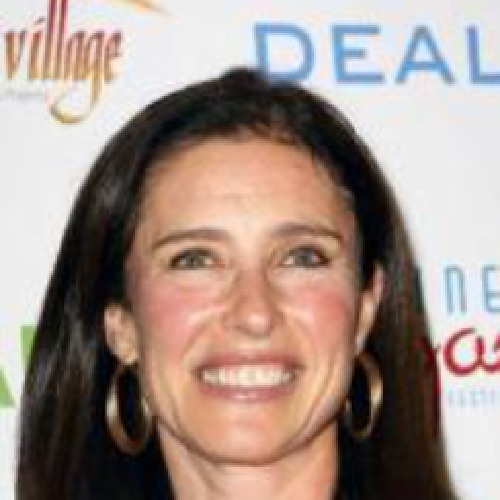} &
    \includegraphics[width=0.090\textwidth]{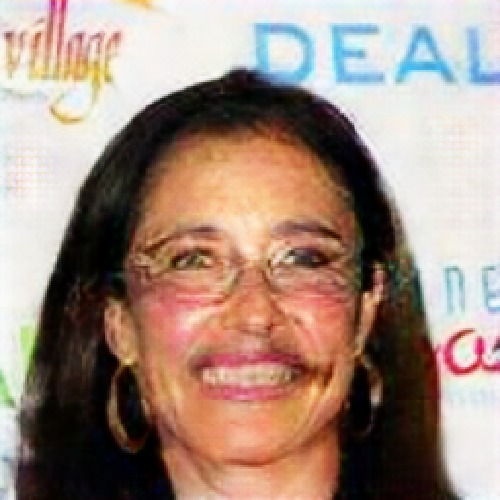} &
    \includegraphics[width=0.090\textwidth]{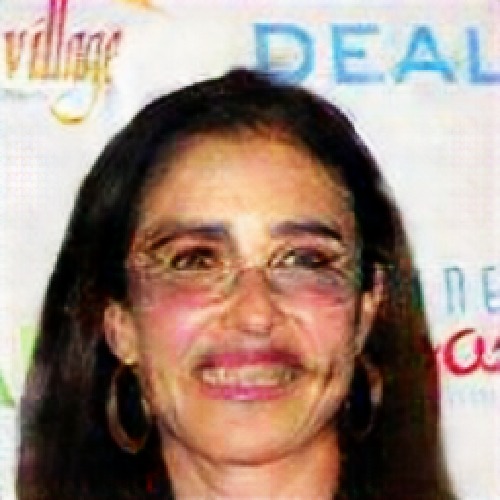} \\ 
    \includegraphics[width=0.090\textwidth]{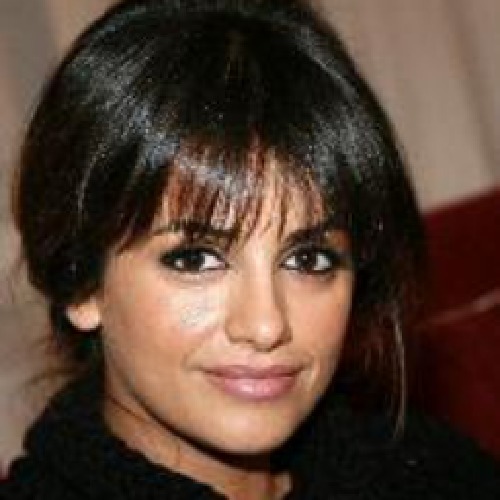} &
    \includegraphics[width=0.090\textwidth]{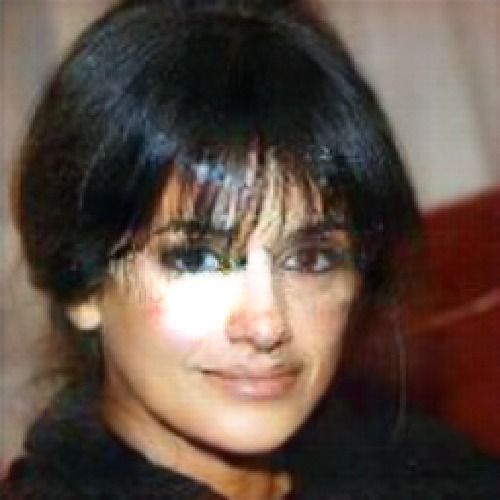} & 
    \includegraphics[width=0.090\textwidth]{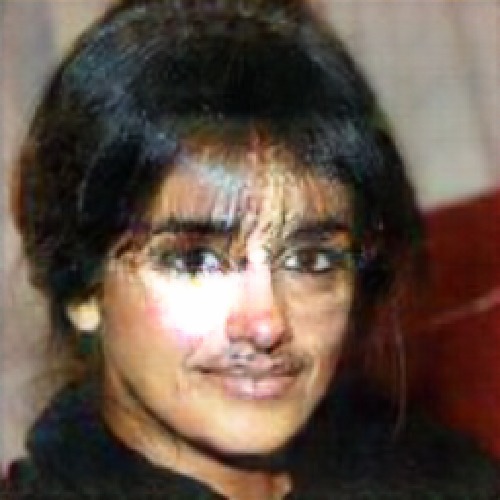} & &
    \includegraphics[width=0.090\textwidth]{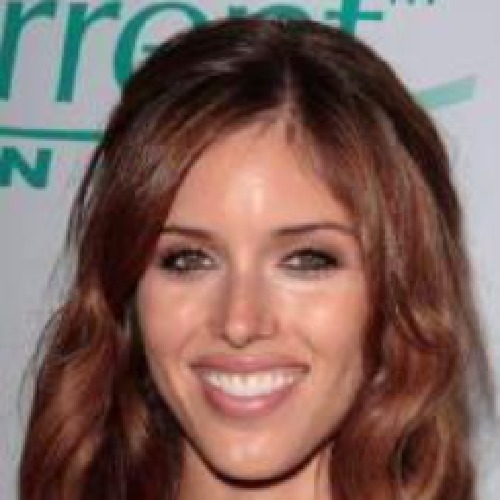} & 
    \includegraphics[width=0.090\textwidth]{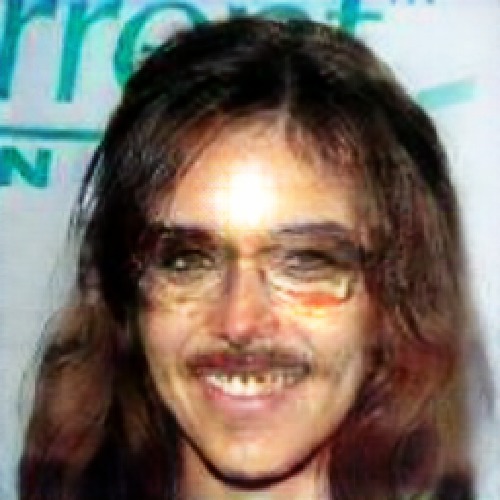} &
    \includegraphics[width=0.090\textwidth]{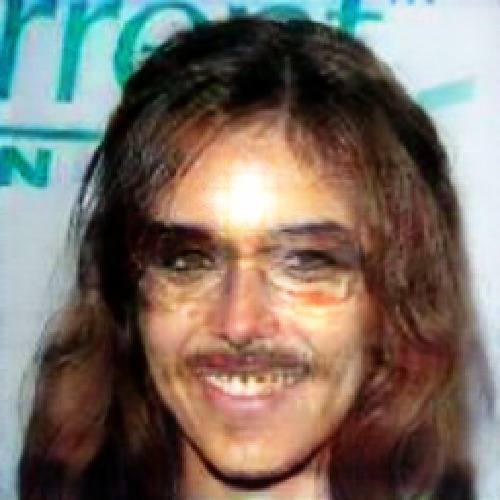} & &
    \includegraphics[width=0.090\textwidth]{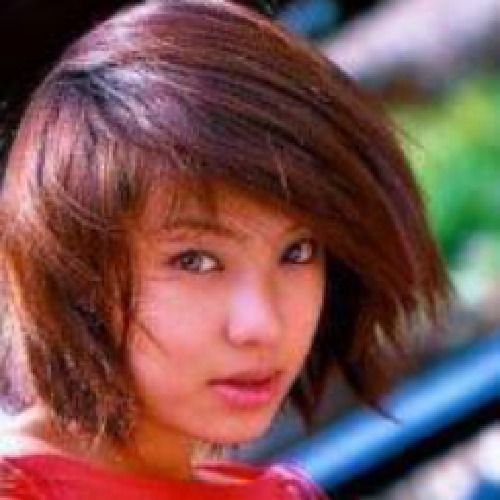} &
    \includegraphics[width=0.090\textwidth]{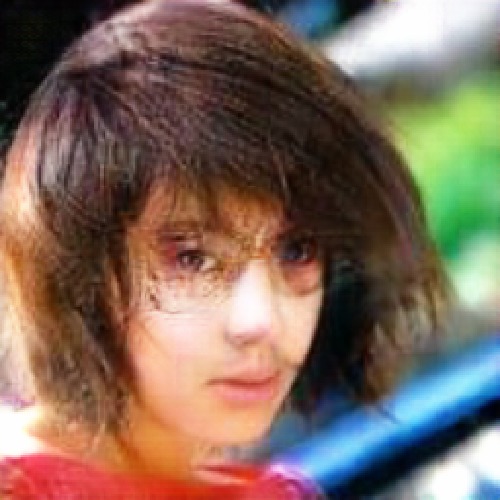} &
    \includegraphics[width=0.090\textwidth]{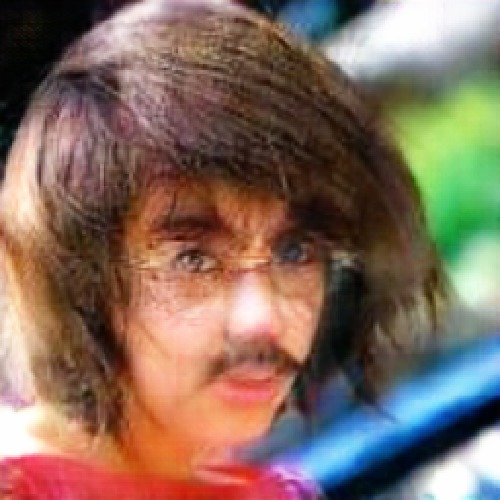} \\ 
    \includegraphics[width=0.090\textwidth]{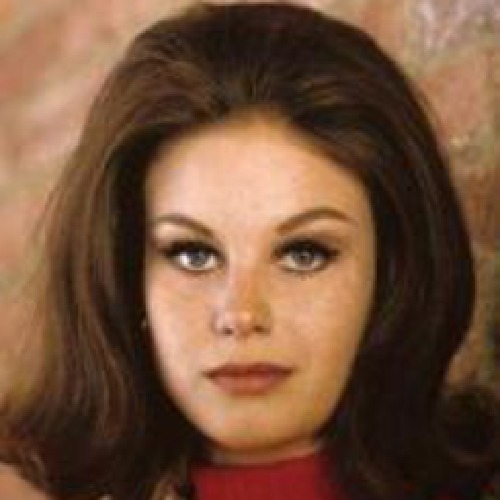} &
    \includegraphics[width=0.090\textwidth]{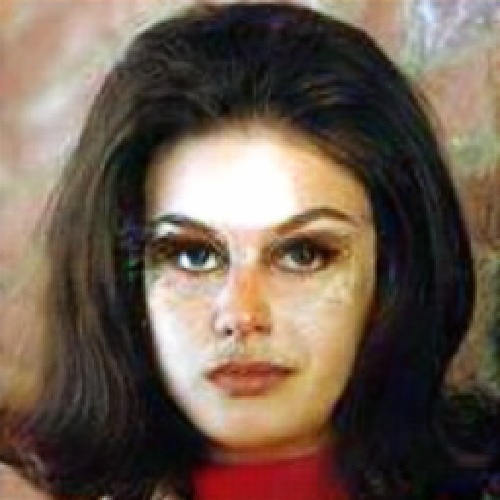} & 
    \includegraphics[width=0.090\textwidth]{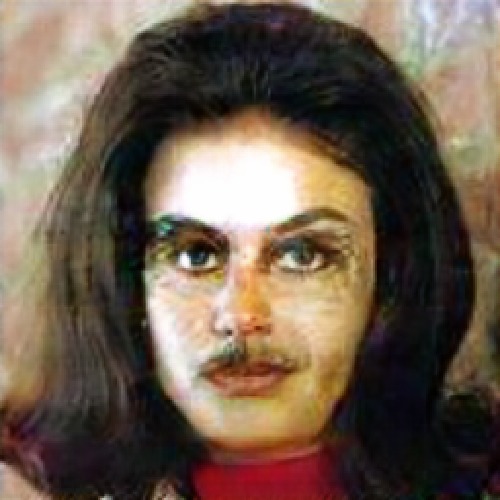} & &
    \includegraphics[width=0.090\textwidth]{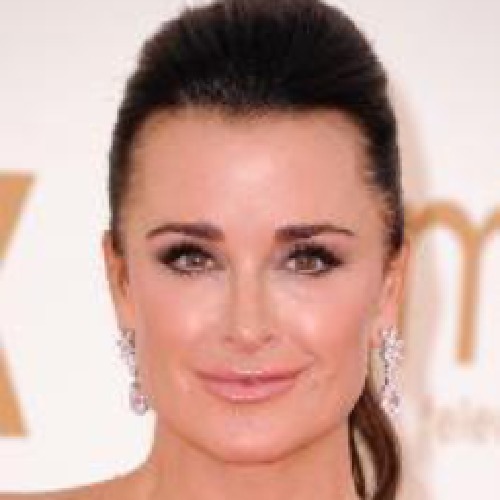} & 
    \includegraphics[width=0.090\textwidth]{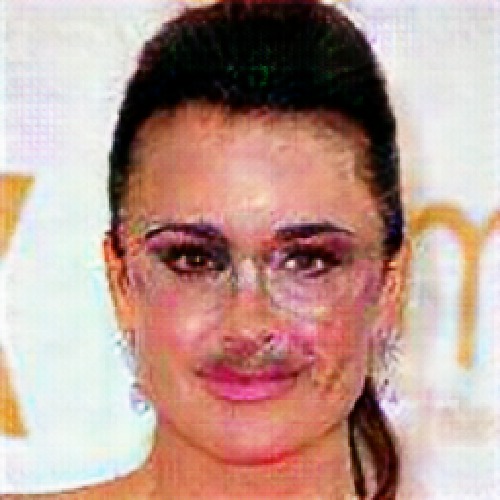} &
    \includegraphics[width=0.090\textwidth]{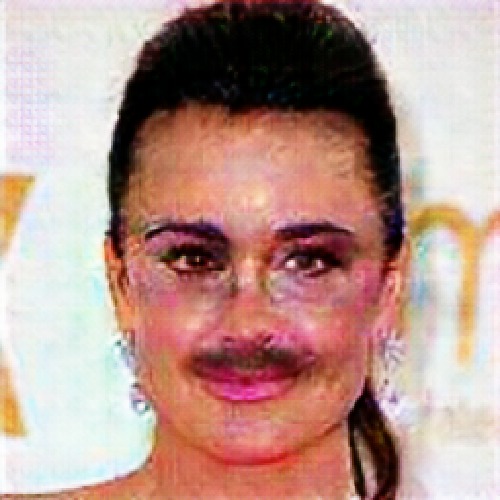} & &
    \includegraphics[width=0.090\textwidth]{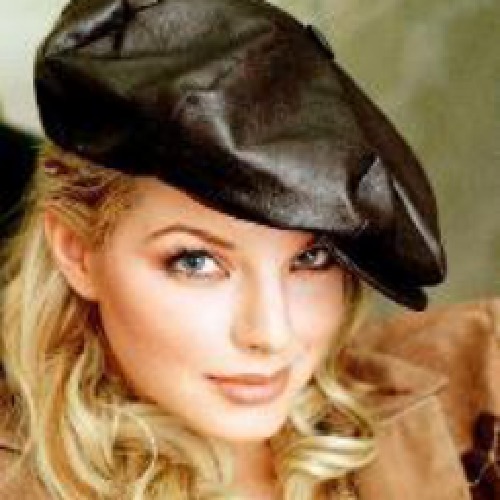} &
    \includegraphics[width=0.090\textwidth]{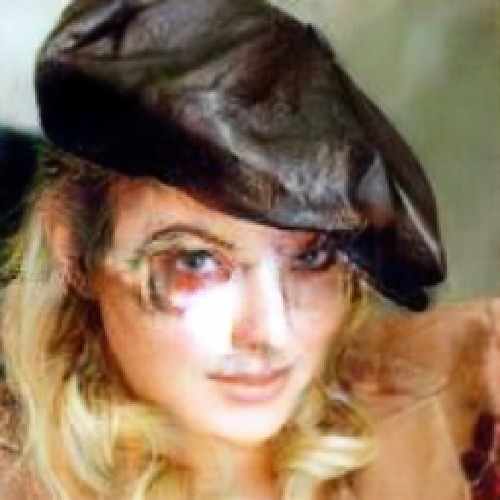} &
    \includegraphics[width=0.090\textwidth]{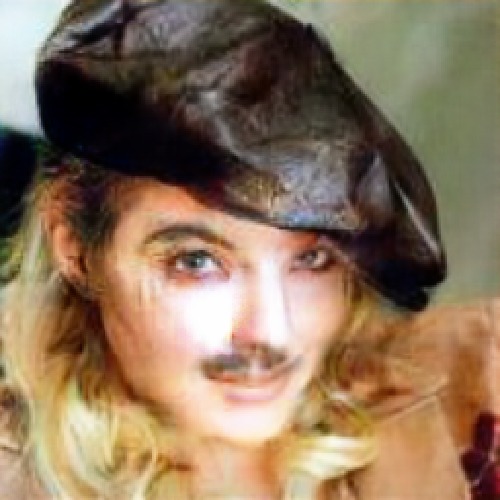} \\ 
    \includegraphics[width=0.090\textwidth]{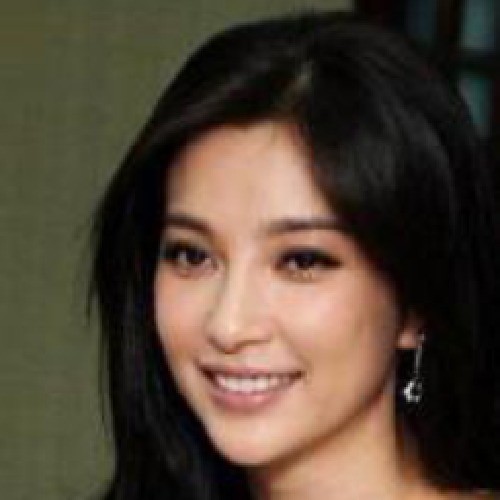} &
    \includegraphics[width=0.090\textwidth]{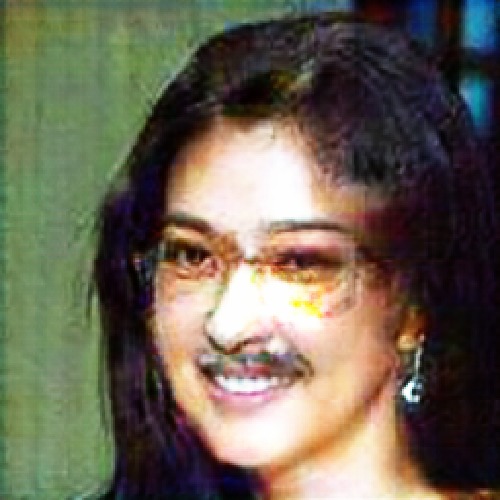} & 
    \includegraphics[width=0.090\textwidth]{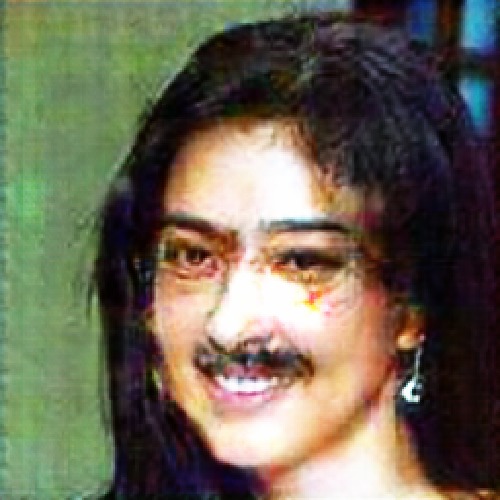} & &
    \includegraphics[width=0.090\textwidth]{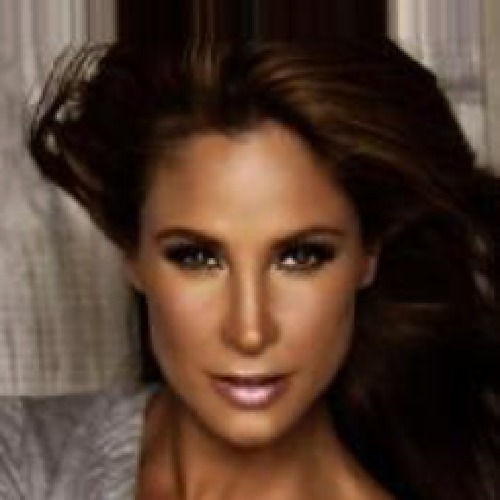} & 
    \includegraphics[width=0.090\textwidth]{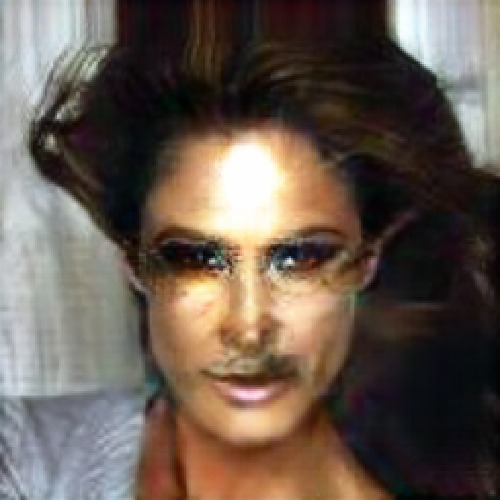} &
    \includegraphics[width=0.090\textwidth]{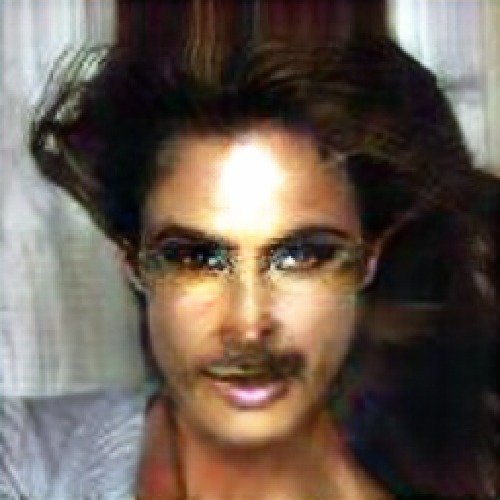} & &
    \includegraphics[width=0.090\textwidth]{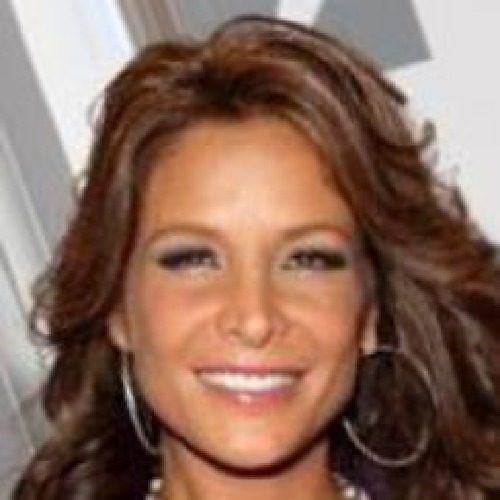} &
    \includegraphics[width=0.090\textwidth]{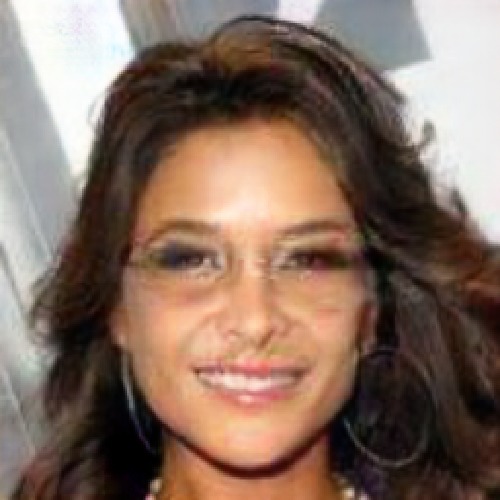} &
    \includegraphics[width=0.090\textwidth]{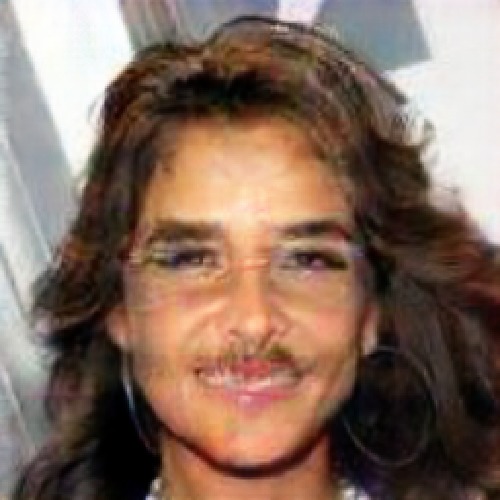} \\ 
    \includegraphics[width=0.090\textwidth]{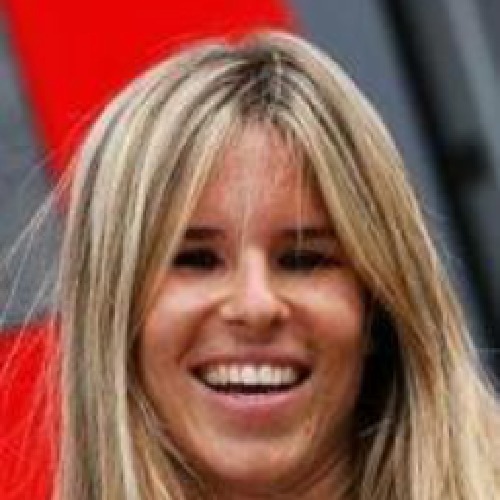} &
    \includegraphics[width=0.090\textwidth]{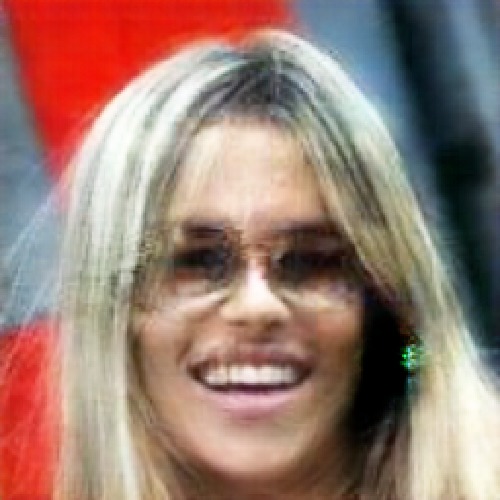} & 
    \includegraphics[width=0.090\textwidth]{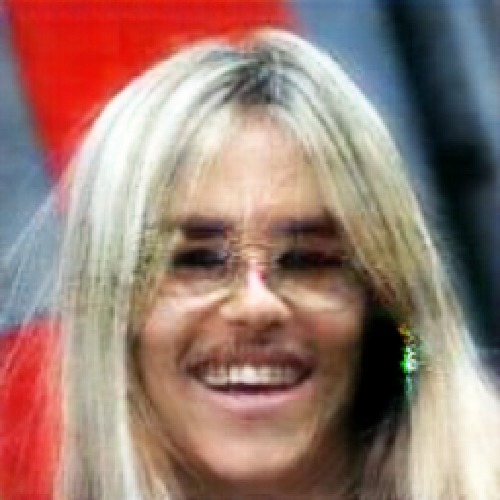} & &
    \includegraphics[width=0.090\textwidth]{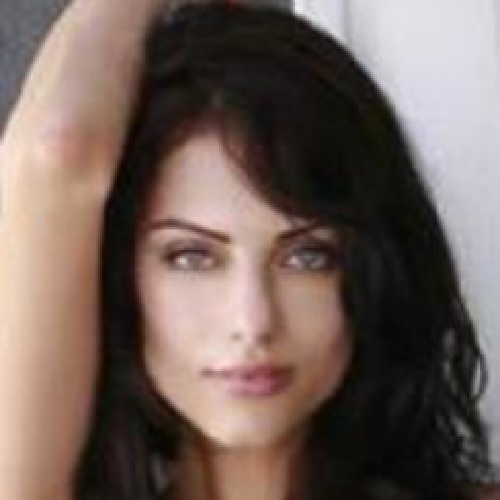} & 
    \includegraphics[width=0.090\textwidth]{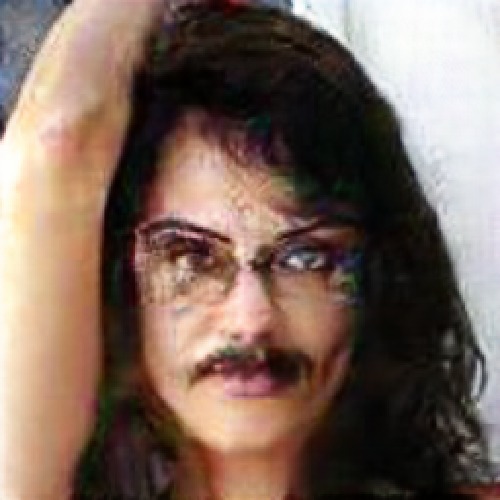} &
    \includegraphics[width=0.090\textwidth]{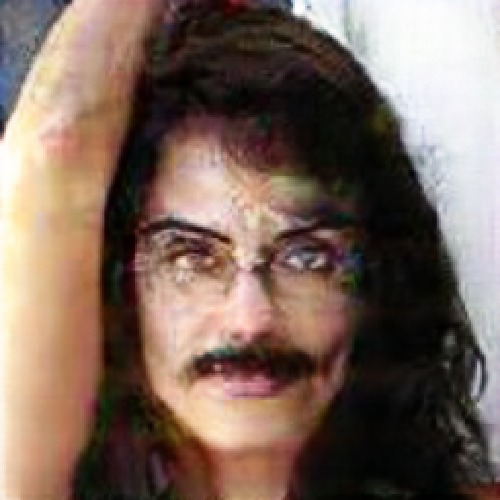} & &
    \includegraphics[width=0.090\textwidth]{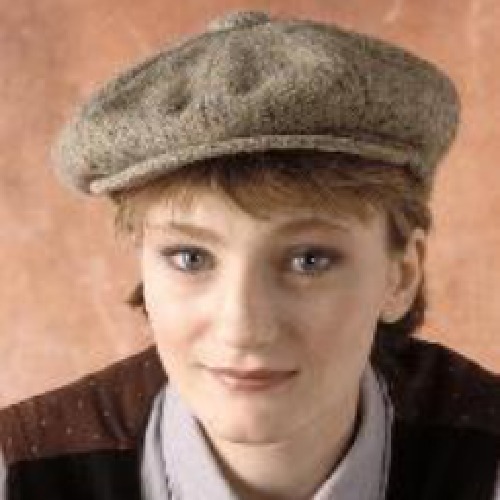} &
    \includegraphics[width=0.090\textwidth]{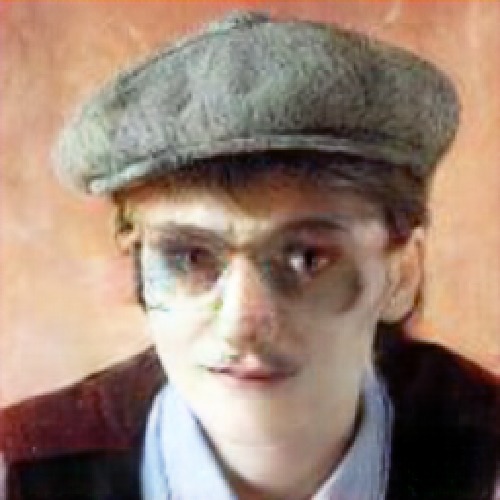} &
    \includegraphics[width=0.090\textwidth]{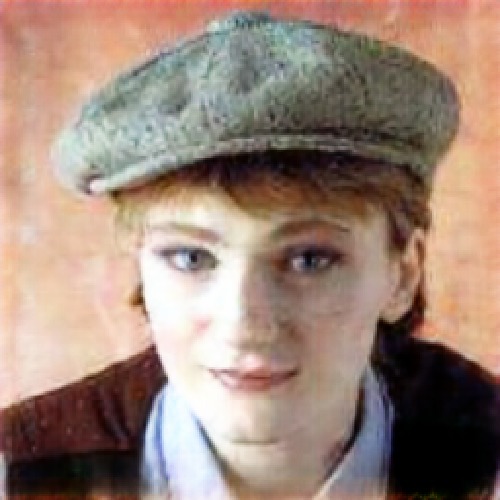} \\ 
    \includegraphics[width=0.090\textwidth]{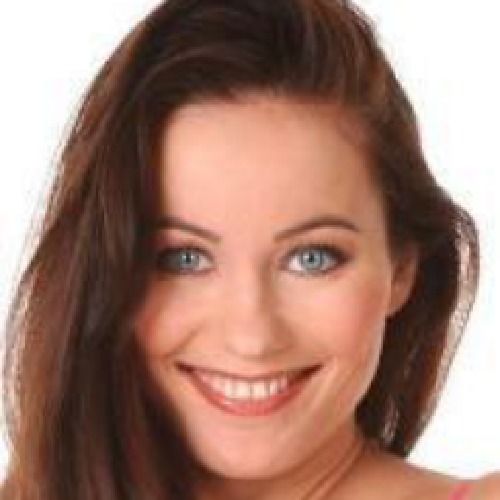} &
    \includegraphics[width=0.090\textwidth]{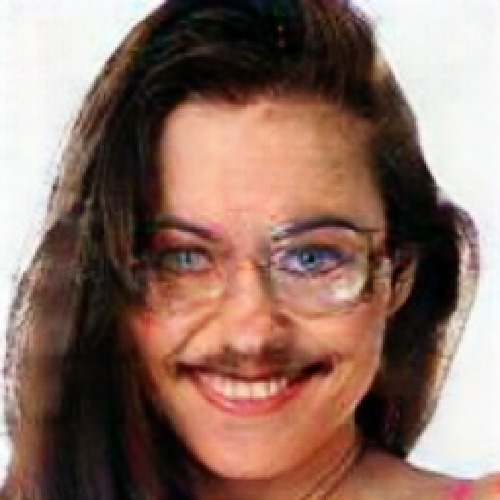} & 
    \includegraphics[width=0.090\textwidth]{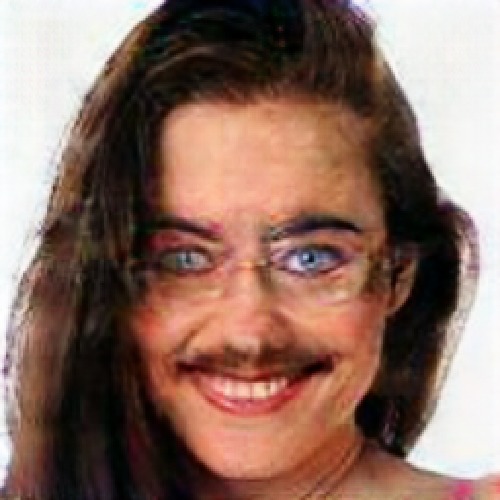} & &
    \includegraphics[width=0.090\textwidth]{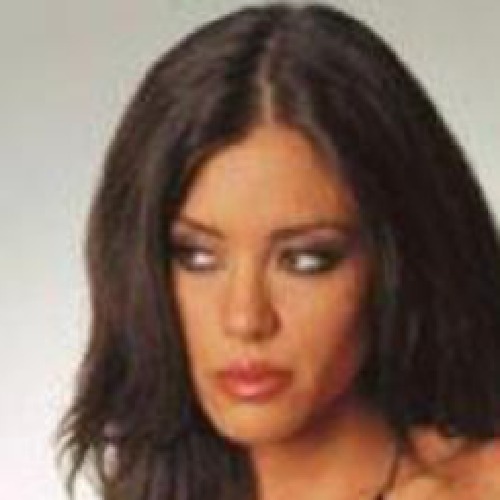} & 
    \includegraphics[width=0.090\textwidth]{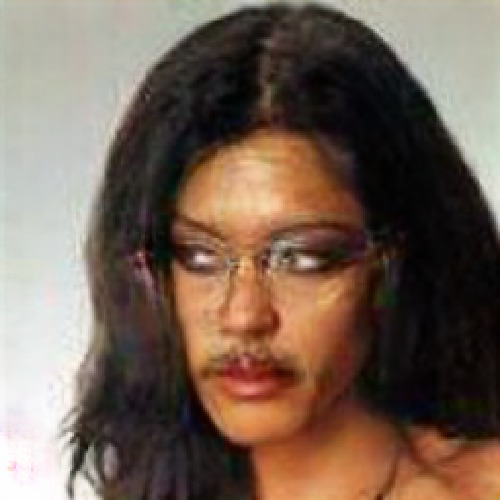} &
    \includegraphics[width=0.090\textwidth]{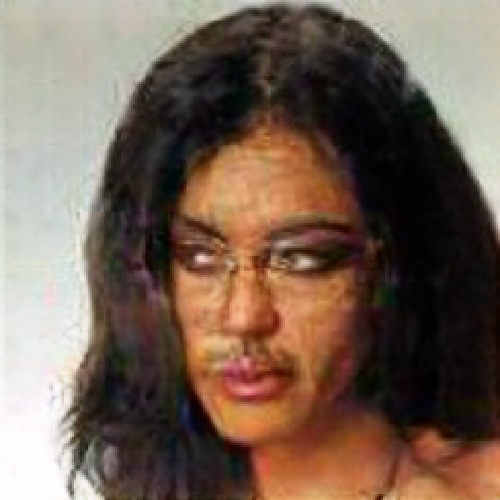} & &
    \includegraphics[width=0.090\textwidth]{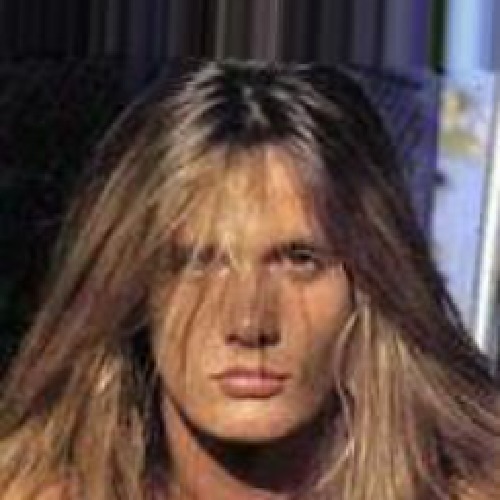} &
    \includegraphics[width=0.090\textwidth]{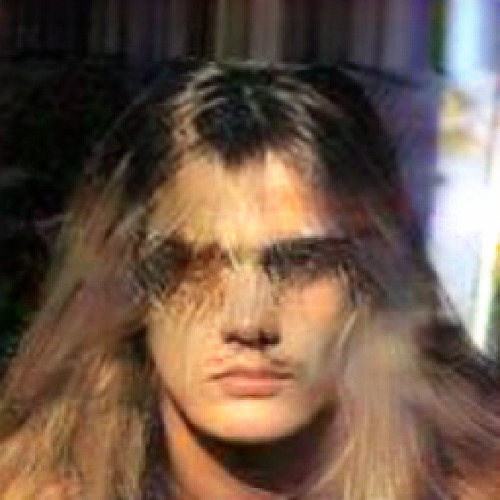} &
    \includegraphics[width=0.090\textwidth]{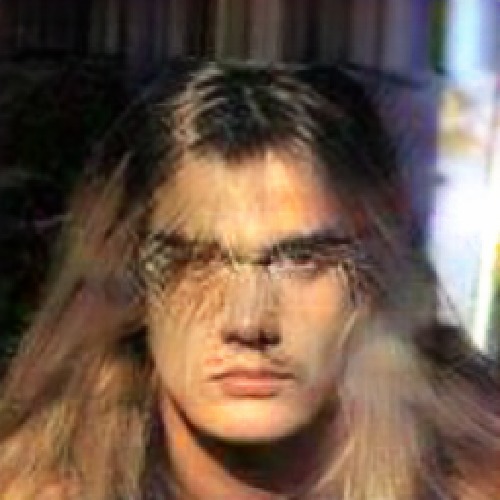} \\ 
    \includegraphics[width=0.090\textwidth]{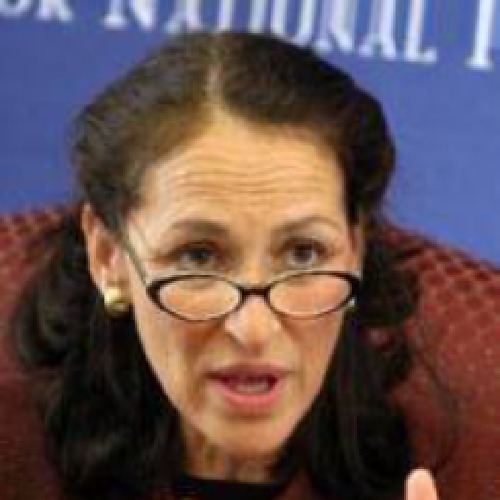} &
    \includegraphics[width=0.090\textwidth]{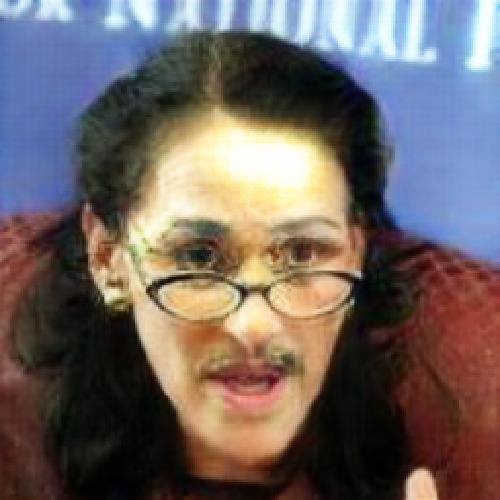} & 
    \includegraphics[width=0.090\textwidth]{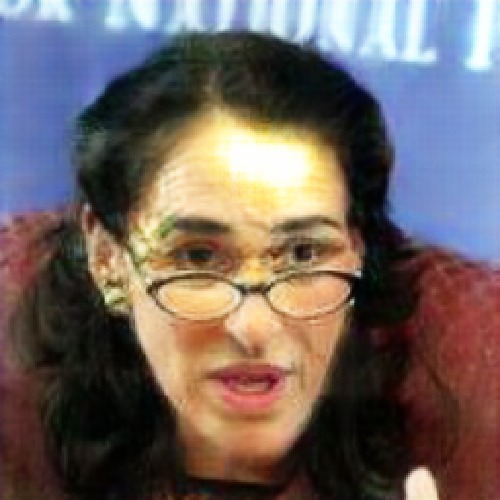} & &
    \includegraphics[width=0.090\textwidth]{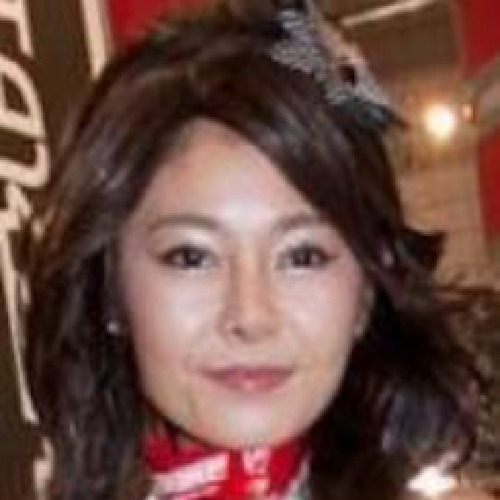} & 
    \includegraphics[width=0.090\textwidth]{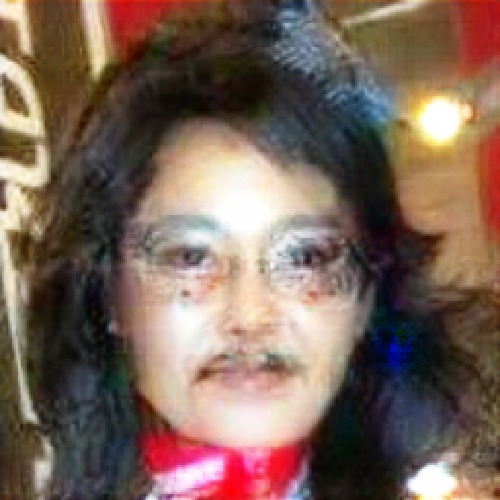} &
    \includegraphics[width=0.090\textwidth]{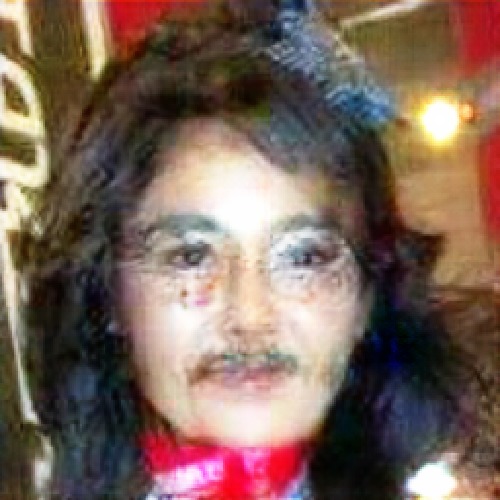} & &
    \includegraphics[width=0.090\textwidth]{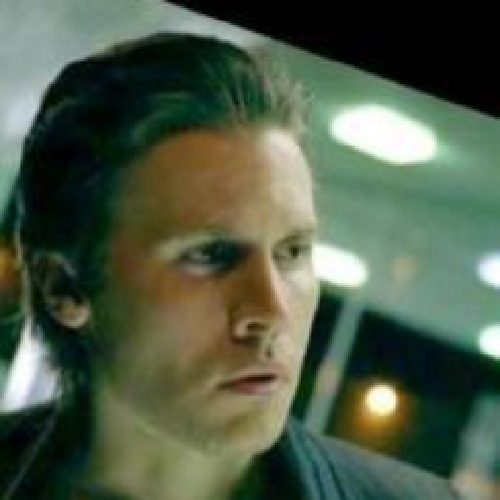} &
    \includegraphics[width=0.090\textwidth]{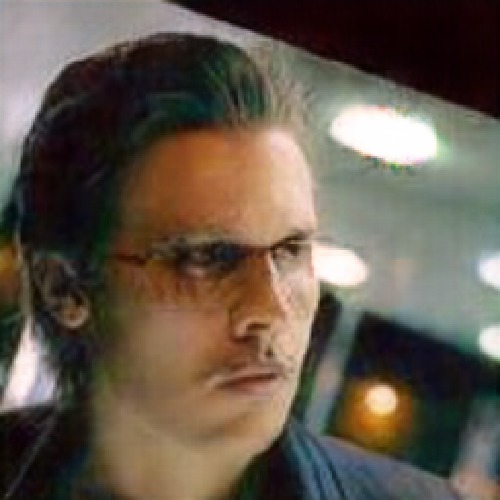} &
    \includegraphics[width=0.090\textwidth]{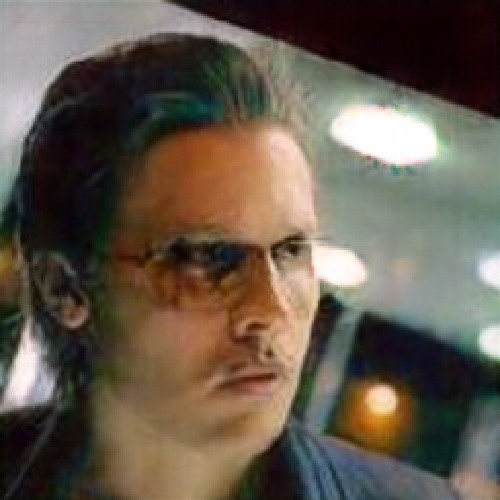} \\ 
    \includegraphics[width=0.090\textwidth]{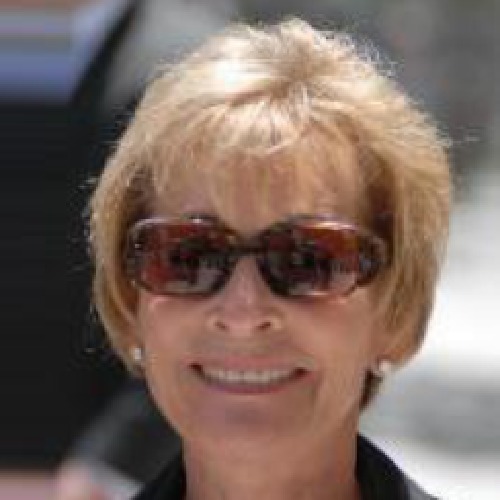} &
    \includegraphics[width=0.090\textwidth]{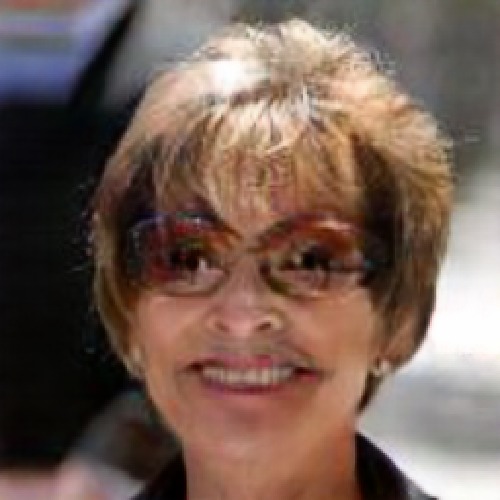} & 
    \includegraphics[width=0.090\textwidth]{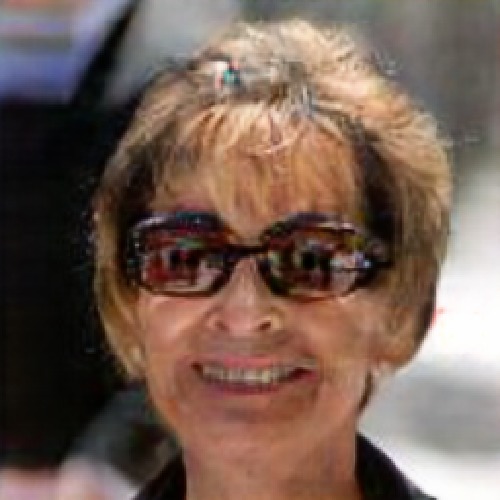} & &
    \includegraphics[width=0.090\textwidth]{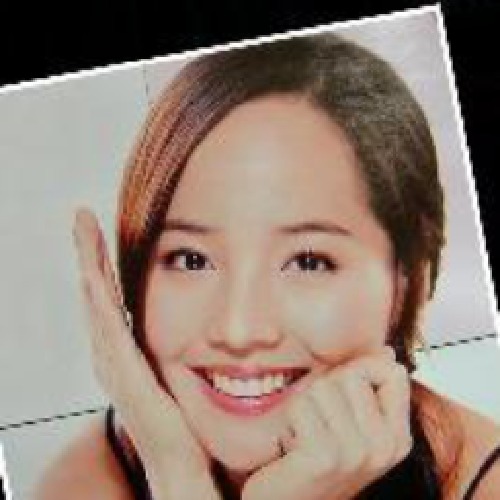} & 
    \includegraphics[width=0.090\textwidth]{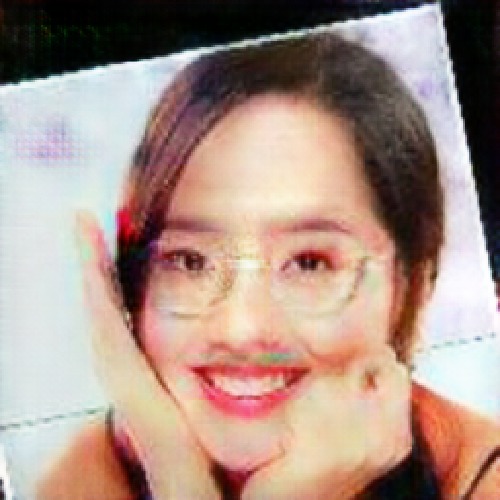} &
    \includegraphics[width=0.090\textwidth]{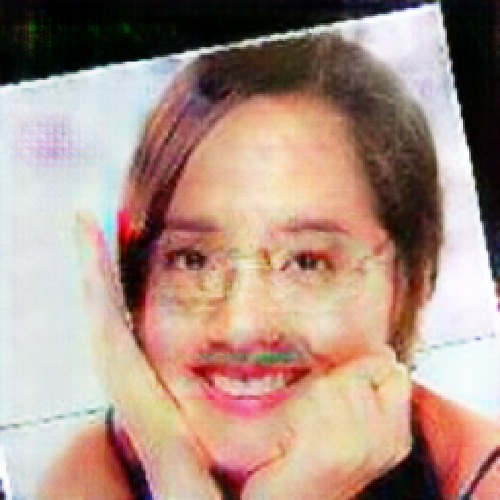} & &
    \includegraphics[width=0.090\textwidth]{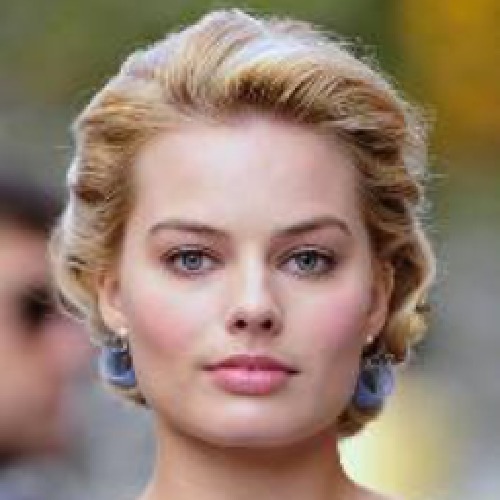} &
    \includegraphics[width=0.090\textwidth]{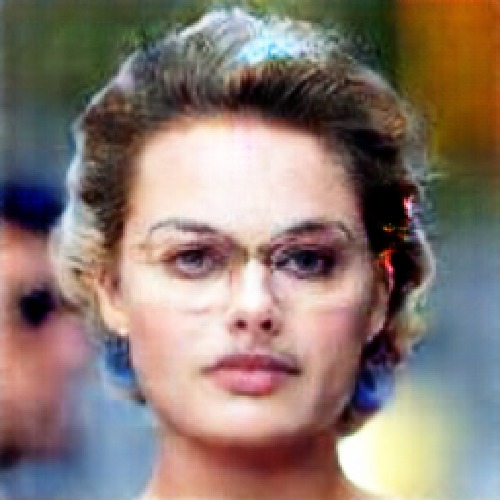} &
    \includegraphics[width=0.090\textwidth]{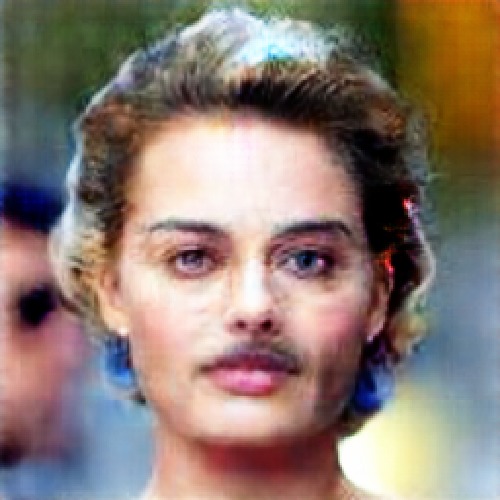} \\ 
    \includegraphics[width=0.090\textwidth]{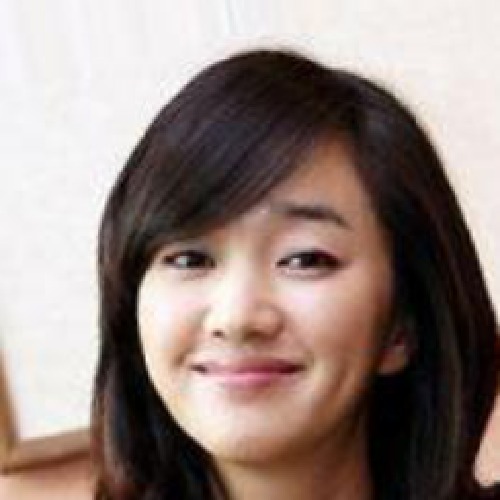} &
    \includegraphics[width=0.090\textwidth]{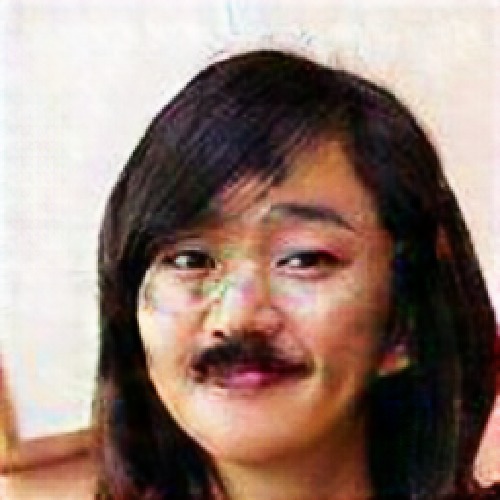} & 
    \includegraphics[width=0.090\textwidth]{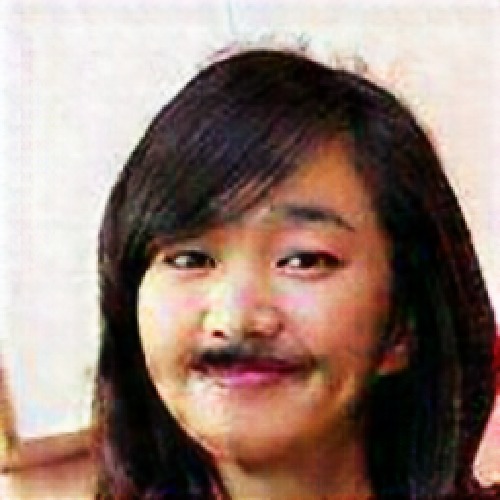} & &
    \includegraphics[width=0.090\textwidth]{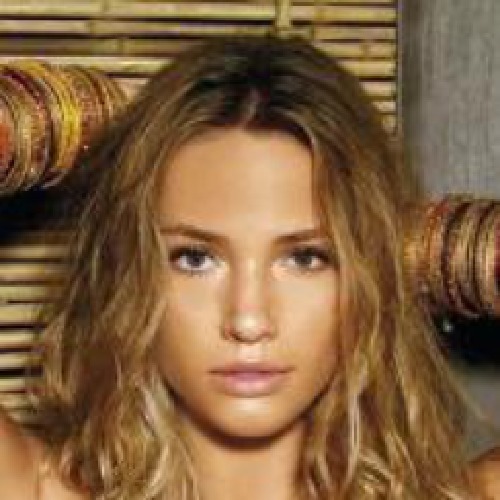} & 
    \includegraphics[width=0.090\textwidth]{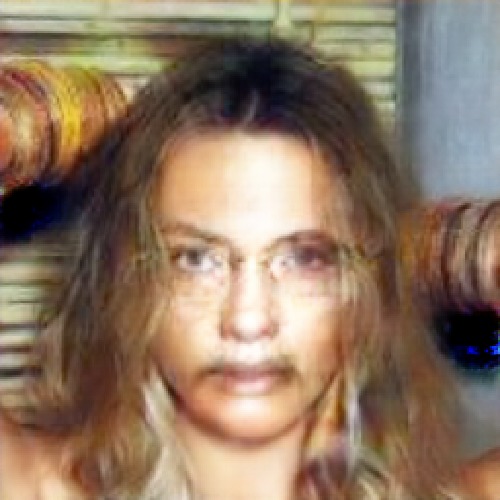} &
    \includegraphics[width=0.090\textwidth]{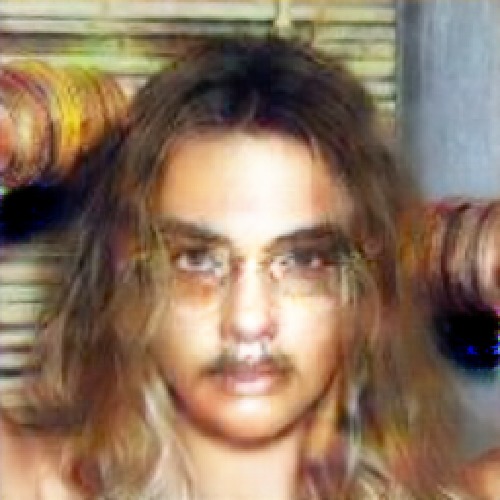} & &
    \includegraphics[width=0.090\textwidth]{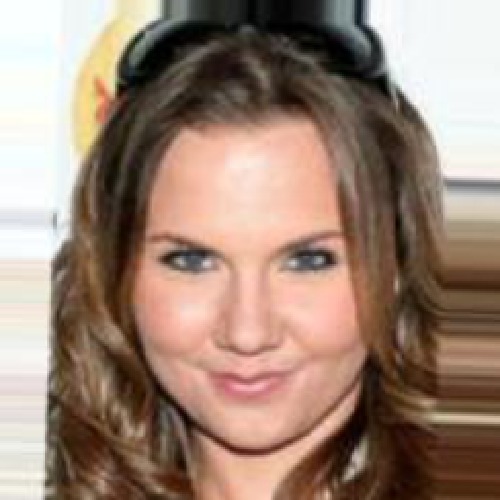} &
    \includegraphics[width=0.090\textwidth]{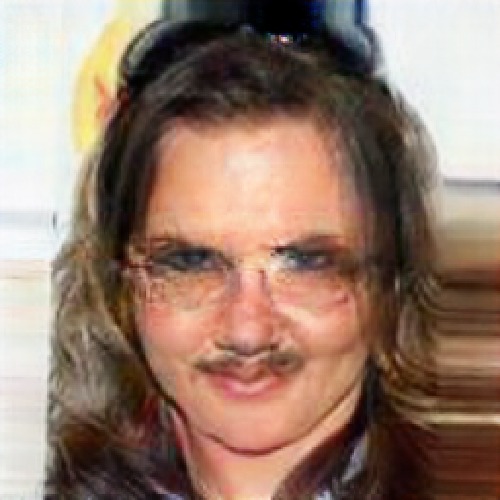} &
    \includegraphics[width=0.090\textwidth]{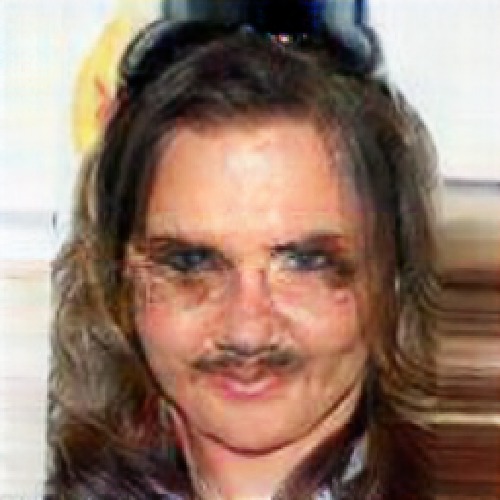} \\ 
    \includegraphics[width=0.090\textwidth]{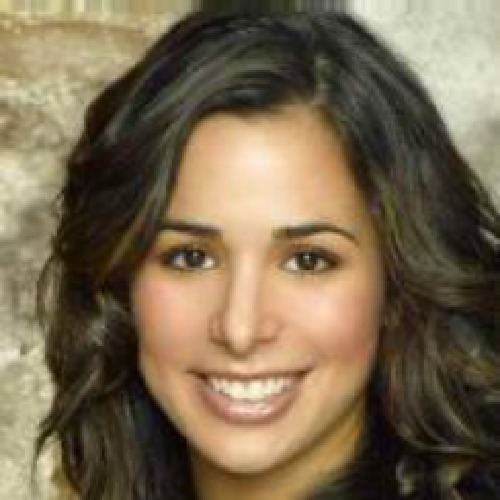} &
    \includegraphics[width=0.090\textwidth]{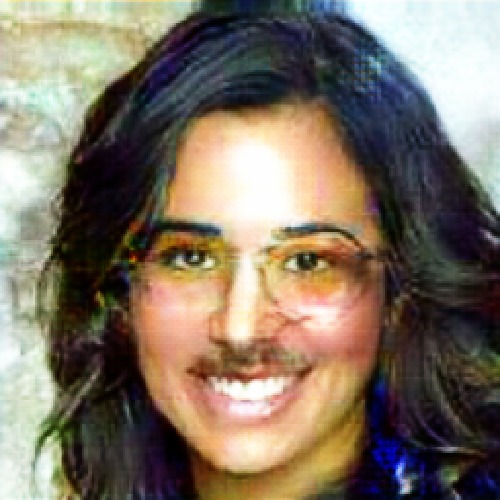} & 
    \includegraphics[width=0.090\textwidth]{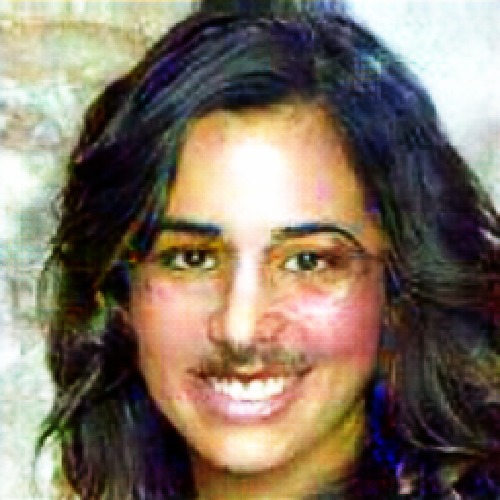} & &
    \includegraphics[width=0.090\textwidth]{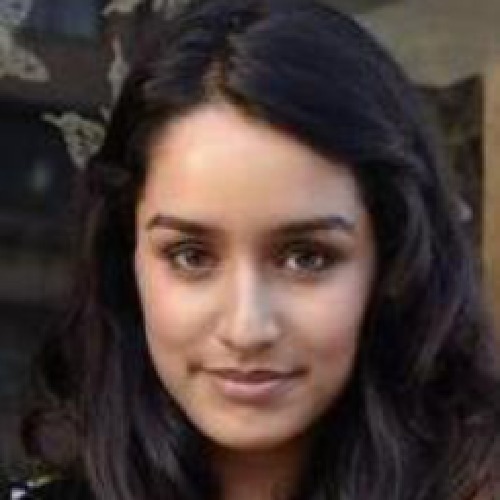} & 
    \includegraphics[width=0.090\textwidth]{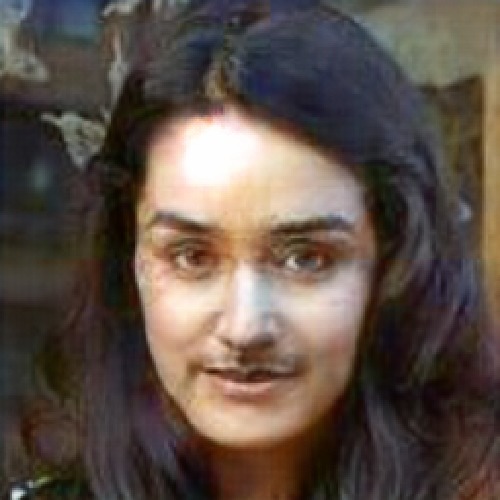} &
    \includegraphics[width=0.090\textwidth]{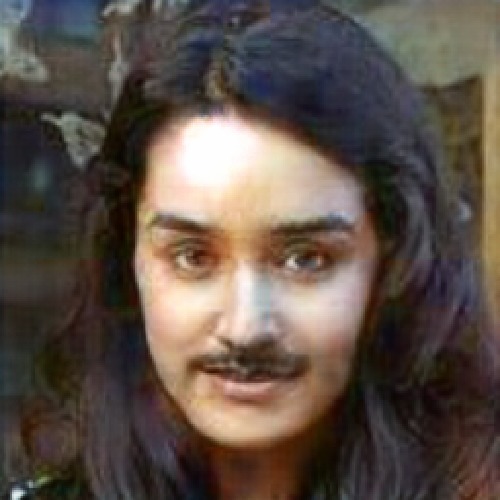} & &
    \includegraphics[width=0.090\textwidth]{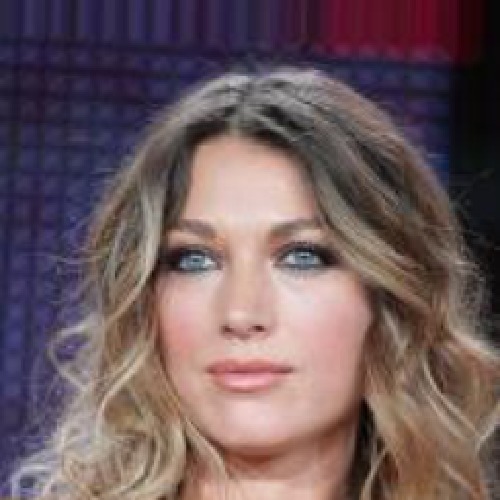} &
    \includegraphics[width=0.090\textwidth]{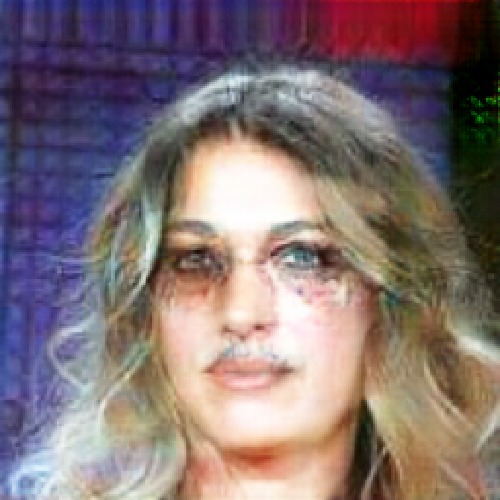} &
    \includegraphics[width=0.090\textwidth]{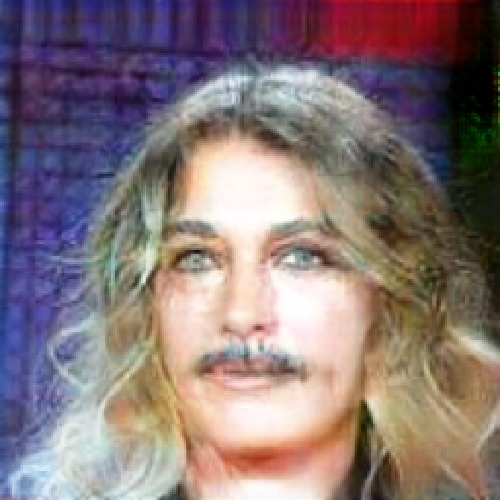} \\ 
    \includegraphics[width=0.090\textwidth]{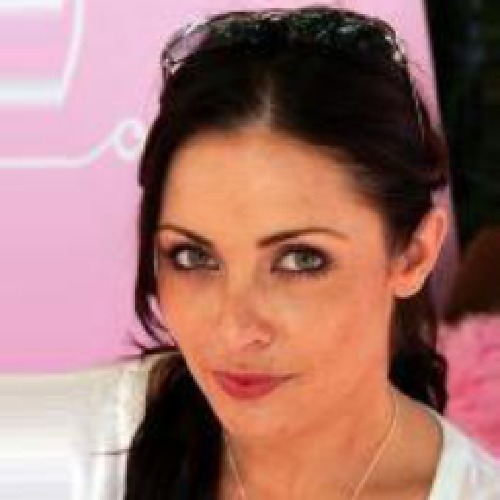} &
    \includegraphics[width=0.090\textwidth]{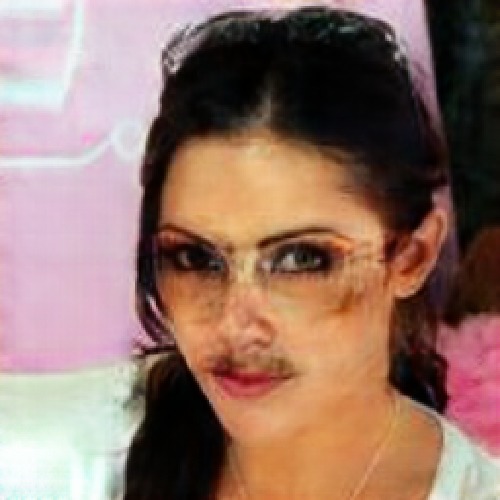} & 
    \includegraphics[width=0.090\textwidth]{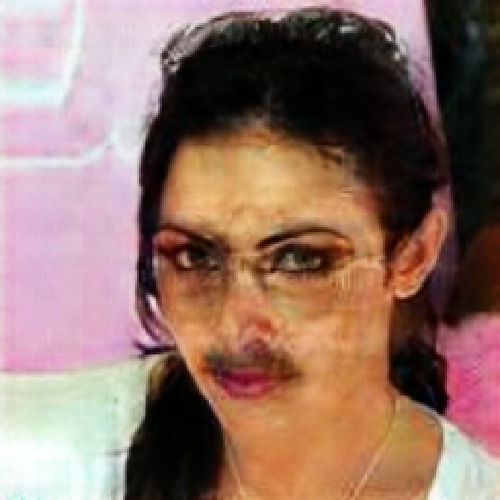} & &
    \includegraphics[width=0.090\textwidth]{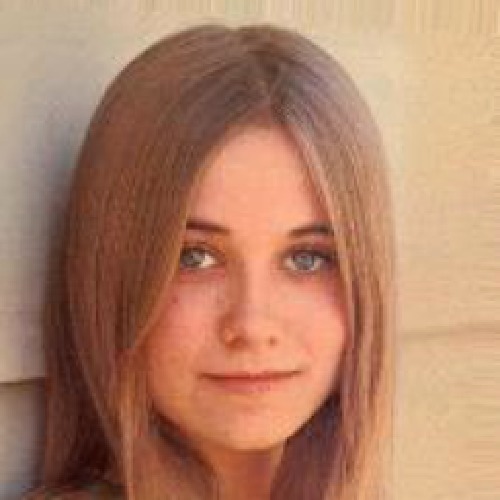} & 
    \includegraphics[width=0.090\textwidth]{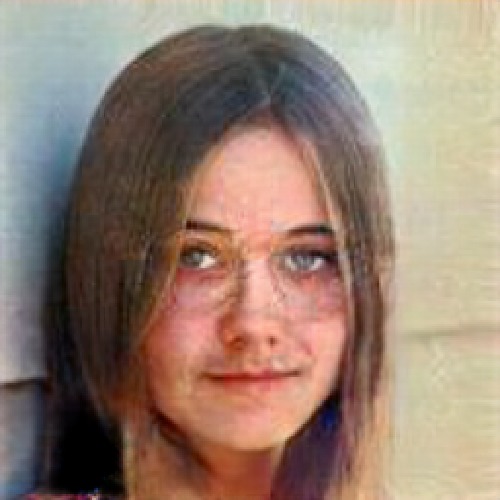} &
    \includegraphics[width=0.090\textwidth]{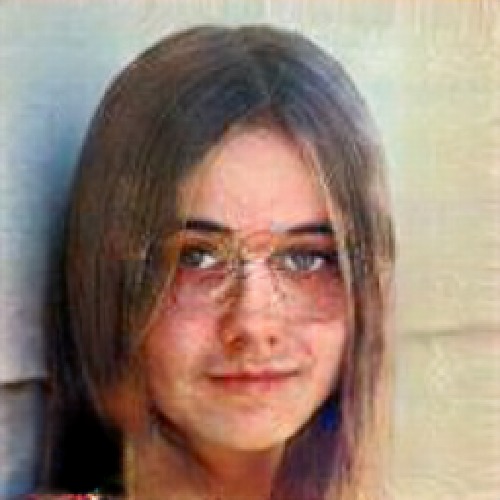} & &
    \includegraphics[width=0.090\textwidth]{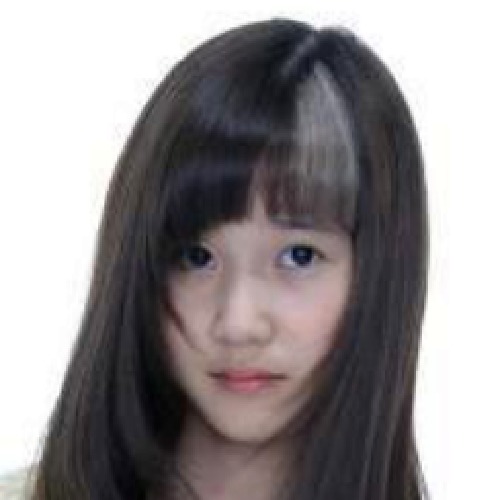} &
    \includegraphics[width=0.090\textwidth]{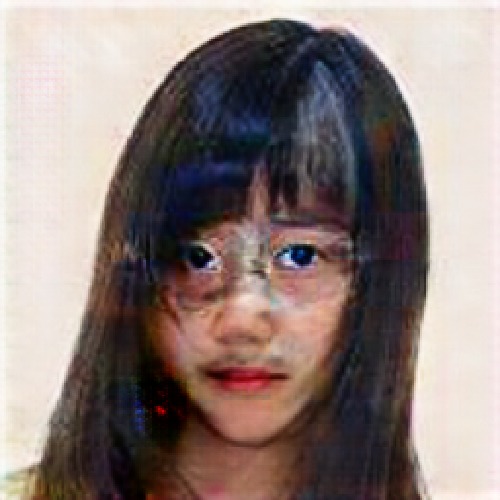} &
    \includegraphics[width=0.090\textwidth]{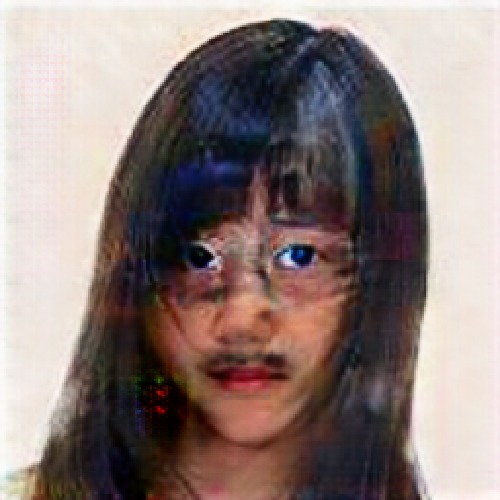} \\ 
    \includegraphics[width=0.090\textwidth]{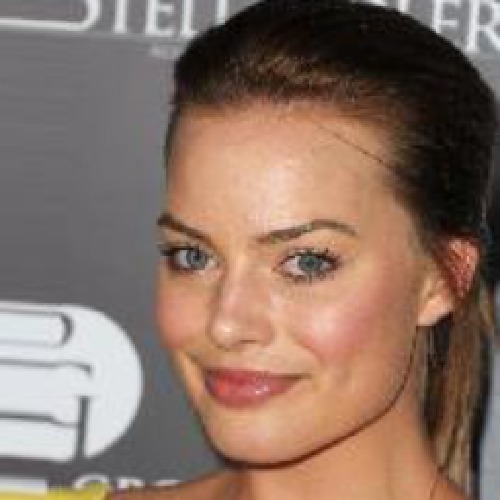} &
    \includegraphics[width=0.090\textwidth]{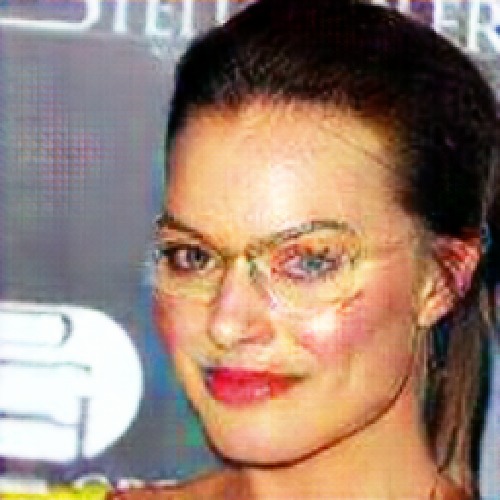} & 
    \includegraphics[width=0.090\textwidth]{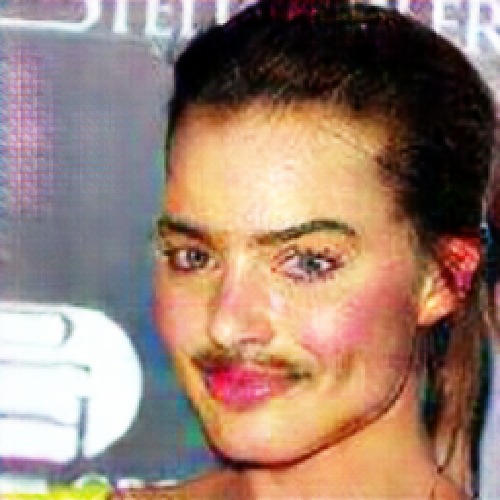} & &
    \includegraphics[width=0.090\textwidth]{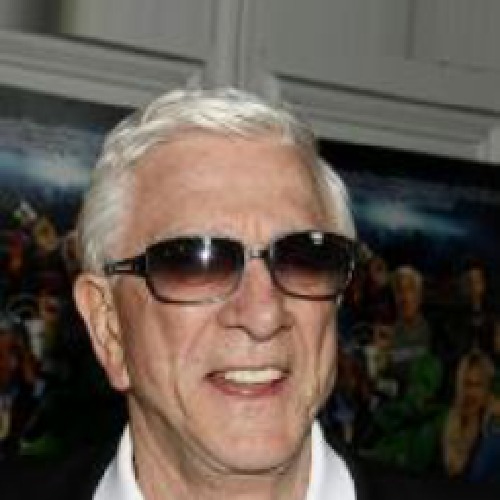} & 
    \includegraphics[width=0.090\textwidth]{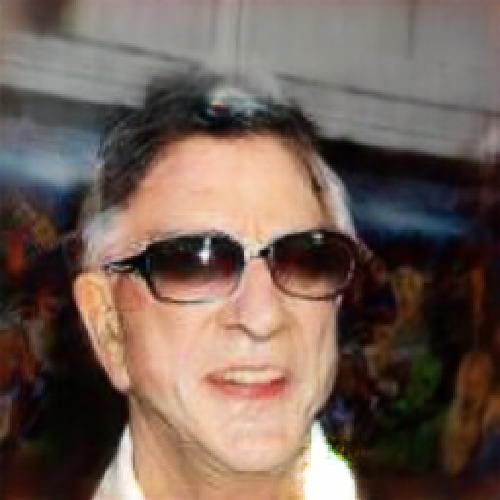} &
    \includegraphics[width=0.090\textwidth]{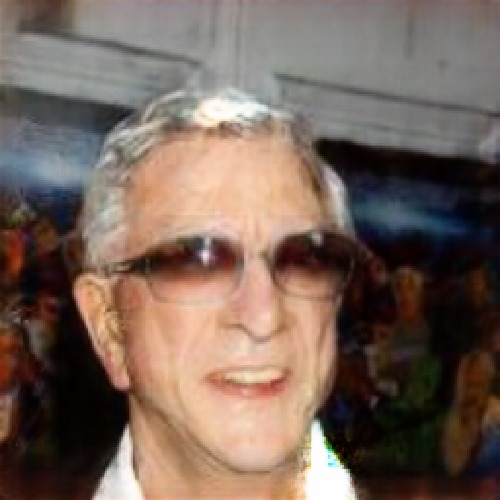} & &
    \includegraphics[width=0.090\textwidth]{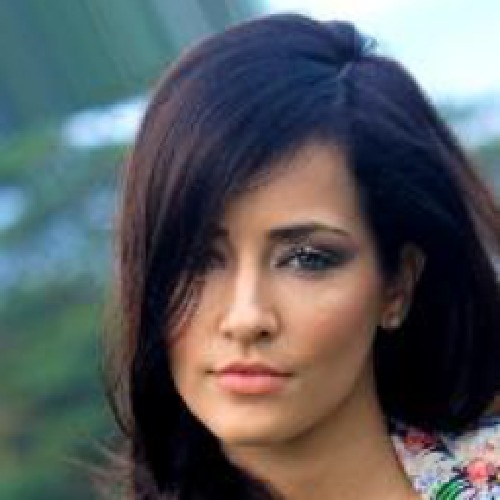} &
    \includegraphics[width=0.090\textwidth]{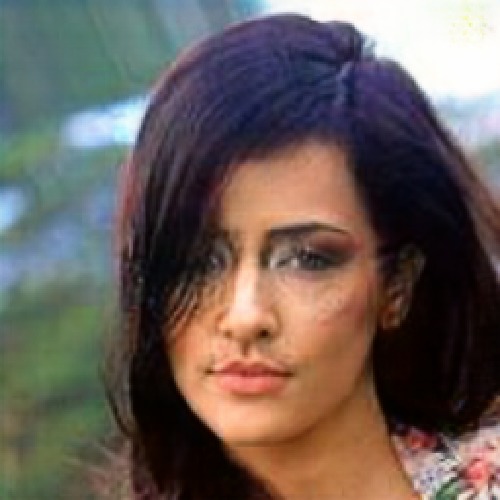} &
    \includegraphics[width=0.090\textwidth]{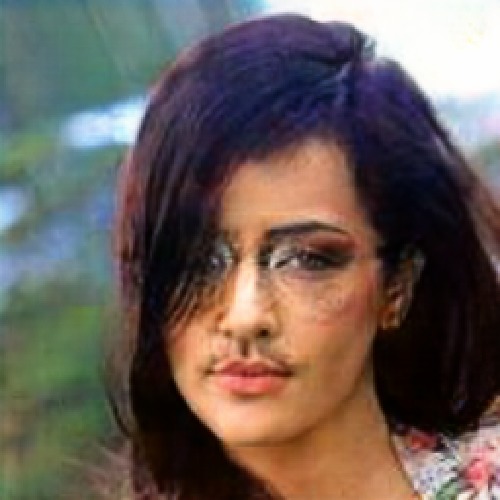} \\

    (a) &  (b) & (c) & & (d) & (e) & (f) & & (g) & (h) & (i) 
    \end{tabular}
    
    \caption{\sl Semantic adversarial examples generated with multiple attribute implementation using Adversarial AttGAN. The first, fourth and seventh columns contain the original images. We show adversarial examples generated under the attributes: (b),(e) and (h) Eyeglasses-Mustache-Age-Pale Skin-Young-Black Hair and (c),(f) and (i)Eyeglasses-Mustache-Pale skin-Age-Bushy eyebrows-Black hair. The quality of the images produced by the Adversarial AttGAN are sharper than those produced by the Adversarial Fader Networks.}
    \label{fig:attgan2}
\end{figure}
\endgroup

\begingroup
\begin{figure}
    \centering
    \setlength{\tabcolsep}{5pt}
    \renewcommand{\arraystretch}{0.5}
    \begin{tabular}{c c c  || c c c  || c c }
    \includegraphics[width=0.090\textwidth]{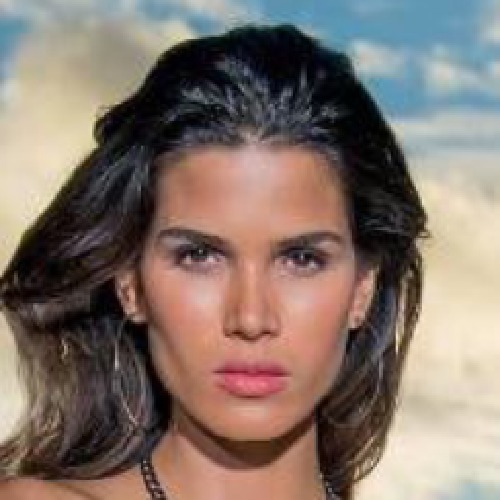} &
    \includegraphics[width=0.090\textwidth]{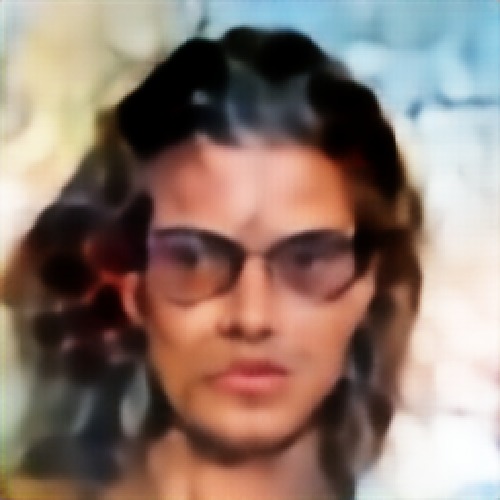} & &
    \includegraphics[width=0.090\textwidth]{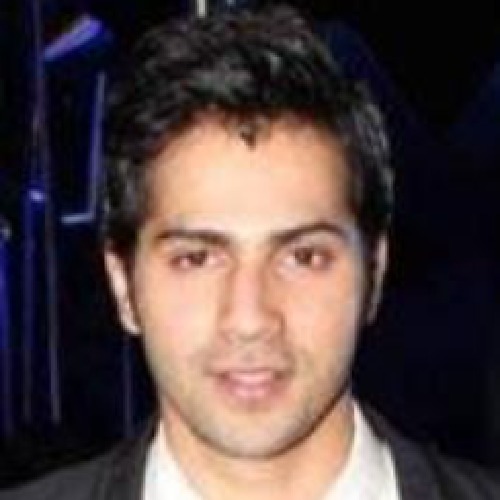} & 
    \includegraphics[width=0.090\textwidth]{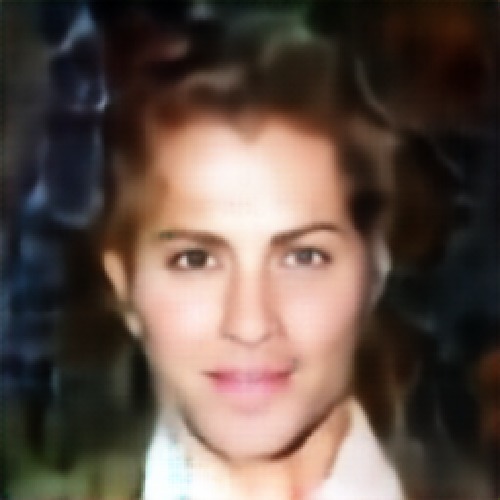} & &
    \includegraphics[width=0.090\textwidth]{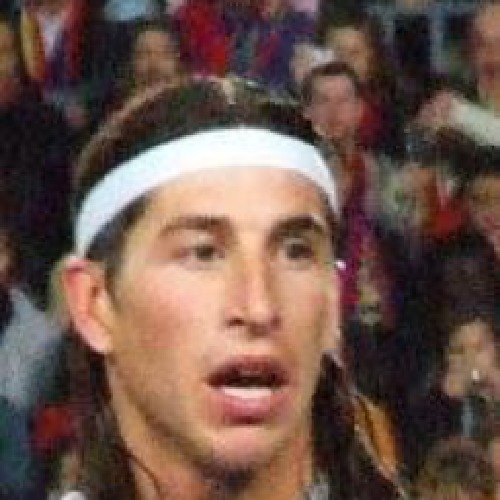} & 
    \includegraphics[width=0.090\textwidth]{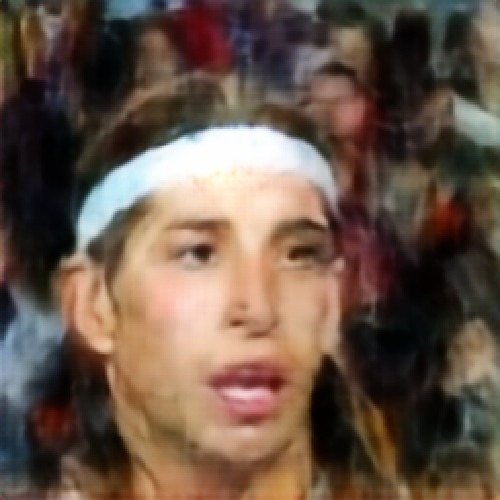} \\
    \includegraphics[width=0.090\textwidth]{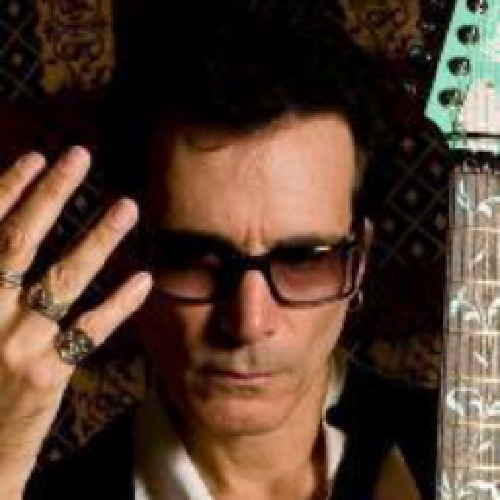} &
    \includegraphics[width=0.090\textwidth]{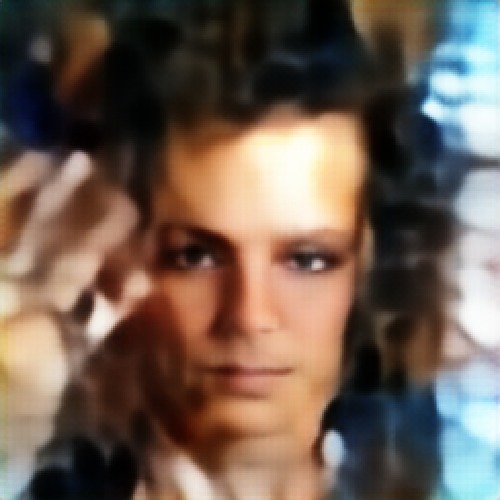}l & &
    \includegraphics[width=0.090\textwidth]{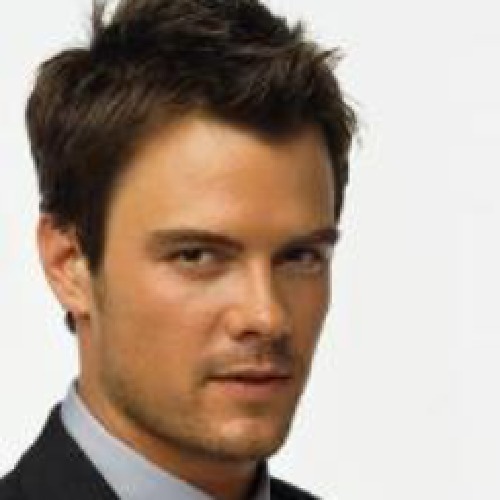} & 
    \includegraphics[width=0.090\textwidth]{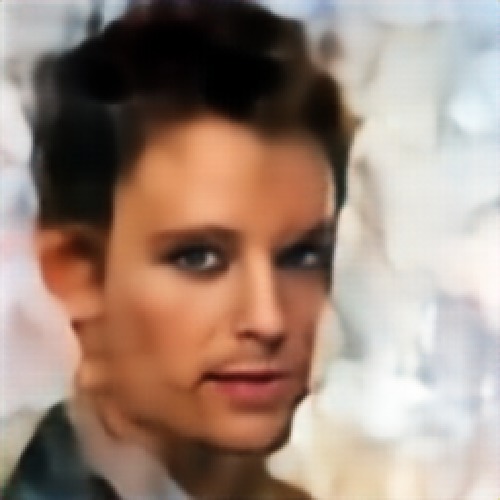} & &
    \includegraphics[width=0.090\textwidth]{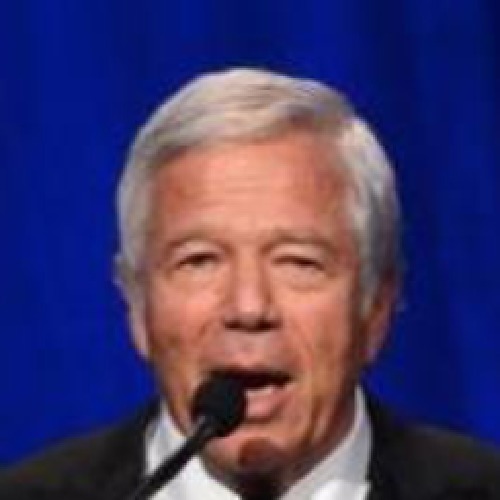} & 
    \includegraphics[width=0.090\textwidth]{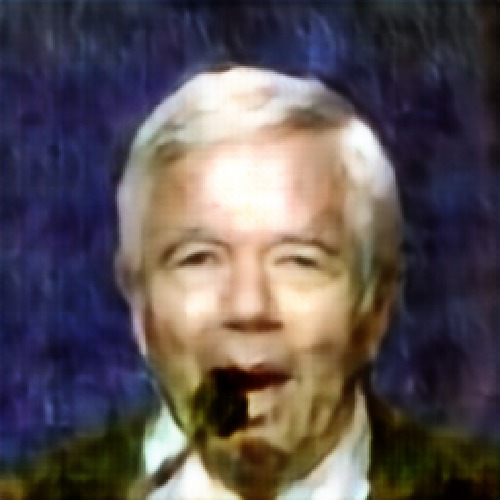} \\
    \includegraphics[width=0.090\textwidth]{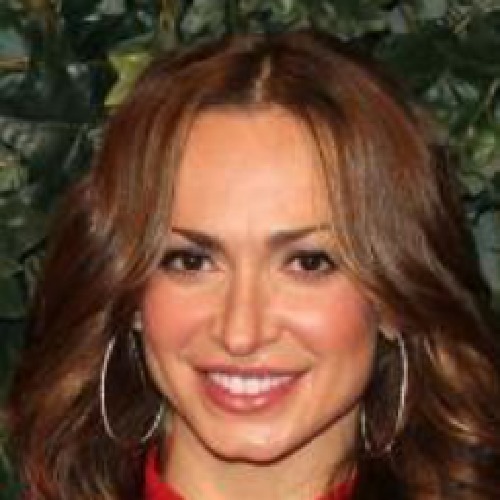} &
    \includegraphics[width=0.090\textwidth]{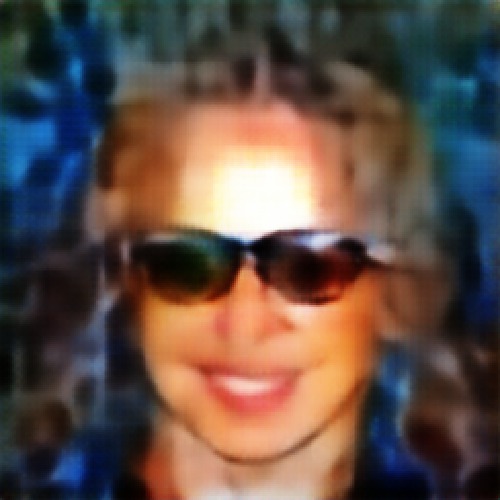} & &
    \includegraphics[width=0.090\textwidth]{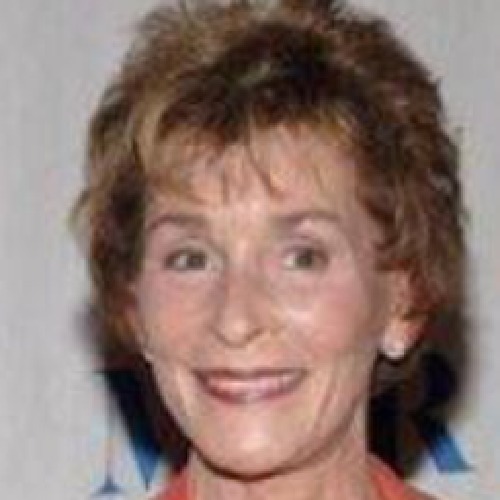} & 
    \includegraphics[width=0.090\textwidth]{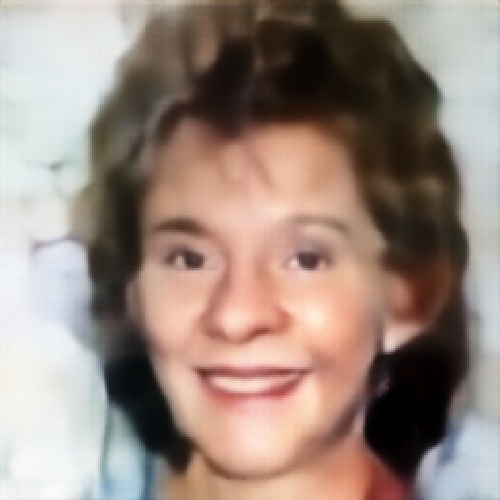} & &
    \includegraphics[width=0.090\textwidth]{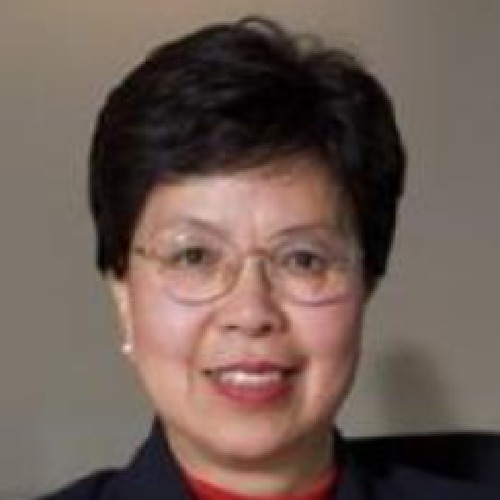} & 
    \includegraphics[width=0.090\textwidth]{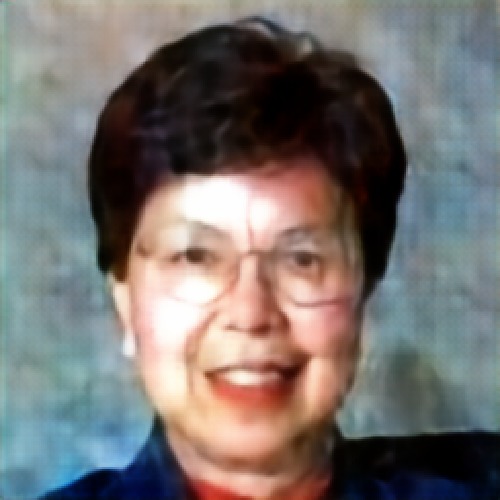} \\
    \includegraphics[width=0.090\textwidth]{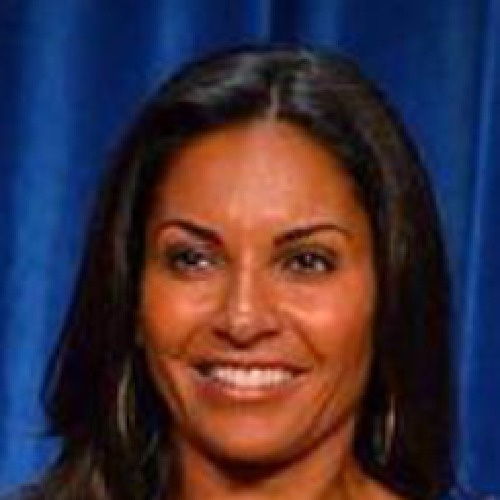} &
    \includegraphics[width=0.090\textwidth]{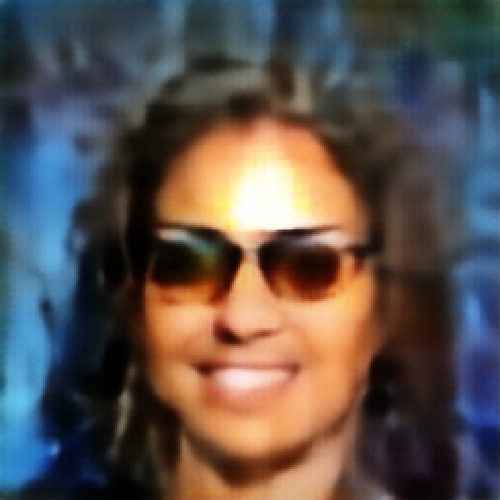} & &
    \includegraphics[width=0.090\textwidth]{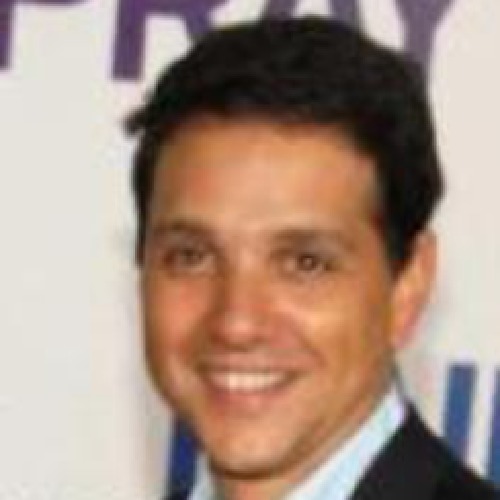} & 
    \includegraphics[width=0.090\textwidth]{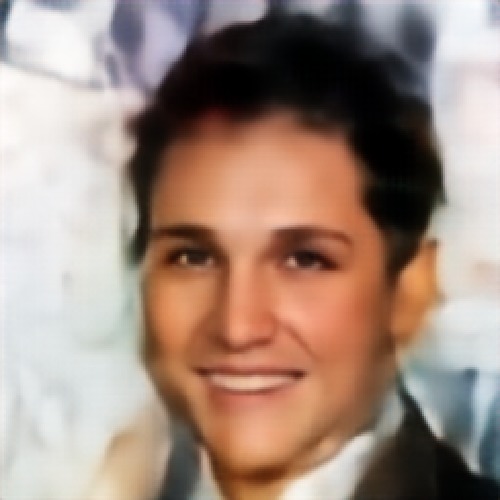} & &
    \includegraphics[width=0.090\textwidth]{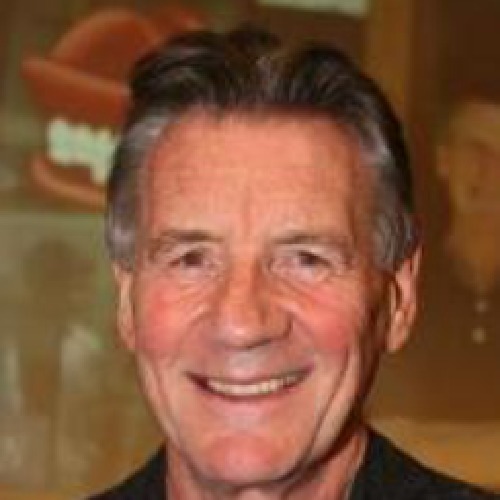} & 
    \includegraphics[width=0.090\textwidth]{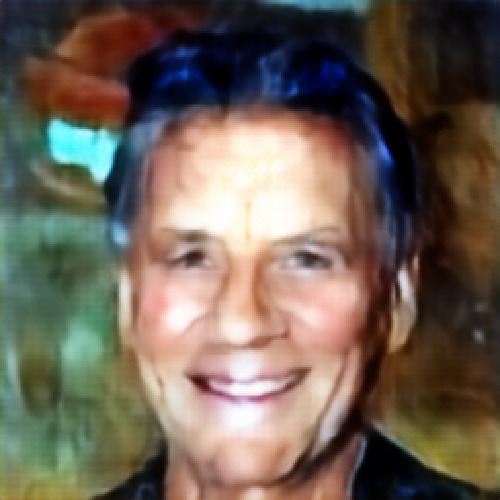} \\
    \includegraphics[width=0.090\textwidth]{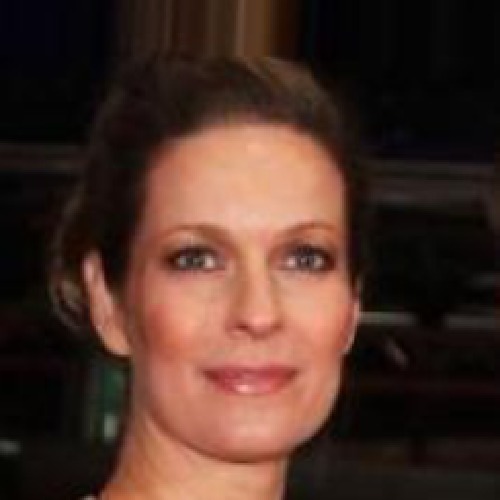} &
    \includegraphics[width=0.090\textwidth]{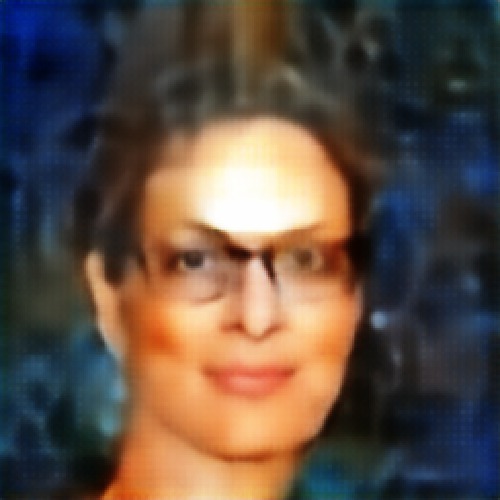} & &
    \includegraphics[width=0.090\textwidth]{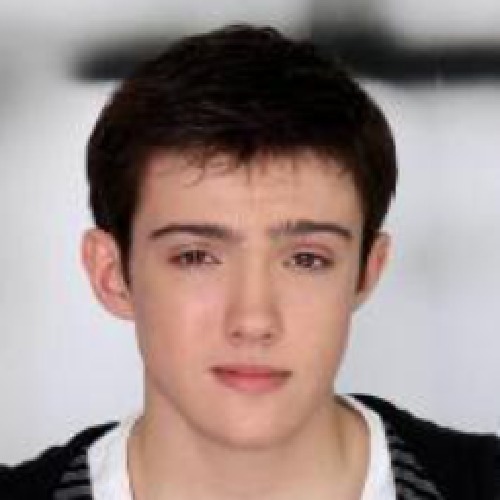} & 
    \includegraphics[width=0.090\textwidth]{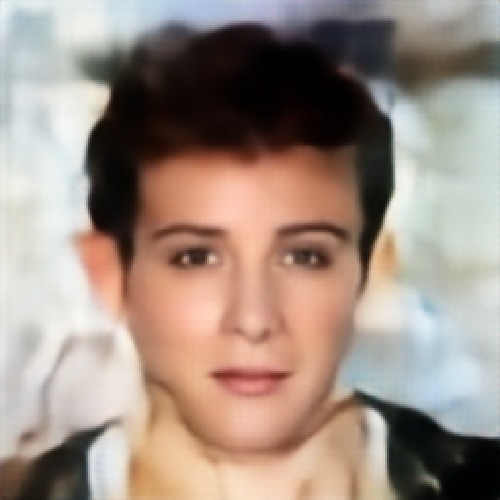} & &
    \includegraphics[width=0.090\textwidth]{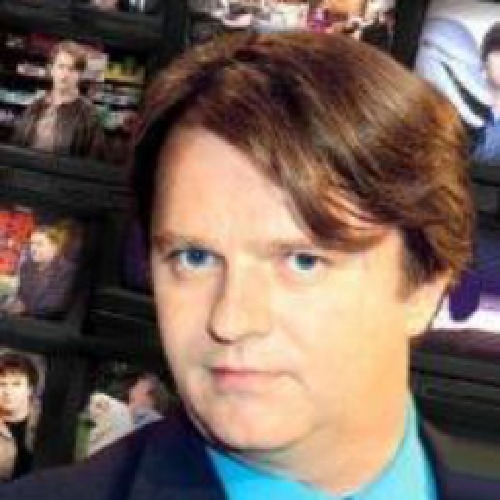} & 
    \includegraphics[width=0.090\textwidth]{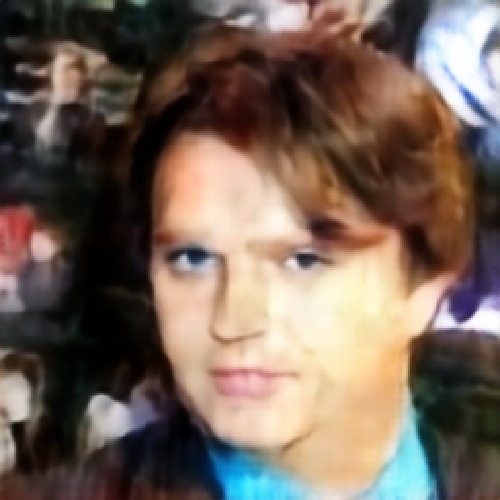} \\
    \includegraphics[width=0.090\textwidth]{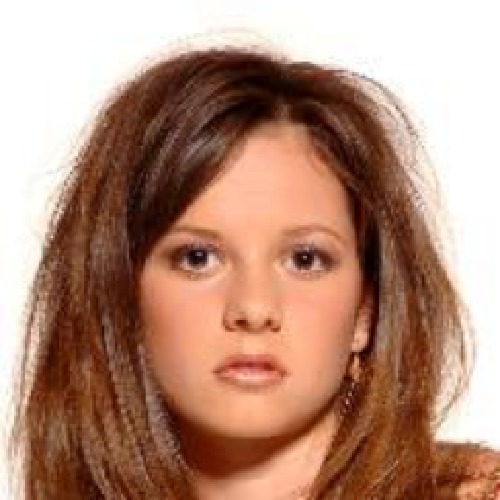} &
    \includegraphics[width=0.090\textwidth]{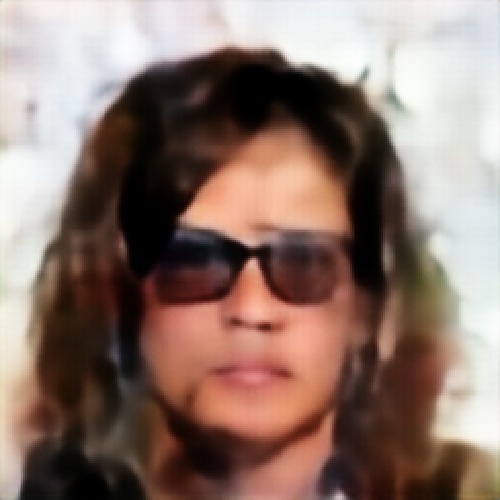} & &
    \includegraphics[width=0.090\textwidth]{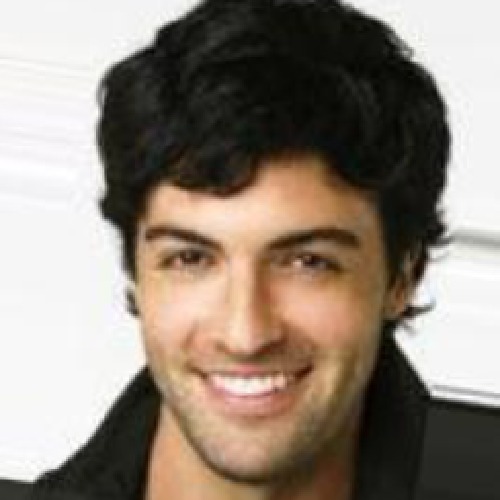} & 
    \includegraphics[width=0.090\textwidth]{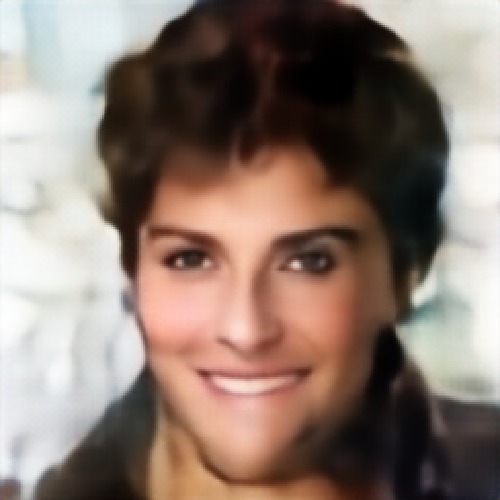} & &
    \includegraphics[width=0.090\textwidth]{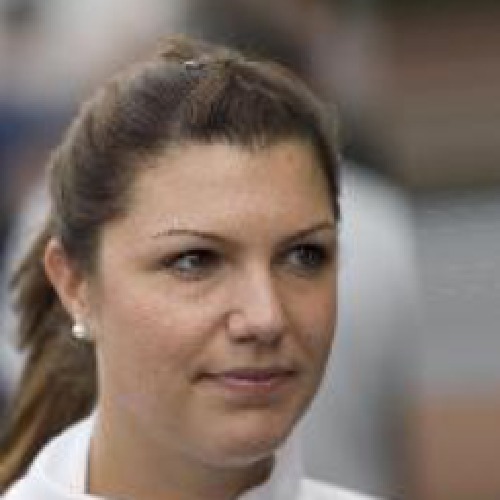} & 
    \includegraphics[width=0.090\textwidth]{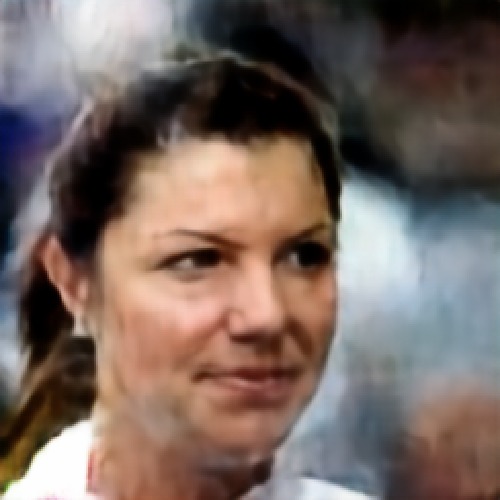} \\
    \includegraphics[width=0.090\textwidth]{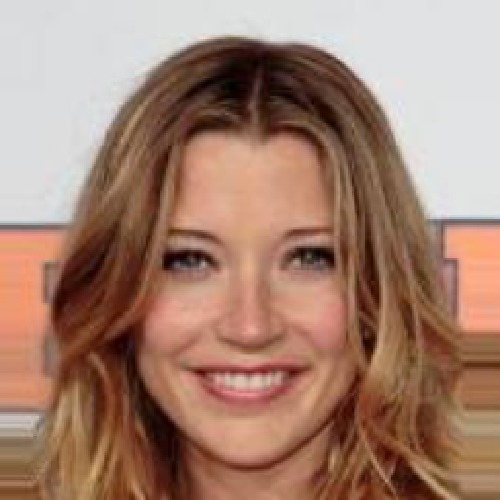} &
    \includegraphics[width=0.090\textwidth]{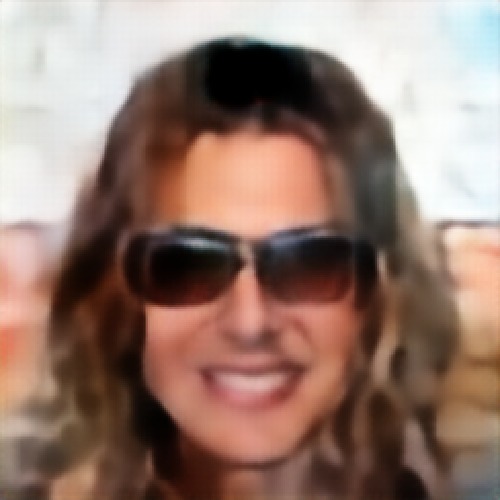} & &
    \includegraphics[width=0.090\textwidth]{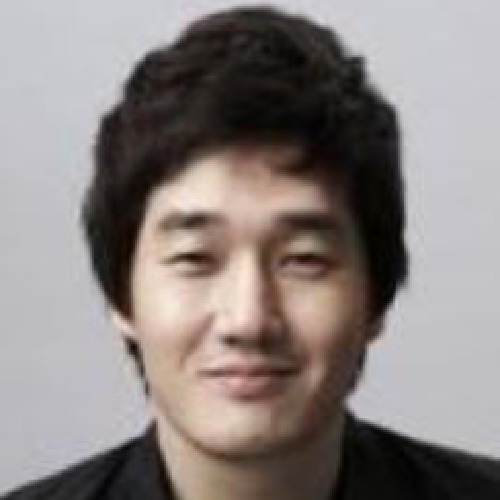} & 
    \includegraphics[width=0.090\textwidth]{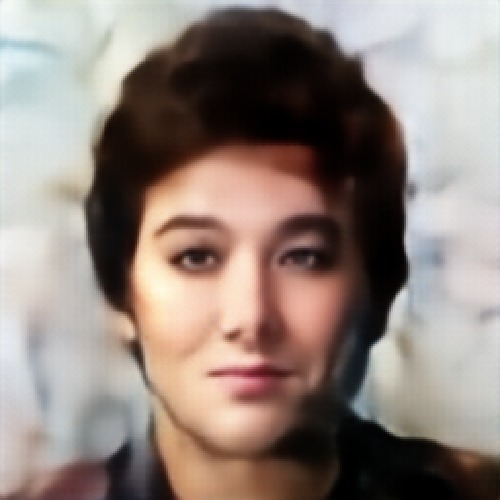} & &
    \includegraphics[width=0.090\textwidth]{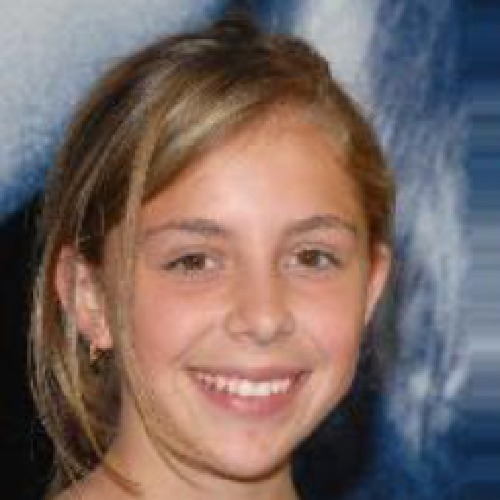} & 
    \includegraphics[width=0.090\textwidth]{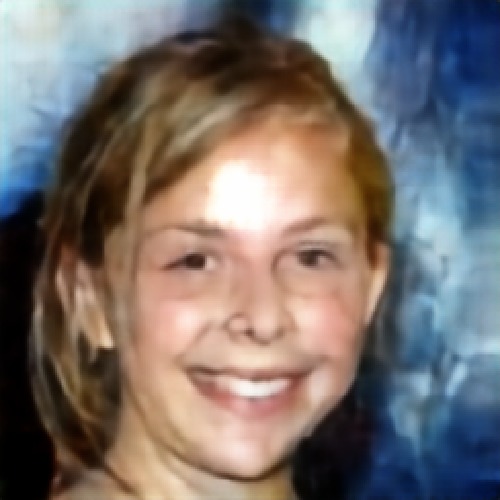} \\
    \includegraphics[width=0.090\textwidth]{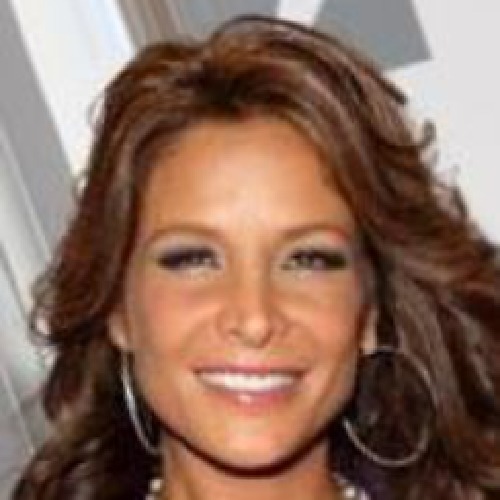} &
    \includegraphics[width=0.090\textwidth]{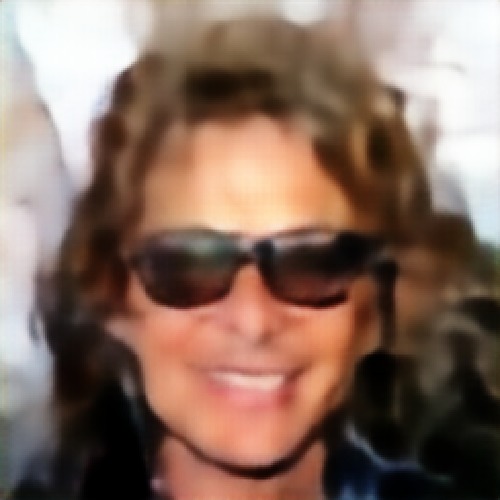} & &
    \includegraphics[width=0.090\textwidth]{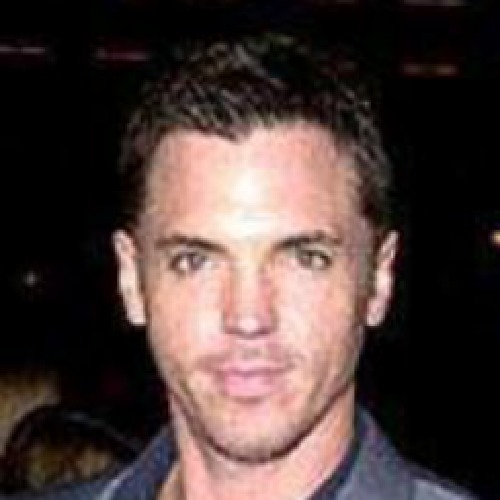} & 
    \includegraphics[width=0.090\textwidth]{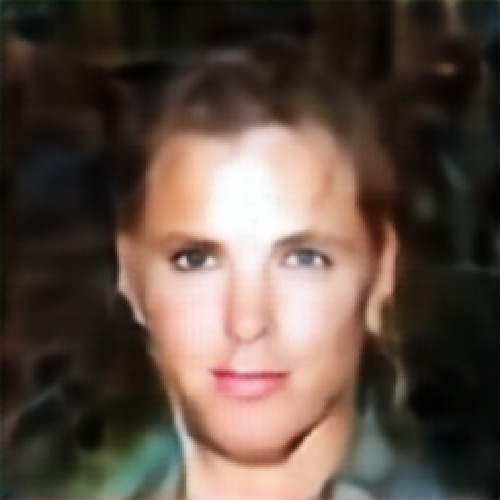} & &
    \includegraphics[width=0.090\textwidth]{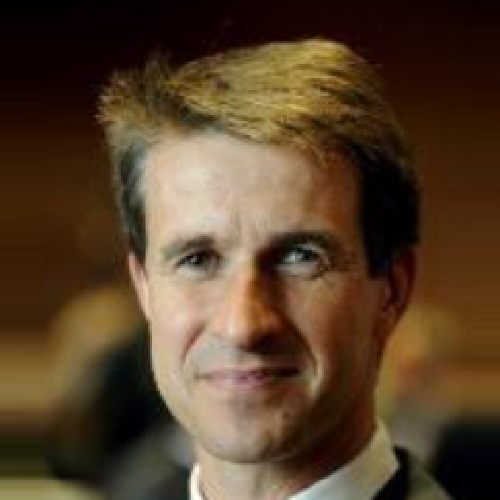} & 
    \includegraphics[width=0.090\textwidth]{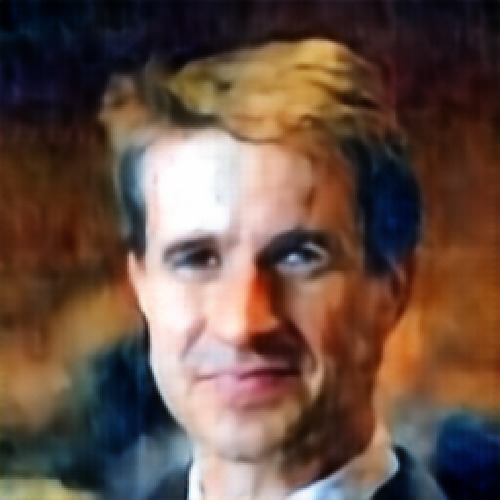} \\
    \includegraphics[width=0.090\textwidth]{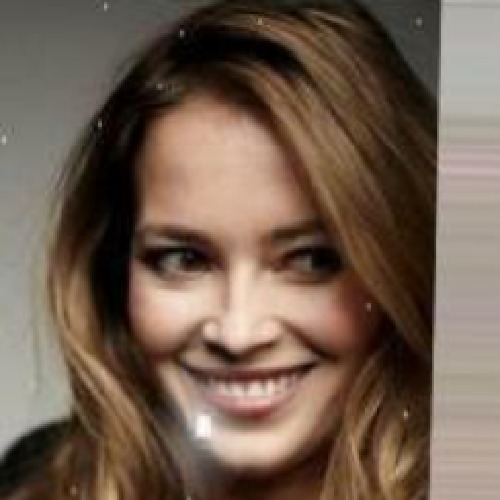} &
    \includegraphics[width=0.090\textwidth]{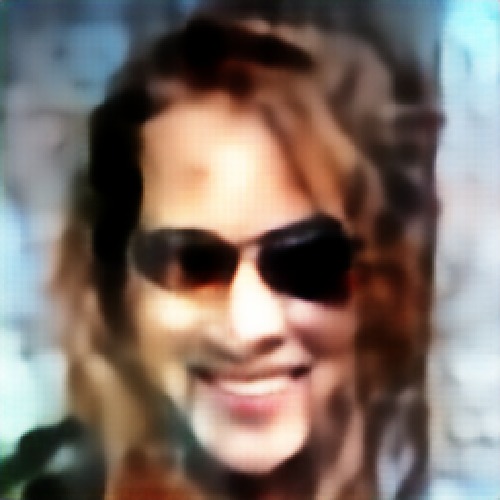} & &
    \includegraphics[width=0.090\textwidth]{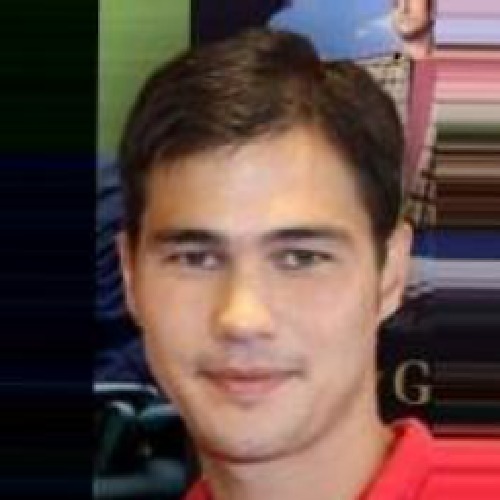} & 
    \includegraphics[width=0.090\textwidth]{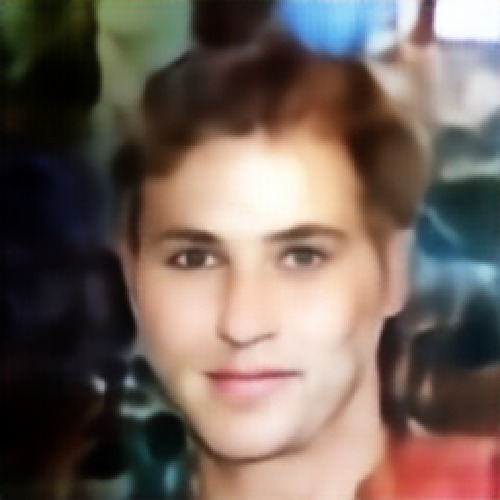} & &
    \includegraphics[width=0.090\textwidth]{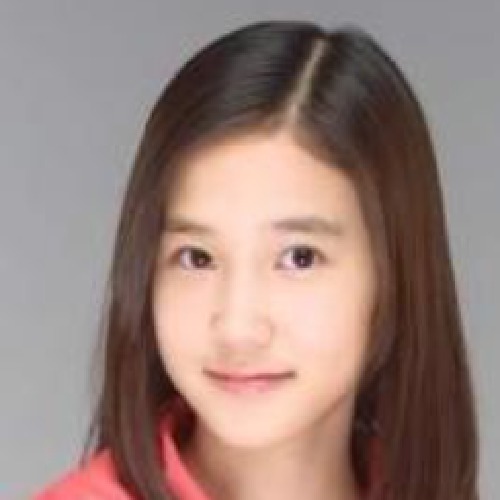} & 
    \includegraphics[width=0.090\textwidth]{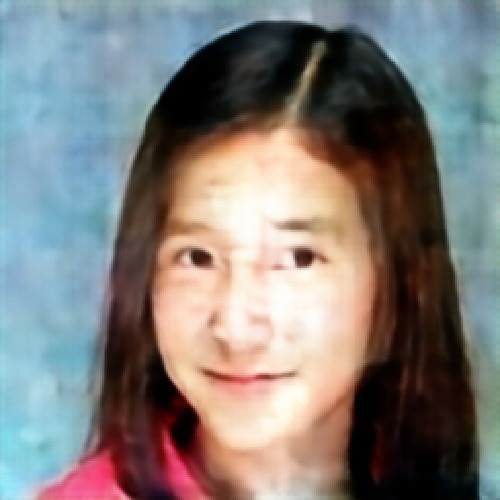} \\
    \includegraphics[width=0.090\textwidth]{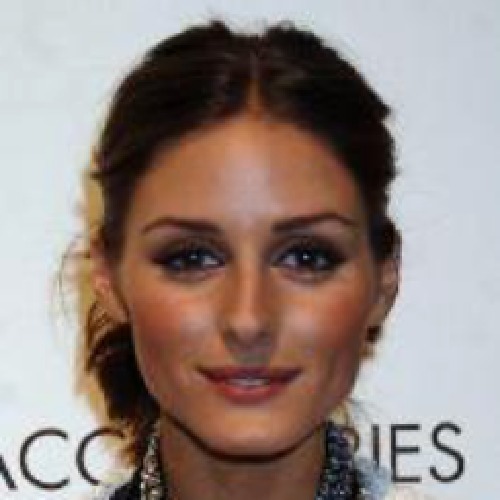} &
    \includegraphics[width=0.090\textwidth]{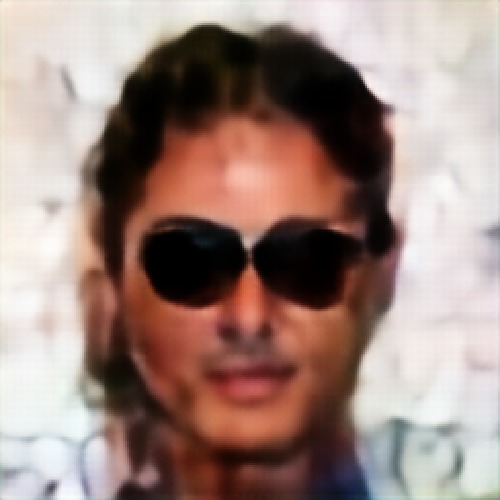} & &
    \includegraphics[width=0.090\textwidth]{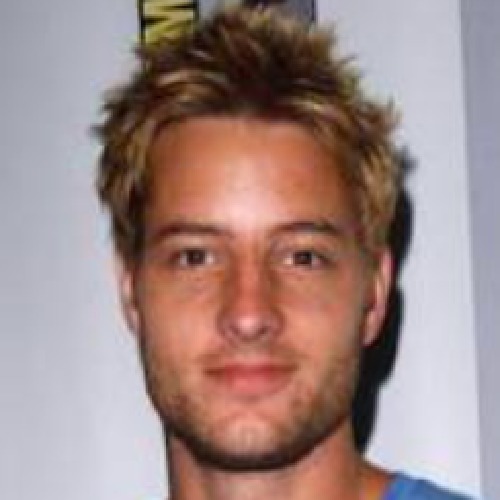} & 
    \includegraphics[width=0.090\textwidth]{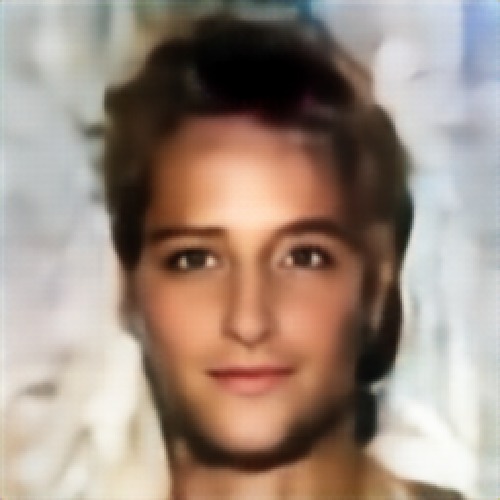} & &
    \includegraphics[width=0.090\textwidth]{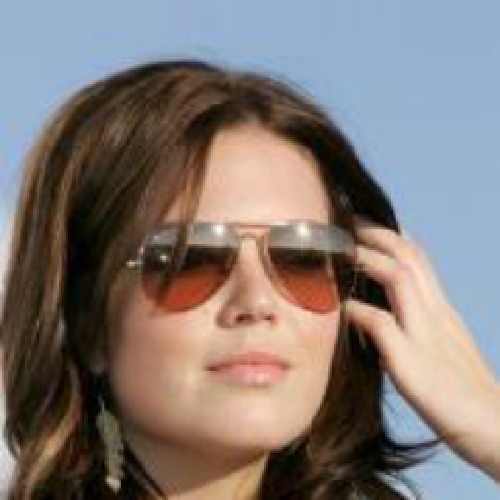} & 
    \includegraphics[width=0.090\textwidth]{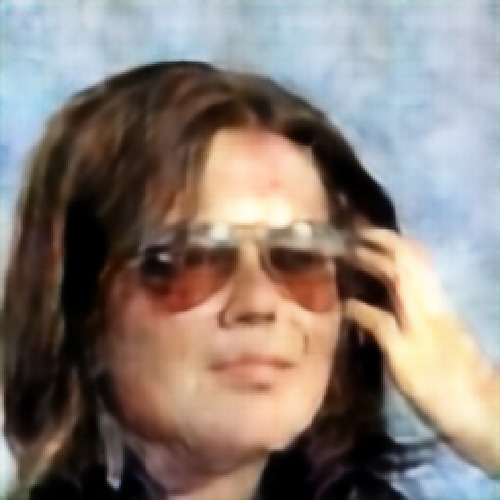} \\
    \includegraphics[width=0.090\textwidth]{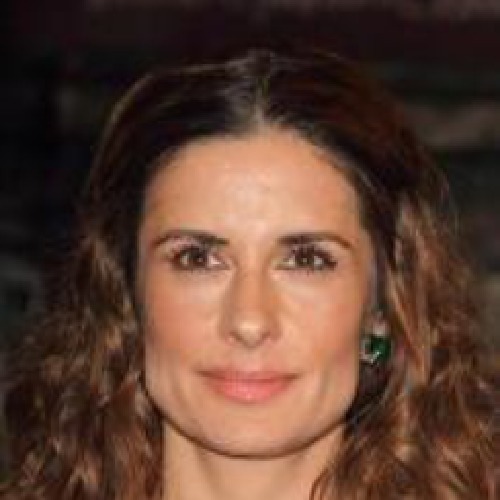} &
    \includegraphics[width=0.090\textwidth]{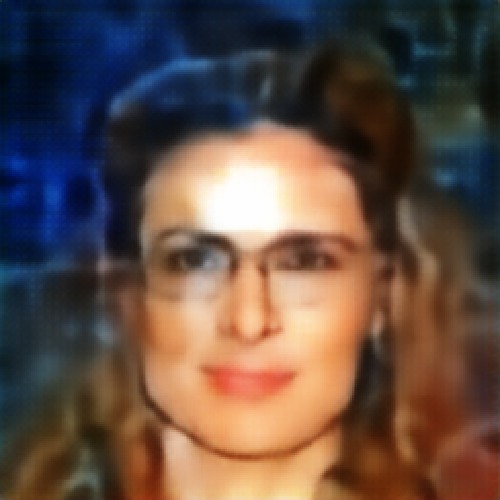} & &
    \includegraphics[width=0.090\textwidth]{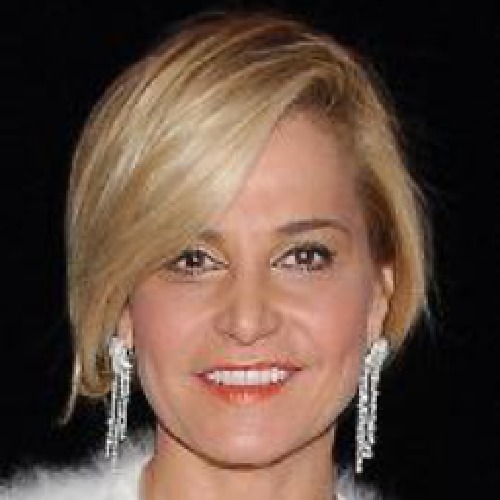} & 
    \includegraphics[width=0.090\textwidth]{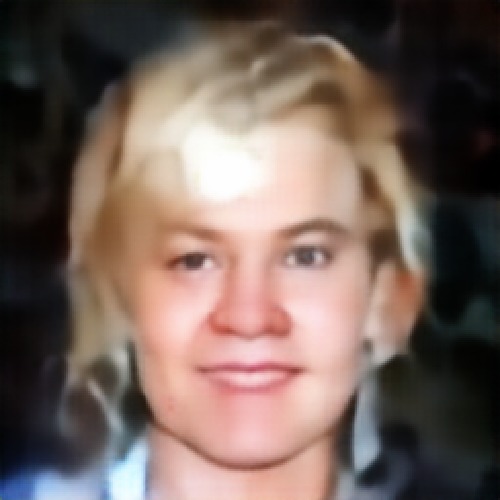} & &
    \includegraphics[width=0.090\textwidth]{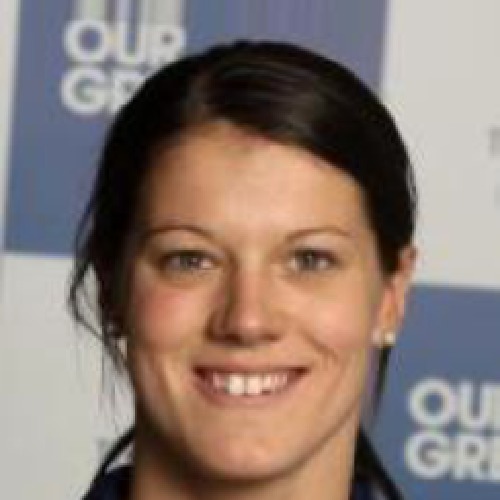} & 
    \includegraphics[width=0.090\textwidth]{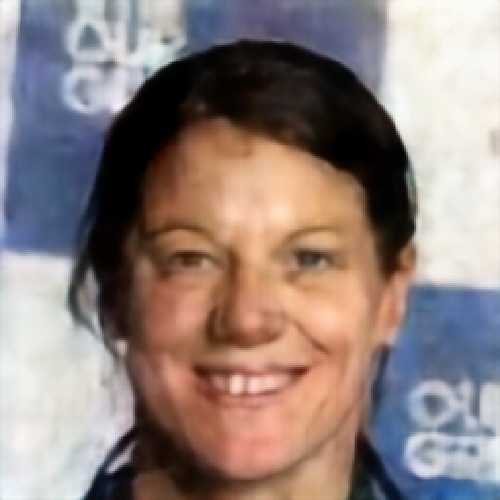} \\
    \includegraphics[width=0.090\textwidth]{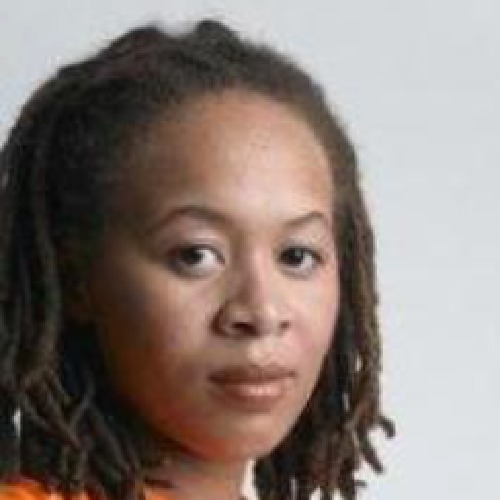} &
    \includegraphics[width=0.090\textwidth]{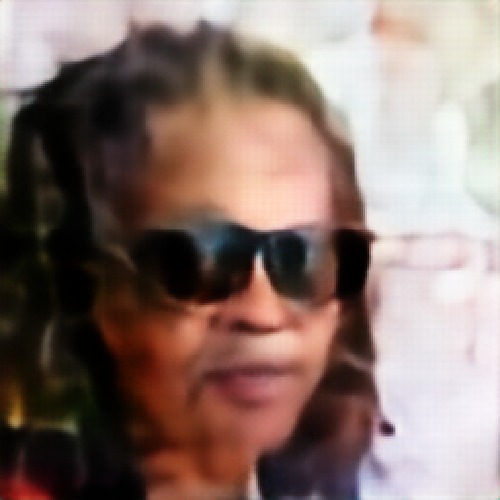} & &
    \includegraphics[width=0.090\textwidth]{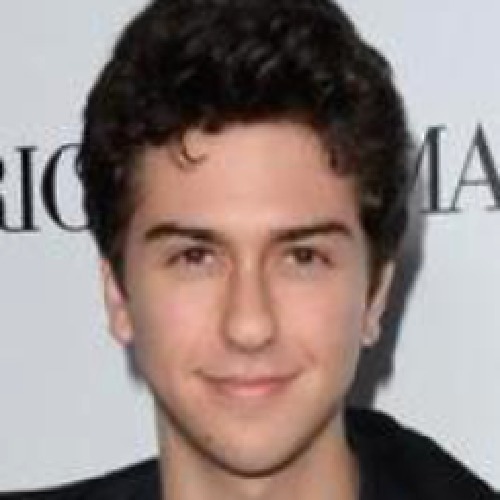} & 
    \includegraphics[width=0.090\textwidth]{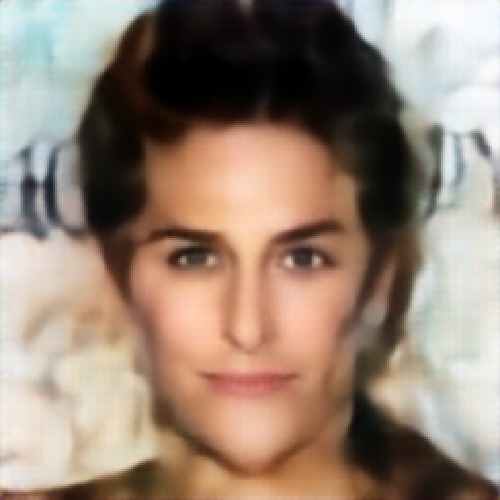} & &
    \includegraphics[width=0.090\textwidth]{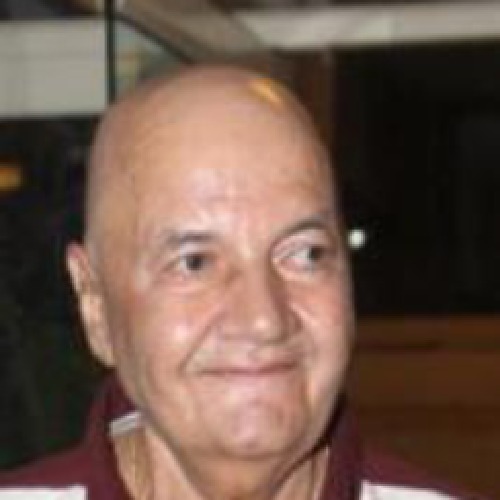} & 
    \includegraphics[width=0.090\textwidth]{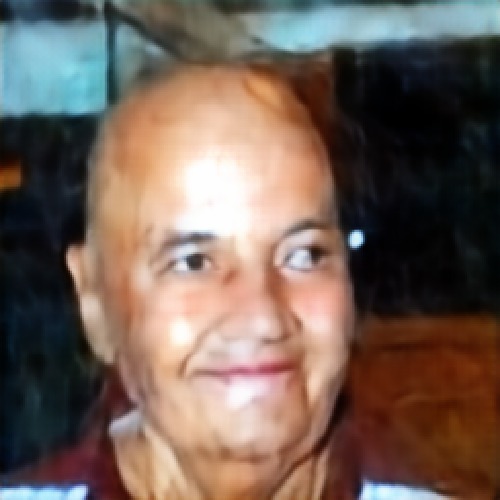} \\

    (a) &  (b) & & (c) & (d) & & (e) & (f)  
    \end{tabular}
    
    \caption{\sl Semantic adversarial examples generated with Single attribute implementation using Adversarial Fader Networks. Columns (a),(c) and (e) contain the original images. We show adversarial examples generated under the attributes: (b),(d) and (f) Eyeglasses, Nose shape and Age respectively.}
    \label{fig:fn_single}
\end{figure}
\endgroup

\begingroup
\begin{figure}[htp]
    \centering
    \setlength{\tabcolsep}{1pt}
    \renewcommand{\arraystretch}{0.5}
    \begin{tabular}{c c c || c c c || c c c || c c c || c c }
    \includegraphics[width=0.09\textwidth]{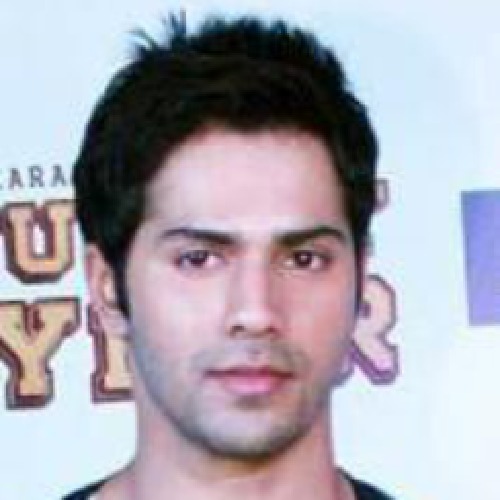}&
    \includegraphics[width=0.09\textwidth]{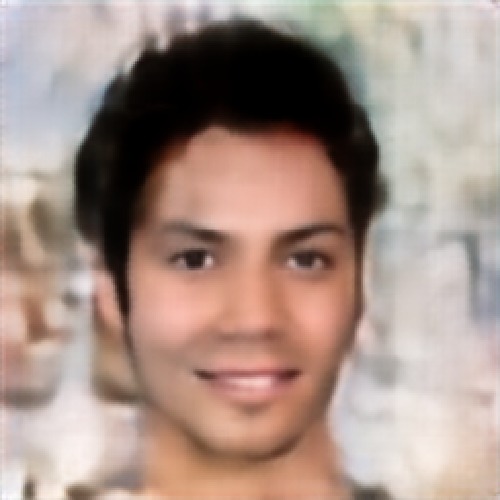} & &
    \includegraphics[width=0.09\textwidth]{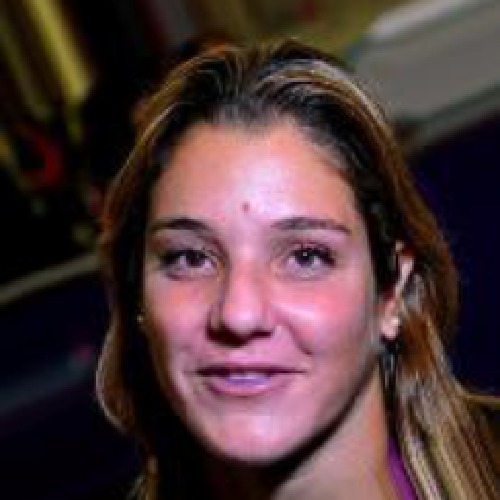}&
    \includegraphics[width=0.09\textwidth]{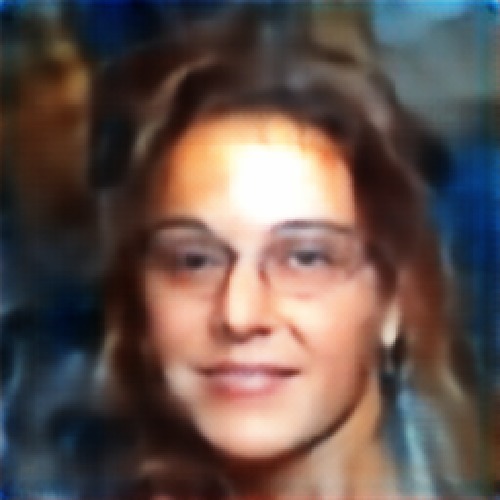} & &
    \includegraphics[width=0.09\textwidth]{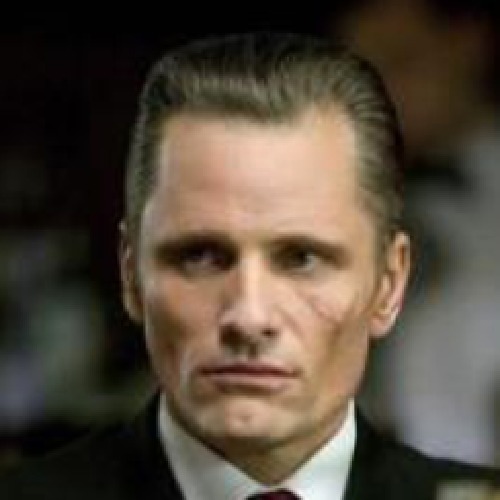}&
    \includegraphics[width=0.09\textwidth]{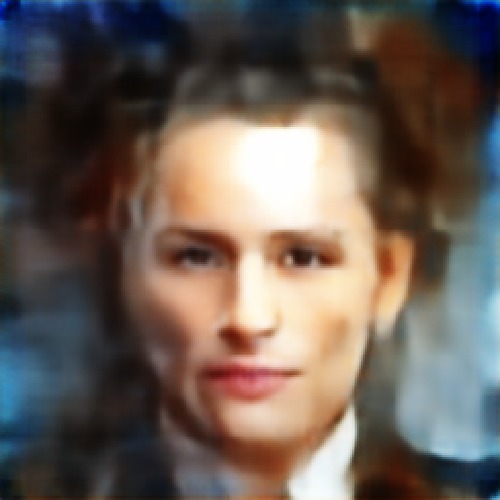}& &
    \includegraphics[width=0.09\textwidth]{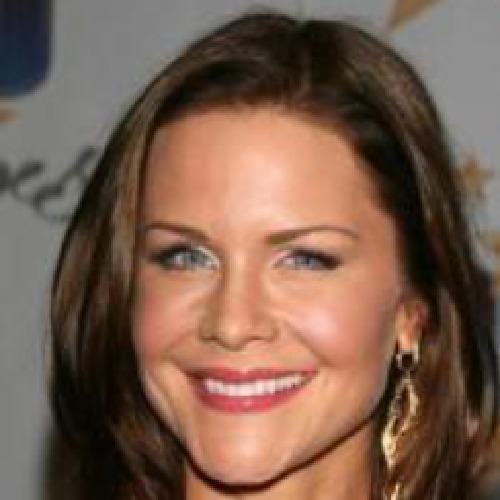}&
    \includegraphics[width=0.09\textwidth]{images/Appendix_Images/fadernet/seq_glasses_young_nose/3_orig.jpg} & &
    \includegraphics[width=0.09\textwidth]{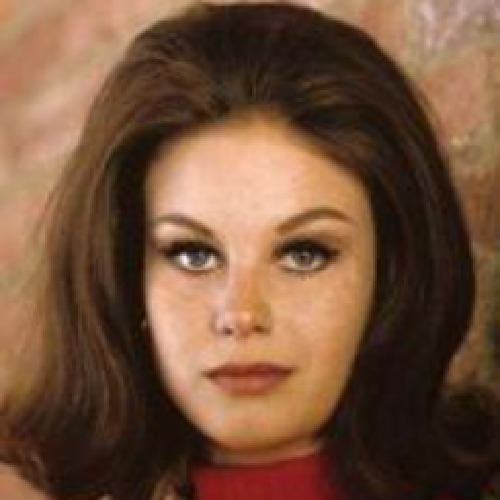}&
    \includegraphics[width=0.09\textwidth]{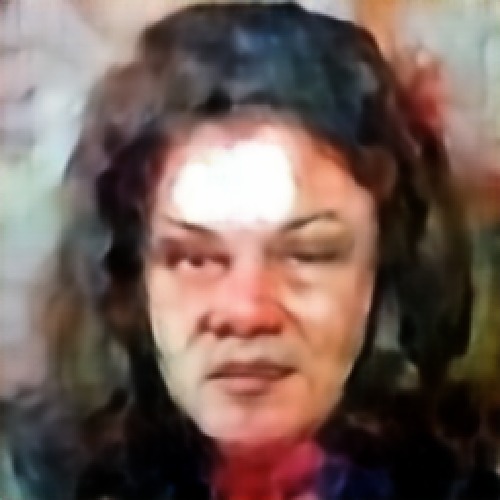}  \\
    \includegraphics[width=0.09\textwidth]{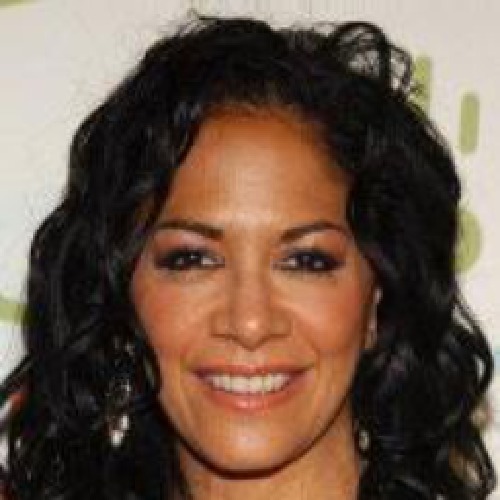}&
    \includegraphics[width=0.09\textwidth]{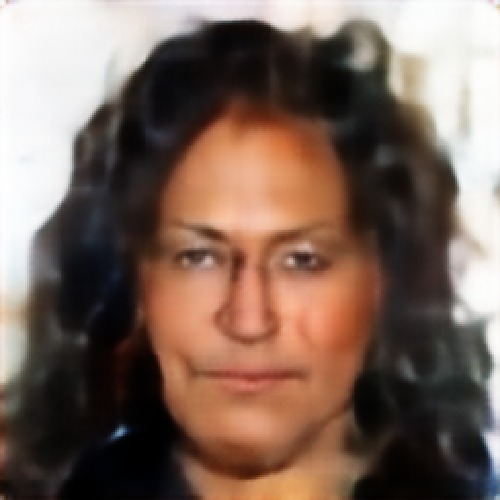} & &
    \includegraphics[width=0.09\textwidth]{images/Appendix_Images/fadernet/6.jpg}&
    \includegraphics[width=0.09\textwidth]{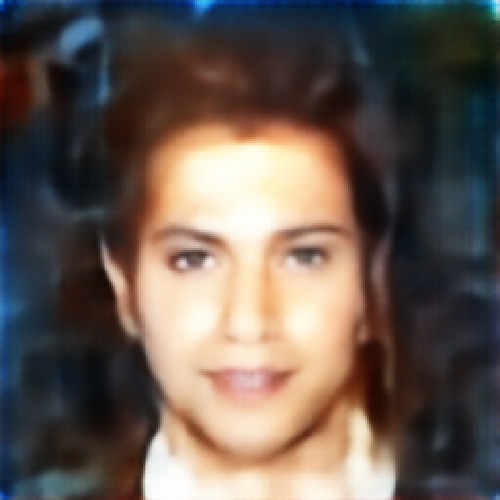} & &
    \includegraphics[width=0.09\textwidth]{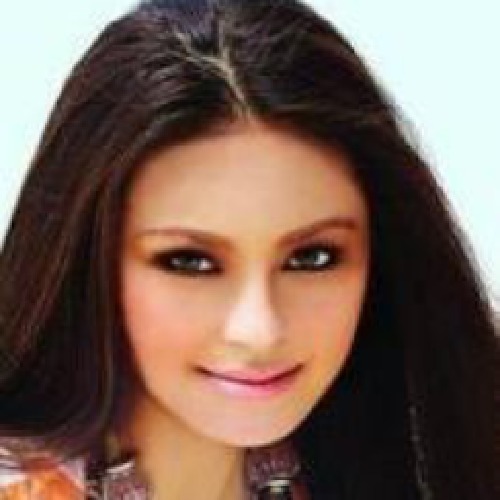}&
    \includegraphics[width=0.09\textwidth]{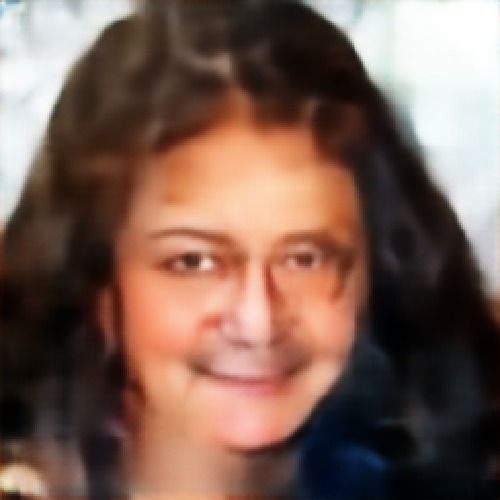} & &
    \includegraphics[width=0.09\textwidth]{images/Appendix_Images/fadernet/seq_glasses_young_nose/3_orig.jpg}&
    \includegraphics[width=0.09\textwidth]{images/Appendix_Images/fadernet/seq_glasses_young_nose/3_orig.jpg} & &
    \includegraphics[width=0.09\textwidth]{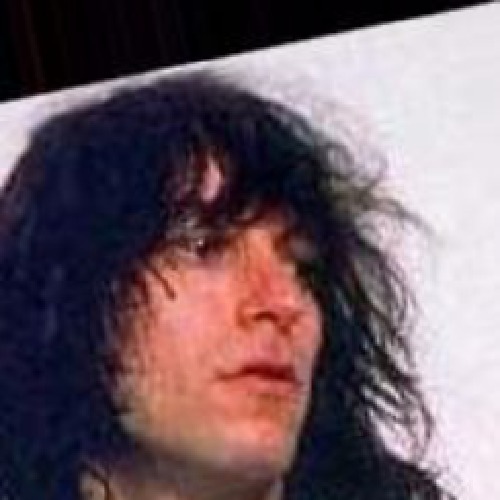}&
    \includegraphics[width=0.09\textwidth]{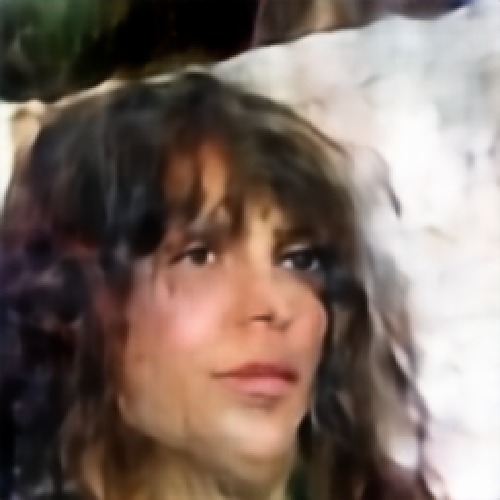}  \\
      \includegraphics[width=0.09\textwidth]{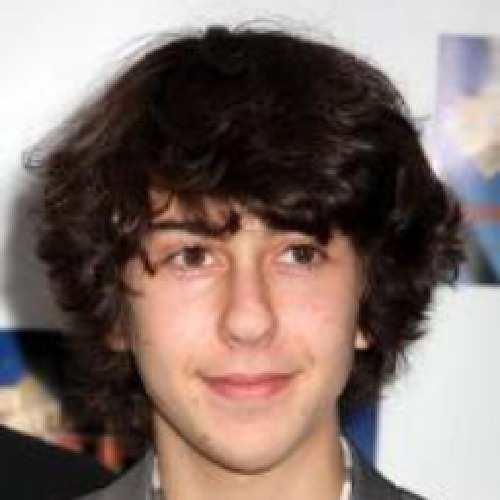}&
    \includegraphics[width=0.09\textwidth]{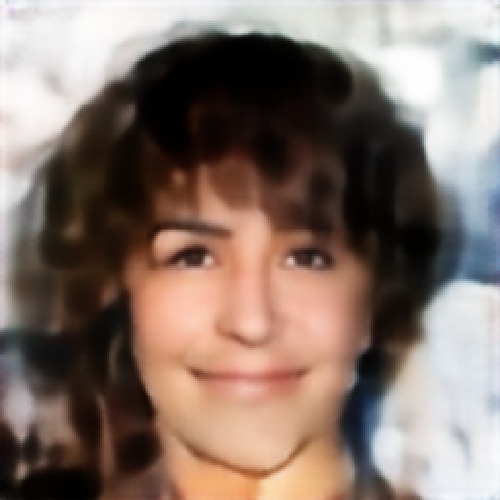} & &
    \includegraphics[width=0.09\textwidth]{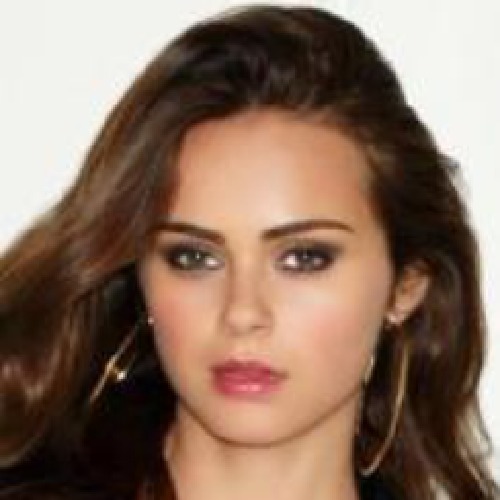}&
    \includegraphics[width=0.09\textwidth]{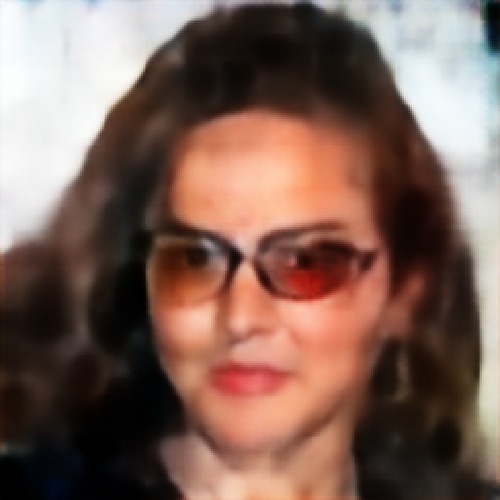} & &
    \includegraphics[width=0.09\textwidth]{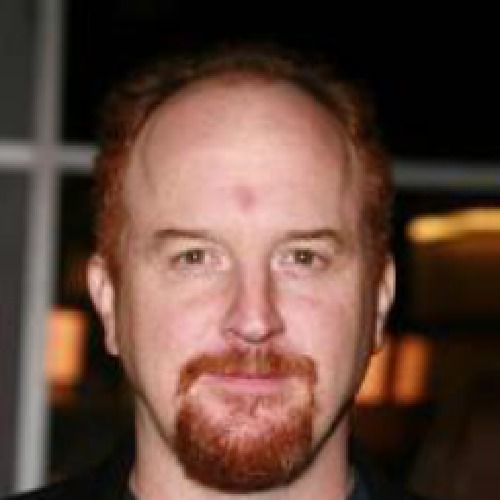}&
    \includegraphics[width=0.09\textwidth]{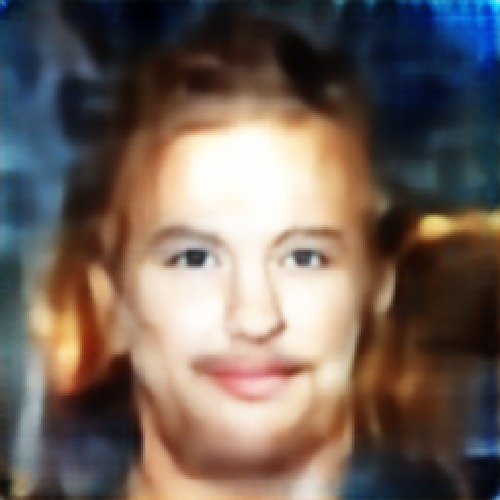} & &
   \includegraphics[width=0.09\textwidth]{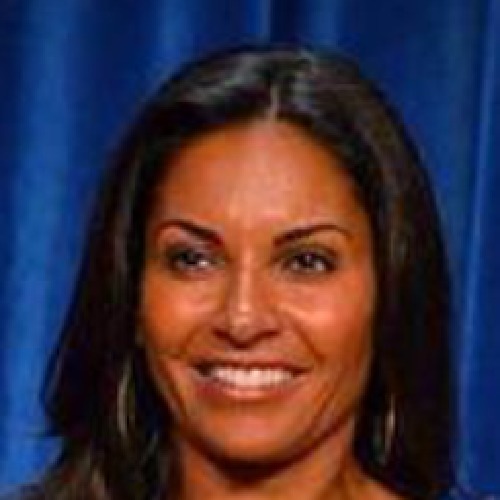}&
    \includegraphics[width=0.09\textwidth]{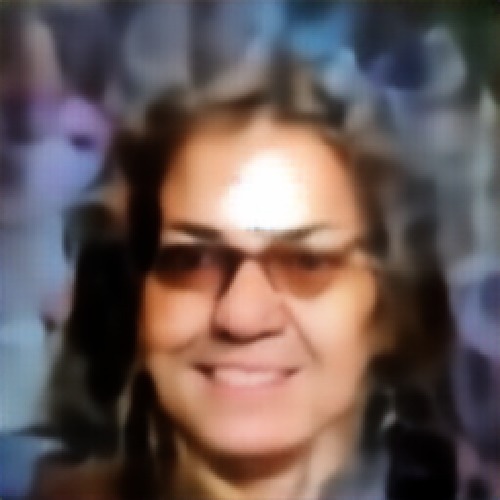} & &
    \includegraphics[width=0.09\textwidth]{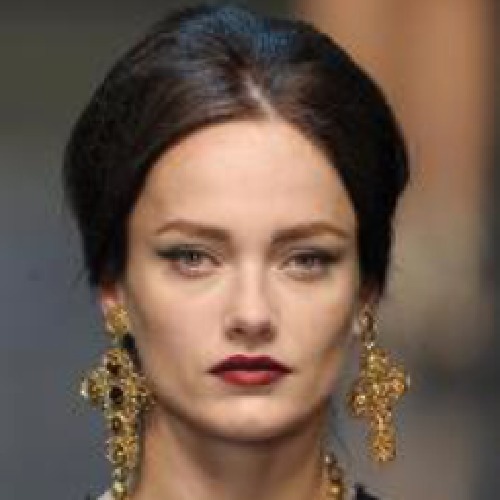}&
    \includegraphics[width=0.09\textwidth]{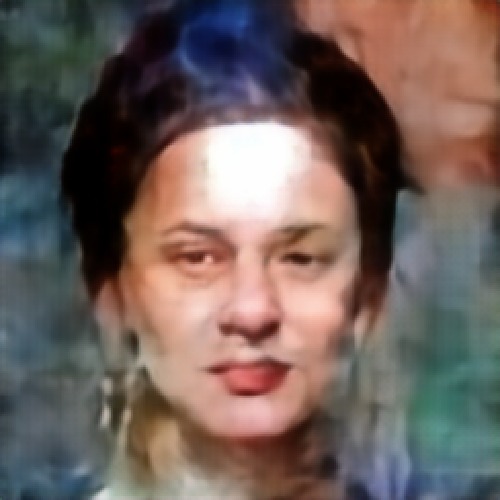}  \\
      \includegraphics[width=0.09\textwidth]{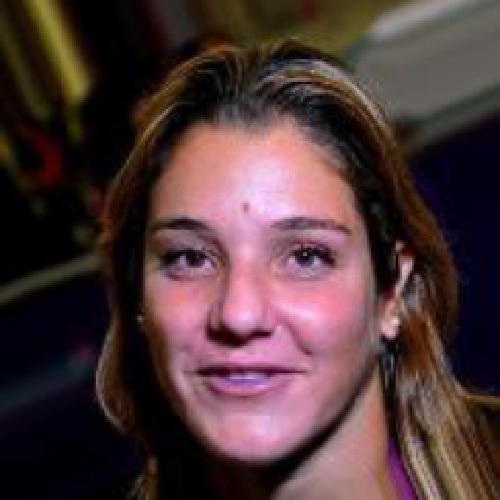}&
    \includegraphics[width=0.09\textwidth]{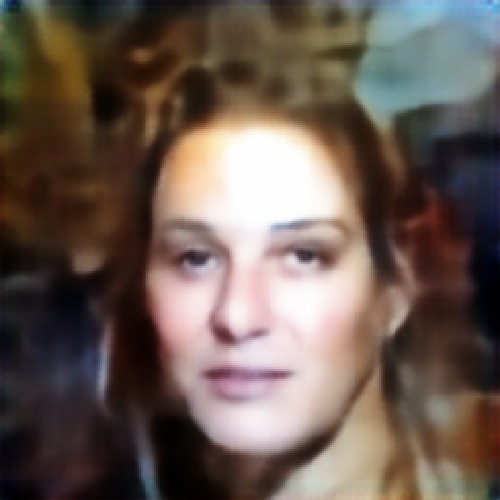} & &
    \includegraphics[width=0.09\textwidth]{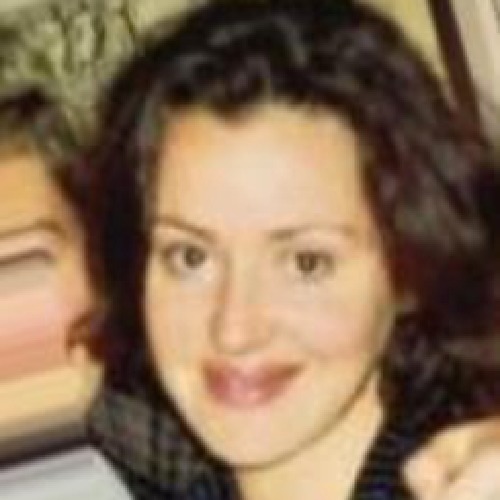}&
    \includegraphics[width=0.09\textwidth]{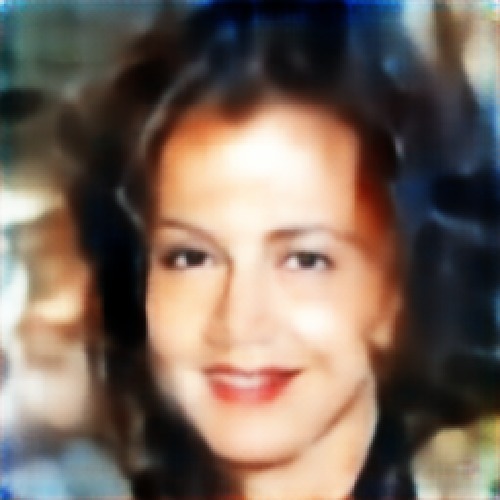} & &
    \includegraphics[width=0.09\textwidth]{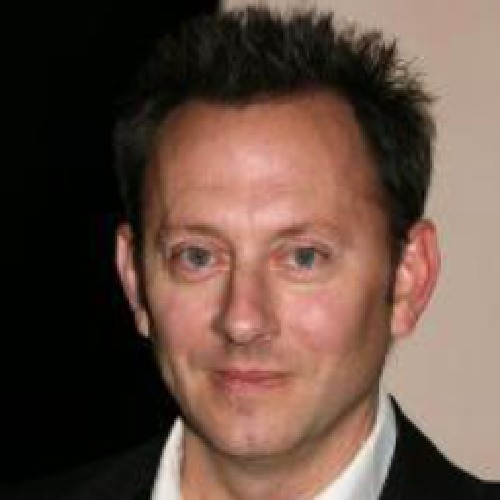}&
    \includegraphics[width=0.09\textwidth]{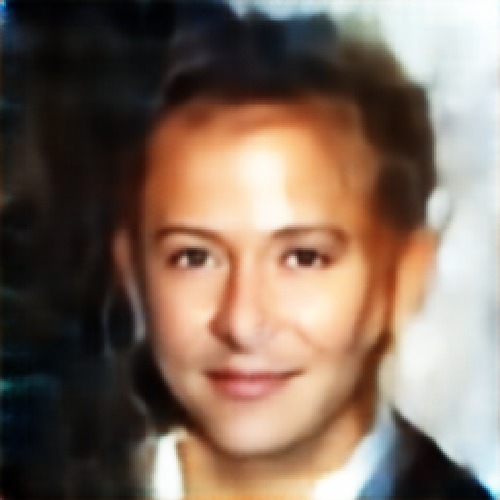} & &
\includegraphics[width=0.09\textwidth]{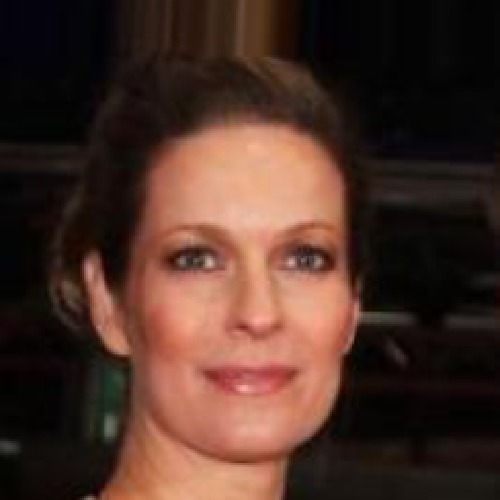}&
    \includegraphics[width=0.09\textwidth]{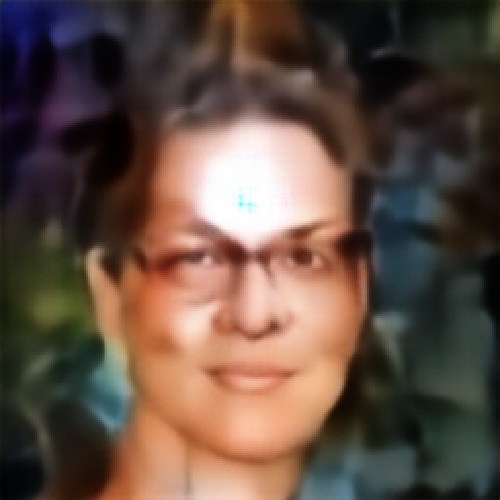} & &
    \includegraphics[width=0.09\textwidth]{images/Appendix_Images/fadernet/190.jpg}&
    \includegraphics[width=0.09\textwidth]{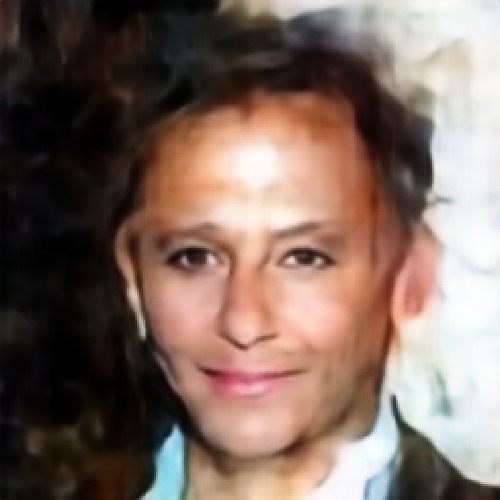}  \\
      \includegraphics[width=0.09\textwidth]{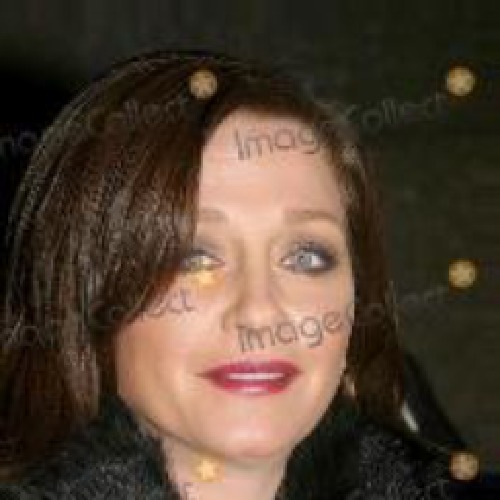}&
    \includegraphics[width=0.09\textwidth]{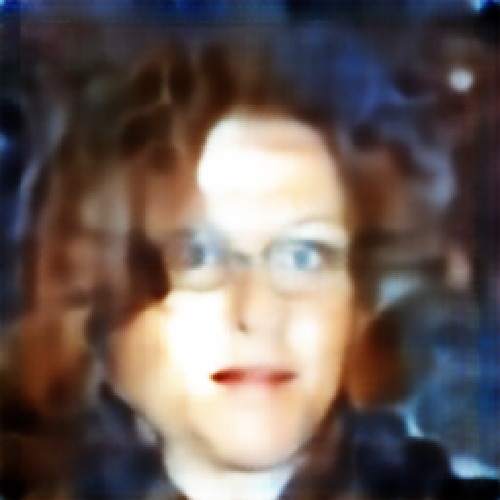} & &
    \includegraphics[width=0.09\textwidth]{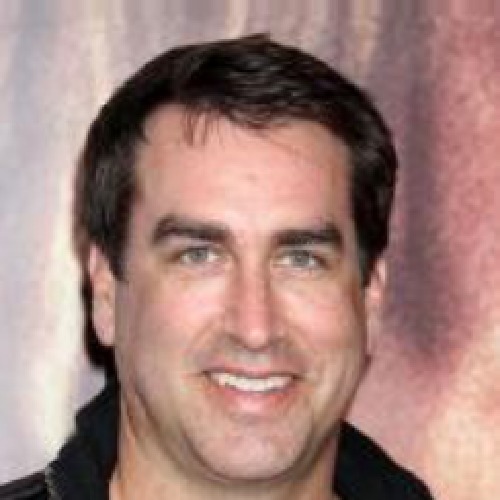}&
    \includegraphics[width=0.09\textwidth]{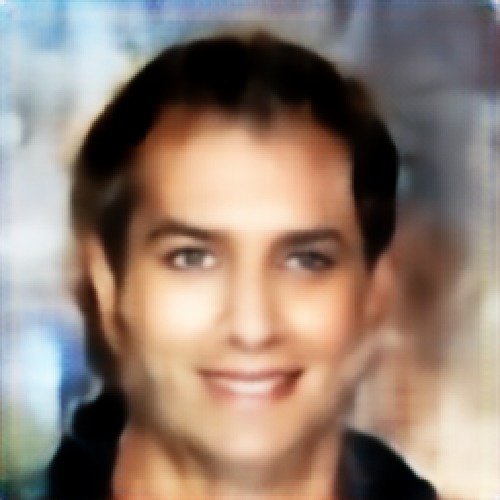} & &
    \includegraphics[width=0.09\textwidth]{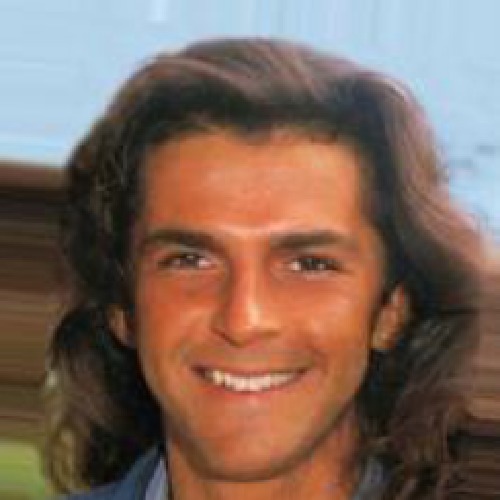}&
    \includegraphics[width=0.09\textwidth]{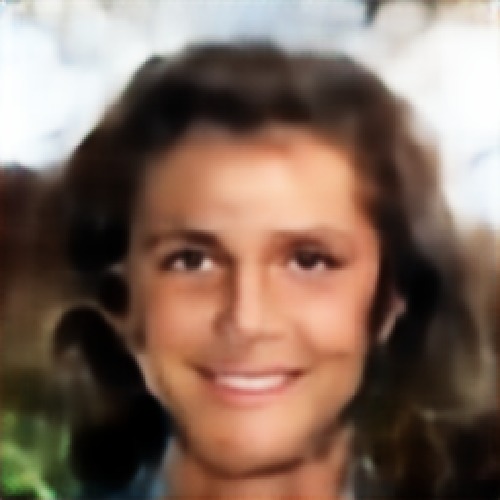} & &
 \includegraphics[width=0.09\textwidth]{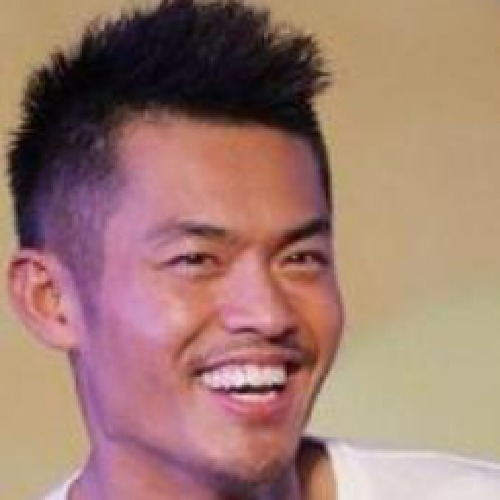}&
    \includegraphics[width=0.09\textwidth]{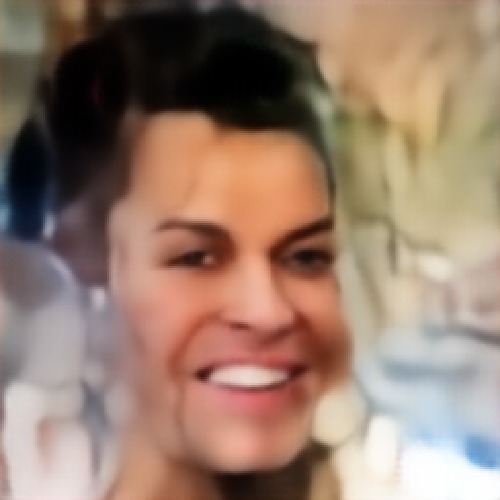} & &
    \includegraphics[width=0.09\textwidth]{images/Appendix_Images/fadernet/174.jpg}&
    \includegraphics[width=0.09\textwidth]{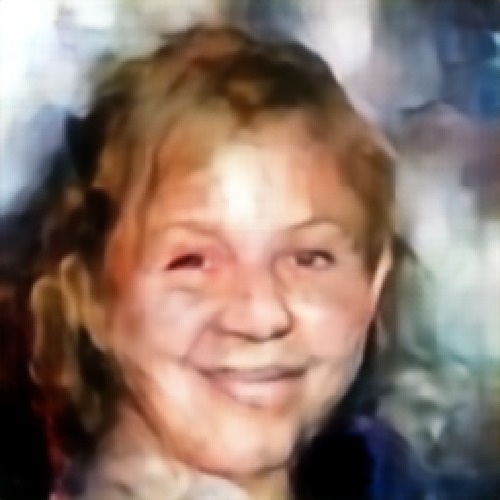}  \\
      \includegraphics[width=0.09\textwidth]{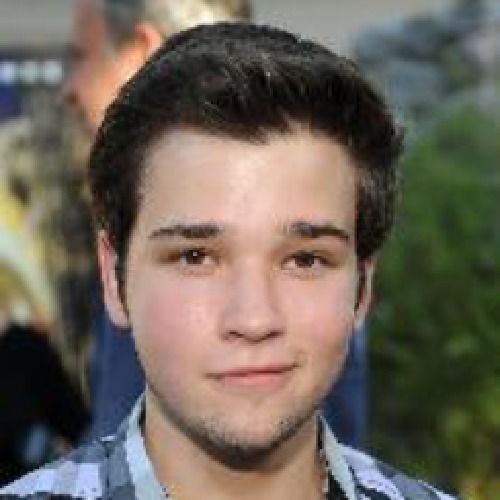}&
    \includegraphics[width=0.09\textwidth]{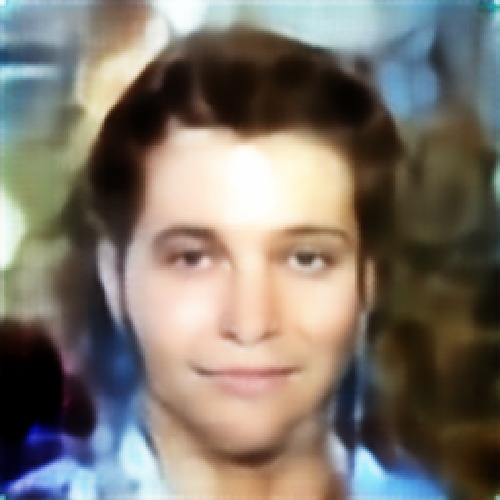} & &
    \includegraphics[width=0.09\textwidth]{images/Appendix_Images/fadernet/69.jpg}&
    \includegraphics[width=0.09\textwidth]{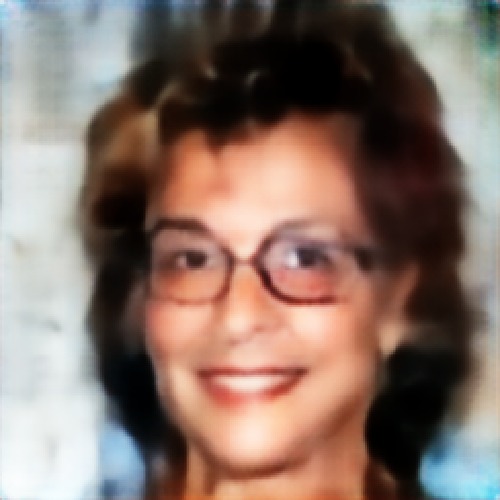} & &
    \includegraphics[width=0.09\textwidth]{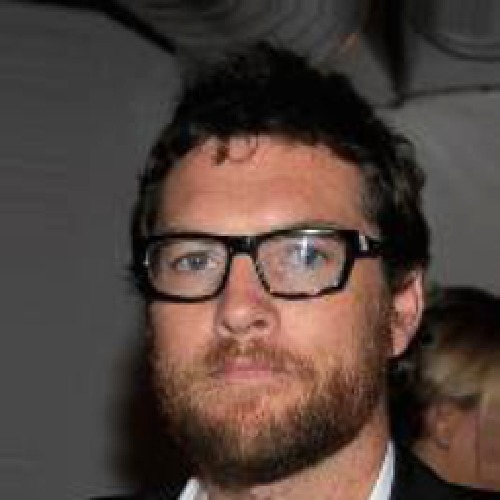}&
    \includegraphics[width=0.09\textwidth]{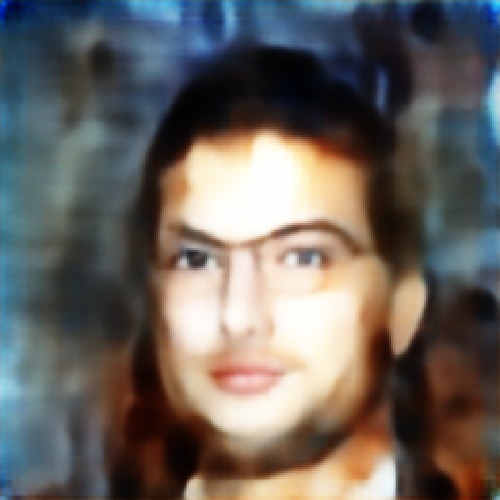} & &
  \includegraphics[width=0.09\textwidth]{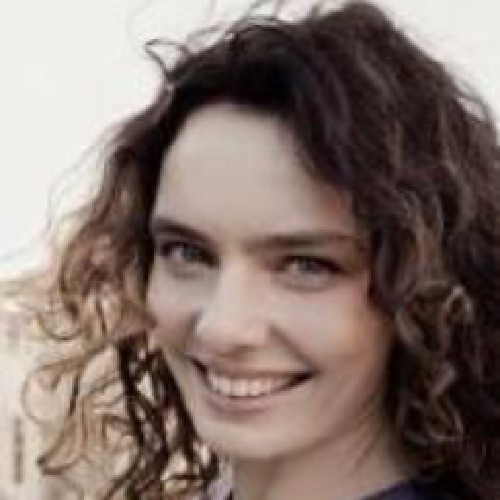}&
    \includegraphics[width=0.09\textwidth]{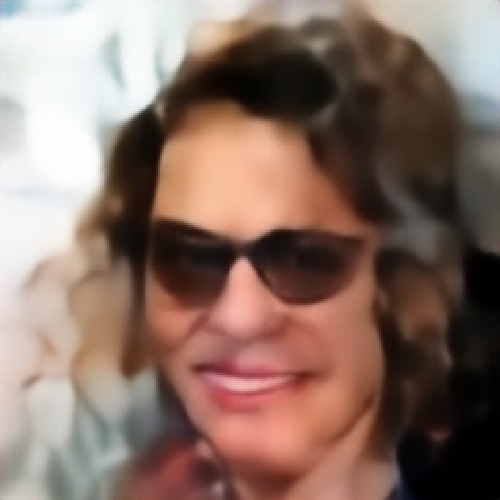} & &
    \includegraphics[width=0.09\textwidth]{images/Appendix_Images/fadernet/152.jpg}&
    \includegraphics[width=0.09\textwidth]{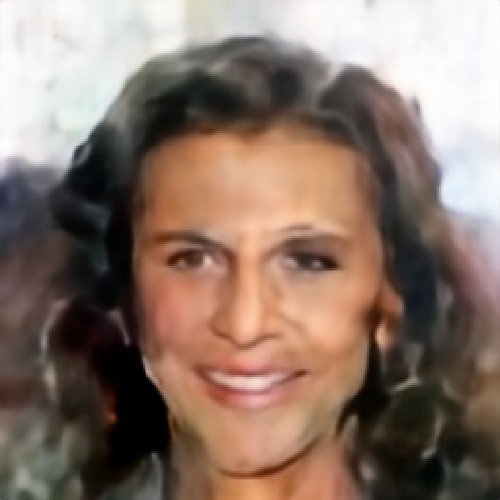}  \\
      \includegraphics[width=0.09\textwidth]{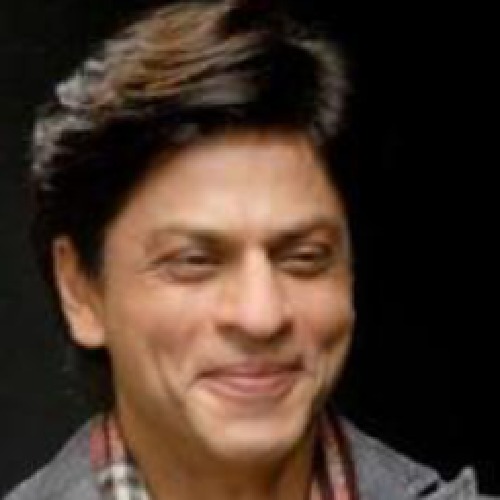}&
    \includegraphics[width=0.09\textwidth]{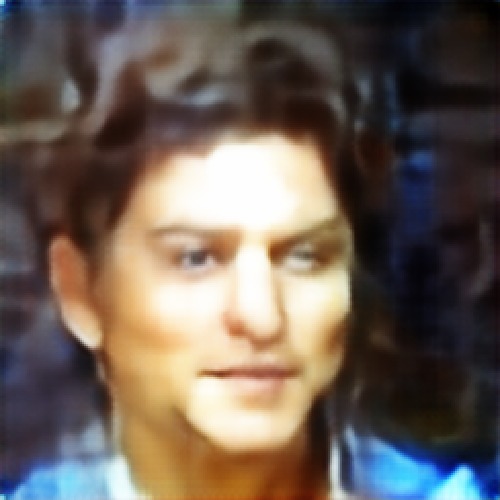} & &
    \includegraphics[width=0.09\textwidth]{images/Appendix_Images/fadernet/118.jpg}&
    \includegraphics[width=0.09\textwidth]{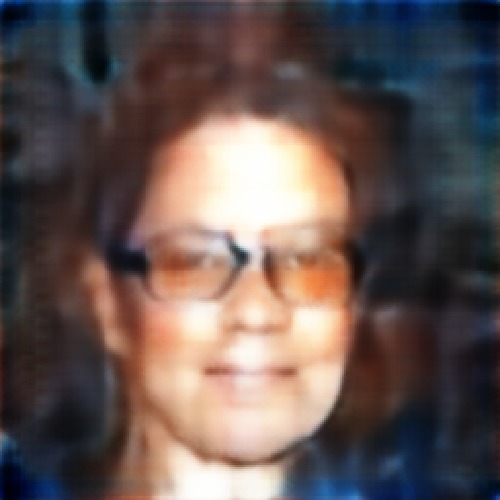} & &
    \includegraphics[width=0.09\textwidth]{images/Appendix_Images/fadernet/237.jpg}&
    \includegraphics[width=0.09\textwidth]{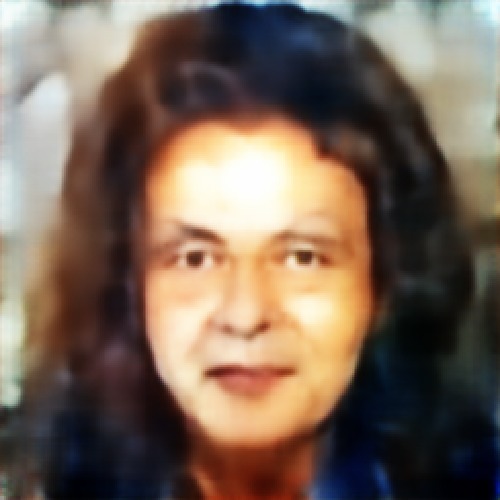} & &
   \includegraphics[width=0.09\textwidth]{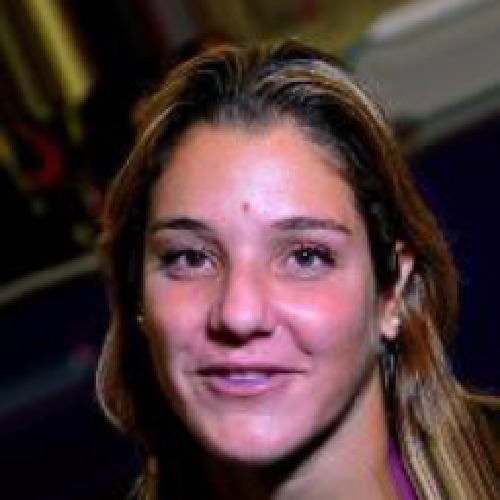}&
    \includegraphics[width=0.09\textwidth]{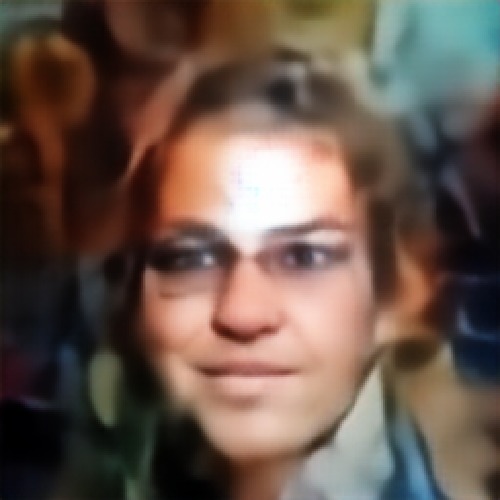} & &
    \includegraphics[width=0.09\textwidth]{images/Appendix_Images/fadernet/107.jpg}&
    \includegraphics[width=0.09\textwidth]{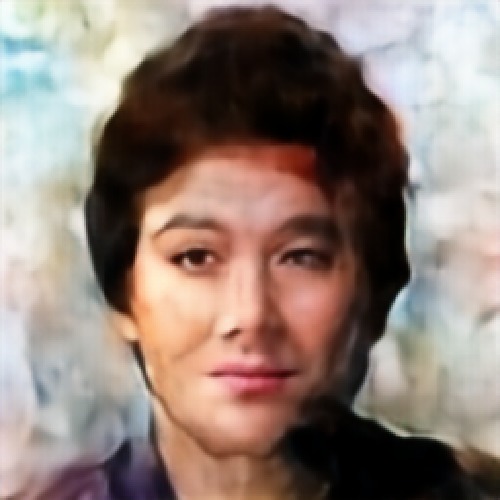}  \\
      \includegraphics[width=0.09\textwidth]{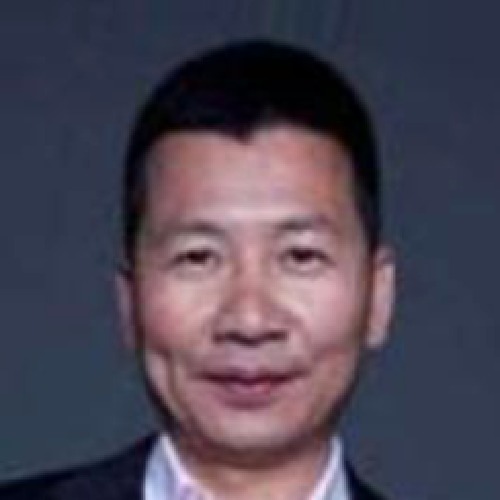}&
    \includegraphics[width=0.09\textwidth]{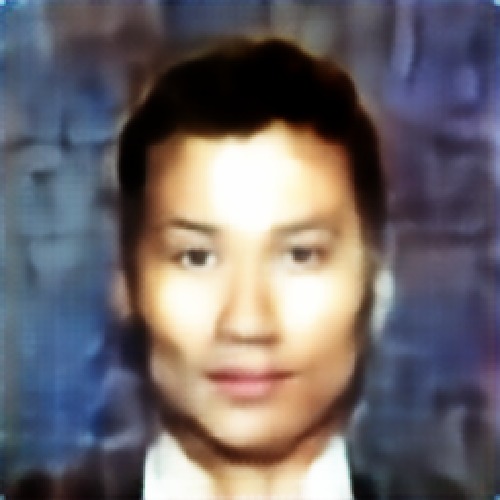} & &
    \includegraphics[width=0.09\textwidth]{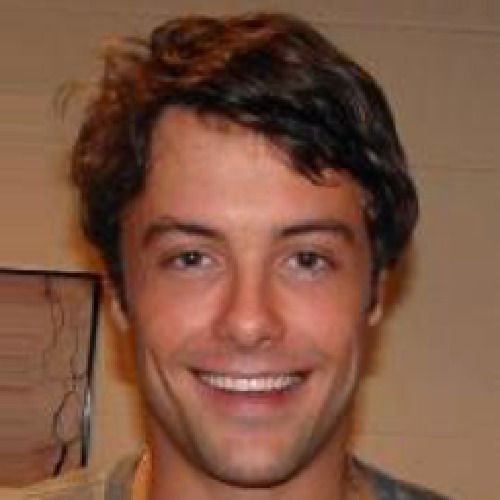}&
    \includegraphics[width=0.09\textwidth]{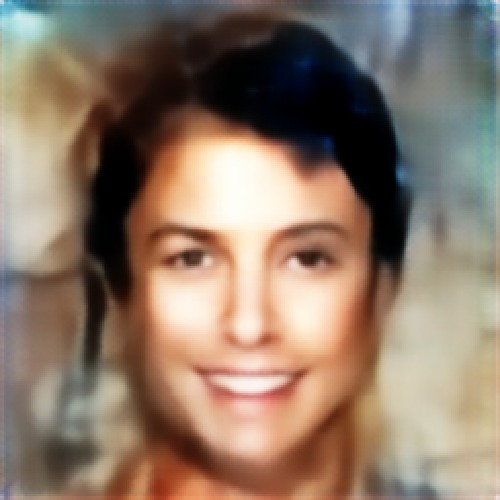} & &
    \includegraphics[width=0.09\textwidth]{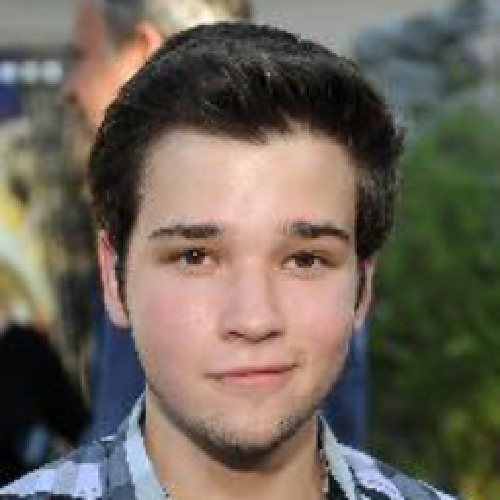}&
    \includegraphics[width=0.09\textwidth]{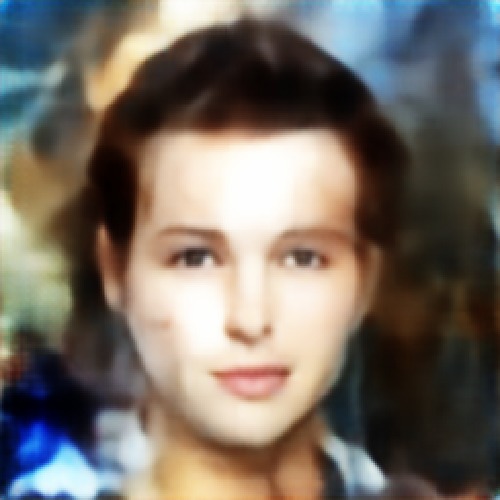} & &
\includegraphics[width=0.09\textwidth]{images/Appendix_Images/fadernet/seq_glasses_young_nose/153_orig.jpg}&
    \includegraphics[width=0.09\textwidth]{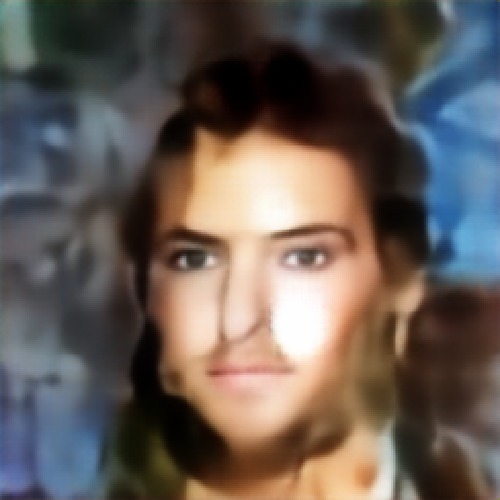} & &
    \includegraphics[width=0.09\textwidth]{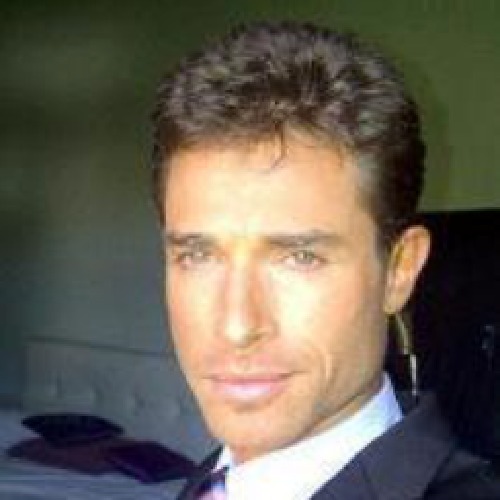}&
    \includegraphics[width=0.09\textwidth]{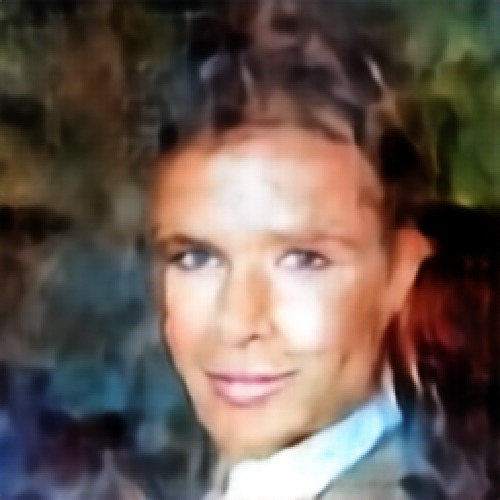}  \\
      \includegraphics[width=0.09\textwidth]{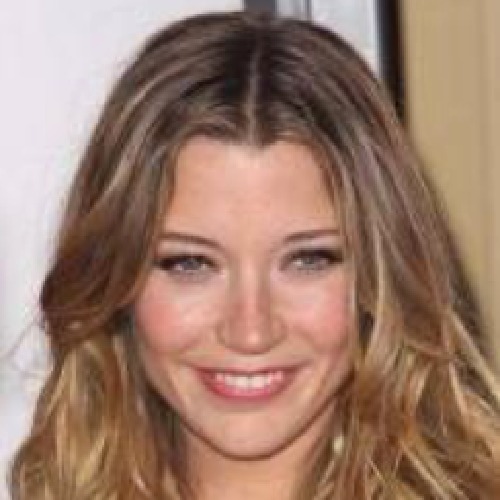}&
    \includegraphics[width=0.09\textwidth]{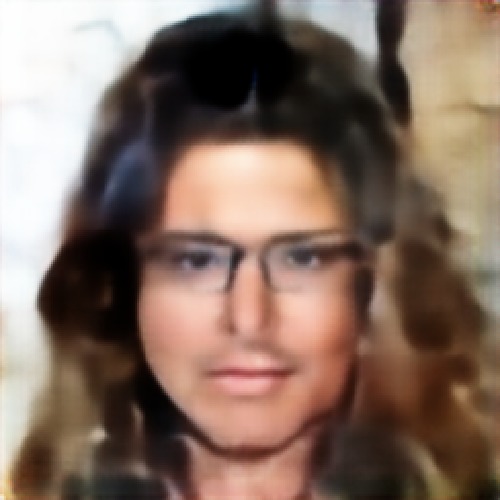} & &
    \includegraphics[width=0.09\textwidth]{images/Appendix_Images/fadernet/140.jpg}&
    \includegraphics[width=0.09\textwidth]{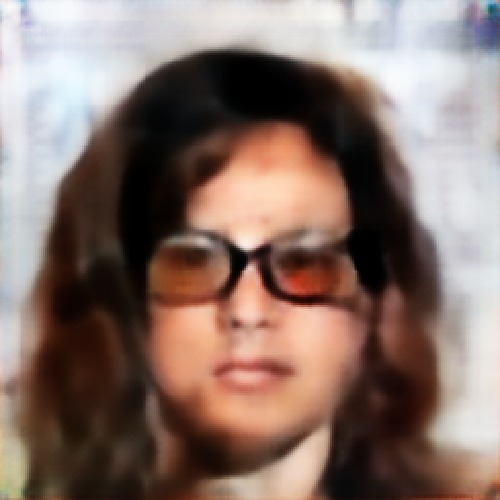} & &
    \includegraphics[width=0.09\textwidth]{images/Appendix_Images/fadernet/215.jpg}&
    \includegraphics[width=0.09\textwidth]{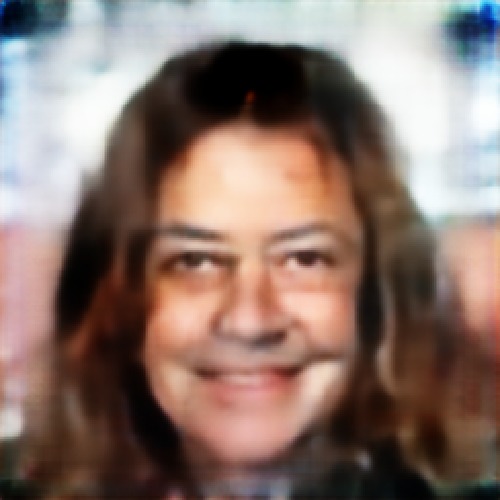} & &
\includegraphics[width=0.09\textwidth]{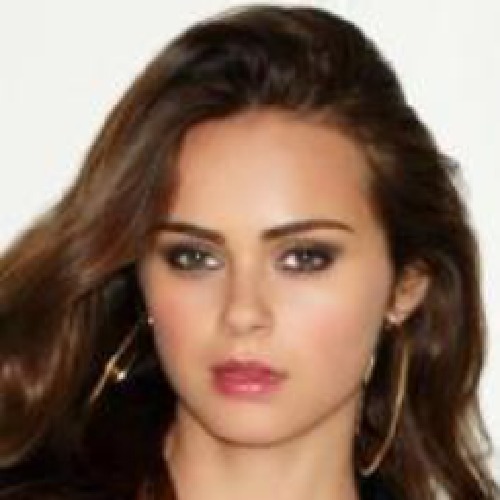}&
    \includegraphics[width=0.09\textwidth]{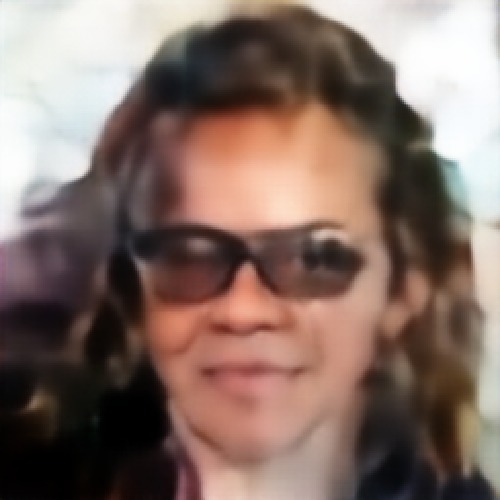} & &
    \includegraphics[width=0.09\textwidth]{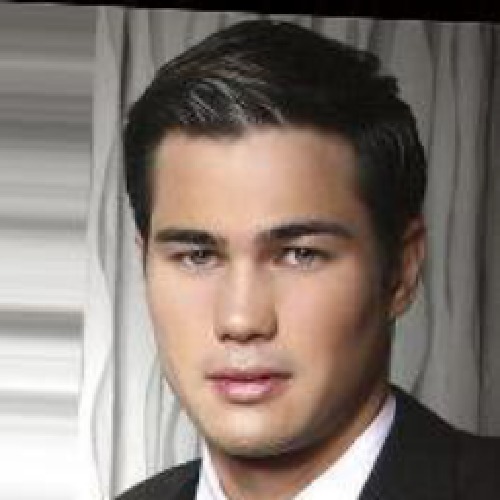}&
    \includegraphics[width=0.09\textwidth]{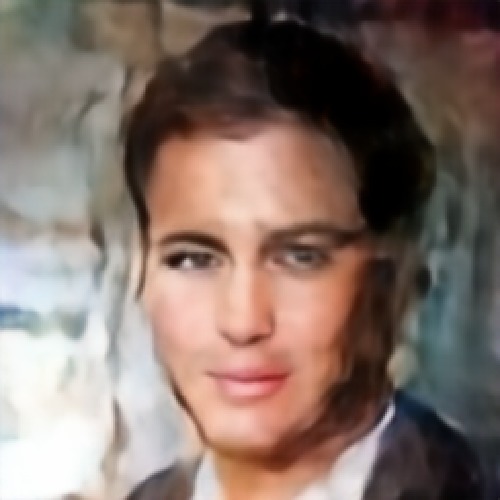}  \\
      \includegraphics[width=0.09\textwidth]{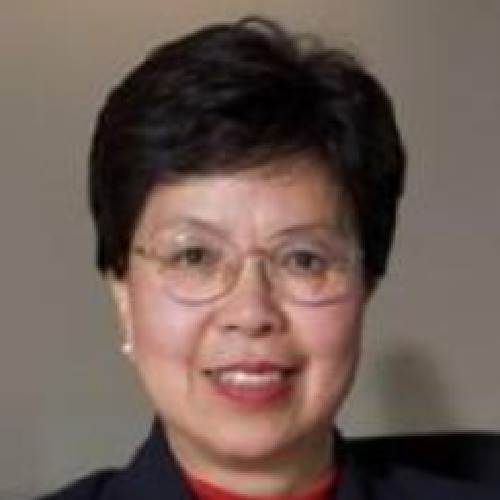}&
    \includegraphics[width=0.09\textwidth]{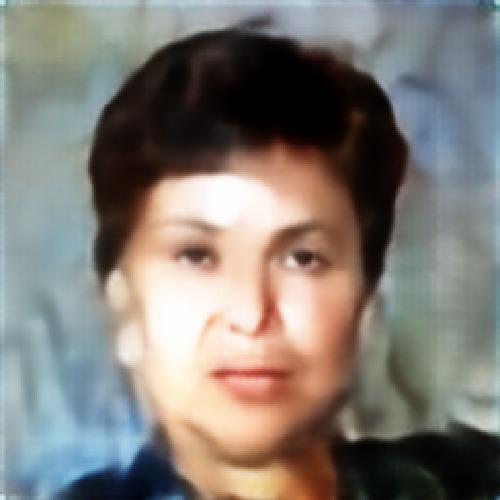} & &
    \includegraphics[width=0.09\textwidth]{images/Appendix_Images/fadernet/192.jpg}&
    \includegraphics[width=0.09\textwidth]{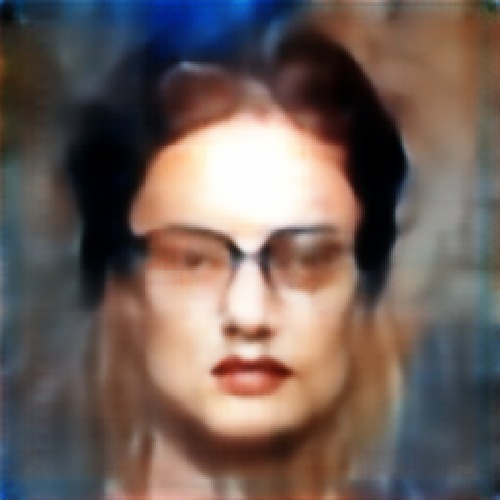} & &
    \includegraphics[width=0.09\textwidth]{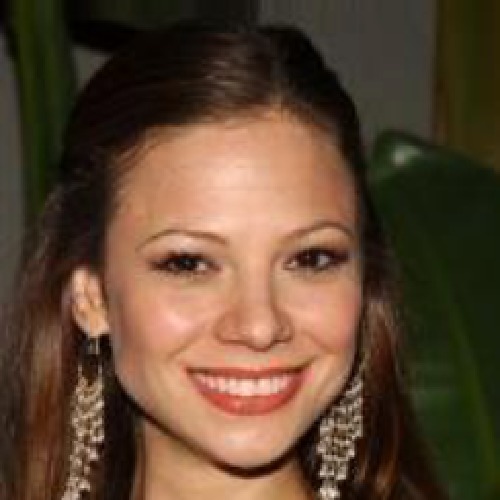}&
    \includegraphics[width=0.09\textwidth]{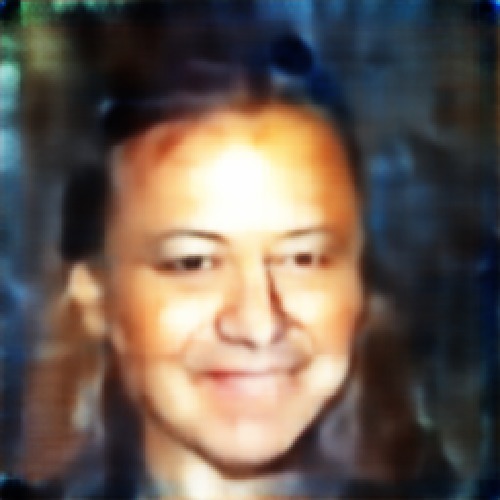} & &
\includegraphics[width=0.09\textwidth]{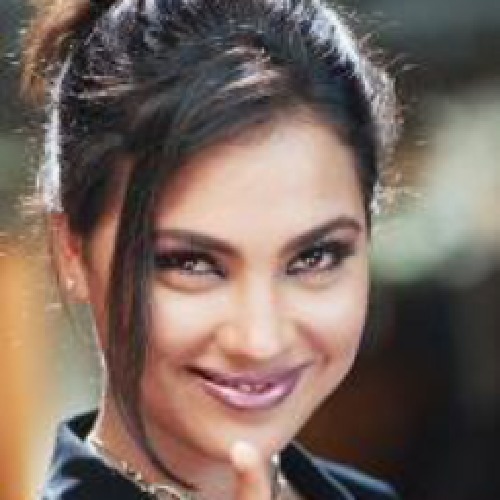}&
    \includegraphics[width=0.09\textwidth]{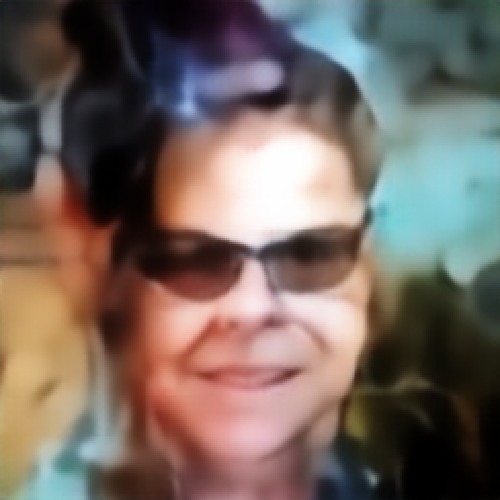} & &
    \includegraphics[width=0.09\textwidth]{images/Appendix_Images/fadernet/34.jpg}&
    \includegraphics[width=0.09\textwidth]{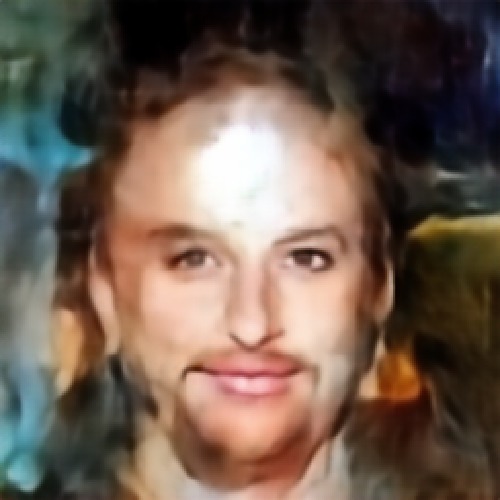}  \\
      \includegraphics[width=0.09\textwidth]{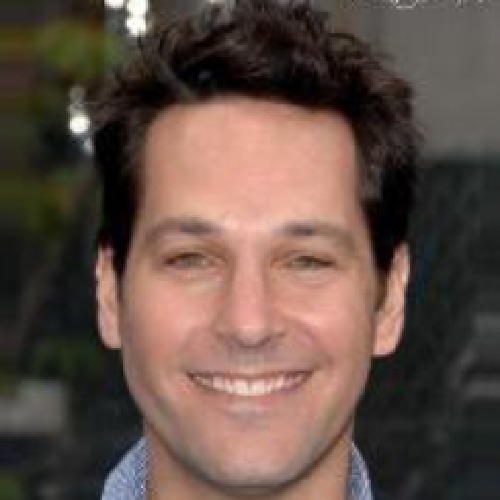}&
    \includegraphics[width=0.09\textwidth]{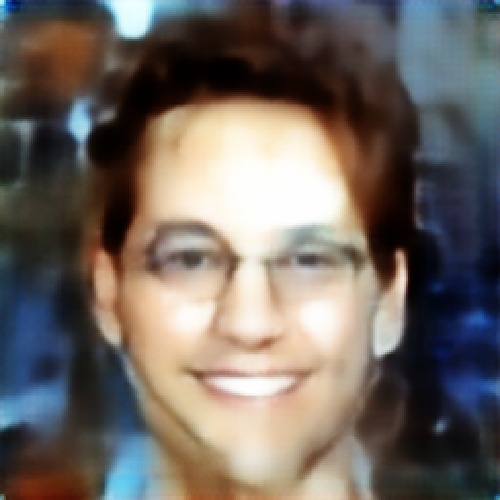} & &
    \includegraphics[width=0.09\textwidth]{images/Appendix_Images/fadernet/194.jpg}&
    \includegraphics[width=0.09\textwidth]{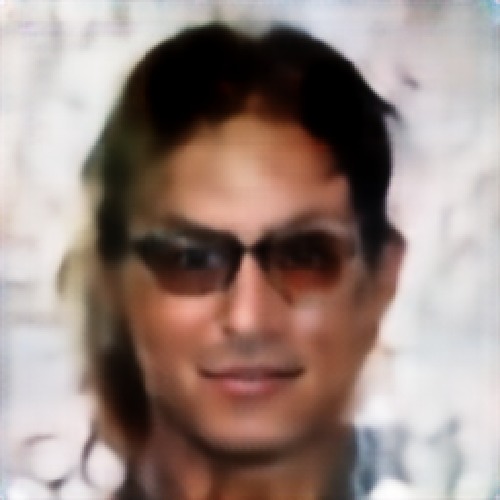} & &
    \includegraphics[width=0.09\textwidth]{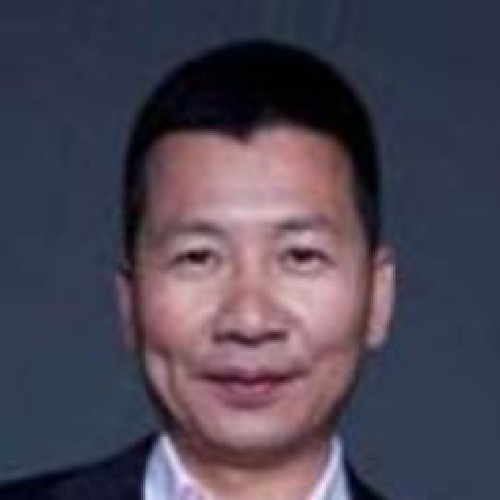}&
    \includegraphics[width=0.09\textwidth]{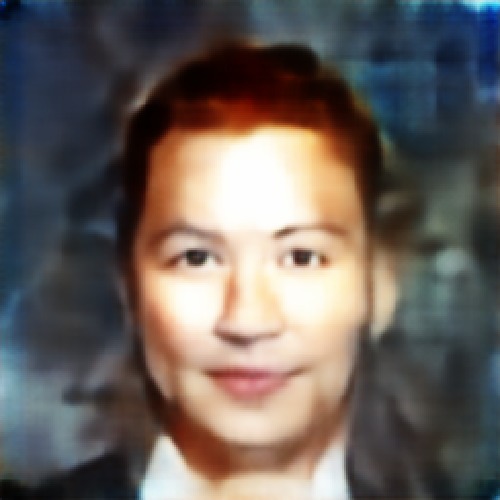} & &
\includegraphics[width=0.09\textwidth]{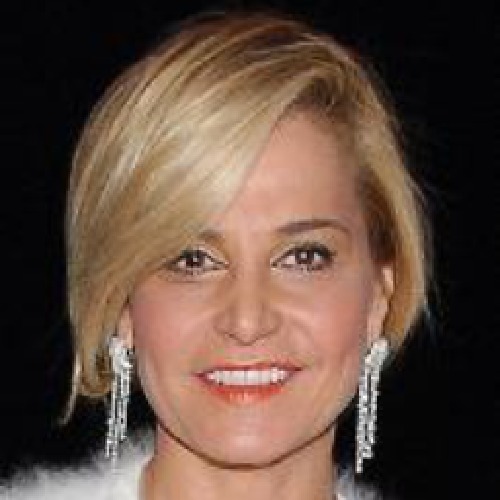}&
    \includegraphics[width=0.09\textwidth]{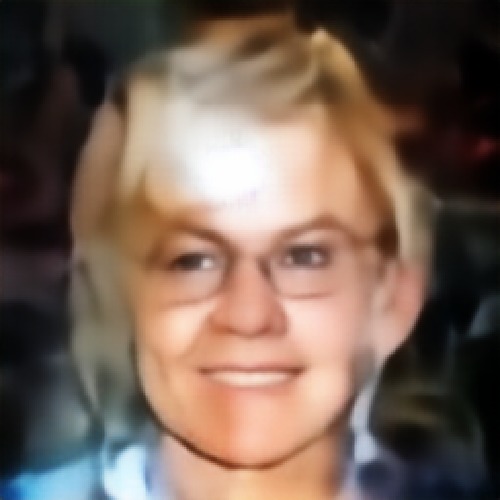} & &
    \includegraphics[width=0.09\textwidth]{images/Appendix_Images/fadernet/33.jpg}&
    \includegraphics[width=0.09\textwidth]{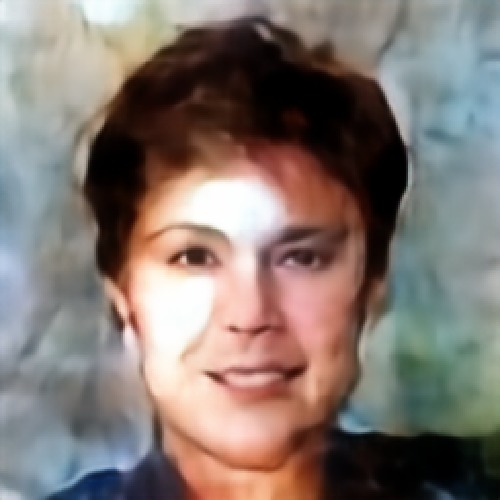}  \\
      \includegraphics[width=0.09\textwidth]{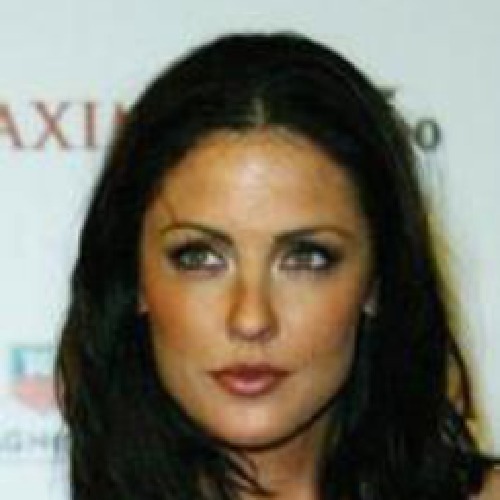}&
    \includegraphics[width=0.09\textwidth]{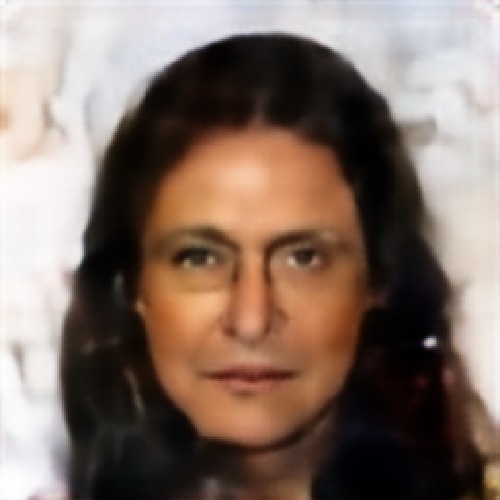} & &
    \includegraphics[width=0.09\textwidth]{images/Appendix_Images/fadernet/215.jpg}&
    \includegraphics[width=0.09\textwidth]{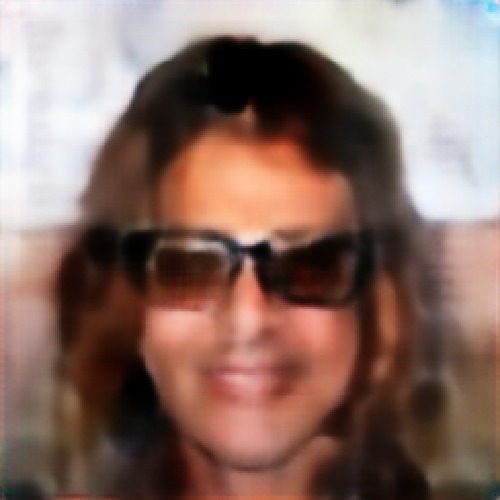} & &
    \includegraphics[width=0.09\textwidth]{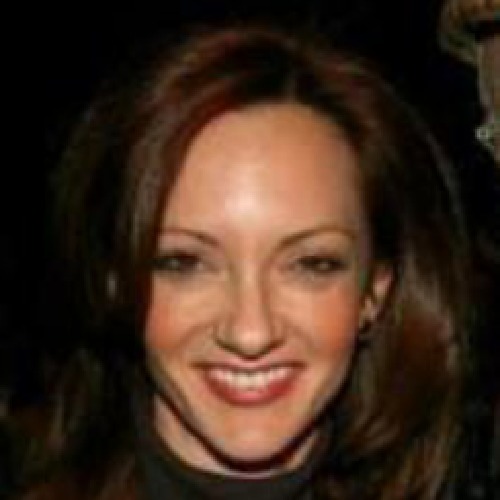}&
    \includegraphics[width=0.09\textwidth]{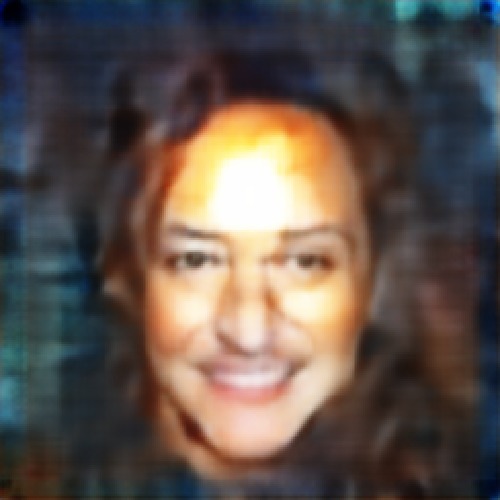} & &
\includegraphics[width=0.09\textwidth]{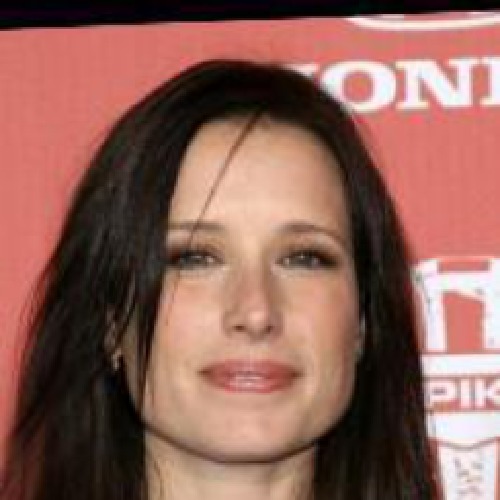}&
    \includegraphics[width=0.09\textwidth]{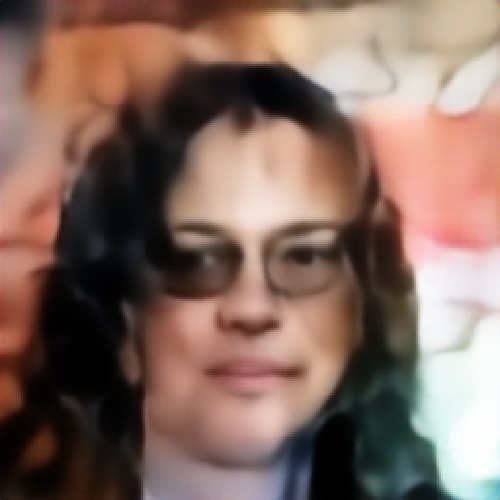} & &
    \includegraphics[width=0.09\textwidth]{images/Appendix_Images/fadernet/21.jpg}&
    \includegraphics[width=0.09\textwidth]{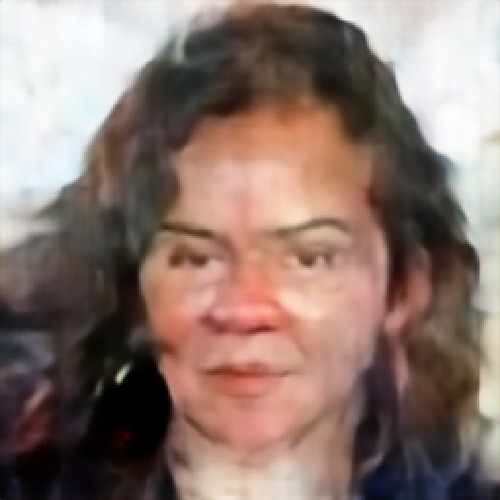} \\
    (a) & (b) & & (c) & (d) & & (e) & (f) & & (g) & (h) & & (i) & (j)
    \end{tabular}
    \caption{\sl Semantic adversarial examples for Multi-attribute and Cascaded Adversarial Fader network attacks. Columns (a), (c), (e), (g), (i) are original images. Columns (b): Multi-attribute Eyeglasses,Age,Smile, (d): Multi-attribute Pale Skin,Eyeglasses,Chubbiness, (f): Multi-attribute Age,Chubbiness,Pale Skin, (h): Cascaded Eyeglasses-Age-Nose shape, (j): Cascaded Nose shape-Narrow Eyes-Age }
\end{figure}
\endgroup

\begingroup
\begin{figure}[htp]
	\centering
	\setlength{\tabcolsep}{1pt}
	\renewcommand{\arraystretch}{0.5}
	\begin{tabular}{c c c || c c c || c c c || c c  }
		\includegraphics[width=0.10\textwidth]{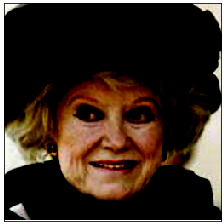}&
		\includegraphics[width=0.10\textwidth]{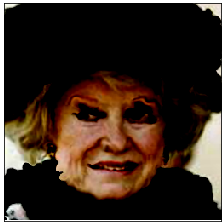} & &
		\includegraphics[width=0.10\textwidth]{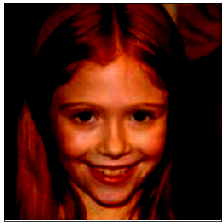}&
		\includegraphics[width=0.10\textwidth]{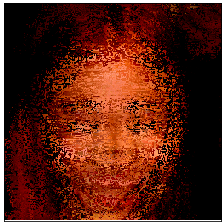} & &
		\includegraphics[width=0.10\textwidth]{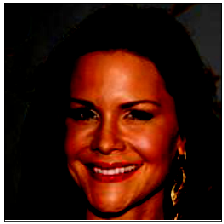}&
		\includegraphics[width=0.10\textwidth]{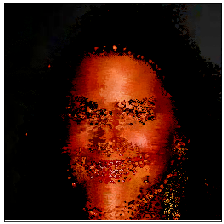}& &
		\includegraphics[width=0.10\textwidth]{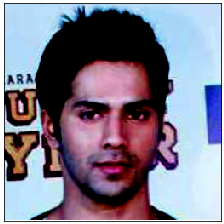}&
		\includegraphics[width=0.10\textwidth]{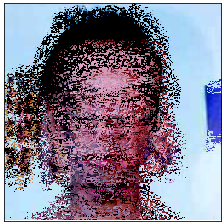}
		\\
		\includegraphics[width=0.10\textwidth]{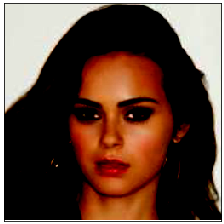}&
		\includegraphics[width=0.10\textwidth]{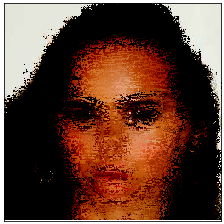} & &
		\includegraphics[width=0.10\textwidth]{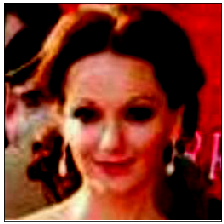}&
		\includegraphics[width=0.10\textwidth]{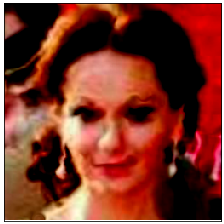} & &
		\includegraphics[width=0.10\textwidth]{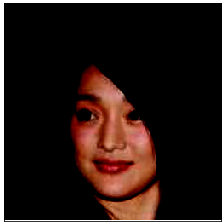}&
		\includegraphics[width=0.10\textwidth]{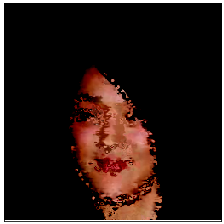}& &
		\includegraphics[width=0.10\textwidth]{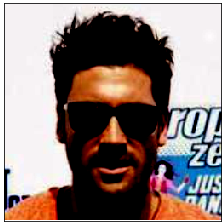}&
		\includegraphics[width=0.10\textwidth]{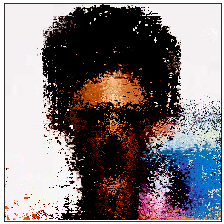}
		\\
		\includegraphics[width=0.10\textwidth]{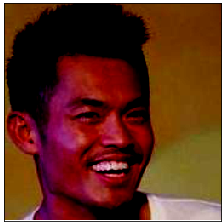}&
		\includegraphics[width=0.10\textwidth]{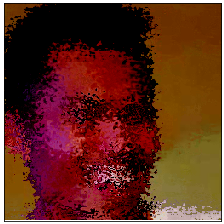} & &
		\includegraphics[width=0.10\textwidth]{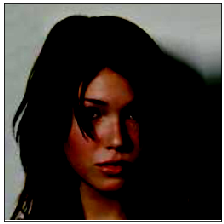}&
		\includegraphics[width=0.10\textwidth]{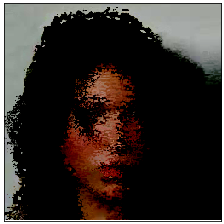} & &
		\includegraphics[width=0.10\textwidth]{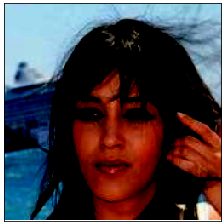}&
		\includegraphics[width=0.10\textwidth]{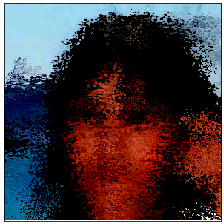}& &
		\includegraphics[width=0.10\textwidth]{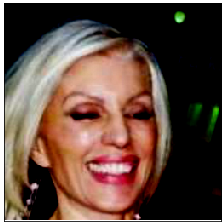}&
		\includegraphics[width=0.10\textwidth]{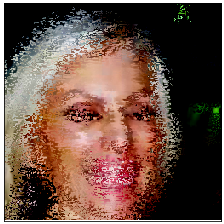}
		\\
		(a) & (b) & & (c) & (d) & & (e) & (f) & & (g) & (h)
	\end{tabular}
	\caption{\sl Columns (a), (c), (e), (g) show original images and columns (b), (d), (f) and (h) show the corresponding adversarial images produced by implementing the algorithm from Xiao \etal \cite{Xiao2018SpatiallyTA}. As can be clearly seen although the adversarial examples are missclassified by the deep gender classifier, the produced adversarial images are not semantically valid.} 
\end{figure}
\endgroup

\begingroup
\begin{figure}[htp]
	\centering
	\setlength{\tabcolsep}{1pt}
	\renewcommand{\arraystretch}{0.5}
	\begin{tabular}{c c c || c c c || c c c || c c  }
		\includegraphics[width=0.10\textwidth]{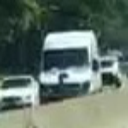}&
		\includegraphics[width=0.10\textwidth]{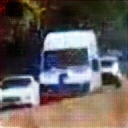} & &
		\includegraphics[width=0.10\textwidth]{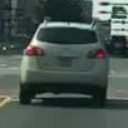}&
		\includegraphics[width=0.10\textwidth]{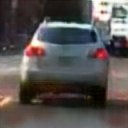} & &
		\includegraphics[width=0.10\textwidth]{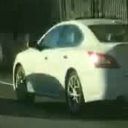}&
		\includegraphics[width=0.10\textwidth]{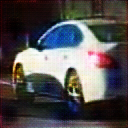}& &
		\includegraphics[width=0.10\textwidth]{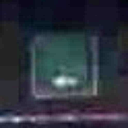}&
		\includegraphics[width=0.10\textwidth]{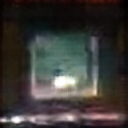}
		\\
		\includegraphics[width=0.10\textwidth]{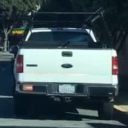}&
		\includegraphics[width=0.10\textwidth]{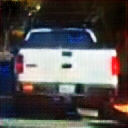} & &
		\includegraphics[width=0.10\textwidth]{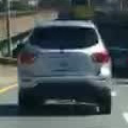}&
		\includegraphics[width=0.10\textwidth]{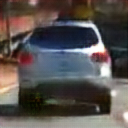} & &
		\includegraphics[width=0.10\textwidth]{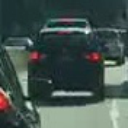}&
		\includegraphics[width=0.10\textwidth]{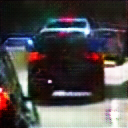} & &
		\includegraphics[width=0.10\textwidth]{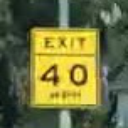}&
		\includegraphics[width=0.10\textwidth]{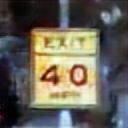}
		\\
		\includegraphics[width=0.10\textwidth]{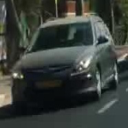}&
		\includegraphics[width=0.10\textwidth]{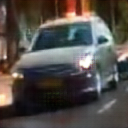} & &
		\includegraphics[width=0.10\textwidth]{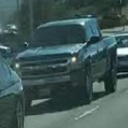}&
		\includegraphics[width=0.10\textwidth]{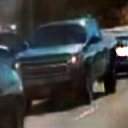} & &
		\includegraphics[width=0.10\textwidth]{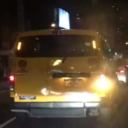}&
		\includegraphics[width=0.10\textwidth]{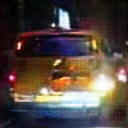} & &
		\includegraphics[width=0.10\textwidth]{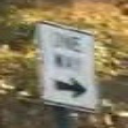}&
		\includegraphics[width=0.10\textwidth]{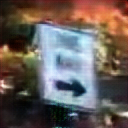}
		\\
		\includegraphics[width=0.10\textwidth]{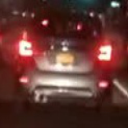}&
		\includegraphics[width=0.10\textwidth]{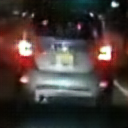} & &
		\includegraphics[width=0.10\textwidth]{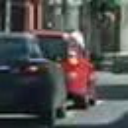}&
		\includegraphics[width=0.10\textwidth]{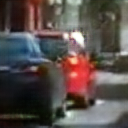} & &
		\includegraphics[width=0.10\textwidth]{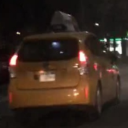}&
		\includegraphics[width=0.10\textwidth]{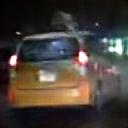} & &
		\includegraphics[width=0.10\textwidth]{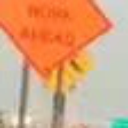}&
		\includegraphics[width=0.10\textwidth]{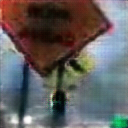}
		\\
		\includegraphics[width=0.10\textwidth]{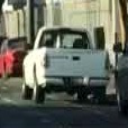}&
		\includegraphics[width=0.10\textwidth]{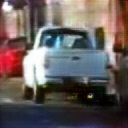} & &
		\includegraphics[width=0.10\textwidth]{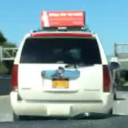}&
		\includegraphics[width=0.10\textwidth]{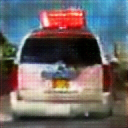} & &
		\includegraphics[width=0.10\textwidth]{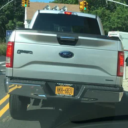}&
		\includegraphics[width=0.10\textwidth]{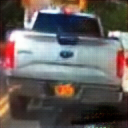} & &
		\includegraphics[width=0.10\textwidth]{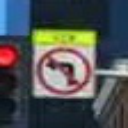}&
		\includegraphics[width=0.10\textwidth]{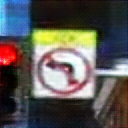}
		\\
		\includegraphics[width=0.10\textwidth]{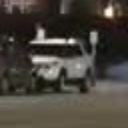}&
		\includegraphics[width=0.10\textwidth]{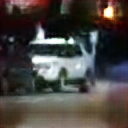} & &
		\includegraphics[width=0.10\textwidth]{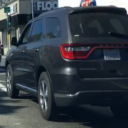}&
		\includegraphics[width=0.10\textwidth]{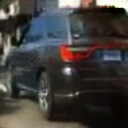} & &
		\includegraphics[width=0.10\textwidth]{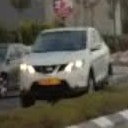}&
		\includegraphics[width=0.10\textwidth]{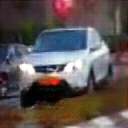} & &
		\includegraphics[width=0.10\textwidth]{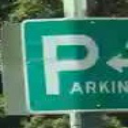}&
		\includegraphics[width=0.10\textwidth]{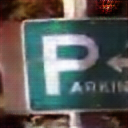}
		\\
		\includegraphics[width=0.10\textwidth]{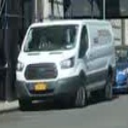}&
		\includegraphics[width=0.10\textwidth]{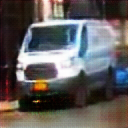} & &
		\includegraphics[width=0.10\textwidth]{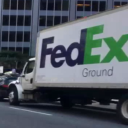}&
		\includegraphics[width=0.10\textwidth]{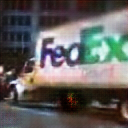} & &
		\includegraphics[width=0.10\textwidth]{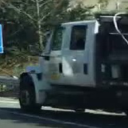}&
		\includegraphics[width=0.10\textwidth]{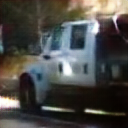} & &
		\includegraphics[width=0.10\textwidth]{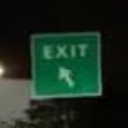}&
		\includegraphics[width=0.10\textwidth]{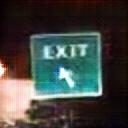}
		\\
		(a) & (b) & & (c) & (d) & & (e) & (f) & & (g) & (h)
	\end{tabular}
	\caption{\sl Semantic adversarial examples produced by Attribute GAN trained on Time of Day labels from BDD dataset\cite{bdd100k}. Columns (a), (c), (e), (g) are original images. Rows (1) through (6) and columns (a) to (f) show adversarial examples of cars getting miss-classified as traffic signs or trucks. Column (h) shows adversarial examples of traffic signs being miss-classified as cars. Row 7, columns (b) and (f) shows examples as trucks getting miss-classified as cars. Row 7, column d shows an adversarial example of a truck getting miss-classified as a traffic sign.}
\end{figure}
\endgroup

\end{document}